\documentclass[final,12pt]{colt2026} %

\title[The Asymptotics of Self-Supervised Pre-training]{On the Asymptotics of Self-Supervised Pre-training: \\
  Two-Stage $M$-Estimation and Representation Symmetry\protect\footnotemark[1]}
\usepackage{times}
\usepackage{mathtools}
\usepackage{microtype}
\usepackage{subcaption}
\usepackage{booktabs}
\usepackage{tcolorbox}

\usepackage{enumitem}

\usepackage{hyperref}
\hypersetup{
    colorlinks,
    linkcolor={red!50!black},
    citecolor={red!50!black},
    urlcolor={red!80!black}
}
\usepackage{tikz}
\usetikzlibrary{arrows.meta, positioning, calc, decorations.pathmorphing, shapes.geometric, fit, backgrounds}
\usepackage{wrapfig}
\usepackage{comment}
\usepackage{color}
\usepackage{xcolor}



\usepackage[capitalize]{cleveref}

\crefname{appendix}{Appendix}{Appendices}
\Crefname{appendix}{Appendix}{Appendices}

\usepackage{thmtools}



 \newtheorem{theorem}{Theorem}[section]
\newtheorem{lemma}[theorem]{Lemma}
\newtheorem{proposition}[theorem]{Proposition}
\newtheorem{corollary}[theorem]{Corollary}
\newtheorem{remark}[theorem]{Remark}
\newtheorem{definition}[theorem]{Definition}
\newtheorem{assmp}[theorem]{Assumption}

\renewcommand{\ge}{\geqslant}
\renewcommand{\geq}{\geqslant}
\renewcommand{\le}{\leqslant}
\renewcommand{\leq}{\leqslant}

\renewcommand{\succeq}{\succcurlyeq}

\DeclareMathOperator*{\argmin}{arg\min}

\DeclareMathOperator*{\rank}{rank}

\DeclareMathOperator*{\diag}{diag}
\DeclareMathOperator*{\Tr}{\mathrm{tr}}

\newcommand{\calL}{\mathcal{L}}

\newcommand{\R}{\ensuremath{\mathbb{R}}}

\newcommand{\N}{\ensuremath{\mathbb{N}}}

\newcommand{\Unif}{\mathrm{Unif}}

\newcommand{\norm}[1]{\lVert #1 \rVert}

\newcommand{\ip}[2]{\ensuremath{\langle #1, #2 \rangle}}

\newcommand{\E}{\mathbb{E}}

\renewcommand{\Pr}{\mathbb{P}}
\newcommand{\T}{\mathsf{T}}

\newcommand{\calN}{\mathcal{N}}
\newcommand{\calE}{\mathcal{E}}

\newcommand{\calF}{\mathcal{F}}

\newcommand{\calH}{\mathcal{H}}

\newcommand{\Cov}{\mathrm{Cov}}

\numberwithin{equation}{section}

\DeclarePairedDelimiterX{\infdivx}[2]{(}{)}{%
  #1\;\delimsize\|\;#2%
}

\newcommand{\opnorm}[1]{\norm{#1}_{\mathrm{op}}}

\newcommand{\distconv}{\stackrel{d}{\rightsquigarrow}}

\newcommand{\rmd}{\mathrm{d}}

\DeclareMathOperator{\tr}{tr}

\newcommand{\mycomment}[1]{}

\newcommand{\iid}{\mathrm{i.i.d.}}

\definecolor{myblue}{RGB}{0,90,156}
\definecolor{myred}{RGB}{176,48,48}
\definecolor{mygreen}{RGB}{34,110,60}
\definecolor{mygray}{RGB}{120,120,120}

\allowdisplaybreaks[2]

\coltauthor{%
 \Name{Mohammad Tinati} \Email{tinati@usc.edu}\\
 \addr Department of Electrical and Computer Engineering, University of Southern California
 \AND
 \Name{Stephen Tu} \Email{stephen.tu@usc.edu}\\
 \addr Department of Electrical and Computer Engineering, University of Southern California%
}

\begin{document}

\maketitle
\begingroup
\renewcommand{\thefootnote}{\fnsymbol{footnote}}
\footnotetext[1]{Accepted for presentation at the Conference on Learning Theory (COLT) 2026.}
\endgroup
\begin{abstract}%
Self-supervised pre-training, where large corpora of unlabeled data are used to learn representations for downstream fine-tuning, has become a cornerstone of modern machine learning. 
While a growing body of work has begun to analyze this paradigm, existing bounds leave open the question of how sharp current rates are, and whether they accurately capture the
complex interaction between pre-training and fine-tuning. 
In this paper, we address this gap by developing an asymptotic theory of pre-training via two-stage $M$-estimation. A key challenge is that the pre-training estimator is often identifiable only up to a group symmetry, a feature common in representation learning that requires careful treatment.
We address this issue using tools from Riemannian geometry to 
study the \emph{intrinsic} parameters of the pre-training representation,
which we link with the downstream predictor through a notion of
\emph{orbit-invariance}, precisely characterizing
the limiting distribution of the downstream test risk.
We apply our results to spectral pre-training, factor models, and Gaussian mixture models,
obtaining substantial improvements in problem-specific factors over prior art when applicable.
\end{abstract}


\section{Introduction}

Self-supervised pre-training has emerged as a powerful paradigm for learning representations from large corpora of unlabeled data, which are subsequently adapted to downstream tasks via fine-tuning. This approach has achieved striking empirical success across a wide range of domains in modern machine learning. For instance in computer vision, a growing body of contrastive,
masked reconstruction, and self-distillation methods~\citep{chen2020simclr,zbontar2021barlowtwins,grill2020BYOL,oord2018representation,he2020momentum,wang2020alignmentuniformity,chen2021exploring,bardes2022vicreg,he2022masked}
have demonstrated that high-quality features can be learned without manual annotation and transferred effectively across tasks. More broadly, large language models and vision-language models trained on massive unlabeled or weakly labeled corpora have shown that pre-training can endow models with general-purpose capabilities that substantially reduce the amount of labeled data required for downstream adaptation~\citep{devlin2019bert,brown2020gpt3,radford2021CLIP,oquab2024dinov2}.

Motivated by these empirical advances, a growing body of theoretical work has begun to investigate self-supervised pre-training, including contrastive learning and other variants, from a statistical perspective~\citep{saunshi2019analysis,tosh2021contrastive,lee2021predicting,haochen2021provable,cabannes23SSLinterplay,saunshi2022inductivebias,ge2023provable,lin2025statistical}. 
Despite the varying problem setups,
loss functions, and structural assumptions in these works,
a central question across much of this literature is: when does the two-stage pipeline of pre-training on unlabelled data followed by fine-tuning on downstream task data \emph{provably outperform training from scratch on the downstream data alone?} 
A closely related question involves the marginal value of pre-training data: 
when is the downstream task error fundamentally bottlenecked by the amount of labeled fine-tuning data, so that additional pre-training samples yield \emph{diminishing improvement for downstream task performance?}

While recent works have made some progress towards answering these questions,
we still lack an instance-optimal theory that precisely characterizes the role of pre-training loss, data distribution, and representation properties in downstream task performance.
Indeed, much of the existing theory focuses on sufficient conditions and upper bounds, leaving open the question of how sharply current rates capture true behavior.
Moreover, available results are typically not instance-adaptive: they do not explicitly reflect the interaction between the specific pre-training and fine-tuning distributions, losses, models, and representation structure. 
Contrast this to standard supervised learning, where classical $M$-estimation theory provides instance-specific asymptotic characterization of the excess risk;
these bounds then serve as a benchmark for deriving sharp non-asymptotic results~\citep{spokoiny2012parametric,frostig15competing,ostrovskii2021finite}. 

\begin{wrapfigure}[14]{r}{0.38\textwidth}
\vspace{-1.0em}
\centering
\begin{tikzpicture}[
  >=Latex,
  line width=0.75pt,
  scale=1,
  every node/.style={font=\scriptsize}
]
  \def\xmin{0.20}
  \def\xmax{3.60}
  \def\xstart{0.40}
  \def\xzero{1.49}
  \def\yfloor{0.50}
  \def\ybase{1.25}
  \def\yhigh{2.25}
  \def\decayk{1.270}

  \draw[->, mygray!70] (-0.2,0) -- (3.85,0)
    node[right, text=mygray] {$\log \alpha$};
  \draw[->, mygray!70] (0,-0.2) -- (0,3.05);

  \node[mygray, rotate=90, font=\scriptsize] at (-0.42,1.45)
    {Asymptotic scaled risk};

  \fill[mygreen!12] (\xzero,\yfloor) rectangle (\xmax,\ybase);

  \draw[myblue!60!black, dashed] (\xmin,\yfloor) -- (\xmax,\yfloor);
  \node[myblue!60!black, right=2pt, font=\tiny] at (\xmax,\yfloor)
    {$\sigma^2 d_\mathrm{eff}$};

  \draw[myred, thick] (\xmin,\ybase) -- (\xmax,\ybase);
  \node[myred, right=2pt, font=\tiny] at (\xmax,\ybase) {$\sigma^2 d_\mathrm{base}$};

  \draw[myblue!80!black, thick] (\xstart,\yhigh) -- (\xzero,\ybase);
  \draw[myblue!80!black, thick, domain=\xzero:\xmax, samples=60, smooth]
    plot (\x, {\yfloor + (\ybase-\yfloor)*exp(-\decayk*(\x-\xzero))});

  \fill[myred] (\xzero,\ybase) circle (1.7pt);
  \draw[mygray!60, densely dotted] (\xzero,0) -- (\xzero,\ybase);
  \node[myred, font=\tiny, below=1pt] at (\xzero,-0.02) {$\alpha_0$};

  \node[black!70!black, font=\tiny, align=center] at (2.70,0.95)
    {pre-training\\ wins};

  \begin{scope}[shift={(1.22,2.30)}]
    \draw[black!55, fill=white, line width=0.4pt]
      (0,0) rectangle (3.42,0.62);
    \draw[myred, thick] (0.10,0.44) -- (0.42,0.44);
    \node[anchor=west, font=\tiny, inner sep=1pt] at (0.44,0.44)
      {downstream-only ($\theta \circ \psi$)};
    \draw[myblue!80!black, thick] (0.10,0.18) -- (0.42,0.18);
    \node[anchor=west, font=\tiny, inner sep=1pt] at (0.44,0.18)
      {pre-training ($\psi$) + fine-tuning ($\theta)$};
  \end{scope}
\end{tikzpicture}
\caption{Schematic risk crossover. Here, $d_{\mathrm{base}}$ and $d_{\mathrm{eff}}$ denote the effective numbers of parameters in the downstream-only and known-representation limits,
respectively.}
\label{fig:intro-transition-pretraining}
\vspace{-1.0em}
\end{wrapfigure}

\paragraph{Our contribution.} In this paper, we take a step towards bridging this gap by developing an asymptotic theory of self-supervised pre-training followed by fine-tuning with linear regression, in the joint limit of pre-training and fine-tuning data. Let $\alpha := m/n$ denote the ratio of pre-training samples to downstream labeled samples. Our main result shows that the scaled downstream excess risk of the two-stage estimator decomposes into
two leading contributions: an intrinsic fine-tuning term, corresponding to well-specified least-squares regression on the limiting representation, and a pre-training interaction term that decays as $1/\alpha$ (see \Cref{fig:intro-transition-pretraining}). Thus, as the amount of pre-training data grows relative to downstream data, the pre-training interaction term vanishes, and the risk approaches the well-specified linear regression floor. Importantly, our result  identifies a precise threshold characterization (i.e., $\alpha > \alpha_0$) where pre-training plus fine-tuning provides strict improvement over a downstream-only baseline.
Applying our main result to several case studies---including pre-training with a spectral loss, a latent factor model, and
a Gaussian mixture model---illuminates
substantial gaps in problem-specific factors in prior art.

A key technical challenge which arises in our analysis is that pre-training estimators are often identifiable only up to a group symmetry, which complicates the direct application of two-stage $M$-estimation theory~\citep[see e.g.,][]{pagan1986twostage,newey1994largesample}.
We address this challenge for a general pre-training loss that learns a representation used in downstream linear regression. We first establish asymptotic normality of the \emph{intrinsic} pre-training representation by building on recent results in Riemannian $M$-estimation~\citep{brunel2023geodesically}.
We then link pre-training and downstream regression 
together via a key conceptual step that identifies
structural conditions on the learned features that ensure \emph{orbit-invariance} of the downstream predictor. 
Our general proof strategy should be broadly applicable beyond analyzing self-supervised pre-training pipelines, and we discuss future extensions in the conclusion.

\section{Related Work}
\label{sec:related_work}

Our work draws on ideas from self-supervised learning, the asymptotic theory of two-stage estimators, and Riemannian $M$-estimation. We defer a more comprehensive discussion of related work to \Cref{sec:appendix:related_work}, and focus here on the prior work most directly relevant to our contributions.

Most related to our work are \citet{cabannes23SSLinterplay,ge2023provable,zhai2024understanding}.
First, \citet{cabannes23SSLinterplay} studies a VICReg-style~\citep{balestriero2022contrastive} pre-training loss combined with downstream RKHS regression.
They control the downstream test risk by
the (scaled) pretrain loss, which they bound
using Rademacher complexity arguments.
While their downstream RKHS setup is more flexible, 
our analysis holds for more general pre-training losses.
Next, \citet{ge2023provable} 
combine MLE pre-training with ERM fine-tuning. 
Their \emph{$\kappa$-informative condition} shares high-level similarity with our goal of quantifying invariance in pre-training; however, 
the details differ substantially from our geometric approach.
Finally, \citet{zhai2024understanding} study downstream error through the lens of RKHS approximation,
showing that downstream error is influenced by two key terms: (a) 
the complexity of the RKHS induced by the augmentation distribution, and (b) how well
the pre-trained encoder approximates the induced augmentation RKHS.
In the aforementioned works, whether or not the upper bounds achieve optimal dependence on problem-specific constants is left open.
In \Cref{sec:examples}, we show non-trivial gaps between the upper bounds provided by \cite{cabannes23SSLinterplay,ge2023provable} and those
that arise from our asymptotic analysis in several settings.\footnote{The bounds from \citet{zhai2024understanding} are not directly 
comparable, as discussed further in \Cref{sec:appendix:related_work}.}

\section{Problem Formulation}
\label{sec:problem-formulation}

Let $\mu_{\mathrm{pre}}$ and $\mu_{\mathrm{down}}$ be probability measures on input spaces $\mathcal{Z}$ and $\mathcal X$, respectively.
We consider two training datasets:
\emph{(i)} a \emph{pre-training} dataset
$D_{\mathrm{pre}}^{(m)} \coloneqq  \{ z_j \}_{j=1}^{m}$,
where $z_j \stackrel{\iid}{\sim} \mu_{\mathrm{pre}}$,
and \emph{(ii)} a \emph{downstream} dataset
$D_{\mathrm{down}}^{(n)} \coloneqq  \{ (x_i, y_i) \}_{i=1}^{n}$,
where $(x_i,y_i) \stackrel{\iid}{\sim} (X,Y)$ and $X \sim \mu_{\mathrm{down}}$.
The datasets $D_\mathrm{pre}^{(m)}$ and $D_{\mathrm{down}}^{(n)}$ are drawn independently.
The pair $(X, Y)$ is further assumed to satisfy:
\begin{align}
Y &= f_\star(X) + \varepsilon,
\qquad
\mathbb{E}[\varepsilon \mid X] = 0,
\qquad
\sigma^2 \coloneq \mathbb{E}[\varepsilon^2 \mid X] < \infty ,
\label{eq:X_and_Y_pair}
\end{align}
for some unknown regression function $f_\star : \mathcal{X} \mapsto \R$.
Towards parameterizing $f_\star$,
we fix a feature dimension $p \in \N_+$, and consider a differentiable representation $\psi(x, w) \in \R^p$,
where $w\in\R^{q_0}$ is the representation parameter.
For each $w$, define the linear hypothesis class
$\mathcal{H}_w \coloneqq \{ f_{\theta,w} \mid \theta \in \mathbb{R}^p \}$
with $f_{\theta,w}(x) \coloneq \langle \theta, \psi(x,w) \rangle$.
We assume that $f_\star$ in \eqref{eq:X_and_Y_pair} is \emph{well-specified} with respect to 
$\calF \coloneqq  \bigcup_{w \in \R^{q_0}} \mathcal{H}_w$,
i.e., $f_\star \in \calF$.
Let $(w_\star, \theta_\star)$ denote a pair such that
$f_\star = f_{\theta_\star, w_\star}$.

\paragraph{Notation.} %
Throughout, $L^2\coloneqq L^2(\mu_{\mathrm{down}})$ denotes the Hilbert space of real-valued square-integrable
functions $g:\mathcal X\mapsto \R$ with inner product
$\langle g,h\rangle \coloneqq  \E_{X \sim \mu_{\mathrm{down}}}[g(X)h(X)]$.
The notation $\distconv$ denotes convergence in distribution, and $\xrightarrow{\mathbb P}$ denotes
convergence in probability.
The set $B_d(w, r) \coloneq \{ w \in \R^d \mid \norm{w} \leq r \}$ denotes the closed $\ell_2$-ball of radius $r$ in $\R^d$; we drop the subscript $d$ when the dimension is implicit.
The set $O(p)$ denotes the orthogonal group $O(p) \coloneq \{ Q \in \R^{p \times p} \mid Q^\top Q = QQ^\top = I \}$.
Finally, $(\cdot)^+$ denotes the Moore-Penrose pseudo-inverse for a matrix.

\subsection{Pre-training Loss, Downstream Least-Squares Estimation, and Final Test Risk}
\label{sec:problem_formulation:losses}

\paragraph{Pre-training objective.}
Let 
$\ell_{\mathrm{pre}}:\R^{q_0}\times\mathcal Z\mapsto\R$
denote a 
pre-training loss which is twice continuously differentiable with respect to $w$
for almost every $Z$,
and let 
$L_{\mathrm{pre}}(w)\coloneqq \E_{Z\sim\mu_{\mathrm{pre}}}\big[\ell_{\mathrm{pre}}(w;Z)\big]$
denote the corresponding population-level pre-training loss.
The pre-training stage solves:
\begin{align}
\hat w_m \in \argmin_{w\in\R^{q_0}}\hat L_{\mathrm{pre}}(w;D_{\mathrm{pre}}^{(m)}) :=  \frac{1}{m}\sum_{j=1}^m \ell_{\mathrm{pre}}(w;z_j). \label{eq:pretrain_estimator}
\end{align}
 Our notation deliberately abstracts away the specific form of the pre-training loss; the analysis applies broadly to
standard contrastive and representation-learning losses used in practice.

\paragraph{Downstream estimation.}
The downstream estimator uses both the pre-trained parameter
$\hat w_m \in \mathbb R^{q_0}$ and the downstream training data
$D_{\rm down}^{(n)}$ to compute:\footnote{When $\hat\theta_{m,n}$ is not unique, we define it as the minimum Euclidean norm solution.
}
\begin{align}
\hat{\theta}_{m,n}
\in \argmin_{\theta \in \mathbb{R}^p}
\hat{L}_{\mathrm{down}}(\theta; \hat{w}_m, D_{\mathrm{down}}^{(n)})\coloneq
\frac{1}{n}\sum_{i=1}^n
\bigl(y_i-\langle \theta,\psi(x_i,w)\rangle\bigr)^2  . \label{eq:downstream_estimator}
\end{align}
The resulting predictor for the downstream task is then
$\hat{f}_{m,n}(\cdot) \coloneq  \langle \hat{\theta}_{m,n}, \psi(\cdot,\hat{w}_m) \rangle \in \calH_{\hat{w}_m}$.

\paragraph{Final test risk.}
Let $(X_{\mathrm{new}},Y_{\mathrm{new}})$ be an independent test pair with the same distribution as $(X,Y)$ (cf.~\eqref{eq:X_and_Y_pair}).
We focus on a \emph{conditional} notion of test-time risk that conditions on the realized pre-training dataset and downstream design,
while still averaging over downstream label noise.
Specifically, write $X_{1:n} \coloneq (X_1,\dots,X_n)$ and define the (conditional) test-time risk:
\begin{align}
R(D_{\mathrm{pre}}^{(m)},X_{1:n})
\coloneqq \E[\bigl(Y_{\mathrm{new}}-\hat f_{m,n}(X_{\mathrm{new}})\bigr)^2
\,|\,D_{\mathrm{pre}}^{(m)},\,X_{1:n}].
\label{eq:main_risk_cond}
\end{align}
With this notation in place, the main goal of this work is the following asymptotic characterization.
\begin{tcolorbox}[colback=black!3,colframe=black!60,boxrule=1pt,arc=2pt,left=3pt,right=3pt,top=3pt,bottom=3pt]
\textbf{Goal:}
Characterize as $(m,n)\to(\infty,\infty)$, with $m/n\to\alpha\in(0,\infty)$, the joint-sample limit:
\begin{align}\label{eq:risk_convergence_goal}
\mathcal{E}_{m,n}
\coloneq
n\Bigl(R(D_{\mathrm{pre}}^{(m)},X_{1:n})-\sigma^2\Bigr)
\;\distconv\;
\mathcal{E}_\alpha.
\end{align}
\end{tcolorbox}
\paragraph{Interpreting the distributional limit \eqref{eq:risk_convergence_goal}.}
From \eqref{eq:risk_convergence_goal}, several important implications follow.
By Fatou's lemma, we have 
the lower bound $\E[ \mathcal{E}_\alpha ] \leq \lim_{m,n} \E[ \mathcal{E}_{m,n} ]$ (here, $\lim_{m,n}$ is understood as $(m,n) \to \infty$ with $m/n\to \alpha$), which gives statistical \emph{lower bounds} on the downstream performance. If the sequence $\{ \mathcal{E}_{m,n} \}_{m,n}$ can be shown to be uniformly integrable, then this lower bound can be upgraded to equality, which yields an \emph{exact characterization} of the asymptotic excess risk in expectation.
Absent uniform integrability, we can still compute exact asymptotic high-probability upper bounds: since $\Pr( \mathcal{E}_\alpha \geq t ) = \lim_{m,n} \Pr( \mathcal{E}_{m,n} \geq t )$ for any $t > 0$ (assuming $\mathcal{E}_\alpha$ is a continuous distribution), letting $t(\delta)$ be such that $\Pr( \mathcal{E}_\alpha \geq t(\delta)) = \delta$, we have that
$\Pr( \mathcal{E}_{m,n} \geq t(\delta) ) = \delta + o_{m,n}(1)$.
In this work, we do not show uniform integrability,
as this generally requires additional small-ball assumptions~\citep{mourtada2022exact} on the features,
which are not actually needed for \eqref{eq:risk_convergence_goal} to hold (cf.~\Cref{rem:fully-avg-risk}).

The risk definition \eqref{eq:main_risk_cond} takes the conditional expectation over the randomness in $(X_{\mathrm{new}},Y_{\mathrm{new}})$ and over the downstream label noise in
$D_{\mathrm{down}}^{(n)}$ (i.e., over $(Y_1,\dots,Y_n)$ conditional on $X_{1:n}$), holding fixed $D_{\mathrm{pre}}^{(m)}$ and $X_{1:n}$.
This is intentional, as it allows the limit analysis to focus on the interaction
between the pre-train and fine-tune covariates $(D^{(m)}_{\mathrm{pre}}, X_{1:n})$; see \Cref{app:decomp-risk-general} for more discussion.

\section{Symmetries of the Two-Stage Pipeline}
\label{sec:identifiability}

\Cref{sec:problem_formulation:losses}
defines a two-stage $M$-estimation procedure
via \eqref{eq:pretrain_estimator} and \eqref{eq:downstream_estimator},
which in principle can be analyzed
using the standard toolkit for classical $M$-estimation: consistency, asymptotic normality, and delta-method expansions~\citep{van2000asymptotic}. 
However, a crucial technical hurdle is that many pre-training objectives are invariant under symmetries, i.e., there exists a compact Lie group $G$ acting \emph{smoothly} on feature parameters $\R^{q_0}$ such that
\begin{align}
\ell_{\mathrm{pre}}(g \cdot w;z) = \ell_{\mathrm{pre}}(w;z) ,\;
\text{for all } g \in G,\ w \in \R^{q_0},\ z \in \mathcal{Z}.
\label{eq:pretrain_group_invariance}
\end{align}
Concretely, consider a simple setting
where $\mathcal{Z} = \R^d$ and we aim to learn a linear representation $\psi(x, A) = A x$ with $A \in \R^{p \times d}$
in pre-training.
Now, consider any family of pre-training losses (e.g., 
contrastive losses such as SimCLR~\citep{chen2020simclr,oord2018representation})
that act on this representation 
through a {similarity measure} $\mathrm{sim}(x, x'; A) := \langle \psi(x, A), \psi(x', A) \rangle$.
This similarity measure, and hence the pre-training loss, is invariant under any orthogonal transform $Q \in O(p)$, i.e., 
$\mathrm{sim}(x, x'; A) = \mathrm{sim}(x, x'; QA)$.

Consequently, population minimizers are typically not identifiable: if $w$ minimizes $L_{\mathrm{pre}}(w)$, then the entire orbit
$[w] \coloneqq \{ g \cdot w\mid g \in G \}$ does as well. This lack of identifiability rules out both consistency
and asymptotic normality  
for the direct parameter $w$.
One of our key contributions
is to make explicit the 
types of symmetry encountered in self-supervised pre-training---in a way that remains compatible
with asymptotic analysis---utilizing concepts
from Riemannian geometry and smooth manifolds~\citep{lee2018introduction}.

\subsection{Manifold Identifiability and Asymptotic Normality}
\label{sec:manifold_identifiability}

The invariance \eqref{eq:pretrain_group_invariance} implies that the intrinsic parameter is an orbit $[w]$,
i.e., an element of the quotient $\R^{q_0}/G$. Rather than work directly with the quotient, we represent it through a
\emph{descriptor map} $D:\R^{q_0}\mapsto\R^q$ that is \emph{(a)} constant along orbits: $D(g\cdot w)=D(w)$ for all $g\in G,\ w\in\R^{q_0}$,
and \emph{(b)} {separates orbits} locally around $w_\star$:
$\exists r>0$ such that for $w,w'\in B(w_\star,r)$, we have
$D(w)=D(w')$ iff $w' \in [w]$.
We will also require that
$\mathcal{M}\coloneq D\bigl(B(w_\star,r')\bigr)\subset\mathbb R^q$, for some $0 < r' \leq r$, is a well-defined $C^2$ embedded sub-manifold
with $\Omega_\star \coloneq D(w_\star)$ in its interior;
\Cref{ass:local-quotient-chart} details the minimal assumptions needed
for $D$ to satisfy these requirements.
We endow $\mathcal M$ with the Riemannian metric inherited from the ambient Euclidean space $\R^q$.
Accordingly, for $\Omega\in\mathcal M$ we write $T_\Omega\mathcal M$ for the tangent space, and denote by
$\mathrm{grad}\,L_{\mathrm{pre}}(\Omega)\in T_\Omega\mathcal M$ and
$\mathrm{Hess}\,L_{\mathrm{pre}}(\Omega):T_\Omega\mathcal M\mapsto T_\Omega\mathcal M$ the Riemannian gradient and Hessian of
$L_{\mathrm{pre}}$ restricted to $\mathcal M$.

\paragraph{Identifiability in descriptor coordinates.}
We will study pre-training through the induced estimator
$\hat\Omega_m\coloneqq D(\hat w_m)$, rather than through the non-identifiable representative $\hat w_m$ itself.
We start by assuming the population minimizer
$\Omega_\star
:=\argmin_{\Omega\in D(\R^{q_0}) \coloneqq  \{D(w)\mid w\in\R^{q_0}\}}
L_{\mathrm{pre}}(\Omega)$ is unique in the descriptor space.
We overload the notation $L_{\mathrm{pre}}(\Omega)\coloneqq L_{\mathrm{pre}}(w)$ for any $w\in D^{-1}(\Omega)$, which is
well-defined as $L_{\mathrm{pre}}$ is $G$-invariant (cf.~\eqref{eq:pretrain_group_invariance}).
Since $\Omega_\star$ is unique and lies in the interior of the manifold $\mathcal{M}$, 
we have
$\mathrm{grad}\,L_{\mathrm{pre}}(\Omega_\star)=0$.
To obtain local quadratic control and a well-posed second-order expansion on $\mathcal{M}$, we further assume that 
$\mathrm{Hess}\,L_{\mathrm{pre}}(\Omega_\star)$
is invertible on the tangent space: there exists $\mu>0$ such that for all $v \in T_{\Omega_\star} \mathcal{M}$,
$\big\langle v,\ \mathrm{Hess}\,L_{\mathrm{pre}}(\Omega_\star)\,v \big\rangle
 \ge
 \mu\,\|v\|^2$.

\paragraph{Asymptotic normality on the descriptor manifold.}
Let
$H_\star \coloneq \mathrm{Hess}\,L_{\mathrm{pre}}(\Omega_\star)$
denote the Hessian of the population pre-training loss, and let
$\Sigma_\star \coloneq
\E\!\left[
\mathrm{grad}\,\ell_{\mathrm{pre}}(\Omega_\star;Z)
\mathrm{grad}\,\ell_{\mathrm{pre}}(\Omega_\star;Z)^\top
\right]$
denote the Fisher Information of the pre-training score.
Writing $v_m \coloneqq\log_{\Omega_\star}(\hat\Omega_m)\in T_{\Omega_\star}\mathcal M$, then under the regularity conditions stated formally in~\Cref{ass:riem_mestimation}, 
we show (\Cref{thm:riem_mestimator_clt}) that:
\begin{align}
\sqrt m\,v_m
\;\distconv\;
\mathcal N\!\left(0,\;H_\star^{-1}\Sigma_\star H_\star^{-1}\right)
\qquad\text{in }T_{\Omega_\star}\mathcal M. \label{eq:manifold_dist_conv}
\end{align}
\Cref{eq:manifold_dist_conv} is the manifold analogue of classical asymptotic normality for the pre-training estimator. It follows by extracting out the asymptotic normality argument from \cite{brunel2023geodesically}, replacing geodesic convexity assumptions with
local smooth regularity conditions that
apply to our pre-training setting; see 
\Cref{app:mestimation_riem} for a detailed discussion.

\subsection{Relating Pre-training and Downstream Estimation via Orthogonal Equivariance}
\label{sec:orthogonal_equivariance}

As detailed in \Cref{sec:problem-formulation},
the downstream stage depends on the pre-training stage 
through the representation map $\psi(\cdot,\hat{w}_m)$.
However, as the symmetry in \eqref{eq:pretrain_group_invariance}
precludes $\hat{w}_m$ from asymptotically converging to a fixed optimal parameter value, this implies that the convergence of both 
$\psi(\cdot, \hat{w}_m)$ and its induced linear hypothesis class
$\mathcal{H}_{\hat{w}_m}$ are not well-defined
without extra structure.
To handle this issue, we assume that the symmetry of pre-training is compatible with downstream prediction in the sense that the induced
hypothesis class is \emph{orbit-invariant}:
$\mathcal H_{g\cdot w}=\mathcal H_w$ for all 
$g\in G,\ w\in\R^{q_0}$.
This condition expresses that different representatives within the same orbit $[w]$ yield the same family of predictors.
However, this assumption alone is insufficient:
when the OLS minimizer \eqref{eq:downstream_estimator} is not
unique, different pre-training parameters $\hat{w}_m$
within the same orbit can still lead to different minimum-norm choices.
Therefore, we introduce the following condition,
which we call \emph{orthogonal equivariance} to address
this issue. 
Specifically, we assume there exists a homomorphism
$\rho:G\mapsto O(p)$ (i.e., $\rho(g_1 g_2)=\rho(g_1)\rho(g_2)$ for all $g_1,g_2 \in G$)
such that
\begin{align}
\psi(x,g\cdot w)=\rho(g)\,\psi(x,w),
\; \text{for all }g\in G,\ w\in\R^{q_0},\ x\in\mathcal X.
\label{eq:psi_orth_equiv}
\end{align}
Under condition \eqref{eq:psi_orth_equiv}, 
since different representatives $w$ in the same orbit correspond to orthogonal coordinate changes in feature space, 
the corresponding OLS minimizer will be constant on the orbit.
\begin{lemma}[Orbit-invariance of the minimum-norm downstream predictor]
\label{lem:orbit_invariant_min_norm}
Assume the orthogonal equivariance condition \eqref{eq:psi_orth_equiv}.
Fix a downstream dataset $D_{\mathrm{down}}^{(n)}$ and a parameter $w\in\R^{q_0}$. Let $\hat\theta_w$ denote the minimum norm solution of the downstream OLS~\eqref{eq:downstream_estimator}
with features $\psi(\cdot, w)$, and $\hat f_w(x)\coloneqq \langle \hat\theta_w,\psi(x,w)\rangle$. Then for every $g\in G$,
$\hat f_{g\cdot w}(\cdot)=\hat f_w(\cdot)$.
\end{lemma}

\Cref{lem:orbit_invariant_min_norm} states that under \eqref{eq:psi_orth_equiv}, an \emph{intrinsic} downstream feature map can be naturally defined on the orbit $[w]$ of each parameter. To see this, since $D$ is constant on orbits, it induces a map on the quotient, and we use the descriptor $\Omega:=D(w)$ as a representative coordinate for the orbit $[w]$. Then, \Cref{lem:orbit_invariant_min_norm} implies 
the feature map $\phi(x,\Omega)\coloneqq \psi\bigl(x,s(\Omega)\bigr)$ is intrinsic for $\Omega \in \mathcal{M}$,
where
$s:\mathcal{M}\cap B(\Omega_\star,r'')\mapsto B(w_\star,r')$ is a $C^2$ local \emph{lift} for some $r''>0$, satisfying $D(s(\Omega))=\Omega$.\footnote{The choice of $s$ is not unique; \Cref{app:riem_bundle} develops a vector-bundle viewpoint showing that
\emph{(i)} the induced representation is well-defined on the quotient and
\emph{(ii)} the feature map $\phi(x, \Omega)$ is differentiable whenever $s$ and $\psi$ are.}

\section{Main Result: Asymptotic Behavior of the Test Risk}
\label{sec:main}

We now state our main result, which characterizes the joint-sample asymptotic
behavior of the conditional test risk \eqref{eq:main_risk_cond}. We first introduce
the operator notation used in the theorem statement.

\subsection{Operator Notation and First-Order Residual Expansions}
Let $\mathcal H_\Omega\subset L^2$ denote the downstream function class induced by $\phi(\cdot,\Omega)$ and $\Pi_\Omega$ be the population $L^2$-orthogonal projector onto
$\mathcal H_\Omega$. Define the residual $e_\Omega \coloneq (I-\Pi_\Omega)f_\star$ and $\mathrm{Rep}(\Omega)\coloneq\|e_\Omega\|_{L^2}^2$. Note that under our well-specified setting $f_\star \in \calF$, $e_{\Omega_\star}=0$. Define the effective dimension
$d_{\mathrm{eff}}(\Omega) \coloneqq \mathrm{dim}(\mathcal H_\Omega) =  \rank(\Sigma(\Omega))$ for a descriptor $\Omega$.

The pretraining contribution is controlled by the first-order behavior of the residual $e_\Omega$ around \(\Omega_\star\). We assume that this residual admits a first-order expansion in normal coordinates, i.e., there exists a bounded linear map $\mathcal L:T_{\Omega_\star}\mathcal M \mapsto L^2$ such that $e_{\exp_{\Omega_\star}(v)}
=
\mathcal L(v)+o(\norm{v})$ in $L^2$. \Cref{app:rep} gives sufficient structural conditions for such an expansion, which we now highlight. %
Let $N_\star\coloneq\ker(T_{\Omega_\star})$, $E_\star:=N_\star^\perp$, $\mathcal A_v:=\operatorname{Im}\left(T_{\exp_{\Omega_\star}(v)}|_{E_\star}\right)$
and define the \emph{activated null-direction span}:
\[
\mathcal B_v
:=
\operatorname{Im}\left(
(I-\Pi_{\mathcal A_v})T_{\exp_{\Omega_\star}(v)}|_{N_\star}
\right).
\]
We assume that there exists a stable limiting span of null directions
(\Cref{ass:stable-null-span}), meaning that the projectors
\(\Pi_{\mathcal B_v}\) converge in operator norm to a limiting projector
\(\Pi_{\mathcal B_0}\), with \(\mathcal B_0\subseteq \mathcal H_{\Omega_\star}^{\perp}\). Define the active dimension $d_{\mathrm{act}}(\Omega_\star)\coloneq
\dim\big(\mathcal H_{\Omega_\star}\oplus \mathcal B_0\big)$, which denotes the total limiting downstream degrees of freedom: the original effective dimension $d_\mathrm{eff}(\Omega_\star)$ plus the activated null directions. Let \(\Pi_\star\coloneq\Pi_{\Omega_\star}\), and let \(\theta_\star\in E_\star\) denote the unique coefficient vector with \(T_{\Omega_\star}\theta_\star=f_\star\). Under these conditions, the residual linearization has the form:
\begin{equation}
\label{eq:linear_map_interaction}
\mathcal L(v)
=
-\big(I-\Pi_\star-\Pi_{\mathcal B_0}\big)DT_{\Omega_\star}[v]\theta_\star,
\qquad v\in T_{\Omega_\star}\mathcal M.
\end{equation}
The projector \(I-\Pi_\star-\Pi_{\mathcal B_0}\) removes both the original well-specified span and the limiting active null-direction span; only the remaining first-order variation contributes to representation error. For the regular case \(\mathcal B_0=\{0\}\), 
the active and effective dimension coincide, i.e., $d_{\mathrm{act}}(\Omega) = d_{\mathrm{eff}}(\Omega)$.

\subsection{Asymptotic Behavior of the Conditional Excess Test Risk}
With this first-order expansion in place, we now turn to our main object of study: the scaled excess test-risk
$\mathcal{E}_{m,n}$ in \eqref{eq:risk_convergence_goal},
obtained by conditioning on the realized
pre-training sample and downstream design, and averaging only over the downstream label noise and the test pair.
Accordingly, $\mathcal{E}_{m,n}$ is a random variable measurable with respect to the joint law $(D_{\mathrm{pre}}^{(m)},X_{1:n})$,
which we take all the convergences below with respect to. The following is our main result,
which characterizes the asymptotic behavior of the conditional excess test risk.

\begin{tcolorbox}[colback=black!3,colframe=black!60,boxrule=1pt,arc=2pt,left=3pt,right=3pt,top=3pt,bottom=3pt]
\begin{theorem}[Main result: asymptotic behavior of the conditional excess test risk]
\label{thm:master-compatible}
Assume both the pre-training regularity conditions in \Cref{ass:riem_mestimation}, in addition to the downstream regularity conditions Assumption~\ref{ass:well-posedness}--\ref{ass:stable-null-span}. Then, along any joint sequence $(m,n)\to(\infty,\infty)$ with $m/n\to \alpha\in(0,\infty)$,
the (scaled) conditional excess risk $\mathcal{E}_{m,n}$ admits the distributional limit
\begin{equation}
\label{eq:master-compatible-dist}
\mathcal{E}_{m,n} = n\Big(R(D_{\mathrm{pre}}^{(m)},X_{1:n})-\sigma^2\Big)
\ \distconv\
\underbrace{\sigma^2\,d_{\mathrm{act}}(\Omega_\star)\vphantom{ \|\mathcal L(Z)\|_{L^2}^2}}_{\text{OLS term on active limiting class}}
+ 
\underbrace{\alpha^{-1}\,\|\mathcal L(Z)\|^2_{L^2}}_{\text{Pre-training interaction term}},
\end{equation}
where $Z\sim\mathcal N(0,V)$ with $V := H_\star^{-1} \Sigma_\star H_\star^{-1}$, where $H_\star$ (resp.\ $\Sigma_\star$) denotes the Hessian (resp.\ Fisher Information matrix) of the pre-training loss (cf.~\Cref{sec:manifold_identifiability}).
\end{theorem}
\end{tcolorbox}
\Cref{thm:master-compatible} shows that 
the (scaled) conditional excess risk $\mathcal{E}_{m,n}$ 
converges in distribution to a random variable 
with two distinct terms.
The first term, $\sigma^2 d_{\mathrm{act}}(\Omega_\star)$,
is the downstream OLS degrees-of-freedom contribution on the limiting active class $\mathcal H_{\Omega_\star}\oplus \mathcal B_0$. The second term
$\alpha^{-1} \norm{ \mathcal{L}(Z) }^2_{L^2}$ on the other hand captures the \emph{interaction} between the pre-training loss and the
downstream regression problem.
As the pre-training data starts to dominate the
downstream data (i.e., $\alpha \to \infty$),
the contribution of the second term correctly vanishes. The remaining term is the downstream OLS degrees-of-freedom contribution on the limiting active class; in the regular case $\mathcal B_0=\{0\}$, this reduces to the risk of well-specified OLS on $\mathcal H_{\Omega_\star}$. On the other hand, when pre-training data is relatively scarce compared to downstream (i.e., $\alpha \to 0$), the second term is dominant in \eqref{eq:master-compatible-dist}, as the bias of the learned pre-training features 
becomes the limiting factor in the two-stage estimator.
In \Cref{sec:examples}, we instantiate
\Cref{thm:master-compatible} on several examples to illustrate how the assumptions translate over to 
concrete problem instances, and how the pre-training
interaction term scales in practice.

\paragraph{Downstream-only baseline.}
To further interpret \Cref{thm:master-compatible} and the risk-crossover picture in \Cref{fig:intro-transition-pretraining}, we compare to the empirical risk minimizer over $\mathcal{F}$ using only $D_{\mathrm{down}}^{(n)}$, in the regular case $\mathcal B_0 = \{0\}$. Specifically, define $f_n^{\mathrm{base}}\in\argmin_{f\in\mathcal F}\frac1n\sum_{i=1}^n (y_i-f(x_i))^2$. For a regular realization $(\theta_\star,w_\star)$ of $f_\star$, define $d_{\mathrm{base}}$ as the rank of the linear map
\[(\delta\theta,\delta w)\mapsto\left.\frac{\rmd}{\rmd t}\right|_{t=0}f_{\theta_\star+t\delta\theta,\;w_\star+t\delta w}.\]
Thus $d_{\mathrm{base}}$ is the dimension of the tangent space to the full downstream-only class \(\mathcal F\) at \(f_\star\). This should be contrasted with \(d_{\mathrm{eff}}(\Omega_\star)\), which is the local dimension of the fixed representation class \(\mathcal H_{\Omega_\star}\) appearing after pre-training. Assuming the standard regularity conditions for nonlinear least squares at the
function level---i.e., differentiability of the parameterization, finite moments, and constant rank of the tangent map in a neighborhood of a regular realization
$(\theta_\star,w_\star)$---the downstream-only baseline has leading scaled excess-risk limit $\sigma^2 d_{\mathrm{base}}$, whereas \Cref{thm:master-compatible} gives the two-stage limit $\sigma^2 d_{\mathrm{eff}}(\Omega_\star)+\alpha^{-1}\|\calL(Z)\|_{L^2}^2$. Therefore, if $d_{\mathrm{base}}>d_{\mathrm{eff}}(\Omega_\star)$,\footnote{
By \Cref{prop:dimension-gap-L}, if
$d_{\mathrm{base}}=d_{\mathrm{eff}}(\Omega_\star)$, then $\mathcal L\equiv 0$. In this case, pre-training does not reduce the local downstream degrees of freedom; it only identifies an equivalent parametrization
(e.g., invertible coordinate change).
} then, under uniform integrability of the downstream-only and two-stage scaled excess risks, the two-stage procedure has smaller asymptotic risk in expectation whenever $\alpha>\alpha_0\coloneq \frac{\E\|\calL(Z)\|_{L^2}^2}{\sigma^2(d_{\mathrm{base}}-d_{\mathrm{eff}}(\Omega_\star))}$.

\section{Examples}
\label{sec:examples}

In this section, we apply our main result (\Cref{thm:master-compatible}) on a few concrete self-supervised learning examples.
For each setting there are two key steps: \emph{(i)} defining the minimal problem-specific structure to satisfy the assumptions of \Cref{thm:master-compatible},
and \emph{(ii)} calculating the instance-specific bound from \eqref{eq:master-compatible-dist}. Here, we describe the main setup and assumptions, deferring specific computations to the appendix. 
For our examples, we restrict attention to the regular case $\mathcal B_0=\{0\}$. Geometrically, this rules out the activation of coefficient directions that are
null at \(\Omega_\star\) under infinitesimal perturbations of the descriptor.
Under Assumptions~\ref{ass:moments} and~\ref{ass:phi-C1-M}, this is precisely the setting in which the population projector
\(\Omega\mapsto\Pi_\Omega\) is Fréchet differentiable at \(\Omega_\star\). In this case, the residual linearization is given by
\begin{equation}
\label{eq:linear_map_interaction_zero}
\mathcal L(v)
=
-D\Pi_{\Omega_\star}[v]f_\star= \big(I-\Pi_{\Omega_\star}\big)DT_{\Omega_\star}[v]f_\star,
\qquad v\in T_{\Omega_\star}\mathcal M.
\end{equation}

\subsection{Linear Spectral Pre-training}
\label{sec:example_linear_spectral}

Inspired by the problem setting considered in \citet{cabannes23SSLinterplay}, we first consider a linear spectral contrastive objective; proofs for this case study are given in \Cref{app:proof-cor-linear}. 

\paragraph{Data model.}
Let $x\in\R^d$ denote a generic unlabeled pre-training sample with mean zero and covariance
$\Sigma_{\mathrm{pre}}\coloneqq \E[xx^\top]\in\R^{d\times d}$.
A \emph{positive pair} consists of two augmented views $(x,x^+)$ of the same underlying instance.
We define the cross-covariance matrix $\Sigma_{\mathrm{pre}}^{+}\coloneqq \E[x^+ x^\top]$ and assume that $(x,x^+)$ is exchangeable
which implies that
$\Sigma_{\mathrm{pre}}^{+}$ is symmetric. %
A \emph{negative sample} $x^-$ is an independent copy of $x$, independent of $(x,x^+)$.
We group the samples as $z\coloneqq (x,x^+,x^-)$.

\paragraph{Linear features and spectral loss.}
Fix a representation dimension $k\in[d]$ and consider linear {representative-level} feature maps $\psi(x,A)\coloneqq Ax\in\R^k$ for $A\in\R^{k\times d}$. For any such $A$, define the Gram matrix (descriptor) $M_A\coloneqq A^\top A\in\R^{d\times d}$. Here, we retain the notation $A$ for the representative parameter (corresponding to $w$)
and write $M$ for the Gram descriptor (corresponding to $\Omega$).
Motivated by prior work on spectral contrastive objectives~\citep{haochen2021provable}, we consider the following (single-negative) spectral loss and define the per-sample objective
\begin{align}
\label{eq:spec-loss-sample}
\ell_{\mathrm{spec}}(A;z)
&\coloneqq
-2\,\langle \psi(x, A),\psi(x^+,A)\rangle
+\langle \psi(x,A),\psi(x^-,A)\rangle^2. %
\end{align}

\paragraph{Symmetry, quotient, and descriptor.}
The loss \eqref{eq:spec-loss-sample} depends on $w$ only through $M_A=A^\T A$.
Let $G = O(k)$ denote the group of $k\times k$ orthogonal matrices acting on $\R^{k\times d}$ by left multiplication.
Then $\ell_{\mathrm{spec}}(QA;z)=\ell_{\mathrm{spec}}(A;z)$ for all $Q\in O(k)$ and all $z$,
and hence $\hat L_{\mathrm{pre}}$ and $L_{\mathrm{pre}}$ are invariant under this action
(cf.~\eqref{eq:pretrain_group_invariance}).
Moreover, the representative-level feature map is orthogonally equivariant (cf.~\eqref{eq:psi_orth_equiv}), that is, $\psi(x,Q\cdot A)=Q\psi(x,A)$ for all $x\in\R^d$ and $Q\in O(k)$. 
A natural orbit-invariant descriptor map is $D(A)\coloneqq M_A=A^\T A$ which is constant along $O(k)$-orbits.
On the regular regime $\mathrm{rank}(A)=k$, the action is free, and $D$ locally identifies nearby $O(k)$-orbits
with nearby points in the rank-$k$ PSD cone $\mathcal{M}_{d,k}\coloneqq \{M\in\R^{d\times d}\mid M\succeq 0,\ \mathrm{rank}(M)=k\}$.
In particular, $\mathcal{M}_{d,k}$ is a smooth embedded submanifold of the space of symmetric matrices, and we may
endow it with the induced Riemannian metric from the ambient Frobenius inner product. On a neighborhood of $M_\star$ there exists a $C^2$ local section $s$ with $s(M)^\top s(M)=M$, and we define the quotient feature map by $\phi(x,M):=s(M)x$; see Appendix~\ref{sec:quot-feat-linear} for the construction and smoothness.

\begin{assmp}%
\label{ass:C-positive-k}
$\Sigma_{\mathrm{pre}}$ is invertible and 
$\lambda_k(C) - \max\{\lambda_{k+1}(C),0\} > 0$ for $C \coloneq \Sigma_{\mathrm{pre}}^{-1/2}\Sigma_{\mathrm{pre}}^{+}\Sigma_{\mathrm{pre}}^{-1/2}$.
\end{assmp}
Under \Cref{ass:C-positive-k}, the population descriptor minimizer $M_\star\in\mathcal M_{d,k}$ exists and is unique,
which leads to the following verification of the assumptions.
\begin{proposition}[Linear spectral model]
\label{cor:linear-spectral-model}
Assume \Cref{ass:C-positive-k} and $\E_{\mu_{\mathrm{pre}}}\|Z\|^{4},\E_{\mu_{\mathrm{down}}}\|x\|^4<\infty$. Then the linear spectral model satisfies the assumptions of \Cref{thm:master-compatible}.
\end{proposition}

\paragraph{Concrete example.} While the limiting expression in \Cref{thm:master-compatible} admits an explicit characterization in this linear spectral model, its general form is involved. To obtain explicit closed-form expressions, we consider a simplified linear model inspired by \citet[Ex.~1]{saunshi2022inductivebias}. We focus on a diagonal setting that captures the essential spectral structure while allowing for precise calculations. 
Assume that $\Sigma_{\mathrm{pre}} = I_d$ and $\Sigma_{\mathrm{pre}}^{+} = \diag(1,1/2,\ldots,1/d)$, so that the whitened cross-covariance matrix is
$C = \diag(1,1/2,\ldots,1/d)$.
The top-$k$ population representation is therefore given by the first $k$
coordinates. We consider a linear downstream target $f_\star(x) = \beta_\star^\top A_\star x$ where $A_\star = [\diag(1,1/\sqrt{2},\dots,1/\sqrt{k}),0_{k\times (d-k)}]$ and $\beta_\star = (1,\ldots,1)^\top$. This ensures that the task is realizable. We further assume the downstream data distribution satisfies $\mu_{\mathrm{down}}=\mathcal{N}(0,I_d)$.
\begin{corollary}\label{cor:linear-concrete}
    Under the above setup, as $(m,n)\to(\infty,\infty)$ with $m/n\to\alpha\in(0,\infty)$,
    \[   n\Big(R(D_{\mathrm{pre}}^{(m)},X_{1:n})-\sigma^2\Big)
\ \distconv\
\sigma^2\,k+\frac{2}{\alpha}\,\left(\sum_{i=1}^k i(1+i^{-2}+\tau) \right)\,\chi^2_{d-k}, \quad \tau = \sum_{j=1}^k j^{-2}. 
\]
In particular, the pre-training interaction term $\frac{1}{m}\|\mathcal L(Z)\|_{L^2}^2$ scales as $\Theta\big(\frac{k^2(d-k)}{m}\big)$ w.h.p.
\end{corollary}
We compare~\Cref{cor:linear-concrete} with \citet[Thm.~4]{cabannes23SSLinterplay}. Applied to this setup, 
\citet[Thm.~4]{cabannes23SSLinterplay} yields (ignoring $\log{n}$ factors) that
$\E[\calE_{m,n}] \lesssim \sigma^2 k + \alpha^{-1} k(d-k)\|C^{-1}A_\star^\top\beta_\star\|^2_2 \asymp \sigma^2 k + \alpha^{-1} k^3(d-k)$ (see~\Cref{sec:comparison-ssl}).
On the other hand, \Cref{cor:linear-concrete} implies that for large $m,n$, $\calE_{m,n} \asymp \sigma^2 k + \alpha^{-2} k^2(d-k)$ w.h.p.
Thus, our result improves the dependence in the second term by
factor of $k$.

\subsection{Factor Model Pre-training}
\label{sec:ex:factor}

We next specialize our main result to the latent factor example from \citet[Sec.~4]{ge2023provable}. This setting is structurally similar to \Cref{sec:example_linear_spectral}: the latent low-rank structure makes unsupervised pretraining natural, while a rotational invariance renders representation-level parameters non-identifiable. Proofs for this case study are presented in \Cref{app:proof-factor}.

\paragraph{Factor model and pre-training.}
Let $x\in\mathbb R^d$ be generated according to the factor model $x = A_\star h + \mu$ where $h\sim\mathcal N(0,I_k)$ is a latent factor, $\mu\sim\mathcal N(0,I_d)$ is independent noise,
and $A_\star\in\mathbb R^{d\times k}$ is an unknown full-rank factor loading matrix. We assume that $k\ll d$. 
Pre-training observes $m$ i.i.d.\ draws $\{z_i\}_{i=1}^{m}$ from the same distribution as $x$, and uses MLE to estimate the factor loading matrix,
i.e., $\ell_{\mathrm{pre}}(A; z_i) = \frac{1}{2}\left( \log\det( I_d + AA^\top ) + z_i^\top (I_d + AA^\top)^{-1} z_i \right)$.

\paragraph{Downstream regression.}
We observe labeled samples $(x,y)$ satisfying $y = \beta_\star^\top h + \nu$ where $\nu\sim\mathcal N(0,\sigma_{\nu}^2)$ and independent of $(h,\mu)$. The downstream learner does not observe $h$ and fits a linear predictor based on features extracted from $x$
using the pretrained representation.
Specifically, under the Gaussian factor model, the regression function is linear and admits a closed form $f_\star(x)= \beta_\star^\top A_\star^\top (I_d + A_\star A_\star^\top)^{-1}x$. Accordingly, we can write the downstream labels as $Y = f_\star(X) + \varepsilon$ where $\varepsilon \mid X\sim \mathcal{N}(0,\sigma^2)$ and $\sigma^2\coloneq\sigma_{\nu}^2 + \beta_\star^\top(I_d+A_\star A_\star^\top)^{-1}\beta_\star$.

\paragraph{Reduction to the linear case.}
For any candidate loading matrix $A\in\R^{d\times k}$ 
we define the 
{representative-level} feature map $\psi(x,A)\coloneqq W(A) x \in \R^k$ with $W(A) \coloneq A^\top(I_d+AA^\top)^{-1}$.
Writing the orbit-invariant descriptor as $M=AA^\top\in\mathcal M_{d,k}$ and choosing any local section $s(M)$ with $s(M)s(M)^\top=M$, any representative $A$ on the same orbit can be expressed as $A=s(M)Q$ for some $Q\in O(k)$, which yields $\psi(x,A)=Q^\top \phi(x,M)$ with $\phi(x,M)\coloneqq s(M)^\top(I_d+M)^{-1}x$. Thus, passing from $A$ to $M$ removes the rotational non-identifiability. Consequently, this factor-model instance is a direct specialization of the linear example at the descriptor level, with the downstream class determined solely by the $k$-dimensional subspace $\mathrm{range}(M)$. Let $M_\star = U_\star
\diag(\Sigma_\star,0_{d-k})U_\star^\top$ with $\Sigma_\star=\operatorname{diag}(\sigma_1,\dots,\sigma_k)$ and
$U_\star=[U_1\ U_2]$, where $U_1\in\R^{d\times k}$ spans $\operatorname{range}(M_\star)$.

\begin{corollary}
\label{cor:factor}
For the factor model example,
as $(m,n)\to(\infty,\infty)$ with $m/n\to\alpha\in(0,\infty)$,
\[
n\Big(R(D_{\mathrm{pre}}^{(m)},X_{1:n})-\sigma^2\Big)
\ \distconv\
\sigma^2\,k+\frac{1}{\alpha}\,\|(I_k+\Sigma_\star)^{-1/2}\,U_1^\top A_\star\,\beta_\star\|_2^2\,\chi^2_{d-k}.
\]
\end{corollary}
Compared to \citet[Thm.~4.4]{ge2023provable},
which states that w.h.p.,\footnote{Technically, \citet[Thm.~4.4]{ge2023provable} controls the excess risk \emph{without} averaging over the conditional labels, as we do in \eqref{eq:main_risk_cond}.
Since we can bound their excess risk definition
by a constant factor of $R(D_{\mathrm{pre}}^{(m)},X_{1:n})$ via Markov's inequality (on a constant probability event), we ignore this difference in our comparison.
} 
$\mathcal{E}_{m,n} \lesssim D^4 k + \alpha^{-1} D^{12} (D^4 + \sigma_{\min}^{-4}(A_\star)) d$ whenever $m \gtrsim D^4 d$, $n \geq D^4 k$ where $D \coloneq \max\{ \opnorm{A_\star}, \norm{\beta_\star} , 1 \}$
(here, we ignore all $\log(1/\delta)$ for simplicity),
we see that \Cref{cor:factor} provides a substantial improvement. 
In particular, it implies a sharper bound of the form
$\mathcal{E}_{m,n} \asymp (\sigma_\nu^2 + D^2) k + \alpha^{-1} D^2 (d-k)$ w.h.p. for large $m,n$, yielding a significant improvement on the dependence of $D$.

\subsection{Gaussian Mixture Pre-training with Subspace-Aware Gating}
\label{sec:ex-mog}

Our final example considers unlabeled pretraining data drawn from a Gaussian mixture (MoG) with
unknown centers, while the downstream predictor uses \emph{subspace-aware} posterior responsibilities that depend
only on a low-dimensional centered-mean subspace.
This example is motivated by the latent classification models studied in e.g.\ \citet{wei2021pretraining,ge2023provable,lin2025statistical}, which we naturally extend to the regression setting. 
From a technical perspective, it also
instantiates the quotient-descriptor viewpoint in a setting with discrete non-identifiability. 
As before, we discuss only the
minimal ingredients needed to invoke~\Cref{thm:master-compatible}, and we defer many details to \Cref{app:proof-cor-mog}.

\paragraph{Pre-training data and loss.}
Fix $d\ge 1$ and $K\ge 2$. Let $U^\star=(u_1^\star,\dots,u_K^\star)\in(\R^d)^K$ be unknown centers and let
$\tau\sim\Unif([K])$. The unlabeled distribution is the MoG $Z \mid (\tau=i) \sim \mathcal N(u_i^\star, I_d)$ for $i\in[K]$. 
Given $m$ i.i.d.\ pre-training samples $Z_i \stackrel{\iid}{\sim} \frac{1}{K}\sum_{i=1}^K \mathcal{N}(u_i^\star,I_d)$, we estimate $U^\star$ by MLE using $\ell_{\mathrm{pre}}(U;Z) \coloneqq -\log(\frac{1}{K} \sum_{i=1}^K \varphi(Z-u_i))$ where $\varphi(\cdot)$ is the density of $\calN(0, I)$.

\paragraph{Subspace-aware features.}
Define the empirical mean $\bar u(U):=\frac1K\sum_{i=1}^K u_i$ and centered second-moment matrix $S(U):=\sum_{i=1}^K \big(u_i-\bar u(U)\big)\big(u_i-\bar u(U)\big)^\top$. Let $r_\star:=\mathrm{rank}(S(U_\star))$ and define $P_U\in O(d)$ to be the orthogonal projector onto the leading $r_\star$-dimensional eigenspace of $S(U)$
(on the regular neighborhood where this eigenspace is well-defined);
centering via $S(U)$ removes an irrelevant global shift and isolates the \emph{effective} subspace in which the mixture geometry varies.
For $U$ in the regular neighborhood, define responsibilities based on the \emph{projected} mixture:
\begin{equation}
\label{eq:ex-mog-gating}
\pi_i(x;U)
:=
\frac{\exp\!\big(\langle P_U u_i,\;P_U x\rangle-\tfrac12\|P_Uu_i\|_2^2\big)}
{\sum_{j=1}^K \exp\!\big(\langle P_U u_j,\;P_U x\rangle-\tfrac12\|P_Uu_j\|_2^2\big)},
\; i\in[K].
\end{equation}
These are exactly the Bayes posteriors $\pi_i(x;U)=\Pr_U(Z=i\mid P_U X=P_U x)$ for the projected model
$P_UX\mid (Z=i)\sim \mathcal N(P_Uu_i, I_{r_\star})$.
Define the feature map $\psi_U:\R^d\mapsto\R^{K(d+1)}$ by
\begin{align}
\label{eq:ex-mog-feature}
\psi_U(x)
:=
\Big(
\pi_1(x;U) P_U (x-u_1),\, \pi_1(x;U),\;
\dots,
\pi_K(x;U) P_U (x-u_K),\, \pi_K(x;U)
\Big).
\end{align}
For parameters $\theta=(\theta_1,b_1,\dots,\theta_K,b_K)$ with $\theta_i\in\R^{r_\star}$ and $b_i\in\R$, the
induced predictor is the linear model in features:
$f_{\theta,U}(x) = \ip{\theta}{\psi_U(x)}$.

\begin{figure}[t]
    \centering
    \includegraphics[width=\textwidth]{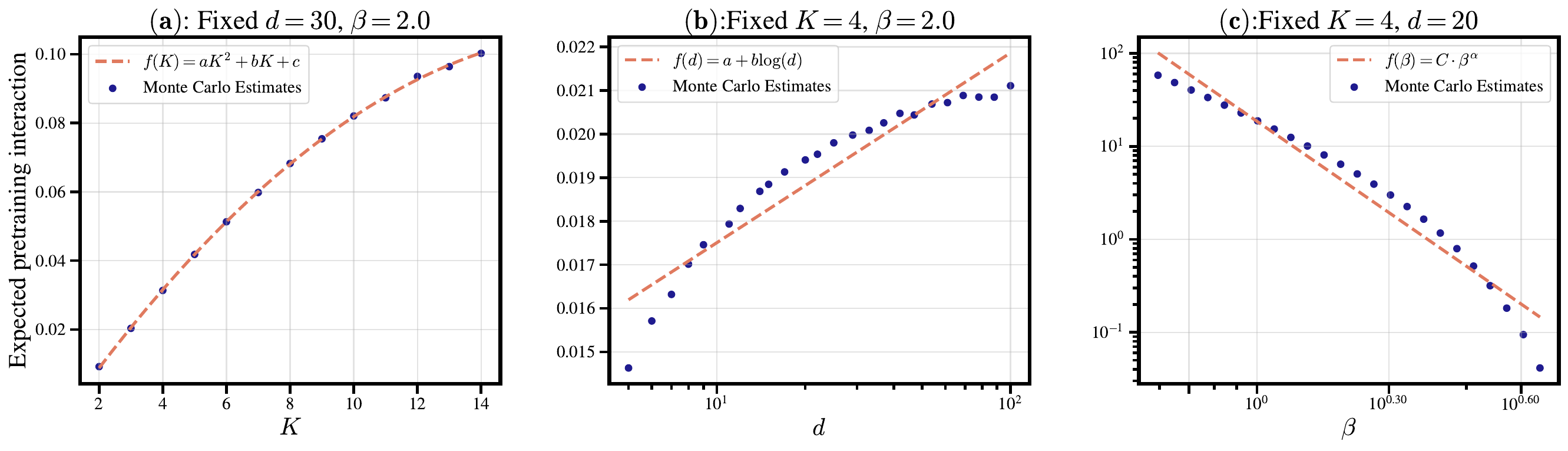}
    \caption{
        Monte Carlo evaluation of $\mathbb{E}[\|\mathcal{L}(Z)\|_{L^2}^2]$ in the MoG example with a block-structured downstream signal. (a) Varying \(K\) with \(d=30\) and \(\beta=2.0\). (b) Varying \(d\) with \(K=4\) and \(\beta=2.0\). (c) Varying \(\beta\) with \(K=4\) and \(d=20\). Dots denote Monte Carlo estimates and dashed curves denote simple fitted trends. See \Cref{app:experiment-details} for the full experimental setup and discussion.}
    \label{fig:gmm-rates-scan}
\end{figure}

\paragraph{Descriptor.}
The pretraining objective for the unlabeled mixture is invariant under permutations of the $K$ components. We take $G=S_K$, the permutation group, and define its action on the parameter $U=(u_1,\cdots,u_k)$ by relabeling.
The downstream hypothesis class depends on $U$ only through its orbit $\underline U=[U]\in (\R^d)^K$. \Cref{ass:ex-mog-regular} below guarantees that $U_\star$ lies in the regular regime of the group action. Thus, we can define the quotient feature via any local lift, i.e., choice of ordering. See~\Cref{app:proof-cor-mog:model} for the exact constructions. 

\begin{assmp}[Regular set for the mixture example]
\label{ass:ex-mog-regular}
The centers are distinct and the centered-mean subspace is locally stable,
i.e., 
(i) $u_i^\star\neq u_j^\star$ for all $i\neq j$,
(ii) the matrix 
has an eigengap
between its $r_\star$-th and $(r_\star+1)$-th eigenvalues, and
(iii) $\mathrm{rank}(\E[\psi_{U_\star}(X)\psi_{U_\star}(X)^\top])=K(r_\star+1)$.
\end{assmp}

\begin{proposition}%
\label{cor:GoM-model}
Under \Cref{ass:ex-mog-regular}, the MoG example satisfies
the assumptions of \Cref{thm:master-compatible}.
\end{proposition}

\paragraph{Explicit computations.}
While \Cref{cor:GoM-model} allows us to invoke \Cref{thm:master-compatible} in this MoG example, the limiting expression in \eqref{eq:master-compatible-dist} does not admit a simple closed form. The main difficulty is the term $\mathcal{L}(Z)$, which depends on derivatives of the projection operator $\Pi_\Omega$ and is not analytically tractable even in simple problem instances. Nevertheless, the limit in \eqref{eq:master-compatible-dist} can be evaluated numerically via Monte-Carlo simulation. \Cref{fig:gmm-rates-scan} reports such an evaluation for a simple instance with $u_i^\star=\beta e_i$, where $e_i\in\R^d$ denotes $i$-th coordinate. It shows a monotone, concave dependence on the number of blocks $K$, slow growth with the ambient dimension $d$, and rapid decay as $\beta$ increases.

To complement the scaling study in Figure~\ref{fig:gmm-rates-scan}, we also provide a finite-sample distributional comparison of the prediction in~\Cref{thm:master-compatible}. Figure~\ref{fig:gmm-cdfs} reports empirical cumulative distribution functions (CDFs) of the total scaled excess risk and its two leading components for a fixed GMM instance. As $n$ grows with $m/n=\alpha$ held fixed, the empirical distributions move toward the asymptotic laws predicted by~\Cref{thm:master-compatible}. Further experimental details are given in~\Cref{app:experiment-details}. Code for reproducing the simulations is available at \url{https://github.com/mtinati/mog-subspace-jax}.

\begin{figure}[t]
    \centering
    \includegraphics[width=\textwidth]{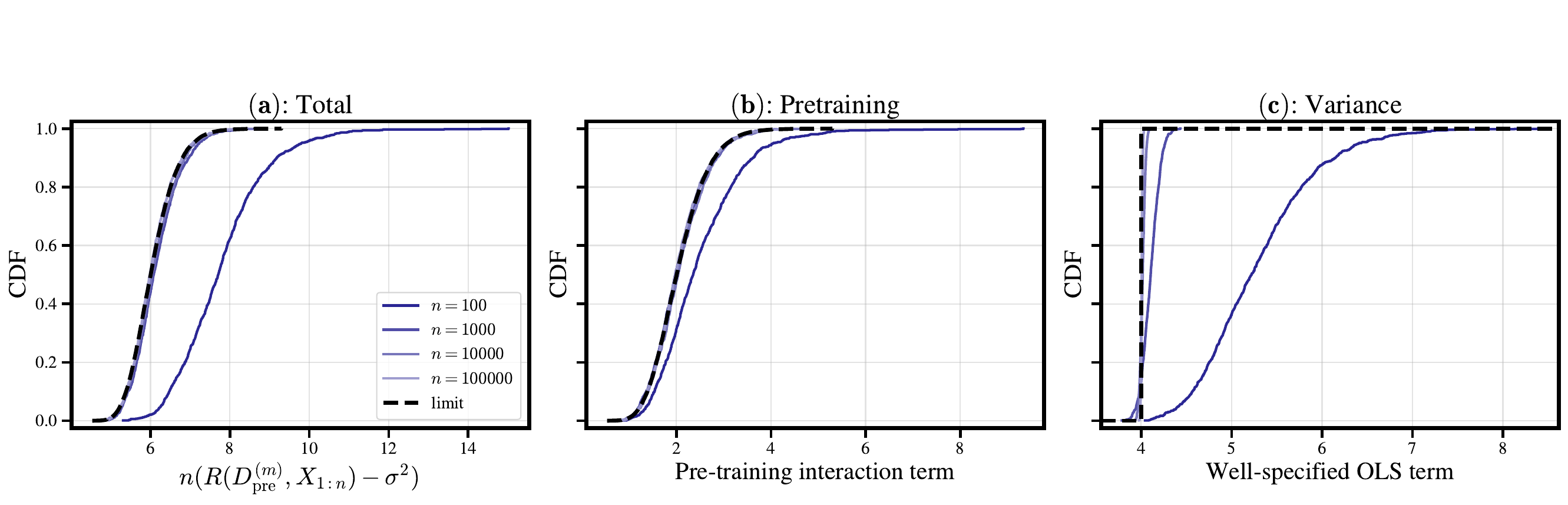}
    \caption{Finite-sample distributional convergence in the GMM example. We fix \(K=4\), \(d=20\), \(\beta=2\), and \(\alpha=m/n=2\), and compare empirical CDFs from repeated two-stage simulations with the asymptotic predictions of~\Cref{thm:master-compatible}. Plot (a) shows the total scaled excess risk \(n(R(D_{\mathrm{pre}}^{(m)},X_{1:n})-\sigma^2)\) and its limiting law \(\sigma^2 d_{\mathrm{eff}}+\alpha^{-1}\|\mathcal{L}(Z)\|_{L^2}^2\). Plot (b) isolates the pretraining contribution and compares it with the limiting fluctuation \(\alpha^{-1}\|\mathcal{L}(Z)\|_{L^2}^2\). Plot (c) shows the scaled variance contribution concentrating around the deterministic limit \(\sigma^2 d_{\mathrm{eff}}\). The CDF plots display representative values of $n$ from a logarithmic grid; dashed black curves denote the corresponding asymptotic limits.}
    \label{fig:gmm-cdfs}
\end{figure}

\section{Conclusion and Discussion}
\label{sec:conclusion}

We developed an asymptotic theory of self-supervised pre-training through a two-stage $M$-estimation framework, leveraging tools from Riemannian geometry to handle group symmetries arising in the pre-training stage. Our work opens up several promising future directions.
On the technical side, a natural next step is to extend the downstream model to more general parametric regression settings; the main challenge lies in generalizing our notion of orthogonal equivariance (cf.~\Cref{sec:orthogonal_equivariance}) to accommodate orbit-invariance for estimators beyond OLS solutions. Another direction
is developing non-asymptotic bounds for downstream risk whose leading-order terms match the asymptotic limits derived in this work.
Extending our framework to pre-training objectives with auxiliary pre-training heads that are discarded before transfer is also of interest.
Such objectives arise naturally in multi-task learning, where a shared representation $w$ is trained with task-specific heads $\{\theta_k\}$,
\[
\widehat L_{\mathrm{pre}}(w,\theta_1,\ldots,\theta_K)
=
\frac1m\sum_{i=1}^m
\sum_{k=1}^K
\ell_k(w,\theta_k;z_i),
\]
and in teacher-student settings, where a student representation $\psi(\cdot,w)$ is matched to a teacher signal through an auxiliary readout $\theta$. Since only $w$ is transferred downstream, the natural object is the profiled loss
\[
\widehat Q_m(w)
=
\inf_\theta \widehat L_{\mathrm{pre}}(w,\theta),
\]
viewed on the quotient space of representation parameters. Developing a quotient-level profile $M$-estimation theory for such objectives would allow auxiliary pre-training parameters to enter the asymptotic covariance through profile scores and Hessians, while keeping the downstream analysis centered on the learned representation descriptor.

More broadly, our asymptotic characterizations suggest a principled way to guide the design of pre-training losses and data-augmentation strategies, by directly optimizing the bound on the downstream risk over a diverse family of problem instances. This connection between asymptotic theory and practical pre-training design is an especially exciting avenue for future exploration.

\section*{Acknowledgments}
We thank the anonymous reviewers for their thoughtful comments and suggestions, which helped improve the manuscript. This work was partially supported by a grant from Coefficient Giving and an Okawa Foundation Research Grant.

\bibliography{paper}

\newpage

\setcounter{tocdepth}{2} %
\tableofcontents
\newpage 

\appendix
\crefalias{section}{appendix}
\crefalias{subsection}{appendix}
\section{More Detailed Related Work Discussion}
\label{sec:appendix:related_work}

\paragraph{Algorithmic approaches to pre-training and representation learning.}
Self-supervised pre-training has been developed through a diverse set of algorithmic paradigms, with the common goal of learning representations from unlabeled data that transfer effectively to downstream tasks. In natural language processing, masked-token prediction and autoregressive language modeling have been especially influential, as exemplified by BERT and GPT-style language models \citep{devlin2019bert,brown2020gpt3}.
Analogous masked-prediction and reconstruction-based approaches have also been highly influential in vision~\citep{he2022masked}. Contrastive learning methods---e.g.,
contrastive predictive coding \citep{oord2018representation}, SimCLR \citep{chen2020simclr}, MoCo \citep{he2020momentum}, and CLIP \citep{radford2021CLIP}---learn representations by bringing together semantically related positive pairs while separating unrelated negatives.
Algorithmic variants include modifying the comparison structure~\citep{caron2020unsupervised,dwibedi2021little,zhang2023trifactor} and negative-sampling mechanism~\citep{chuang2020debiased,robinson2021contrastive}.
A second line of joint-embedding methods avoids explicit negative samples, instead using architectural, statistical, or optimization mechanisms to prevent collapse; representative examples include BYOL \citep{grill2020BYOL}, SimSiam \citep{chen2021exploring}, Barlow Twins \citep{zbontar2021barlowtwins}, VICReg \citep{bardes2022vicreg}, and DINOv2 \citep{oquab2024dinov2}. 
Our work is complementary to algorithmic development;
our framework is agnostic to the specific pre-training loss used in the first stage.

\paragraph{Theoretical studies of pre-training and fine-tuning.}
A growing theoretical literature seeks to explain when
unsupervised pre-training produces representations that are useful for downstream prediction. Early theoretical work on contrastive learning introduced latent-class models and showed that representations with small contrastive loss can support sample-efficient downstream classification on related tasks \citep{saunshi2019analysis}. Related analyses study contrastive estimation in topic models, showing that contrastive representations contain topic-posterior information~\citep{tosh2021contrastive}, and pretext-task learning, where predicting one part of the input can provably reduce the labeled sample complexity of downstream learning under suitable conditional-independence assumptions \citep{lee2021predicting}. Spectral and graph-based perspectives have also played an important role: \citet{haochen2021provable} analyze a spectral contrastive loss through an augmentation graph, while \citet{balestriero2022contrastive} connect 
various SSL objectives to spectral embedding methods.
\citet{bansal2025understanding} analyze contrastive learning in Gaussian mixture models where the augmentation is biased towards the mixture component of the sample.

Many follow up works refine this picture by emphasizing that downstream performance is not determined by the SSL objective alone, but also by augmentations, model class, inductive bias, and memoization. \citet{saunshi2022inductivebias} show that analyses depending only on augmentation structure and contrastive loss can be vacuous without accounting for the function class and training algorithm. \citet{haochen2022theoretical} further study how model-class inductive biases affect the structures recovered by contrastive learning.  \citet{chen2020group} give a group-theoretic treatment of data augmentation and show that it can provably lead to variance reduction.
\citet{wei2021pretraining} analyze why pretrained language models can help downstream head and prompt tuning 
using a latent-variable generative model. 
\citet{wang2024memorization} provide empirical evidence that some level of memorization 
improves generalization in SSL.
Several works take an operator-theoretic approach to SSL: 
\citet{cabannes23SSLinterplay} analyze the interplay between augmentations, inductive bias, and generalization using RKHS,
\citet{johnson2023contrastive} 
show that several existing contrastive learning methods are actually approximating a positive-pair kernel,
\citet{zhai2024understanding} study SSL through the lens of RKHS approximation,
and \citet{esser2024nonparametric} develop kernelized SSL objectives.
\citet{ge2023provable} study a broad latent-variable framework in which maximum-likelihood pre-training is followed by empirical-risk minimization for downstream prediction.
More general loss-transfer settings
where the pre-training and downstream losses are allowed to have different orders is considered in \citet{deng2024generalization}. 
\citet{jones2025provable}
shows that using PCA to initialize SGD improves the sample complexity of learning a single-index model.
Finally, \citet{lin2025statistical} bound downstream
task error in terms of an approximate sufficiency loss, based on a theory of approximate sufficient statistics~\citep{oko2025statistical}, and which is shown to be well-controlled with SimCLR pre-training. 

As mentioned previously, the focus of our work is to develop an asymptotic analysis of the two-stage pre-training and fine-tuning pipeline, enabling precise computation of the limiting downstream risk. As discussed
in \Cref{sec:related_work}, the works mentioned in the previous paragraph most directly related to our work are \citet{ge2023provable,cabannes23SSLinterplay,zhai2024understanding}. These works provide 
sufficient conditions and upper bounds on the downstream risk.
For \citet{ge2023provable,cabannes23SSLinterplay}, we showed in \Cref{sec:examples} that the upper bounds 
are loose by problem-specific factors even in simple parametric examples; the bounds in our work however
are not directly comparable to \citet{zhai2024understanding},
as discussed in the next paragraph.
Specific details aside, we believe that extending our analysis to cover the full generality of the nonparametric RKHS framework is exciting future work.

\paragraph{Detailed comparison with \citet{zhai2024understanding}.}
\citet{zhai2024understanding} provides an RKHS approximation theory for augmentation-based SSL. Their analysis starts from the kernel induced by the augmentation distribution and assumes that the downstream target is soft invariant, or regular with respect to the corresponding RKHS. Their main downstream result \cite[Thm.~1]{zhai2024understanding} applies to an arbitrary $d$-dimensional encoder: the error bound decomposes into an approximation term, measuring how well the encoder aligns with the leading eigenspace of the augmentation-induced kernel, and an estimation term from fitting the downstream predictor with finitely many labeled samples. At the population level, the optimal $d$-dimensional representation---in their RKHS approximation sense---is the span of the top $d$ eigenfunctions of the augmentation-induced kernel. They estimate this population augmentation kernel from finite pre-training data by taking the top $d$ eigenfunctions of the corresponding empirical augmentation kernel.

Our work shares the high-level perspective that a population pre-training object determines the useful downstream representation. However, we measure the quality of pre-training in a different way. Although the bound of \citet[Thm.~1]{zhai2024understanding} can in principle be evaluated for any fixed encoder, their analysis measures the encoder quality through augmentation-induced RKHS approximation quantities, such as alignment with the leading eigenspace or the associated trace gap. In contrast, our analysis measures pre-training quality through the statistical error of the empirical pre-training optimizer relative to its population counterpart. Thus, we study how finite sample fluctuations in the pre-training objective propagate to the downstream estimator. This also leads to different downstream guarantees.~\citet{zhai2024understanding}'s 
analysis works under a weaker RKHS regularity condition and retains both an approximation term and a nonparametric downstream estimation term. Our downstream analysis, by contrast, imposes a stronger realizability structure assumption; under this stronger assumption, we obtain a faster downstream rate. Thus, these rates are not directly comparable, since they correspond to different assumptions and different notions of pre-training error.

\paragraph{Classical two-stage estimation in econometrics and statistics.}
Classical econometrics work on two-stage estimators 
provides the natural foundation for viewing pre-training followed by fine-tuning as a two-stage estimation problem. 
Pagan's work on generated regressors and two-stage estimation~\citep{pagan1984generated,pagan1986twostage} highlights a key theme that remains
prevalent in modern representation learning:
the first-stage estimation errors generally propagate
into the downstream \emph{limiting} distribution
of the second-stage estimator.
\citet{murphy1985twostep} derive asymptotically correct covariance expressions that account for the propagation of first-stage uncertainty, and \citet{newey1984methodofmoments} showed that sequential (e.g., two-stage) estimators can be interpreted within a method-of-moments framework.
These broad ideas are encapsulated in 
Newey and McFadden's general large-sample theory~\citep{newey1994largesample},
which contains analysis of consistency, asymptotic normality, and variance estimation for two-stage estimation.

There are several key differences between the classic work and ours. On the one hand, classic
two-stage theory does \emph{not} assume that the data used in stage one is independent of the data used in stage two. On the other hand, 
the classic theory does typically assume that the first-stage estimator is identifiable in the direct parameter space, and does not 
address the structural issues of orbit invariance that we do in our work. Regarding the first point, there is related literature
on two-sample, two-stage least squares~\citep{klevmarken1982missing,pacini2016robust} and two-sample instrumental variable estimators~\citep{inoue2010twosampleIV} where, similar to our
pre-train and fine-tune pipeline, the data used in each stage is independent. Again, since these works do not deal with 
orbit identifiability, 
the asymptotic analysis is able to leverage e.g., arguments from \citet{newey1994largesample}. 
Finally, we mention more recent work of \citet{zhang2019semi}
which studies asymptotic analysis of semi-supervised mean estimation.

\paragraph{Riemannian limit theorems and $M$-estimation.}
As mentioned in \Cref{sec:related_work} and \Cref{sec:manifold_identifiability}, the work of \citet{brunel2023geodesically} on geodesically-convex $M$-estimation 
provides key technical tools used to establish our
limit theorems for the pre-training stage. 
Other prior works studying estimation on Riemannian
manifolds include limit theorems for 
Fr{\'{e}}chet means~\citep{bhattacharya2003meansI,bhattacharya2005meansII,bhattacharya2008statistics}, geodesic principal components~\citep{huckemann2010intrinsic,huckemann2011intrinsic}, 
Fr{\'{e}}chet regression~\citep{petersen2019frechetregression},
and low-rank matrix sensing~\citep{bastani2024asymptotic}.
Our work is complementary to this line of research,
and can be viewed as leveraging these results (specifically the work of \citet{brunel2023geodesically}) to study a particular two-sample, two-stage $M$-estimation problem.

\section{Structure of the Proof of Theorem~\ref{thm:master-compatible}}
\label{app:proof-structure}

This appendix gives a high-level roadmap for the proof of \Cref{thm:master-compatible}. The purpose is to isolate the main statistical mechanisms behind the limit in \eqref{eq:master-compatible-dist}; the full technical arguments are given in \Cref{app:proof-main-thm}.

\subsection{Exact conditional risk decomposition}

Recalling the pre-trained descriptor
$\hat\Omega_m=\hat\Omega_m(D_{\mathrm{pre}}^{(m)})\in\mathcal M$ from
\Cref{sec:identifiability}, we condition on $D_{\mathrm{pre}}^{(m)}$ and treat
$\Omega\coloneq \hat\Omega_m$ as fixed. Assume
$\E\|\phi(X,\Omega)\|_2^2<\infty$. Define the population forward operator
$T_\Omega:\R^p\to L^2_{\mathrm{down}}$ and its adjoint
$T_\Omega^{\mathrm{adj}}:L^2_{\mathrm{down}}\to\R^p$ by
\begin{align}\label{eq:feature-operator}
(T_\Omega\theta)(x)\coloneq \langle \theta,\phi(x,\Omega)\rangle,
\qquad
T_\Omega^{\mathrm{adj}}g\coloneq \E[g(X)\phi(X,\Omega)].    
\end{align}
These operators induce the population feature covariance
\[
\Sigma(\Omega)
\coloneq T_\Omega^{\mathrm{adj}}T_\Omega
=
\E[\phi(X,\Omega)\phi(X,\Omega)^\top].
\]
Let $\mathcal H_\Omega=\mathrm{Im}(T_\Omega)\subseteq L^2_{\mathrm{down}}$ be
the downstream function class induced by $\Omega$. The population
$L^2_{\mathrm{down}}$-orthogonal projector onto $\mathcal H_\Omega$ is
\[
\Pi_\Omega
\coloneq
T_\Omega\,\Sigma(\Omega)^+\,T_\Omega^{\mathrm{adj}}.
\]
We use this projector to define the residual and representation error
\begin{align}\label{eq:representation error}
e_\Omega\coloneq (I-\Pi_\Omega)f_\star,
\qquad
\mathrm{Rep}(\Omega)\coloneq \|e_\Omega\|_{L^2_{\mathrm{down}}}^2.    
\end{align}
Since the model is well-specified at the population descriptor, we have
$e_{\Omega_\star}=0$.

Similarly, we can define the (downstream) empirical adjoint and empirical covariance as
\[T_{\Omega,n}^{\mathrm{adj}}g
\coloneqq  \frac{1}{n}\sum_{i=1}^n g(x_i)\phi(x_i,\Omega),\]
and $\Sigma_n(\Omega)
\coloneqq \frac{1}{n}\sum_{i=1}^n \phi(x_i,\Omega)\phi(x_i,\Omega)^\top$, 
and the empirical projector
\[\Pi_{\Omega,n}g
\coloneqq T_\Omega\,\Sigma_n(\Omega)^+\,T_{\Omega,n}^{\mathrm{adj}}g.\]
By construction, the minimum-norm OLS predictor $\hat f_{\Omega,n}\coloneqq T_\Omega\hat\theta_{\Omega,n}$ satisfies
$\hat f_{\Omega,n}=\Pi_{\Omega,n}y$
for any $y(x_i)=y_i$ for $i \in [n]$ (cf.~\Cref{app:prelim}). Similarly, let $\varepsilon(\cdot)$ denote the noise function on the design, defined by $\varepsilon(x_i)=\varepsilon_i$ for $i\in[n]$.
Then $y=f_\star+\varepsilon$ on $\{x_i\}_{i=1}^n$.

With this notation, the following exact
finite-sample decomposition is the starting point of the proof.

\begin{restatable}[Exact conditional risk decomposition]{proposition}{exactdecomp}
\label{prop:exact-decomp}
For the minimum-norm OLS predictor $\hat f_{\Omega,n}$, the conditional test risk
admits the decomposition
\begin{align}
&\E\!\left[
\bigl(Y_{\mathrm{new}}-\hat f_{\Omega,n}(X_{\mathrm{new}})\bigr)^2
\,\middle|\, D_{\mathrm{pre}}^{(m)},X_{1:n}
\right]
=
\sigma^2+\mathrm{Rep}(\Omega)
\nonumber\\
&\qquad+
\underbrace{
\E\!\left[
\bigl(\Pi_{\Omega,n}f_\star-\Pi_\Omega f_\star\bigr)(X_{\mathrm{new}})^2
\,\middle|\,D_{\mathrm{pre}}^{(m)},X_{1:n}
\right]
}_{=:~\mathrm{Leakage}_n(\Omega)}
+
\underbrace{
\frac{\sigma^2}{n}
\Tr\!\left(\Sigma(\Omega)\Sigma_n(\Omega)^+\right)
}_{=:~\mathrm{Var}_n(\Omega)} .
\label{eq:exact-risk-general-main}
\end{align}
\end{restatable}

The three terms in \eqref{eq:exact-risk-general-main} have distinct roles.
The representation term $\mathrm{Rep}(\Omega)$ measures the population
approximation error of the learned feature class $\mathcal H_\Omega$. This term
is zero at $\Omega_\star$, but is typically nonzero for the finite-sample
pre-training estimator $\hat\Omega_m$. The leakage term
$\mathrm{Leakage}_n(\Omega)$ measures the discrepancy between the empirical and
population projection operators on $f_\star$. Under a standard well-posedness
condition that the empirical inner product is non-degenerate on
$\mathcal H_\Omega$, this term can be written as
\[
\Pi_{\Omega,n}f_\star-\Pi_\Omega f_\star
=
\Pi_{\Omega,n}e_\Omega ,
\]
so leakage is caused by projecting the population residual $e_\Omega$ using the
empirical downstream geometry. Finally, $\mathrm{Var}_n(\Omega)$ is the usual
least-squares variance contribution from fitting the downstream labels. See~\Cref{app:decomp} for the proof.

\subsection{Main proof idea}

We apply \Cref{prop:exact-decomp} with
$\Omega=\hat\Omega_m$. Subtracting $\sigma^2$ and multiplying by $n$ gives
\[
\mathcal E_{m,n}
=
n\,\mathrm{Var}_n(\hat\Omega_m)
+
n\,\mathrm{Leakage}_n(\hat\Omega_m)
+
n\,\mathrm{Rep}(\hat\Omega_m).
\]
The proof of \Cref{thm:master-compatible} analyzes these three terms separately
under the joint limit $m,n\to\infty$ with $m/n\to\alpha$.

First, the variance term behaves as a well-specified least-squares variance
term on the \emph{activated} limiting feature space. In the stable-null-span
regime, perturbations of directions that are null at \(\Omega_\star\) may
activate an additional limiting subspace \(\mathcal B_0\subseteq
\mathcal H_{\Omega_\star}^{\perp}\). Thus the relevant limiting degrees of
freedom are
\[
d_{\mathrm{act}}(\Omega_\star)
\coloneqq
\dim\bigl(\mathcal H_{\Omega_\star}\oplus\mathcal B_0\bigr),
\]
rather than merely \(d_{\mathrm{eff}}(\Omega_\star)=\dim(\mathcal H_{\Omega_\star})\).
By \Cref{ass:stable-null-span} and concentration of the downstream empirical
covariance,
\[
n\,\mathrm{Var}_n(\hat\Omega_m)
=
\sigma^2
\Tr\!\left(
\Sigma(\hat\Omega_m)\Sigma_n(\hat\Omega_m)^+
\right)
\xrightarrow{\mathbb P}
\sigma^2 d_{\mathrm{act}}(\Omega_\star).
\]
Thus the downstream label noise contributes the same leading constant one would
obtain from ordinary least squares on the limiting activated representation
\(\mathcal H_{\Omega_\star}\oplus\mathcal B_0\). In the regular special case
\(\mathcal B_0=\{0\}\), this reduces to the usual
\(\sigma^2 d_{\mathrm{eff}}(\Omega_\star)\) term; see
\Cref{app:downstream-terms:noise}.

Second, the leakage term is asymptotically negligible at the \(n^{-1}\) scale. Indeed, in the compatible regime \(e_{\Omega_\star}=0\), and the residual
continuity implied by the representation-side assumptions gives
\[
e_{\hat\Omega_m}\to0
\qquad
\text{in }L^2(\mu_{\mathrm{down}}).
\]
Since the empirical projection is uniformly well behaved on
\(\mathcal H_{\hat\Omega_m}\), the conditional second moment of the projected
residual is \(o_{\mathbb P}(n^{-1})\), and hence
\[
n\,\mathrm{Leakage}_n(\hat\Omega_m)
\xrightarrow{\mathbb P}
0.
\]
This shows that empirical projection error does not contribute to the limiting
distribution; see~\Cref{app:downstream-terms:signal}.

The remaining term is the representation error. This is where pre-training
randomness enters. By the Riemannian \(M\)-estimation CLT for the descriptor,
\[
\sqrt m\,\log_{\Omega_\star}(\hat\Omega_m)
\distconv
Z,
\qquad
Z\sim\mathcal N(0,V),
\qquad
V=H_\star^{-1}\Sigma_\star H_\star^{-1}.
\]
Since \(f_\star\in\mathcal H_{\Omega_\star}\), the residual
\(e_{\Omega_\star}=(I-\Pi_{\Omega_\star})f_\star\) vanishes. The relevant
first-order object is therefore the linearization of the residual map
\[
e_\Omega=(I-\Pi_\Omega)f_\star,
\]
rather than necessarily the full projector map \(\Omega\mapsto\Pi_\Omega\).
In normal coordinates, the structural condition in \Cref{app:rep} gives
\[
e_{\exp_{\Omega_\star}(v)}
=
\mathcal L(v)+o(\|v\|)
\qquad
\text{in }L^2(\mu_{\mathrm{down}}),
\]
where \(\mathcal L:T_{\Omega_\star}\mathcal M\to L^2(\mu_{\mathrm{down}})\)
is the first-order residual map. Under the stable limiting span of null
directions (\Cref{ass:stable-null-span}), this map is
\[
\mathcal L(v)
=
-
(I-\Pi_\star-\Pi_{\mathcal B_0})
DT_{\Omega_\star}[v]\theta_\star.
\]
In the stronger special case where the projector map
\(\Omega\mapsto\Pi_\Omega\) is Fr\'echet differentiable at \(\Omega_\star\) in
operator norm, this reduces to
\[
\mathcal L(v)
=
-D\Pi_{\Omega_\star}[v]f_\star.
\]
Consequently,
\[
m\,\mathrm{Rep}(\hat\Omega_m)
=
m\|e_{\hat\Omega_m}\|_{L^2(\mu_{\mathrm{down}})}^2
\distconv
\|\mathcal L(Z)\|_{L^2(\mu_{\mathrm{down}})}^2 .
\]
See~\Cref{app:rep}.

Combining the three limits, and using $n/m\to 1/\alpha$, gives
\[
n\,\mathrm{Rep}(\hat\Omega_m)
=
\frac{n}{m}
\bigl(m\,\mathrm{Rep}(\hat\Omega_m)\bigr)
\distconv
\alpha^{-1}
\|\mathcal L(Z)\|_{L^2(\mu_{\mathrm{down}})}^2 .
\]
Slutsky's theorem then yields
\begin{align}\label{eq:distlimit-sketch}   
\mathcal E_{m,n}
\distconv
\sigma^2 d_{\mathrm{eff}}(\Omega_\star)
+
\alpha^{-1}
\|\mathcal L(Z)\|_{L^2(\mu_{\mathrm{down}})}^2,
\end{align}

which is precisely \eqref{eq:master-compatible-dist}.

\begin{remark}[From conditional to fully averaged risk]\label{rem:fully-avg-risk}

The distributional limit in~\eqref{eq:distlimit-sketch} is stated for the
conditional excess risk~\eqref{eq:main_risk_cond}, where we condition on the
realized pre-training sample and downstream design. Passing from this conditional
statement to a fully averaged risk statement requires interchanging limits and
expectations. Indeed,
\[
\E[\mathcal E_{m,n}]
=
n\left(
\E\!\left[(Y_{\mathrm{new}}-\hat f_{m,n}(X_{\mathrm{new}}))^2\right]
-\sigma^2
\right),
\]
since
\[
\E\!\left[(Y_{\mathrm{new}}-\hat f_{m,n}(X_{\mathrm{new}}))^2\right]
=
\E\!\left[R(D_{\mathrm{pre}}^{(m)},X_{1:n})\right].
\]
This upgrade is typically delicate because the conditional expansion contains
inverse empirical covariance terms, with
\[
\widehat\Sigma_{n,m}
=
\frac1n\sum_{i=1}^n
\phi(x_i,\hat\Omega_m)\phi(x_i,\hat\Omega_m)^\top .
\]
When $\widehat\Sigma_{n,m}$ is ill-conditioned, its smallest nonzero eigenvalue
can be very small with non-negligible probability, producing heavy-tailed
inverse-covariance contributions.

A sufficient route to convergence of the fully averaged risk is to establish
uniform integrability of $\{\mathcal E_{m,n}\}_{m,n}$. For instance, sufficiently
strong lower-tail or anti-concentration bounds for
$\widehat\Sigma_{n,m}$, as in the small-ball analysis of random design least
squares~\citep[cf.][]{mourtada2022exact}, can control the inverse-covariance
terms and allow the conditional expansion to be integrated. In that case,
\[
\E[\mathcal E_{m,n}]
\longrightarrow
\E[\mathcal E_\alpha],
\]
so the same limiting expression also characterizes the scaled fully averaged
excess risk.
\end{remark}

\section{Empirical Projection Operators and Linear-Algebra Preliminaries}
\label{app:prelim}

This appendix collects basic facts used repeatedly in the proofs of
Proposition~\ref{prop:exact-decomp} and Theorem~\ref{thm:master-compatible}.
The statements are standard; we include them to (i) make explicit which inner
product each projection is taken with respect to, and (ii) clarify what is
unique (fitted values on the sample) versus what is a chosen convention
(a minimum-norm representative in coefficient space) when the design is rank-deficient.
All statements in this appendix are deterministic once the design points $X_{1:n}$ are fixed; in particular,
they are tailored to our conditional-on-design viewpoint in the main text. We also list the norm conventions used for matrices, feature operators, and descriptor derivatives at the end of this appendix.

\subsection{Empirical inner products and evaluation maps}
\label{app:empirical-inner-product}

Let $\mathcal X$ be the input space (e.g.\ $\mathcal X=\R^d$ in the linear--Gaussian example), and let
$\mathcal G$ be a class of measurable functions $g:\mathcal X\to\R$ (for instance $\mathcal G=L^2(\mu_{\mathrm{down}})$,
or any function space containing the hypothesis classes used in the main text).

Fix design points $x_{1:n}\coloneqq (x_1,\dots,x_n)\in\mathcal X^n$.
For $g,h\in\mathcal G$ such that $g(x_i),h(x_i)\in\R$ for all $i\in[n]$, define the empirical bilinear form
\[
\langle g,h\rangle_n \coloneqq  \frac1n\sum_{i=1}^n g(x_i)\,h(x_i),
\qquad
\|g\|_n^2 \coloneqq  \langle g,g\rangle_n.
\]
In general, $\|\cdot\|_n$ is only a \emph{seminorm}: it depends only on the values of a function on the
finite set $\{x_1,\dots,x_n\}$. Consequently, the fitted values of any empirical least-squares projection are uniquely determined
only through these $n$ values.

To isolate this dependence, define the evaluation map at $x_{1:n}$ by
\[
\mathrm{Ev}_n:\mathcal G \to \R^n,
\qquad
\mathrm{Ev}_n(g)\coloneqq \big(g(x_1),\dots,g(x_n)\big)^\top.
\]
Whenever $g(x_i)\in\R$ for all $i$, we have $\mathrm{Ev}_n(g)\in\R^n$ and
\[
\langle g,h\rangle_n=\frac1n\,\mathrm{Ev}_n(g)^\top \mathrm{Ev}_n(h),
\qquad
\|g\|_n^2=\frac1n\|\mathrm{Ev}_n(g)\|_2^2.
\]
Thus, empirical least-squares projection statements can be proved equivalently in the finite-dimensional
space $\R^n$ by working with the vectors of function values $\mathrm{Ev}_n(g)$.

\subsection{Pseudoinverse identities}
\label{app:pseudoinverse}

We use $(\cdot)^+$ to denote the Moore--Penrose pseudoinverse.

\begin{lemma}[Basic pseudoinverse identities]
\label{lem:pinv-basic}
For any matrix $A$ (not necessarily square),
\[
AA^+A=A,\qquad A^+AA^+=A^+,\qquad (AA^+)^\top=AA^+,\qquad (A^+A)^\top=A^+A.
\]
Moreover, $AA^+$ is the Euclidean orthogonal projector onto $\mathrm{Im}(A)$ and $A^+A$ is the Euclidean orthogonal
projector onto $\mathrm{Im}(A^\top)$.
\end{lemma}

\begin{lemma}[Trace equals rank for symmetric PSD matrices]
\label{lem:trace-rank}
If $S\in\R^{p\times p}$ is symmetric positive semidefinite, then
\[
\tr(S^+S)=\rank(S).
\]
\end{lemma}

\begin{remark}[On conditioning and inverse-eigenvalue effects]
\label{rem:pinv-conditioning}
The identities in this section are purely algebraic and hold deterministically for any realization of the design.
When taking expectations over random designs, quantities involving $S^+$ can be sensitive to small eigenvalues of $S$,
and additional tail control is typically needed to justify interchanging limits and expectations.
This sensitivity motivates the conditional-on-design risk formulation adopted in the main text.
\end{remark}

\subsection{Design matrices and hat matrices}
\label{app:design-hat}

Fix a feature parameter $\Omega\in\R^q$ and design points $x_{1:n}$.
Define the feature matrix
\[
\Phi_\Omega \in \R^{n\times p},
\qquad
(\Phi_\Omega)_{i,:} \coloneqq  \phi(x_i,\Omega)^\top,
\]
and the empirical covariance
\[
\Sigma_n(\Omega)=\frac1n\Phi_\Omega^\top \Phi_\Omega.
\]
Define the hat matrix
\[
H_{\Omega,n}
\;\coloneqq \;
\Phi_\Omega(\Phi_\Omega^\top \Phi_\Omega)^+\Phi_\Omega^\top
\;=\;
\frac1n\,\Phi_\Omega \Sigma_n(\Omega)^+ \Phi_\Omega^\top.
\]

\begin{lemma}[Hat matrix is an orthogonal projector]
\label{lem:hat-projector}
$H_{\Omega,n}$ is the Euclidean orthogonal projector in $\R^n$ onto $\mathrm{Im}(\Phi_\Omega)$.
In particular, $H_{\Omega,n}^2=H_{\Omega,n}$, $H_{\Omega,n}^\top=H_{\Omega,n}$, and
\[
\tr(H_{\Omega,n})=\rank(\Phi_\Omega)=\rank(\Sigma_n(\Omega)).
\]
\end{lemma}

\begin{proof}
By Lemma~\ref{lem:pinv-basic}, $\Phi_\Omega(\Phi_\Omega^\top\Phi_\Omega)^+\Phi_\Omega^\top$ is the orthogonal projector onto
$\mathrm{Im}(\Phi_\Omega)$. The trace of an orthogonal projector equals its rank.
Finally, $\rank(\Phi_\Omega)=\rank(\Phi_\Omega^\top\Phi_\Omega)=\rank(\Sigma_n(\Omega))$.
\end{proof}

\subsection{Empirical least-squares projection onto $\mathcal H_\Omega$}
\label{app:empirical-projector}

Recall the linear class
\[
\mathcal H_\Omega\coloneqq \{T_\Omega\theta:\theta\in\R^p\},
\qquad
(T_\Omega\theta)(x)\coloneqq \langle \theta,\phi(x,\Omega)\rangle.
\]

\paragraph{Parameterization versus functions.}
The parametrization $\theta\mapsto T_\Omega\theta$ need not be injective: if $v\in\R^p$ satisfies
$\langle v,\phi(x,\Omega)\rangle=0$ for all $x\in\mathcal X$ (equivalently, $T_\Omega v\equiv 0$), then
$T_\Omega(\theta+v)=T_\Omega\theta$ as functions on $\mathcal X$.
Thus, even when the induced function is unique, the coefficient vector representing it may not be.

Given design points $x_{1:n}$, define the empirical adjoint
\[
T_{\Omega,n}^{\mathrm{adj}}g \coloneqq  \frac1n\sum_{i=1}^n g(x_i)\phi(x_i,\Omega)\in\R^p.
\]

\begin{lemma}[Empirical adjoint identity]
\label{lem:emp-adj}
For any $\theta\in\R^p$ and any function $g$, we have
\[
\langle T_\Omega\theta, g\rangle_n = \langle \theta, T_{\Omega,n}^{\mathrm{adj}} g\rangle_{\R^p}.
\]
\end{lemma}

\begin{proof}
This is simply a consequence of the following equalities:
\[
\langle T_\Omega\theta,g\rangle_n=\frac1n\sum_{i=1}^n \langle \theta,\phi(x_i,\Omega)\rangle g(x_i)
=\left\langle \theta,\frac1n\sum_{i=1}^n g(x_i)\phi(x_i,\Omega)\right\rangle
=\langle \theta, T_{\Omega,n}^{\mathrm{adj}} g\rangle_{\R^p}.
\]
\end{proof}

\paragraph{Empirical projection: what is unique and what is a convention.}
Since $\|\cdot\|_n$ is only a seminorm, the empirical least-squares problem
\[
\min_{h\in\mathcal H_\Omega}\ \|g-h\|_n^2
\qquad\Bigl(\ \|g-h\|_n^2=\frac1n\sum_{i=1}^n (g(x_i)-h(x_i))^2\ \Bigr)
\]
can determine only the values of the minimizer on the sample points. In particular, the objective
depends on $h$ only through the vector $\mathrm{Ev}_n(h)=(h(x_1),\dots,h(x_n))^\top$, and any two
functions that agree on $\{x_i\}_{i=1}^n$ are indistinguishable under $\|\cdot\|_n$.
Writing $h=T_\Omega\theta$ so that $\mathrm{Ev}_n(h)=\Phi_\Omega\theta$, the problem reduces to Euclidean
least squares in $\R^n$.

\begin{lemma}[Least-squares equivalence]
\label{lem:ls-equiv}
The empirical projection problem is equivalent to the Euclidean least-squares problem
\[
\min_{\theta\in\R^p}\ \frac1n\|\mathrm{Ev}_n(g)-\Phi_\Omega\theta\|_2^2.
\]
\end{lemma}

\begin{proof}
By definition,
\[
\|g-h\|_n^2
=\frac1n\sum_{i=1}^n (g(x_i)-h(x_i))^2
=\frac1n\|\mathrm{Ev}_n(g)-\mathrm{Ev}_n(h)\|_2^2.
\]
Every $h\in\mathcal H_\Omega$ can be written as $h=T_\Omega\theta$ for some $\theta\in\R^p$, and then
\[
\mathrm{Ev}_n(h)
=\bigl((T_\Omega\theta)(x_1),\dots,(T_\Omega\theta)(x_n)\bigr)^\top
=\bigl(\langle \theta,\phi(x_1,\Omega)\rangle,\dots,\langle \theta,\phi(x_n,\Omega)\rangle\bigr)^\top
=\Phi_\Omega\theta.
\]
Substituting this identity into the empirical objective gives
\[
\|g-T_\Omega\theta\|_n^2=\frac1n\|\mathrm{Ev}_n(g)-\Phi_\Omega\theta\|_2^2,
\]
which proves the claim.
\end{proof}

Even when $\Phi_\Omega$ is rank-deficient, the \emph{fitted values} are unique: the Euclidean orthogonal projection of $\mathrm{Ev}_n(g)$ onto $\mathrm{Im}(\Phi_\Omega)$
\begin{align}
\label{eq:emp-proj-evaluation}
\hat{g}_{1:n}\;\coloneqq \;H_{\Omega,n}\,\mathrm{Ev}_n(g)\in\R^n
\end{align}
is uniquely determined by Lemma~\ref{lem:hat-projector}.

In contrast, the coefficient vector $\theta$ achieving these fitted values need not be unique when
$\Phi_\Omega$ is rank-deficient: if $\theta^\star$ is one minimizer then all minimizers are
$\theta^\star+v$ with $v\in\ker(\Phi_\Omega)$, and they all satisfy $\Phi_\Omega\theta=\hat g_{1:n}$.
Whether these distinct minimizers define the same function on $\mathcal X$ depends on the feature
representation: if $\ker(\Phi_\Omega)\subseteq\ker(T_\Omega)$ then all minimizers induce the same function
in $\mathcal H_\Omega$, whereas if there exists $v\neq 0$ with $\Phi_\Omega v=0$ but $T_\Omega v\not\equiv 0$,
then different minimizers agree on the sample points but can differ off-sample.

Throughout the paper, we follow the convention fixed in the main text: the empirical projector
$\Pi_{\Omega,n}$ is defined via the Moore--Penrose pseudoinverse (see Section~\ref{sec:main}), so that
the associated coefficient vector is the minimum-Euclidean-norm least-squares solution.
Importantly, all fitted-value identities below hold regardless of this convention.

\begin{definition}[Canonical empirical projector]
\label{def:canonical-emp-proj}
For $g$ with finite evaluations on $\{x_i\}_{i=1}^n$, define
\[
\hat\theta_{\Omega,n}(g)
\;\coloneqq \;
\Sigma_n(\Omega)^+\,T_{\Omega,n}^{\mathrm{adj}}g \in \R^p,
\]
and set
\[
\Pi_{\Omega,n}g \;\coloneqq \; T_\Omega\,\hat\theta_{\Omega,n}(g)\in\mathcal H_\Omega.
\]
\end{definition}

\begin{lemma}[Least-squares characterization and empirical orthogonality]
\label{lem:emp-proj-char}
For any $g$ we have
\[
\Pi_{\Omega,n} g \in \argmin_{h\in\mathcal H_\Omega}\ \|g-h\|_n^2.
\]
Moreover, the fitted values satisfy
\[
\mathrm{Ev}_n(\Pi_{\Omega,n} g)=H_{\Omega,n}\,\mathrm{Ev}_n(g),
\]
and the residual is empirically orthogonal to $\mathcal H_\Omega$ in the sense that
\[
\langle g-\Pi_{\Omega,n}g,\,h\rangle_n=0\qquad \forall h\in\mathcal H_\Omega.
\]
\end{lemma}

\begin{proof}
Write $h=T_\Omega\theta$. Then
\[
\|g-h\|_n^2=\frac1n\|\mathrm{Ev}_n(g)-\Phi_\Omega\theta\|_2^2.
\]
Thus minimizing over $h\in\mathcal H_\Omega$ is equivalent to least squares in $\R^n$.
By Definition~\ref{def:canonical-emp-proj}, $\hat\theta_{\Omega,n}(g)$ is the Moore--Penrose minimum-norm
least-squares solution, hence $\Pi_{\Omega,n}g=T_\Omega\hat\theta_{\Omega,n}(g)$ is a minimizer. Furthermore,
\[
\mathrm{Ev}_n(\Pi_{\Omega,n}g)=\Phi_\Omega\hat\theta_{\Omega,n}(g)
= \Phi_\Omega(\Phi_\Omega^\top\Phi_\Omega)^+\Phi_\Omega^\top\,\mathrm{Ev}_n(g)
=H_{\Omega,n}\mathrm{Ev}_n(g).
\]
For empirical orthogonality, let $r\coloneqq \mathrm{Ev}_n(g)-\mathrm{Ev}_n(\Pi_{\Omega,n}g)$.
Since $\mathrm{Ev}_n(\Pi_{\Omega,n}g)=H_{\Omega,n}\mathrm{Ev}_n(g)$ and $H_{\Omega,n}$ is the orthogonal projector onto
$\mathrm{Im}(\Phi_\Omega)$, we have $r\perp \mathrm{Im}(\Phi_\Omega)$.
For any $h=T_\Omega\theta\in\mathcal H_\Omega$, $\mathrm{Ev}_n(h)=\Phi_\Omega\theta\in\mathrm{Im}(\Phi_\Omega)$, hence
\[
n\langle g-\Pi_{\Omega,n}g,\,h\rangle_n
=
r^\top \mathrm{Ev}_n(h)=0.
\]
\end{proof}

\begin{lemma}[Reproducing property on the sample points]
\label{lem:reproducing}
For any $h\in\mathcal H_\Omega$,
\[
\mathrm{Ev}_n(\Pi_{\Omega,n}h)=\mathrm{Ev}_n(h),
\qquad\text{and hence}\qquad
\|h-\Pi_{\Omega,n}h\|_n=0.
\]
\end{lemma}

\begin{proof}
If $h\in\mathcal H_\Omega$ then $\mathrm{Ev}_n(h)\in\mathrm{Im}(\Phi_\Omega)$, so
$H_{\Omega,n}\mathrm{Ev}_n(h)=\mathrm{Ev}_n(h)$ by Lemma~\ref{lem:hat-projector}.
Now apply Lemma~\ref{lem:emp-proj-char}.
\end{proof}

\begin{definition}
For any vector $v\in\R^n$, we define
$\mathrm{lift}_n(v)$ as any arbitrary measurable function $g:\mathcal X\to\R$ that satisfies
$\mathrm{Ev}_n(g)=v$ (the values of $g$ off $\{x_i\}_{i=1}^n$ are irrelevant). Since $\Pi_{\Omega,n}$ depends on an input
only through its evaluations on the design points, $\Pi_{\Omega,n}\mathrm{lift}_n(v)$ is well-defined.
\end{definition}

Let $y_{1:n}=(y_1,\cdots,y_n)^\top$, and recall the minimum-norm least-squares solution $\hat\theta_{\Omega,n}$ and $\hat f_{\Omega,n}=T_\Omega\hat\theta_{\Omega,n}$.

\begin{lemma}[OLS equals empirical projection]
\label{lem:ols-equals-proj}
The OLS predictor satisfies $\hat f_{\Omega,n}=\Pi_{\Omega,n} \mathrm{lift}_n(y_{1:n})$, and hence
\[
\mathrm{Ev}_n(\hat f_{\Omega,n})=H_{\Omega,n}\,(y_1,\dots,y_n)^\top.
\]
\end{lemma}

\begin{proof}
By definition, $\hat\theta_{\Omega,n}$ minimizes $\|\mathrm{lift}_n(y_{1:n})-T_\Omega\theta\|_n^2$, which is exactly the least-squares problem
encoded in Definition~\ref{def:canonical-emp-proj} with $g=\mathrm{lift}_n(y_{1:n})$.
Therefore $T_\Omega\hat\theta_{\Omega,n}$ coincides with $\Pi_{\Omega,n}\mathrm{lift}_n(y_{1:n})$, and evaluating yields the hat-matrix identity.
\end{proof}

\subsection{Effective dimension and degrees of freedom}
\label{app:effective-dim}

Recall the empirical effective dimension
\[
d_{\mathrm{eff},n}(\Omega)\coloneqq \tr\!\bigl(\Sigma_n(\Omega)^+\Sigma_n(\Omega)\bigr)=\rank(\Sigma_n(\Omega)).
\]

\begin{lemma}[Equivalent characterizations of $d_{\mathrm{eff},n}$]
\label{lem:deffn-eq}
For any $\Omega$ and design points $x_{1:n}$,
\[
d_{\mathrm{eff},n}(\Omega)=\rank(\Sigma_n(\Omega))=\rank(\Phi_\Omega)=\tr(H_{\Omega,n}).
\]
\end{lemma}

\begin{proof}
Combine Lemma~\ref{lem:trace-rank} with Lemma~\ref{lem:hat-projector}.
\end{proof}

\begin{lemma}[Projected-noise identity conditional on the design]
\label{lem:proj-noise-cond}
Let $\varepsilon_{1:n}=(\varepsilon_1,\dots,\varepsilon_n)^\top$ with $\E[\varepsilon_{1:n}\mid x_{1:n}]=0$ and
$\E[\varepsilon_{1:n}\varepsilon_{1:n}^\top\mid x_{1:n}]=\sigma^2 I_n$, and assume $\varepsilon_{1:n}$ is conditionally independent of $x_{1:n}$.
Then, for any $\Omega$,
\[
\E\!\left[\frac1n\|H_{\Omega,n}\varepsilon_{1:n}\|_2^2\,\middle|\,x_{1:n}\right]
=
\frac{\sigma^2}{n}\tr(H_{\Omega,n})
=
\frac{\sigma^2}{n}\,d_{\mathrm{eff},n}(\Omega).
\]
\end{lemma}

\begin{proof}
Condition on $x_{1:n}$ and use $H_{\Omega,n}^\top=H_{\Omega,n}$ and $H_{\Omega,n}^2=H_{\Omega,n}$:
\[
\E\!\left[\|H_{\Omega,n}\varepsilon_{1:n}\|_2^2\,\middle|\,x_{1:n}\right]
=
\E\!\left[\varepsilon_{1:n}^\top H_{\Omega,n}\varepsilon_{1:n}\,\middle|\,x_{1:n}\right]
=
\tr\!\left(H_{\Omega,n}\,\E[\varepsilon_{1:n}\varepsilon_{1:n}^\top\mid x_{1:n}]\right)
=
\sigma^2\tr(H_{\Omega,n}).
\]
Divide by $n$ and invoke Lemma~\ref{lem:deffn-eq}.
\end{proof}

\begin{remark}
Lemma~\ref{lem:deffn-eq} is the linear-algebraic reason the leading noise-estimation constant is a degrees-of-freedom term:
under homoskedastic noise, the conditional expected squared norm of the projected noise depends on the design only through
$\tr(H_{\Omega,n})$, which equals $d_{\mathrm{eff},n}(\Omega)$ by Lemma~\ref{lem:deffn-eq}.
\end{remark}

\subsection{Benefits of an operator-theoretic formulation}
\label{app:benefits_of_operator_formulation}

The downstream stage does not use $\Omega$ as an object in isolation; it uses only the linear function class it induces together with
the least-squares fit of the labels onto this class. Introducing the linear map $T_\Omega$ packages all downstream quantities in a
coordinate-free form.

\begin{enumerate}[label=(\roman*)]
    \item \emph{Invariance becomes explicit.} If two feature parametrizations induce the same subspace $\mathcal H_\Omega\subseteq L^2_{\mathrm{down}}$, then downstream predictions after refitting are identical. In the operator language, all population objects depend on $\Omega$ only through $\mathcal H_\Omega$, summarized by the projector $\Pi_\Omega=T_\Omega\Sigma(\Omega)^+T_\Omega^{\mathrm{adj}}$.

    \item \emph{Population and empirical stages have the same algebraic structure.} The empirical projector $\Pi_{\Omega,n}$ is obtained from $\Pi_\Omega$ by replacing expectations with sample averages, i.e., $T_\Omega^{\mathrm{adj}}$ by $T_{\Omega,n}^{\mathrm{adj}}$ and $\Sigma(\Omega)$ by $\Sigma_n(\Omega)$.

    \item \emph{Bridge to the manifold pre-training limit theory.} Our $m$-asymptotics enter through how population quantities change with $\Omega$ near $\Omega_\star$. When $\mathcal M$ is a Riemannian manifold, the pre-training estimator satisfies a log-map CLT 
    \[
    \sqrt m\,\log_{\Omega_\star}(\hat\Omega_m)
    \distconv
    \mathcal N(0,V)
    \quad
    \text{in }T_{\Omega_\star}\mathcal M,
    \]
    under the assumptions in the main text. The operator viewpoint makes it natural to apply a delta-method argument to maps of the form $\Omega\mapsto \Pi_\Omega$ and $\Omega\mapsto \mathrm{Rep}(\Omega)$, after passing to the appropriate descriptor/quotient parametrization described in Section~\ref{sec:identifiability}.
\end{enumerate}

\subsection{Norm conventions}\label{app:norm-convention}
Throughout the paper, $\norm{\cdot}_2$ denotes the Euclidean norm on finite-dimensional spaces. For a matrix $A$, $\norm{A}_\mathrm{op}$ denotes the induces Euclidean operator norm
\[
\norm{A}_\mathrm{op} = \sup_{\norm{\theta}_2=1} \norm{A\theta}_2.
\]
More generally, if $S:F\to E$ is a bounded linear map between normed spaces, $\norm{S}_\mathrm{op}$ denotes the induced operator norm. In particular, for a feature map $T_\Omega:\R^p\to L^2$,
\[
\norm{T_\Omega}_\mathrm{op}^2=\sup_{\norm{\theta}_2=1}\norm{T_\Omega \theta}_{L^2}^2=\sup_{\norm{\theta}_2=1}\E_X[\langle \theta,\phi(X,\Omega)\rangle].
\]

\section{Riemannian Geometry and $M$-estimation Background}
\label{app:riemannian_mestimation}

This appendix collects the minimal differential-geometric and $M$-estimation background used in our
quotient/descriptor-manifold asymptotic analysis. Our Riemannian conventions follow~\cite{lee2018introduction}.
For classical references on Euclidean $M$-estimation, see~\cite{van2000asymptotic}.

\subsection{Basic Riemannian notions}
\label{app:riem_basics}

\paragraph{Smooth manifolds and tangent spaces.}
A smooth $q$-dimensional manifold $\mathcal M$ is a Hausdorff, second-countable topological space equipped with a smooth atlas
$\{(U_\alpha,\varphi_\alpha)\}$, where $\varphi_\alpha:U_\alpha\to \varphi_\alpha(U_\alpha)\subseteq \R^q$ is a homeomorphism and all transition maps
$\varphi_\beta\circ\varphi_\alpha^{-1}$ are smooth on overlaps. For $\Omega\in\mathcal M$, the tangent space $T_\Omega\mathcal M$ is a $q$-dimensional real vector space.
For a smooth map $F:\mathcal M\to\mathcal N$, we write $dF_\Omega:T_\Omega\mathcal M\to T_{F(\Omega)}\mathcal N$ for its differential.
A smooth vector field $X$ assigns to each $\Omega\in\mathcal M$ a vector $X(\Omega)\in T_\Omega\mathcal M$ smoothly in charts.

\paragraph{Riemannian metrics and induced metrics.}
A Riemannian metric on $\mathcal M$ is a choice of inner product $\langle \cdot,\cdot\rangle_\Omega$ on each tangent space $T_\Omega\mathcal M$ that depends smoothly on $\Omega$.
The pair $(\mathcal M,\langle\cdot,\cdot\rangle)$ is called a Riemannian manifold.
If $\mathcal M\subseteq \R^{q_0}$ is an embedded $C^k$ submanifold of $\R^{q_0}$, the induced (ambient) metric is
$\langle u,v\rangle_\Omega \coloneqq  u^\top v$ for $u,v\in T_\Omega\mathcal M\subseteq \R^{q_0}$.

For a piecewise $C^1$ curve $\gamma:[0,1]\to\mathcal M$, its length is
\[
L(\gamma)=\int_0^1 \|\dot\gamma(t)\|_{\gamma(t)}\,dt,
\qquad
\|v\|_\Omega\coloneqq \sqrt{\langle v,v\rangle_\Omega}.
\]
The associated Riemannian distance is defined by
\[
d_{\mathcal M}(\Omega_1,\Omega_2)=\inf_{\gamma(0)=\Omega_1,\ \gamma(1)=\Omega_2} L(\gamma),
\]
where the infimum ranges over piecewise $C^1$ curves $\gamma:[0,1]\to\mathcal{M}$.

\paragraph{Levi--Civita connection, geodesics, gradients, and Hessians.}
There is a unique connection $\nabla$ on $\mathcal M$ (the Levi--Civita connection) that is torsion-free and metric-compatible.
A $C^2$ curve $\gamma$ is a geodesic if it satisfies $\nabla_{\dot\gamma(t)}\dot\gamma(t)=0$ for all $t$.

For a smooth function $f:\mathcal M\to\R$, the Riemannian gradient $\mathrm{grad}\,f(\Omega)\in T_\Omega\mathcal M$ is defined by the identity
\[
df_\Omega(v)=\langle \mathrm{grad}\,f(\Omega),v\rangle_\Omega,
\qquad v\in T_\Omega\mathcal M.
\]
The Riemannian Hessian at point $\Omega$ is the linear map $\mathrm{Hess}\,f(\Omega):T_\Omega\mathcal M\to T_\Omega\mathcal M$ given by
\[
\mathrm{Hess}\,f(\Omega)[v] \coloneqq  \nabla_V(\mathrm{grad}\,f)(\Omega),
\]
where $V$ is any smooth local extension of $v\in T_\Omega \mathcal M$. We also use the associated bilinear form as
\[
\mathrm{Hess}\,f(\Omega)(u,v)=\langle u,\mathrm{Hess}\,f(\Omega)[v]\rangle_\Omega.
\]

\paragraph{Exponential map, normal neighborhoods, and logarithm map.}
Fix $\Omega\in\mathcal M$. For each $v\in T_\Omega \mathcal{M}$, let $\gamma_v:[0,1]\to \mathcal M$ denote the unique geodesic with $\gamma_v(0)=\Omega$ and $\dot\gamma_v(0)=v$. The exponential map at $\Omega$ is
\[
\exp_\Omega:T_\Omega\mathcal M\to\mathcal M,
\qquad
\exp_\Omega(v)\coloneqq \gamma_v(1).
\]
Moreover, there exists $\delta_\Omega>0$ such that $\exp_\Omega$ restricts to a diffeomorphism from $B(\Omega,\delta_\Omega)\subset T_\Omega\mathcal{M}$ onto its image. On any set where this restriction is invertible, we write $\mathrm{log}_\Omega$ for the local inverse of $\exp_\Omega$.
\begin{definition}[Normal neighborhood and normal coordinates]\label{def:normal-neigh}
    An open subset $U\subseteq\mathcal M$ is a \emph{normal neighborhood} of $\Omega$ if there exists a $\delta>0$ such that $U=\exp_\Omega(B(\Omega,\delta))$ and the restriction $\exp_\Omega:B(\Omega,\delta)\to U$ is a diffeomorphism. On such a $U$, the logarithm map at $\Omega$ is the inverse chart $\log_\Omega: U\to T_\Omega\mathcal M$, and the resulting chart $\log_\Omega$ defines the \emph{normal coordinates} at $\Omega$.
\end{definition}

\subsection{Taylor expansions in normal coordinates}
\label{app:riem_taylor}

This subsection records the Taylor expansions we use to linearize first-order optimality conditions on the descriptor manifold.

\paragraph{Normal coordinates and pullbacks.}
Let $U\subseteq\mathcal M$ be a normal neighborhood of $\Omega$ (Definition~\ref{def:normal-neigh}).
Given a function $f:\mathcal M\to\R$, we pull $f$ back to the tangent space via 
\[
\tilde f(v)\coloneqq f(\exp_\Omega(v)),
\qquad v\in B(\Omega, \delta)\subseteq T_\Omega\mathcal M.
\]
Equivalently, for $\Omega'\in U$, we write $v=\log_\Omega(\Omega')$ and view $\tilde f$ as $f$ expressed in normal coordinates around $\Omega$.

\paragraph{First- and second-order expansions of a smooth function.}
Let $U$ be a normal neighborhood of $\Omega_1$ and write $\Omega_2=\exp_{\Omega_1}(v)$ with $v\in T_{\Omega_1}\mathcal M$. If $f$ is $C^1(\mathcal{M},\mathbb R)$ on $U$, then as $v\to0$ in $T_{\Omega_1}\mathcal{M}$,
\begin{align}
f(\exp_{\Omega_1}(v))&=
f(\Omega_1)+\langle \mathrm{grad},f(\Omega_1),v\rangle_{\Omega_1}
+
o(\|v\|_{\Omega_1}).
\label{eq:riem_taylor_f_first}
\end{align}
If $f$ is $C^2(\mathcal{M},\mathbb R)$ on $U$, then as $v\to0$ in $T_{\Omega_1}\mathcal{M}$,
\begin{align}
f(\exp_{\Omega_1}(v))&=
f(\Omega_1)+\langle \mathrm{grad},f(\Omega_1),v\rangle_{\Omega_1}
+\frac12\big\langle v,\ \mathrm{Hess},f(\Omega_1)[v]\big\rangle_{\Omega_1}
+o(\|v\|^2_{\Omega_1}).
\label{eq:riem_taylor_f_second}
\end{align}
The restriction to a normal neighborhood ensures that nearby points admit a unique representation via the logarithm map, so that the above expansions are well-defined and intrinsic.
\paragraph{Taylor expansion of the gradient via parallel transport.}
Fix $\Omega_1\in\mathcal M$ and let $U$ be a normal neighborhood of $\Omega_1$.
For $v\in T_{\Omega_1}\mathcal M$ sufficiently small, define the unique geodesic
\[
\gamma(t)\coloneqq \exp_{\Omega_1}(t v), \qquad t\in[0,1],
\]
so that $\gamma(0)=\Omega_1$ and $\gamma(1)=\Omega_2\coloneqq \exp_{\Omega_1}(v)$.

\emph{Parallel transport.}
Let $\mathcal P_{\Omega_1\to \Omega_2}:T_{\Omega_1}\mathcal M\to T_{\Omega_2}\mathcal M$
denote parallel transport along $\gamma$.
Given $w\in T_{\Omega_1}\mathcal M$, let $W(t)\in T_{\gamma(t)}\mathcal M$ be the unique vector field along $\gamma$ that has the following property
\[
\nabla_{\dot\gamma(t)} W(t)=0,
\qquad
W(0)=w,
\]
and set $\mathcal P_{\Omega_1\to \Omega_2}w\coloneqq W(1)$.
We write $\mathcal P_{\Omega_2\to \Omega_1}\coloneqq \mathcal P_{\Omega_1\to \Omega_2}^{-1}$.

\emph{Gradient expansion.}
If $f:\mathcal M\to\mathbb R$ is $C^2$ on $U$, then as $v\to 0$,
\begin{align}
\mathcal P_{\Omega_2\to \Omega_1}\big(\mathrm{grad}\,f(\Omega_2)\big)
=
\mathrm{grad}\,f(\Omega_1)+\mathrm{Hess}\,f(\Omega_1)[v] + r_f(v),
\label{eq:riem_taylor_grad}
\end{align}
where $r_f(v)=o(\|v\|_{\Omega_1})$.
If, in addition, $\mathrm{Hess}\,f$ is locally Lipschitz on $U$ (with respect to $d_{\mathcal M}$), then there exist
constants $C>0$ and $\varepsilon>0$ such that
\begin{align}
\|r_f(v)\|_{\Omega_1} \le C\,\|v\|_{\Omega_1}^2
\qquad
\text{for all $\|v\|_{\Omega_1}\le \varepsilon$.}
\label{eq:riem_taylor_grad_bound}
\end{align}

\paragraph{Empirical objectives.}
The expansion \eqref{eq:riem_taylor_grad} applies to empirical objectives of the form
\[
\hat f \coloneqq \frac1m \sum_{i=1}^m f_i,
\]
provided each $f_i$ is $C^2$ on a common normal neighborhood $U$ of $\Omega_\star$, and the corresponding derivatives admit
local bounds on $U$ (so that the remainder is uniform for $\|v\|_{\Omega_\star}$ small).
In particular, for $v\in T_{\Omega_\star}\mathcal M$ small and $\hat \Omega\coloneqq \exp_{\Omega_\star}(v)$,
\begin{align}
\mathcal P_{\hat \Omega\to \Omega_\star}\big(\mathrm{grad}\,\hat f(\hat \Omega)\big)
=
\mathrm{grad}\,\hat f(\Omega_\star)+\mathrm{Hess}\,\hat f(\Omega_\star)[v] + r_{\hat f}(v),
\label{eq:riem_taylor_grad_emp}
\end{align}
where $r_{\hat f}(v)=o(\|v\|_{\Omega_\star})$ as $v\to 0$.
If, in addition, $\mathrm{Hess}\,\hat f$ is locally Lipschitz on $U$, then $\|r_{\hat f}(v)\|_{\Omega_\star}=O(\|v\|_{\Omega_\star}^2)$.

\paragraph{Remark (normal coordinates vs.\ parallel transport).}
In normal coordinates at $\Omega$, one may view \eqref{eq:riem_taylor_grad} as an ordinary Euclidean Taylor expansion
of the pullback $\tilde f(v)=f(\exp_\Omega(v))$ at $v=0$. The parallel-transport form is convenient because it compares vectors in a single space $T_\Omega\mathcal M$.

\subsection{Euclidean $M$-estimation: consistency and asymptotic normality}
\label{app:mestimation_euclid}

We first review standard Euclidean $M$-estimation. Note that all assumptions and results presented in this section are classical~\citep{van2000asymptotic}; 
our purpose for recording these
arguments here is that we will follow
their structure closely when generalizing
to the Riemannian setting in \Cref{app:mestimation_riem}.

Let $(\mathcal Z,\mathcal G)$ be a measurable space and let $Z_1,\dots,Z_m$ be i.i.d.\ with law $\mathbb P$ on $\mathcal Z$.
Let $\Theta\subseteq\mathbb R^p$ be the parameter set. Given a measurable loss $\ell:\Theta\times\mathcal Z\to\mathbb R$, define
\[
\hat L_m(\theta)\coloneqq \frac1m\sum_{i=1}^m \ell(\theta;Z_i),
\qquad
L(\theta)\coloneqq \mathbb E[\ell(\theta;Z)],
\]
where $Z\sim\mathbb P$ and we assume $L(\theta)$ is well-defined (possibly $+\infty$) for all $\theta\in\Theta$.
An \emph{$M$-estimator} is any measurable selection $\hat\theta_m$ from the set of empirical minimizers,
\[
\hat\theta_m \in \argmin_{\theta\in\Theta}\hat L_m(\theta),
\]
whenever this set is nonempty.

\begin{assmp}[Euclidean $M$-estimation conditions]
\label{ass:euclid_mestimation}
There exist $\theta_\star\in\Theta$, an open set $U\subseteq\mathbb R^p$ with $\theta_\star\in U$, and $r>0$ with
$\overline{B}(\theta_\star,r)\subseteq U\cap\Theta$ such that:
\begin{enumerate}[label=(\roman*)]
\item \textbf{(Identification and separation).}
$\theta_\star$ is the unique minimizer of $L$ on $\Theta$ and for every $\epsilon>0$,
\[
\inf_{\theta\in\Theta:\ \|\theta-\theta_\star\|\ge \epsilon}\bigl(L(\theta)-L(\theta_\star)\bigr)>0.
\]

\item \textbf{(Uniform LLN on a compact set).}
On the compact set $\overline{B}(\theta_\star,r)$, we have
\[
\sup_{\theta\in \overline{B}(\theta_\star,r)}\bigl|\hat L_m(\theta)-L(\theta)\bigr|\ \xrightarrow{\ \mathbb P\ }\ 0,
\]
and $\hat\theta_m\in \overline{B}(\theta_\star,r)$ with probability tending to one.

\item \textbf{(Local $C^2$ smoothness and score moments).}
For $\mathbb P$-a.e.\ $z$, the map $\theta\mapsto \ell(\theta;z)$ is $C^2$ on $U$, and
$\mathbb E[\|\nabla \ell(\theta_\star;Z)\|^2]<\infty$.

\item \textbf{(Nondegenerate minimizer).}
The matrix $H_\star\coloneqq \nabla^2 L(\theta_\star)$ is invertible.

\item \textbf{(Uniform Hessian convergence on $\overline{B}(\theta_\star,r)$).}
\[
\sup_{\theta\in \overline{B}(\theta_\star,r)}
\big\|\nabla^2 \hat L_m(\theta)-\nabla^2 L(\theta)\big\|\ \xrightarrow{\ \mathbb P\ }\ 0.
\]
\end{enumerate}
\end{assmp}

Define
\[
\Sigma_\star\coloneqq \mathrm{Var}\big(\nabla \ell(\theta_\star;Z)\big)
=\mathbb E\Big[\nabla \ell(\theta_\star;Z)\nabla \ell(\theta_\star;Z)^\top\Big],
\]
where $\mathbb E[\nabla \ell(\theta_\star;Z)]=\nabla L(\theta_\star)=0$.
\begin{proposition}[Euclidean $M$-estimator consistency]
\label{prop:euclid_consistency}
Assume Assumption~\ref{ass:euclid_mestimation}\,(i)--(ii). Then
\[
\hat\theta_m \ \xrightarrow{\ \mathbb P\ }\ \theta_\star.
\]
\end{proposition}

\begin{proof}
The argument follows standard $M$-estimation proofs (see, e.g.,~\cite{van2000asymptotic}); we include it
as a reference, since we will soon generalize this argument to Riemannian manifolds.

Fix $\epsilon>0$ and define the separation gap
\[
\Delta_\epsilon \coloneqq 
\inf_{\theta\in\Theta:\ \|\theta-\theta_\star\|\ge \epsilon}\bigl(L(\theta)-L(\theta_\star)\bigr),
\]
so that $\Delta_\epsilon>0$ by Assumption~\ref{ass:euclid_mestimation}\,(i). Consider the event
\[
E_m \coloneqq 
\left\{
\sup_{\theta\in \overline{B}(\theta_\star,r)}\bigl|\hat L_m(\theta)-L(\theta)\bigr| \le \frac{\Delta_\epsilon}{3}
\right\}
\cap \{\hat\theta_m\in \overline{B}(\theta_\star,r)\}.
\]
By the uniform law of large numbers on $\overline{B}(\theta_\star,r)$ and the localization $\mathbb P(\hat\theta_m\in \overline{B}(\theta_\star,r))\to 1$,
we have $\mathbb P(E_m)\to 1$.

On the event $E_m$, for any $\theta\in \overline{B}(\theta_\star,r)$ with $\|\theta-\theta_\star\|\ge \epsilon$ we have
\[
\hat L_m(\theta)
\;\ge\;
L(\theta) - \frac{\Delta_\epsilon}{3}
\;\ge\;
L(\theta_\star) + \Delta_\epsilon - \frac{\Delta_\epsilon}{3}
\;=\;
L(\theta_\star) + \frac{2\Delta_\epsilon}{3},
\]
while
\[
\hat L_m(\theta_\star)
\;\le\;
L(\theta_\star) + \frac{\Delta_\epsilon}{3}.
\]
Hence, on $E_m$,
\[
\inf_{\theta\in \overline{B}(\theta_\star,r):\ \|\theta-\theta_\star\|\ge \epsilon}\hat L_m(\theta)
\;>\;
\hat L_m(\theta_\star).
\]
In particular, any minimizer of $\hat L_m$ over $\overline{B}(\theta_\star,r)$ must lie in $B(\theta_\star,\epsilon)$.
Since $\hat\theta_m\in \overline{B}(\theta_\star,r)$ on $E_m$ and $\hat\theta_m$ is (by assumption) an empirical minimizer, we conclude that
$\|\hat\theta_m-\theta_\star\|<\epsilon$ on $E_m$.
Therefore,
\[
\mathbb P\big(\|\hat\theta_m-\theta_\star\|\ge \epsilon\big)
\;\le\;
\mathbb P(E_m^c)\ \to\ 0,
\]
which proves $\hat\theta_m\to\theta_\star$ in probability.
\end{proof}

\begin{theorem}[Euclidean $M$-estimator CLT]
\label{thm:euclid_mestimator_clt}
Under \Cref{ass:euclid_mestimation},
\[
\sqrt n\,(\hat\theta_n-\theta_\star)\ \distconv\ \mathcal N\!\big(0,\ H_\star^{-1}\Sigma_\star H_\star^{-1}\big).
\]
\end{theorem}

\begin{proof}
The argument follows standard $M$-estimation proofs (see, e.g.,~\cite{van2000asymptotic}); we include it
as a reference, since we will soon generalize this argument to Riemannian manifolds.

By \Cref{prop:euclid_consistency} and \Cref{ass:euclid_mestimation}\,(i)--(ii), we have
\[
\hat\theta_m \xrightarrow{\ \mathbb P\ } \theta_\star.
\]
In particular, since $\hat\theta_m\in \overline{B}(\theta_\star,r)$ with probability tending to one by
\Cref{ass:euclid_mestimation}\,(ii), all arguments below may be restricted to the event
$\{\hat\theta_m\in \overline{B}(\theta_\star,r)\}$.

On the event $\{\hat\theta_m\in \overline{B}(\theta_\star,r)\}$, the first-order condition for the empirical minimizer gives
\[
\nabla \hat L_m(\hat\theta_m)=0.
\]
Since $\ell(\,\cdot\,;z)$ is $C^2$ on $U$ for $\mathbb P$-a.e.\ $z$ by \Cref{ass:euclid_mestimation}\,(iii), the map
$\theta\mapsto \nabla \hat L_m(\theta)$ is differentiable on $U$ and we may apply the mean-value form of Taylor's theorem:
there exists a point $\tilde\theta_m$ on the line segment between $\theta_\star$ and $\hat\theta_m$ such that
\begin{equation}
\label{eq:euclid_score_taylor}
0
=
\nabla \hat L_m(\hat\theta_m)
=
\nabla \hat L_m(\theta_\star)
+
\nabla^2 \hat L_m(\tilde\theta_m)\,(\hat\theta_m-\theta_\star).
\end{equation}
Rearranging yields
\begin{equation}
\label{eq:euclid_linearization}
\sqrt m\,(\hat\theta_m-\theta_\star)
=
-\Big(\nabla^2 \hat L_m(\tilde\theta_m)\Big)^{-1}\,\sqrt m\,\nabla \hat L_m(\theta_\star),
\end{equation}
on the event that $\nabla^2\hat L_m(\tilde\theta_m)$ is invertible.

Since $\tilde\theta_m$ lies on the segment between $\theta_\star$ and $\hat\theta_m$, we have
$\tilde\theta_m\in \overline{B}(\theta_\star,r)$ whenever $\hat\theta_m\in \overline{B}(\theta_\star,r)$.
Moreover, by consistency $\hat\theta_m\to\theta_\star$ in probability, hence $\tilde\theta_m\to\theta_\star$ in probability as well.
By the uniform Hessian convergence in \Cref{ass:euclid_mestimation}\,(v),
\[
\sup_{\theta\in \overline{B}(\theta_\star,r)}
\big\|\nabla^2 \hat L_m(\theta)-\nabla^2 L(\theta)\big\|\ \xrightarrow{\ \mathbb P\ }\ 0,
\]
and therefore
\[
\nabla^2 \hat L_m(\tilde\theta_m) - \nabla^2 L(\tilde\theta_m)\ \xrightarrow{\ \mathbb P\ }\ 0.
\]
Since $L$ is twice differentiable at $\theta_\star$ and $\tilde\theta_m\to\theta_\star$ in probability, we also have
$\nabla^2 L(\tilde\theta_m)\to \nabla^2 L(\theta_\star)=H_\star$ in probability. Combining these gives
\begin{equation}
\label{eq:hess_to_Hstar}
\nabla^2 \hat L_m(\tilde\theta_m)\ \xrightarrow{\ \mathbb P\ }\ H_\star.
\end{equation}
By \Cref{ass:euclid_mestimation}\,(iv), $H_\star$ is invertible, hence by continuity of matrix inversion,
\begin{equation}
\label{eq:inv_hess_to_invH}
\Big(\nabla^2 \hat L_m(\tilde\theta_m)\Big)^{-1}\ \xrightarrow{\ \mathbb P\ }\ H_\star^{-1}.
\end{equation}
In particular, $\nabla^2 \hat L_m(\tilde\theta_m)$ is invertible with probability tending to one.

By definition,
\[
\nabla \hat L_m(\theta_\star)
=
\frac1m\sum_{i=1}^m \nabla \ell(\theta_\star;Z_i).
\]
Since $\theta_\star$ is a minimizer of $L$ and $L$ is differentiable at $\theta_\star$, we have
$\mathbb E[\nabla \ell(\theta_\star;Z)]=\nabla L(\theta_\star)=0$, hence the summands are mean-zero.
By \Cref{ass:euclid_mestimation}\,(iii), $\mathbb E[\|\nabla \ell(\theta_\star;Z)\|^2]<\infty$, so the multivariate CLT yields
\begin{equation}
\label{eq:score_clt}
\sqrt m\,\nabla \hat L_m(\theta_\star)\ \distconv\ \mathcal N(0,\Sigma_\star).
\end{equation}

Combining \eqref{eq:euclid_linearization}, \eqref{eq:inv_hess_to_invH}, and \eqref{eq:score_clt}, and applying Slutsky's theorem,
we obtain
\[
\sqrt m\,(\hat\theta_m-\theta_\star)
\ \distconv\
-\,H_\star^{-1}\,G,
\qquad
G\sim \mathcal N(0,\Sigma_\star).
\]
Since $-H_\star^{-1}G \sim \mathcal N(0, H_\star^{-1}\Sigma_\star H_\star^{-1})$, this proves the claim.
\end{proof}

\subsection{$M$-estimation on a Riemannian manifold}
\label{app:mestimation_riem}

Let $(\mathcal M,\langle\cdot,\cdot\rangle)$ be a finite-dimensional $C^2$ Riemannian manifold.
Let $(\mathcal Z,\mathcal G)$ be a measurable space and let $Z_1,\dots,Z_m$ be i.i.d.\ with common law $\mathbb P$ on $\mathcal Z$.
Let $f:\mathcal M\times\mathcal Z\to\mathbb R$ be a measurable loss such that the expectations below are well-defined. Define
\[
\hat F_m(\Omega)\coloneqq \frac1m\sum_{i=1}^m f(\Omega;Z_i),
\qquad
F(\Omega)\coloneqq \mathbb E[f(\Omega;Z)],
\qquad \Omega\in\mathcal M.
\]
An \emph{$M$-estimator} is a measurable map $\hat \Omega_m=\hat \Omega_m(Z_1,\dots,Z_m)$ such that
$\hat \Omega_m\in\argmin_{\Omega\in\mathcal M}\hat F_m(\Omega)$ whenever the argmin set is nonempty.

Fix $\Omega_\star\in\mathcal M$. For $\epsilon>0$, define the tangent-space ball
\[
B(\Omega_\star,\epsilon)\coloneqq \{v\in T_{\Omega_\star}\mathcal M:\ \|v\|_{\Omega_\star}<\epsilon\},
\qquad
\overline{B}(\Omega_\star,\epsilon)\coloneqq \{v\in T_{\Omega_\star}\mathcal M:\ \|v\|_{\Omega_\star}\le\epsilon\}.
\]

\begin{assmp}[Riemannian $M$-estimation conditions]
\label{ass:riem_mestimation}
There exist $\Omega_\star\in\mathcal M$, a normal neighborhood
$U=\exp_{\Omega_\star}\!\big(B(\Omega_\star,\epsilon_0)\big)$, and $\epsilon'\in(0,\epsilon_0)$ such that
$\exp_{\Omega_\star}\!\big(\overline{B}(\Omega_\star,\epsilon')\big)\subseteq U$ and:
\begin{enumerate}[label=(\roman*)]
\item \textbf{(Identification and separation).}
$\Omega_\star$ is the unique minimizer of $F$ on $\mathcal M$ and for every $\epsilon>0$,
\[
\inf_{\Omega\in\mathcal M:\ d_{\mathcal M}(\Omega,\Omega_\star)\ge \epsilon}
\bigl(F(\Omega)-F(\Omega_\star)\bigr)>0.
\]

\item \textbf{(Uniform LLN on a compact set).}
On the compact set $\exp_{\Omega_\star}\!\big(\overline{B}(\Omega_\star,\epsilon')\big)$, we have
\[
\sup_{\Omega\in \exp_{\Omega_\star}(\overline{B}(\Omega_\star,\epsilon'))}
\bigl|\hat F_m(\Omega)-F(\Omega)\bigr|\ \xrightarrow{\ \mathbb P\ }\ 0,
\]
and $\hat\Omega_m\in \exp_{\Omega_\star}\!\big(\overline{B}(\Omega_\star,\epsilon')\big)$ with probability tending to one.

\item \textbf{(Local $C^2$ smoothness and score moments).}
For $\mathbb P$-a.e.\ $z$, the map $\Omega\mapsto f(\Omega;z)$ is $C^2$ on $U$, and
\[
\mathbb E\big[\|\mathrm{grad}\, f(\Omega_\star;Z)\|_{\Omega_\star}^2\big]<\infty.
\]

\item \textbf{(Nondegenerate minimizer).}
The linear map $H_\star\coloneqq \mathrm{Hess}\,F(\Omega_\star):T_{\Omega_\star}\mathcal M\to T_{\Omega_\star}\mathcal M$
is invertible.

\item \textbf{(Uniform transported Hessian convergence on $\exp_{\Omega_\star}(\overline{B}(\Omega_\star,\epsilon'))$).}
For $\Omega\in \exp_{\Omega_\star}\!\big(\overline{B}(\Omega_\star,\epsilon')\big)$, define
\[
\widetilde H_m(\Omega)
\coloneqq 
\mathcal P_{\Omega\to \Omega_\star}\circ \mathrm{Hess}\,\hat F_m(\Omega)\circ \mathcal P_{\Omega_\star\to \Omega},
\qquad
\widetilde H(\Omega)
\coloneqq 
\mathcal P_{\Omega\to \Omega_\star}\circ \mathrm{Hess}\,F(\Omega)\circ \mathcal P_{\Omega_\star\to \Omega}.
\]
Then
\[
\sup_{\Omega\in \exp_{\Omega_\star}(\overline{B}(\Omega_\star,\epsilon'))}
\big\|\widetilde H_m(\Omega)-\widetilde H(\Omega)\big\|\ \xrightarrow{\ \mathbb P\ }\ 0.
\]
\end{enumerate}
\end{assmp}

Define the covariance operator $\Sigma_\star$ on $T_{\Omega_\star}\mathcal M$ by
\[
\langle u,\ \Sigma_\star v\rangle_{\Omega_\star}
=
\mathbb E\Big[\langle u,\ \mathrm{grad}\,f(\Omega_\star;Z)\rangle_{\Omega_\star}\,
               \langle v,\ \mathrm{grad}\,f(\Omega_\star;Z)\rangle_{\Omega_\star}\Big],
\qquad u,v\in T_{\Omega_\star}\mathcal M.
\]
\begin{proposition}[Riemannian $M$-estimator consistency]
\label{prop:riem_consistency}
Assume Assumption~\ref{ass:riem_mestimation}\,(i)--(ii). Then
\[
\hat\Omega_m \ \xrightarrow{\ \mathbb P\ }\ \Omega_\star.
\]
\end{proposition}

\begin{proof}
Fix $\epsilon>0$ and define the separation gap
\[
\Delta_\epsilon \coloneqq 
\inf_{\Omega\in\mathcal M:\ d_{\mathcal M}(\Omega,\Omega_\star)\ge \epsilon}
\bigl(F(\Omega)-F(\Omega_\star)\bigr),
\]
so that $\Delta_\epsilon>0$ by Assumption~\ref{ass:riem_mestimation}\,(i). Consider the event
\[
E_m \coloneqq 
\left\{
\sup_{\Omega\in \exp_{\Omega_\star}(\overline{B}(\Omega_\star,\epsilon'))}
\bigl|\hat F_m(\Omega)-F(\Omega)\bigr|
\le \frac{\Delta_\epsilon}{3}
\right\}
\cap 
\left\{\hat\Omega_m\in \exp_{\Omega_\star}(\overline{B}(\Omega_\star,\epsilon'))\right\}.
\]
By the uniform law of large numbers on $\exp_{\Omega_\star}(\overline{B}(\Omega_\star,\epsilon'))$ and the localization in
Assumption~\ref{ass:riem_mestimation}\,(ii), we have $\mathbb P(E_m)\to 1$.

On the event $E_m$, for any $\Omega\in \exp_{\Omega_\star}(\overline{B}(\Omega_\star,\epsilon'))$ with
$d_{\mathcal M}(\Omega,\Omega_\star)\ge \epsilon$ we have
\[
\hat F_m(\Omega)
\;\ge\;
F(\Omega) - \frac{\Delta_\epsilon}{3}
\;\ge\;
F(\Omega_\star) + \Delta_\epsilon - \frac{\Delta_\epsilon}{3}
\;=\;
F(\Omega_\star) + \frac{2\Delta_\epsilon}{3},
\]
while
\[
\hat F_m(\Omega_\star)
\;\le\;
F(\Omega_\star) + \frac{\Delta_\epsilon}{3}.
\]
Hence, on $E_m$,
\[
\inf_{\Omega\in \exp_{\Omega_\star}(\overline{B}(\Omega_\star,\epsilon')):\ d_{\mathcal M}(\Omega,\Omega_\star)\ge \epsilon}
\hat F_m(\Omega)
\;>\;
\hat F_m(\Omega_\star).
\]
In particular, any minimizer of $\hat F_m$ over $\exp_{\Omega_\star}(\overline{B}(\Omega_\star,\epsilon'))$ must lie in the metric ball
\[
B_{d_{\mathcal M}}(\Omega_\star,\epsilon)
\coloneqq
\{\Omega\in\mathcal M:\ d_{\mathcal M}(\Omega,\Omega_\star)<\epsilon\}.
\]
Since $\hat\Omega_m\in \exp_{\Omega_\star}(\overline{B}(\Omega_\star,\epsilon'))$ on $E_m$ and $\hat\Omega_m$ is (by definition) an empirical minimizer,
we conclude that $d_{\mathcal M}(\hat\Omega_m,\Omega_\star)<\epsilon$ on $E_m$. Therefore,
\[
\mathbb P\big(d_{\mathcal M}(\hat\Omega_m,\Omega_\star)\ge \epsilon\big)
\;\le\;
\mathbb P(E_m^c)\ \to\ 0,
\]
which proves $\hat\Omega_m\to \Omega_\star$ in probability.
\end{proof}

\begin{theorem}[Riemannian $M$-estimator CLT]
\label{thm:riem_mestimator_clt}
Under \Cref{ass:riem_mestimation}, we have
\[
\sqrt m\,\log_{\Omega_\star}(\hat\Omega_m)
\ \distconv\
\mathcal N\!\big(0,\ H_\star^{-1}\Sigma_\star H_\star^{-1}\big),
\]
where
\[
H_\star \coloneqq \mathrm{Hess}\,F(\Omega_\star):T_{\Omega_\star}\mathcal M\to T_{\Omega_\star}\mathcal M
\]
and
\[
\Sigma_\star \coloneqq \mathrm{Var}\big(\mathrm{grad}\,f(\Omega_\star;Z)\big)
=
\mathbb E\Big[\mathrm{grad}\,f(\Omega_\star;Z)\,\mathrm{grad}\,f(\Omega_\star;Z)^\top\Big],
\]
with $\mathbb E[\mathrm{grad}\,f(\Omega_\star;Z)]=\mathrm{grad}\,F(\Omega_\star)=0$.
\end{theorem}
\begin{remark}
    In~\Cref{thm:riem_mestimator_clt}, the logarithm map $\log_{\Omega_\star}$ is defined on a normal neighborhood of $\Omega_\star$. While $\hat\Omega_m$ need not belong to this neighborhood for each finite $m$, consistency ensures that $\hat\Omega_m \to \Omega_\star$ in probability. Consequently, $\log_{\Omega_\star}(\hat\Omega_m)$ is well-defined with probability tending to one.
\end{remark}
\begin{proof}
By \Cref{prop:riem_consistency} and \Cref{ass:riem_mestimation}\,(i)--(ii), we have
\[
\hat\Omega_m \xrightarrow{\ \mathbb P\ } \Omega_\star.
\]
Fix a normal neighborhood $U=\exp_{\Omega_\star}(B_{\Omega_\star}(\epsilon_0))$ and define $A_m\coloneqq\{\hat\Omega_m\in U\}$. Since $\hat\Omega_m \xrightarrow{\mathbb P} \Omega_\star$ and $U$ is an open neighborhood of $\Omega_\star$, we have $\mathbb P(A_m)\to 1$. In~\Cref{def:normal-neigh}, we defined the logarithm map $\log_{\Omega_\star}:U\to T_{\Omega_\star}\mathcal M$ as the unique inverse of the exponential map $\exp_{\Omega_\star}$. Define
\[
v_m :=
\begin{cases}
\log_{\Omega_\star}(\hat{\Omega}_m)\in T_{\Omega_\star}\mathcal M, & \text{on } A_m,\\
0 \in T_{\Omega_\star}\mathcal M, & \text{on } A_m^c.
\end{cases}
\]
Then, $v_m$ is well-defined globally and satisfies $v_m\to 0$ in probability. 

On the event $A_m$, the first-order condition for an empirical minimizer gives
\[
\mathrm{grad}\,\hat F_m(\hat\Omega_m)=0.
\]
Applying the transported gradient expansion \eqref{eq:riem_taylor_grad_emp} with $\Omega_1=\Omega_\star$,
$\Omega_2=\hat\Omega_m=\exp_{\Omega_\star}(v_m)$, and $f=\hat F_m$, we obtain
\[
0
=
\mathcal P_{\hat\Omega_m\to \Omega_\star}\big(\mathrm{grad}\,\hat F_m(\hat\Omega_m)\big)
=
\mathrm{grad}\,\hat F_m(\Omega_\star)
+\mathrm{Hess}\,\hat F_m(\Omega_\star)[v_m]
+r_m,
\]
where $r_m$ denotes the Taylor remainder on $A_m$, and we define $r_m=0$ on $A_m^c$. Moreover,
\[
\|r_m\|_{\Omega_\star}\mathbf{1}_{A_m}
=
o_{\mathbb P}\big(\|v_m\|_{\Omega_\star}\big).
\]
Rearranging yields
\begin{equation}
\label{eq:riem_linearization}
\sqrt m\,v_m
=
-\Big(\mathrm{Hess}\,\hat F_m(\Omega_\star)\Big)^{-1}\,\sqrt m\,\mathrm{grad}\,\hat F_m(\Omega_\star)
\mathbf{1}_{A_m}-\Big(\mathrm{Hess}\,\hat F_m(\Omega_\star)\Big)^{-1}\,\sqrt m\,r_m\mathbf{1}_{A_m},
\end{equation}
on the event $\{ \mathrm{Hess}\,\hat F_m(\Omega_\star) \text{ is invertible}\}$.

By \Cref{ass:riem_mestimation}\,(v) applied at $\Omega=\Omega_\star$ (so that $\mathcal P_{\Omega_\star\to \Omega_\star}=\mathrm{Id}$),
\[
\big\|\mathrm{Hess}\,\hat F_m(\Omega_\star)-\mathrm{Hess}\,F(\Omega_\star)\big\|
=
\big\|\widetilde H_m(\Omega_\star)-\widetilde H(\Omega_\star)\big\|
\ \xrightarrow{\ \mathbb P\ }\ 0.
\]
Hence,
\begin{equation}
\label{eq:riem_hess_to_Hstar}
\mathrm{Hess}\,\hat F_m(\Omega_\star)\ \xrightarrow{\ \mathbb P\ }\ H_\star.
\end{equation}
By \Cref{ass:riem_mestimation}\,(iv), $H_\star$ is invertible, and by continuity of inversion we have
\begin{equation}
\label{eq:riem_inv_hess_to_invH}
\Big(\mathrm{Hess}\,\hat F_m(\Omega_\star)\Big)^{-1}\ \xrightarrow{\ \mathbb P\ }\ H_\star^{-1}.
\end{equation}
In particular, $\mathrm{Hess}\,\hat F_m(\Omega_\star)$ is invertible with probability tending to one.

By definition,
\[
\mathrm{grad}\,\hat F_m(\Omega_\star)
=
\frac1m\sum_{i=1}^m \mathrm{grad}\,f(\Omega_\star;Z_i).
\]
Since $\Omega_\star$ is a minimizer of $F$ and $F$ is differentiable at $\Omega_\star$, we have
$\mathbb E[\mathrm{grad}\,f(\Omega_\star;Z)]=\mathrm{grad}\,F(\Omega_\star)=0$.
By \Cref{ass:riem_mestimation}\,(iii), $\mathbb E[\|\mathrm{grad}\,f(\Omega_\star;Z)\|_{\Omega_\star}^2]<\infty$, so the multivariate CLT yields
\begin{equation}
\label{eq:riem_score_clt}
\sqrt m\,\mathrm{grad}\,\hat F_m(\Omega_\star)\ \distconv\ \mathcal N(0,\Sigma_\star).
\end{equation}

Since $\norm{r_m}_{\Omega_\star} \mathbf{1}_{A_m}=o\big(\|v_m\|_{\Omega_\star}\big)$ in probability on the event $A_m$, we have
\[
\frac{\|r_m\|_{\Omega_\star}}{\|v_m\|_{\Omega_\star}} \mathbf{1}_{A_m} \ \xrightarrow{\ \mathbb P\ }\ 0,
\]
with the convention that the ratio is set to $0$ on the event $\{v_m=0\}$.
Moreover, from \eqref{eq:riem_linearization} and \eqref{eq:riem_inv_hess_to_invH} we have
$\big(\mathrm{Hess}\,\hat F_m(\Omega_\star)\big)^{-1}=O_{\mathbb P}(1)$ and
$\sqrt m\,\mathrm{grad}\,\hat F_m(\Omega_\star)=O_{\mathbb P}(1)$, hence
\[
\sqrt m\,\|v_m\|_{\Omega_\star}=O_{\mathbb P}(1).
\]
Consequently,
\[
\sqrt m\,\|r_m\|_{\Omega_\star} \mathbf{1}_{A_m}
=
\big(\sqrt m\,\|v_m\|_{\Omega_\star}\big)\cdot
\frac{\|r_m\|_{\Omega_\star}}{\|v_m\|_{\Omega_\star}}\mathbf{1}_{A_m}
\ \xrightarrow{\ \mathbb P\ }\ 0,
\]
and therefore
\begin{equation}
\label{eq:riem_remainder_negligible}
\sqrt m\,r_m\mathbf{1}_{A_m} \ \xrightarrow{\ \mathbb P\ }\ 0.
\end{equation}

Combining \eqref{eq:riem_linearization}, \eqref{eq:riem_inv_hess_to_invH}, \eqref{eq:riem_score_clt}, and \eqref{eq:riem_remainder_negligible},
and applying Slutsky's theorem, using that $\mathbf{1}_{A_m}\xrightarrow{\ \mathbb P\ }\ \mathbf{1}$, we obtain
\[
\sqrt m\,v_m
\ \distconv\
-\,H_\star^{-1}\,G,
\qquad
G\sim \mathcal N(0,\Sigma_\star).
\]
Finally, on $A_m$, we have $v_m=\log_{\Omega_\star}(\hat\Omega_m)$, while $\mathbb P(A_m)\to 1$. Hence, $v_m$ and $\log_{\Omega_\star}(\hat\Omega_m)$ agree with probability tending to one, so they have the same asymptotic distribution. Since $-H_\star^{-1}G \sim \mathcal N(0,H_\star^{-1}\Sigma_\star H_\star^{-1})$, this proves the claim.
\end{proof}

\paragraph{Remark (Euclidean case as a special case).}
If $\mathcal M=\R^p$ with the Euclidean metric, then $\exp_{\theta_\star}(v)=\theta_\star+v$,
$\log_{\theta_\star}(\theta)=\theta-\theta_\star$, geodesics are line segments, and parallel transport is the identity.
In this case $\widetilde H_n(\theta)=\nabla^2\hat L_n(\theta)$.

\paragraph{Relation to \cite{brunel2023geodesically}.}
\Cref{thm:riem_mestimator_clt} is closely related to and inspired by the asymptotic-normality result of \cite{brunel2023geodesically} for geodesically convex $M$-estimators. \citeauthor{brunel2023geodesically} proves a tangent-space Gaussian limit on a complete Riemannian manifold under the following conditions: the sample losses are geodesically convex, the population objective has a unique minimizer, the population objective is twice differentiable at that minimizer with positive-definite Hessian, and the
loss admits measurable subgradients satisfying a local second-moment condition. 

The key difference is the mechanism used to obtain the local quadratic expansion underlying asymptotic normality. In \citeauthor{brunel2023geodesically}'s theorem, geodesic convexity is the structural assumption that allows for nonsmooth empirical objectives. The proof uses subgradient inequalities and convexity to compare the localized empirical objective in normal coordinates with a random quadratic approximation. Thus, geodesic convexity replaces the need for differentiability of the empirical criterion, a Taylor expansion of the empirical first-order condition, and uniform convergence of empirical Hessians. It does not however replace population-level second-order differentiability: \citeauthor{brunel2023geodesically} still assumes that the population objective is twice differentiable at the minimizer with positive-definite Hessian.

\Cref{thm:riem_mestimator_clt} takes a complementary route. We do not assume geodesic convexity of $L_{\mathrm{pre}}$.
Instead, after passing to the descriptor manifold that quotients out the symmetry, we directly impose the local regularity conditions needed for a smooth Riemannian $M$-estimation CLT. 
These conditions (cf.~\Cref{ass:riem_mestimation}) include local identification and separation of the population minimizer, localization and a uniform law of large numbers on a compact normal neighborhood, local $C^2$ smoothness of the sample loss, a nonsingular population Riemannian Hessian, score moment control, and uniform convergence of transported empirical Hessians. Under these assumptions, \Cref{thm:riem_mestimator_clt} follows by Taylor expanding the transported first-order condition in $T_{\Omega_\star}\mathcal M$.

Thus, \Cref{thm:riem_mestimator_clt} can be viewed as abstracting the local smooth assumptions needed for asymptotic normality away from the particular sufficient condition of geodesic convexity. 
When $\ell_\mathrm{pre}$ is $C^2$ smooth, geodesic convexity together with additional empirical Hessian regularity verifies Assumption~\ref{ass:riem_mestimation}, and the two approaches yield the same sandwich-form tangent-space CLT. However, \citeauthor{brunel2023geodesically}'s assumptions do not necessarily imply Assumption~\ref{ass:riem_mestimation}, since \citeauthor{brunel2023geodesically}'s theorem allows nonsmooth empirical losses through measurable subgradients. Conversely,~\Cref{thm:riem_mestimator_clt} does not require geodesic convexity: any smooth descriptor-manifold pre-training problem satisfying Assumption~\ref{ass:riem_mestimation}, whether geodesically convex or not, falls under \Cref{thm:riem_mestimator_clt}. Our formulation is necessary for the quotient-manifold pre-training problems considered here, where the natural objective is locally smooth after passing to descriptor coordinates, but need not be globally geodesically convex.

\section{Symmetry, Identifiability, and Quotient Geometry}
\label{app:identifiability}

This appendix formalizes how symmetry-induced non-identifiability in pretraining is handled in our analysis.

\subsection{Quotients by group actions and local descriptor charts}
\label{app:quotients}

\paragraph{Smooth group actions.}
Let $G$ be a Lie group acting smoothly on a smooth manifold $\mathcal A$.
We write the action as $(g,a)\mapsto g\cdot a$.
For $a\in\mathcal A$, the orbit is $[a]=\{g\cdot a:g\in G\}$ and the stabilizer (isotropy subgroup) is
$G_a=\{g\in G:g\cdot a=a\}$.
The orbit space (quotient set) is $\mathcal A/G=\{[a]:a\in\mathcal A\}$ with the quotient topology, and
the canonical projection is denoted $\pi:\mathcal A\to\mathcal A/G$, $\pi(a)=[a]$.

\paragraph{Regular neighborhoods and smooth quotients.}
A smooth action is called \emph{free} if $G_a=\{e\}$ for all $a\in\mathcal A$.
It is called \emph{proper} if the map $G\times\mathcal A\to\mathcal A\times\mathcal A$,
$(g,a)\mapsto(a,g\cdot a)$ is proper.
A sufficient condition for properness is that $G$ is compact (e.g., an orthogonal group) or finite
(e.g., a permutation group).

If the action of $G$ on $\mathcal A$ is smooth, free, and proper on an open set $\mathcal U\subseteq\mathcal A$,
then the orbit space
\[
\mathcal B \coloneqq \mathcal U/G
\]
admits a unique smooth manifold structure such that the projection
$\pi:\mathcal U\to\mathcal B$ is a smooth submersion.
In this case, $\pi:\mathcal U\to\mathcal B$ is a principal $G$-bundle.
If the action is proper but not free, one may restrict attention to a regular stratum on which the orbit type is constant;
on such a neighborhood the quotient is again a smooth manifold.
This is the regime implicitly used in the main text.

\paragraph{Local quotient charts via invariant descriptors.}
Rather than working directly with the abstract quotient manifold $\mathcal B$,
we represent a local neighborhood of $\mathcal B$ using an orbit-invariant \emph{descriptor map}.
The following assumption makes this precise.

\begin{assmp}[Local quotient chart via an invariant descriptor]
\label{ass:local-quotient-chart}
Let $\mathcal U\subseteq\mathcal A$ be an open set on which the action of $G$ is smooth, free, and proper,
and let $\mathcal B=\mathcal U/G$ with projection $\pi:\mathcal U\to\mathcal B$.
There exists a map $D:\mathcal U\to\mathbb R^q$ such that:
\begin{enumerate}[label=(\roman*)]
\item \textbf{(Orbit-constancy).}
For all $a\in\mathcal U$ and $g\in G$,
\[
D(g\cdot a)=D(a).
\]
\item \textbf{(Local chart for the quotient).}
There exist open neighborhoods $V\subseteq\mathcal B$ and $W\subseteq\mathbb R^q$,
and a $C^k$ diffeomorphism $\bar D:V\to W$,
such that on $\pi^{-1}(V)$ we have
\[
D=\bar D\circ \pi .
\]
\end{enumerate}
\end{assmp}

\paragraph{Consequences.}
Under Assumption~\ref{ass:local-quotient-chart}, the following hold.
\begin{enumerate}[label=(\roman*)]
\item \textbf{Local orbit separation.}
For $a,a'\in\pi^{-1}(V)$,
\[
D(a)=D(a') \quad\Longleftrightarrow\quad \pi(a)=\pi(a') \quad\Longleftrightarrow\quad a'\in[a].
\]
\item \textbf{Descriptor manifold.}
We identify the local quotient neighborhood $V\subseteq\mathcal B$ with its descriptor coordinates
$W\subseteq\mathbb R^q$ via $\bar D$, and write
\[
\mathcal M \coloneqq W .
\]
Thus $\mathcal M$ is a smooth manifold (indeed, an open subset of $\mathbb R^q$) equipped with the induced Euclidean metric.
\item \textbf{Existence of smooth lifts.}
Since $\pi:\mathcal U\to\mathcal B$ is a principal bundle, there exists a smooth local section
$\sigma:V\to\mathcal U$.
Defining
\[
s \coloneqq \sigma\circ \bar D^{-1} : \mathcal M \to \mathcal U,
\]
we obtain a $C^k$ \emph{lift} satisfying $D(s(M))=M$ for all $M\in\mathcal M$.
\item \textbf{Well-defined induced objectives.}
Any $G$-invariant function on $\mathcal U$ induces a well-defined function on $\mathcal M$
by evaluation at any representative in $D^{-1}(M)\cap\pi^{-1}(V)$.
\end{enumerate}

\subsection{Vector-bundle viewpoint: quotient-level features and the coordinate feature map $\phi(x,M)$}
\label{app:riem_bundle}

This subsection formalizes how equivariant representative-level features induce intrinsic
quotient-level features, and how coordinate feature maps arise from choosing local lifts.

\subsubsection{Setup: principal bundle and equivariant features}

Let $\mathcal U\subseteq\R^{q_0}$ be an open set on which a Lie group $G$ acts smoothly, freely, and properly,
and let $\mathcal B=\mathcal U/G$ with projection $\pi:\mathcal U\to\mathcal B$.
Fix a feature dimension $p\in\N_+$ and let $\rho:G\to O(p)$ be a smooth group homomorphism.
Let $\psi:\mathcal X\times\mathcal U\to\R^p$ be measurable in $x$ and $C^k$ in its second argument.

We assume the orthogonal equivariance condition
\[
\psi(x,g\cdot A)=\rho(g)\,\psi(x,A),
\qquad
x\in\mathcal X,\ A\in\mathcal U,\ g\in G.
\]

\subsubsection{Associated vector bundle and intrinsic feature section}

Define an equivalence relation on $\mathcal U\times\R^p$ by
\[
(A,v)\sim(g\cdot A,\rho(g)v),\qquad g\in G.
\]
The associated rank-$p$ vector bundle over $\mathcal B$ is
\[
\mathcal E \coloneqq (\mathcal U\times\R^p)/\sim,
\]
with projection $\pi_{\mathcal E}([A,v])=[A]$.
The Euclidean inner product on $\R^p$ descends to a well-defined fiberwise inner product on $\mathcal E$.

For each $x\in\mathcal X$, define the intrinsic feature section
\[
\Phi_x:\mathcal B\to\mathcal E,
\qquad
\Phi_x([A])\coloneqq [A,\psi(x,A)].
\]

\begin{proposition}[Well-definedness and smoothness]
\label{prop:bundle-feature}
Under the equivariance assumption above, $\Phi_x$ is well-defined.
If $A\mapsto\psi(x,A)$ is $C^k$ on $\mathcal U$, then $\Phi_x$ is a $C^k$ section of $\mathcal E$.
\end{proposition}

\begin{proof}
If $A'=g\cdot A$, then by equivariance,
\[
[A',\psi(x,A')] = [g\cdot A,\rho(g)\psi(x,A)] = [A,\psi(x,A)],
\]
so $\Phi_x$ depends only on the orbit.
Smoothness follows by expressing $\Phi_x$ in local trivializations induced by smooth local sections
of the principal bundle $\pi:\mathcal U\to\mathcal B$.
\end{proof}

\subsubsection{Descriptor coordinates and the coordinate feature map}

Let $\mathcal M\subseteq\R^q$ be the descriptor manifold provided by
Assumption~\ref{ass:local-quotient-chart}, and let
$s:\mathcal M\to\mathcal U$ be the associated $C^k$ lift, i.e., $D(s(\Omega))=\Omega$ for all $\Omega$ in a local neighborhood.

\paragraph{Coordinate feature map.}
Define
\[
\phi:\mathcal X\times\mathcal M\to\R^p,
\qquad
\phi(x,M)\coloneqq \psi(x,s(M)).
\]

\begin{lemma}[Differentiability of $\phi$]
If $A\mapsto\psi(x,A)$ is $C^k$ and $s$ is $C^k$, then for each fixed $x\in\mathcal X$
the map $M\mapsto\phi(x,M)$ is $C^k$ on $\mathcal M$.
\end{lemma}

\begin{proof}
Fix $x\in\mathcal X$ and define $\psi_x:\mathcal U\to\R^p$ by $\psi_x(A)\coloneqq \psi(x,A)$.
By assumption, $\psi_x$ is a $C^k$ map on $\mathcal U$, and $s$ is $C^k$ on $\mathcal M$.
Hence $\phi_x:\mathcal M\to\R^p$ given by
\[
\phi_x(M)\coloneqq \phi(x,M)=\psi_x\bigl(s(M)\bigr)
\]
is the composition $\phi_x=\psi_x\circ s$ of two $C^k$ maps between smooth manifolds, and is therefore $C^k$.
\end{proof}

\paragraph{Gauge transformations.}
If $s'$ is another $C^k$ lift on $\mathcal M$, then for each $M\in\mathcal M$ there exists
a unique $g(M)\in G$ such that $s'(M)=g(M)\cdot s(M)$.
The map $M\mapsto g(M)$ is $C^k$.

\begin{lemma}[Gauge transformation rule]
If $\phi$ and $\phi'$ are induced by lifts $s$ and $s'$ respectively, then
\[
\phi'(x,M)=\rho(g(M))\,\phi(x,M),
\qquad
x\in\mathcal X,\ M\in\mathcal M.
\]
\end{lemma}

\begin{proof}
By equivariance,
\[
\phi'(x,M)=\psi(x,s'(M))=\psi(x,g(M)\cdot s(M))=\rho(g(M))\,\phi(x,M).
\]
\end{proof}

\paragraph{Intrinsic meaning.}
The intrinsic object is the bundle section $\Phi_x$.
Choosing a lift $s$ identifies each fiber with $\R^p$ and yields the coordinate representation $\phi(x,M)$.
Different lifts correspond to orthogonal changes of coordinates.

\subsubsection{Orbit-invariance of minimum-norm OLS}

The proof of \Cref{lem:orbit_invariant_min_norm} in the main text follows directly
from orthogonal equivariance and is given below for completeness.

\begin{proof}[Proof of \Cref{lem:orbit_invariant_min_norm}]
Fix a downstream dataset $D_{\mathrm{down}}^{(n)}=\{(x_i,y_i)\}_{i=1}^n$ and write
$Y\coloneqq (y_1,\dots,y_n)^\top\in\R^n$.
For any $w\in\R^{q_0}$, define the design matrix $\Psi_w\in\R^{n\times p}$ by
\[
(\Psi_w)_{i,:}\coloneqq \psi(x_i,w)^\top .
\]
The minimum Euclidean norm solution of the OLS problem is given by
\[
\hat\theta_w=\Psi_w^{+}Y,
\]
where $(\cdot)^+$ denotes the Moore--Penrose pseudoinverse.

By the orthogonal equivariance condition \eqref{eq:psi_orth_equiv}, for any $g\in G$ and each $i$,
\[
\psi(x_i,g\cdot w)^\top
=
(\rho(g)\psi(x_i,w))^\top
=
\psi(x_i,w)^\top \rho(g)^\top .
\]
Therefore,
\[
\Psi_{g\cdot w}=\Psi_w\,\rho(g)^\top .
\]

For any matrix $M\in\R^{n\times p}$ and any orthogonal matrix $Q\in O(p)$,
\[
(MQ)^+=Q^\top M^+ .
\]
Applying this identity with $M=\Psi_w$ and $Q=\rho(g)^\top$ yields
\[
\hat\theta_{g\cdot w}
=
\Psi_{g\cdot w}^+Y
=
(\Psi_w\rho(g)^\top)^+Y
=
\rho(g)\Psi_w^+Y
=
\rho(g)\hat\theta_w .
\]

For any $x\in\mathcal X$,
\begin{align*}
\hat f_{g\cdot w}(x)
&=
\langle \hat\theta_{g\cdot w},\psi(x,g\cdot w)\rangle \\
&=
\langle \rho(g)\hat\theta_w,\rho(g)\psi(x,w)\rangle \\
&=
\langle \hat\theta_w,\psi(x,w)\rangle \\
&=
\hat f_w(x),
\end{align*}
where we used the orthogonality of $\rho(g)$ in the third equality.
This shows that the minimum-norm downstream predictor depends on $w$
only through its orbit $[w]$.
\end{proof}

\section{Proof of Proposition~\ref{prop:exact-decomp}}
\label{app:decomp}

Recall the downstream regression model in Equation~\eqref{eq:X_and_Y_pair}
\[
Y=f_\star(X)+\varepsilon,
\qquad
X\sim \mu_{\mathrm{down}},
\qquad
\E[\varepsilon\mid X]=0,
\qquad
\sigma^2\coloneqq\E[\varepsilon^2\mid X]<\infty.
\]
Let $\{(x_i,y_i)\}_{i=1}^n$ be i.i.d.\ copies of $(X,Y)$ and write
$D_{\mathrm{down}}^{(n)}=\{(x_i,y_i)\}_{i=1}^n$ for the labeled sample and
$X_{1:n}=(x_1,\dots,x_n)$ for the downstream design.
Let $(X_{\mathrm{new}},Y_{\mathrm{new}})$ be an independent copy of $(X,Y)$.
Define $\varepsilon_i\coloneqq y_i-f_\star(x_i)$ for $i\in[n]$ and
$\varepsilon_{\mathrm{new}}\coloneqq Y_{\mathrm{new}}-f_\star(X_{\mathrm{new}})$.

\paragraph{Manifold-valued feature parameters and quenched conditioning.}
In the main text, the feature parameter is learned in pre-training: $\Omega=\hat \Omega_m(D_{\mathrm{pre}}^{(m)})$, and $\Omega$ may take values on
a Riemannian manifold $\mathcal M$. All identities below are deterministic once $(D_{\mathrm{pre}}^{(m)},X_{1:n})$ is fixed, because conditioning on $D_{\mathrm{pre}}^{(m)}$
freezes $\Omega$, and conditioning on $X_{1:n}$ freezes the empirical projection operator $\Pi_{\Omega,n}$.
This viewpoint isolates the downstream label noise and the fresh test pair randomness, without averaging over pre-training and downstream design.

\subsection{Empirical projection notation}
\label{app:decomp-proj-meaning}

Recall the empirical inner product
\[
\langle g,h\rangle_n=\frac1n\sum_{i=1}^n g(x_i)h(x_i).
\]
Let $\mathcal H_\Omega=\{T_\Omega\theta:\theta\in\R^p\}$ denote the induced linear class and $\Pi_{\Omega,n}$ be the Moore--Penrose empirical least-squares map defined in Appendix~\ref{app:prelim} (Definition~\ref{def:canonical-emp-proj}).
We will invoke the following properties (Lemma~\ref{lem:emp-proj-char} and
Lemma~\ref{lem:reproducing}): for any $g$ with finite evaluations on $\{x_i\}_{i=1}^n$,
\begin{enumerate}
\item $\Pi_{\Omega,n}g\in\mathcal H_\Omega$ and $\langle g-\Pi_{\Omega,n}g,\,h\rangle_n=0$ for all $h\in\mathcal H_\Omega$;
\item $\mathrm{Ev}_n(\Pi_{\Omega,n}h)=\mathrm{Ev}_n(h)$ for all $h\in\mathcal H_\Omega$
(equivalently, $\|h-\Pi_{\Omega,n}h\|_n=0$).
\end{enumerate}
When $\langle\cdot,\cdot\rangle_n$ is non-degenerate on $\mathcal H_\Omega$ (equivalently, $\mathrm{Ev}_n$ is injective on
$\mathcal H_\Omega$), these properties imply $\Pi_{\Omega,n}h=h$ for all $h\in\mathcal H_\Omega$, i.e.\ $\Pi_{\Omega,n}$ is the unique
empirical orthogonal projector onto $\mathcal H_\Omega$.

\paragraph{OLS as empirical projection.}
Let $\hat\theta_{\Omega,n}$ be the minimum-norm OLS solution and set $\hat f_{\Omega,n}=T_\Omega\hat\theta_{\Omega,n}$.
Write the sample-value vectors
\[
y_{1:n}\coloneqq (y_1,\dots,y_n)^\top,\qquad \varepsilon_{1:n}\coloneqq (\varepsilon_1,\dots,\varepsilon_n)^\top,
\qquad f_{\star,1:n}\coloneqq \bigl(f_\star(x_1),\dots,f_\star(x_n)\bigr)^\top,
\]
so that $y_{1:n}=f_{\star,1:n}+\varepsilon_{1:n}$.
By Lemma~\ref{lem:ols-equals-proj},
\[
\hat f_{\Omega,n}=\Pi_{\Omega,n}\,\mathrm{lift}_n(y_{1:n}).
\]
By linearity of $\Pi_{\Omega,n}$ and $y_{1:n}=f_{\star,1:n}+\varepsilon_{1:n}$,
\begin{equation}
\label{eq:ols-split}
\hat f_{\Omega,n}
=
\Pi_{\Omega,n} f_\star
+
\Pi_{\Omega,n}\,\mathrm{lift}_n(\varepsilon_{1:n}).
\end{equation}
For brevity, we will write $\Pi_{\Omega,n}\varepsilon$ instead of $\Pi_{\Omega,n}\,\mathrm{lift}_n(\varepsilon_{1:n})$, with the
understanding that this means applying $\Pi_{\Omega,n}$ to any measurable representative whose evaluations
on $\{x_i\}_{i=1}^n$ equal $\varepsilon_{1:n}$.

\subsection{Population decomposition and orthogonality}
\label{app:decomp-pop}

Recall the population projector $\Pi_\Omega=T_\Omega \Sigma(\Omega)^+ T_{\Omega}^{\mathrm{adj}}$ onto $\mathcal H_\Omega$ in $L^2(\mu_{\mathrm{down}})$ and define
\[
e_\Omega\coloneqq (I-\Pi_\Omega)f_\star,\qquad \mathrm{Rep}(\Omega)\coloneqq \|e_\Omega\|_{L^2(\mu_{\mathrm{down}})}^2.
\]
Since $\Pi_\Omega$ is the $L^2(\mu_{\mathrm{down}})$-orthogonal projector onto $\mathcal H_\Omega$, we have
\[
e_\Omega \perp \mathcal H_\Omega \quad \text{in } L^2(\mu_{\mathrm{down}}),
\qquad\text{i.e.}\qquad
\langle e_\Omega, h\rangle_{L^2(\mu_{\mathrm{down}})}=0\ \ \forall h\in\mathcal H_\Omega.
\]
We will use the decomposition
\begin{equation}
\label{eq:pop-decomp}
f_\star=\Pi_\Omega f_\star+e_\Omega.
\end{equation}

\subsection{Exact risk decomposition conditional on $(D_{\mathrm{pre}}^{(m)},X_{1:n})$}
\label{app:decomp-risk-general}

We now restate and prove \Cref{prop:exact-decomp}.
\exactdecomp*

\begin{proof}
Start from the identity $Y_{\mathrm{new}}=f_\star(X_{\mathrm{new}})+\varepsilon_{\mathrm{new}}$ and write
\[
Y_{\mathrm{new}}-\hat f_{\Omega,n}(X_{\mathrm{new}})
=
\varepsilon_{\mathrm{new}}
+
\bigl(f_\star-\hat f_{\Omega,n}\bigr)(X_{\mathrm{new}}).
\]
Using \eqref{eq:ols-split}, we have
\[
\hat f_{\Omega,n}
=
\Pi_{\Omega,n}f_\star+\Pi_{\Omega,n}\varepsilon.
\]
Substituting and then adding and subtracting $\Pi_\Omega f_\star$ yields
\begin{align}
Y_{\mathrm{new}}-\hat f_{\Omega,n}(X_{\mathrm{new}})
&=
\varepsilon_{\mathrm{new}}
+
\bigl(f_\star-\Pi_{\Omega,n}f_\star-\Pi_{\Omega,n}\varepsilon\bigr)(X_{\mathrm{new}})
\notag\\
&=
\varepsilon_{\mathrm{new}}
+
\bigl((f_\star-\Pi_\Omega f_\star)-(\Pi_{\Omega,n}f_\star-\Pi_\Omega f_\star)-\Pi_{\Omega,n}\varepsilon\bigr)(X_{\mathrm{new}})
\notag\\
&=
\varepsilon_{\mathrm{new}}
+
\bigl(e_\Omega-(\Pi_{\Omega,n}f_\star-\Pi_\Omega f_\star)-\Pi_{\Omega,n}\varepsilon\bigr)(X_{\mathrm{new}}),
\label{eq:test-error-expand}
\end{align}
where in the last step we used \eqref{eq:pop-decomp}.

Square \eqref{eq:test-error-expand} and take conditional expectation given $(D_{\mathrm{pre}}^{(m)},X_{1:n})$:
\begin{align}
&\E\!\left[\bigl(Y_{\mathrm{new}}-\hat f_{\Omega,n}(X_{\mathrm{new}})\bigr)^2 \,\middle|\, D_{\mathrm{pre}}^{(m)},X_{1:n}\right]
\notag\\&=
\E\!\left[\varepsilon_{\mathrm{new}}^2 \,\middle|\, D_{\mathrm{pre}}^{(m)},X_{1:n}\right]
\notag\\
&\quad
+\E\!\left[\Bigl(e_\Omega-(\Pi_{\Omega,n}f_\star-\Pi_\Omega f_\star)-\Pi_{\Omega,n}\varepsilon\Bigr)(X_{\mathrm{new}})^2 \,\middle|\, D_{\mathrm{pre}}^{(m)},X_{1:n}\right]
\notag\\
&\quad
+2\,\E\!\left[\varepsilon_{\mathrm{new}}\Bigl(e_\Omega-(\Pi_{\Omega,n}f_\star-\Pi_\Omega f_\star)-\Pi_{\Omega,n}\varepsilon\Bigr)(X_{\mathrm{new}})
\,\middle|\, D_{\mathrm{pre}}^{(m)},X_{1:n}\right].
\label{eq:three-terms-general}
\end{align}

\paragraph{Cross term with $\varepsilon_{\mathrm{new}}$.}
Condition on $(D_{\mathrm{pre}}^{(m)},X_{1:n},X_{\mathrm{new}},\varepsilon_{1:n})$.
The bracketed term in \eqref{eq:three-terms-general} evaluated at $X_{\mathrm{new}}$ is measurable with respect to
$(D_{\mathrm{pre}}^{(m)},X_{1:n},X_{\mathrm{new}},\varepsilon_{1:n})$, while $\E[\varepsilon_{\mathrm{new}}\mid X_{\mathrm{new}}]=0$.
Therefore, we have
\begin{align*}
&\E\!\left[\varepsilon_{\mathrm{new}}\Bigl(e_\Omega-(\Pi_{\Omega,n}f_\star-\Pi_\Omega f_\star)-\Pi_{\Omega,n}\varepsilon\Bigr)(X_{\mathrm{new}})
\,\middle|\, D_{\mathrm{pre}}^{(m)},X_{1:n}\right]\\
\;&=\E\Bigg[\E\!\left[\varepsilon_{\mathrm{new}}\Bigl(e_\Omega-(\Pi_{\Omega,n}f_\star-\Pi_\Omega f_\star)-\Pi_{\Omega,n}\varepsilon\Bigr)(X_{\mathrm{new}})
\,\middle|\, X_{\mathrm{new}},\varepsilon_{1:n}\right]\bigg|D_{\mathrm{pre}}^{(m)},X_{1:n}\Bigg]\\
\;&=\E\Bigg[\E\!\left[\varepsilon_{\mathrm{new}}
\,\middle|\, X_{\mathrm{new}},\varepsilon_{1:n}\right]\Bigl(e_\Omega-(\Pi_{\Omega,n}f_\star-\Pi_\Omega f_\star)-\Pi_{\Omega,n}\varepsilon\Bigr)(X_{\mathrm{new}})\bigg|D_{\mathrm{pre}}^{(m)},X_{1:n}\Bigg]=0.
\end{align*}

Since $(X_{\mathrm{new}},\varepsilon_{\mathrm{new}})$ is an independent copy of $(X,\varepsilon)$ and is independent of $(D_{\mathrm{pre}}^{(m)},X_{1:n},\varepsilon_{1:n})$,
\[
\E\!\left[\varepsilon_{\mathrm{new}}^2 \,\middle|\, D_{\mathrm{pre}}^{(m)},X_{1:n}\right]=\E[\varepsilon^2]=\sigma^2.
\]

Set
\[
g\coloneqq \Pi_{\Omega,n}f_\star-\Pi_\Omega f_\star\in\mathcal H_\Omega,
\qquad
u\coloneqq \Pi_{\Omega,n}\varepsilon\in\mathcal H_\Omega.
\]
Then pointwise,
\[
(e_\Omega-g-u)^2=e_\Omega^2+g^2+u^2-2e_\Omega g-2e_\Omega u+2gu.
\]
Evaluating at $X_{\mathrm{new}}$ and taking $\E[\cdot\mid D_{\mathrm{pre}}^{(m)},X_{1:n}]$ gives
\begin{align}
\E\!\left[(e_\Omega-g-u)(X_{\mathrm{new}})^2 \,\middle|\, D_{\mathrm{pre}}^{(m)},X_{1:n}\right]
&=
\E\!\left[e_\Omega(X_{\mathrm{new}})^2\,\middle|\, D_{\mathrm{pre}}^{(m)},X_{1:n}\right]
+\E\!\left[g(X_{\mathrm{new}})^2 \,\middle|\, D_{\mathrm{pre}}^{(m)},X_{1:n}\right]
\notag\\&\quad
-2\,\langle e_\Omega,g\rangle_{L^2(\mu_{\mathrm{down}})}
-2\,\langle e_\Omega,u\rangle_{L^2(\mu_{\mathrm{down}})}\notag\\
&\quad
+2\,\E\!\left[g(X_{\mathrm{new}})u(X_{\mathrm{new}})\,\middle|\, D_{\mathrm{pre}}^{(m)},X_{1:n}\right]\notag\\
&\quad +\E\!\left[u(X_{\mathrm{new}})^2 \,\middle|\, D_{\mathrm{pre}}^{(m)},X_{1:n}\right].
\label{eq:square-expand-general}
\end{align}

\paragraph{Orthogonality kills the $e_\Omega$ cross terms.}
Since $g,u\in\mathcal H_\Omega$ and $e_\Omega\perp\mathcal H_\Omega$ in $L^2(\mu_{\mathrm{down}})$, we have
\[
\langle e_\Omega,g\rangle_{L^2(\mu_{\mathrm{down}})}=\langle e_\Omega,u\rangle_{L^2(\mu_{\mathrm{down}})}=0.
\]

\paragraph{The remaining cross term averages to zero.}
Condition on $(D_{\mathrm{pre}}^{(m)},X_{1:n},X_{\mathrm{new}})$.
Given $(D_{\mathrm{pre}}^{(m)},X_{1:n})$, the function $g$ is deterministic (because $\Omega$ is $D_{\mathrm{pre}}^{(m)}$-measurable and $\Pi_{\Omega,n}$ depends only on $(\Omega,X_{1:n})$),
while $u=\Pi_{\Omega,n}\varepsilon$ is linear in the noise vector $\varepsilon_{1:n}$.
Using $\E[\varepsilon_{1:n}\mid X_{1:n}]=0$ and independence of $D_{\mathrm{pre}}^{(m)}$ from the downstream noise, we obtain
\[
\E\!\left[u(X_{\mathrm{new}})\,\middle|\, D_{\mathrm{pre}}^{(m)},X_{1:n},X_{\mathrm{new}}\right]=0,
\]
and hence
\[
\E\!\left[g(X_{\mathrm{new}})u(X_{\mathrm{new}})\,\middle|\, D_{\mathrm{pre}}^{(m)},X_{1:n},X_{\mathrm{new}}\right]=0.
\]
Therefore,
\[
\E\!\left[g(X_{\mathrm{new}})u(X_{\mathrm{new}})\,\middle|\, D_{\mathrm{pre}}^{(m)},X_{1:n}\right]=0.
\]

\paragraph{Conclusion.}
Using $\E[e_\Omega(X_{\mathrm{new}})^2]=\|e_\Omega\|_{L^2(\mu_{\mathrm{down}})}^2=\mathrm{Rep}(\Omega)$ in \eqref{eq:square-expand-general}
and substituting into \eqref{eq:three-terms-general} yields \eqref{eq:exact-risk-general-main}.
\end{proof}

\subsection{Well-posedness specialization}
\label{app:decomp-risk-wp}

\begin{corollary}[Well-posedness implies $\Pi_{\Omega,n}f_\star-\Pi_\Omega f_\star=\Pi_{\Omega,n}f_\Omega$]
\label{cor:exact-decomp-wp}
Assume $\langle\cdot,\cdot\rangle_n$ is non-degenerate on $\mathcal H_\Omega$ (equivalently, $\mathrm{Ev}_n$ is injective on $\mathcal H_\Omega$).
Then $\Pi_{\Omega,n}h=h$ for all $h\in\mathcal H_\Omega$, and therefore
\begin{align*}
\Pi_{\Omega,n}f_\star-\Pi_\Omega f_\star=\Pi_{\Omega,n}e_\Omega.
\end{align*}
Consequently, \eqref{eq:exact-risk-general-main} reduces to
\begin{align*}
\E\!\left[\bigl(Y_{\mathrm{new}}-\hat f_{\Omega,n}(X_{\mathrm{new}})\bigr)^2 \,\middle|\, D_{\mathrm{pre}}^{(m)},X_{1:n}\right]
&=
\E[\varepsilon^2]
+\mathrm{Rep}(\Omega)
+\E\!\left[(\Pi_{\Omega,n}e_\Omega)(X_{\mathrm{new}})^2 \,\middle|\, D_{\mathrm{pre}}^{(m)},X_{1:n}\right]
\\&\quad+\E\!\left[(\Pi_{\Omega,n}\varepsilon)(X_{\mathrm{new}})^2 \,\middle|\, D_{\mathrm{pre}}^{(m)},X_{1:n}\right].  
\end{align*}

\end{corollary}

\begin{proof}
Under the stated condition, $\Pi_{\Omega,n}h=h$ holds for all $h\in\mathcal H_\Omega$.
Since $\Pi_\Omega f_\star\in\mathcal H_\Omega$ and $f_\star=\Pi_\Omega f_\star+e_\Omega$, we have
\[
\Pi_{\Omega,n}f_\star
=
\Pi_{\Omega,n}(\Pi_\Omega f_\star+e_\Omega)
=
\Pi_{\Omega,n}(\Pi_\Omega f_\star)+\Pi_{\Omega,n}e_\Omega
=
\Pi_\Omega f_\star+\Pi_{\Omega,n}e_\Omega,
\]
which implies $\Pi_{\Omega,n}f_\star-\Pi_\Omega f_\star=\Pi_{\Omega,n}e_\Omega$.
Substituting this identity into \eqref{eq:exact-risk-general-main} gives the claimed formula.
\end{proof}

\paragraph{Bridge to the Riemannian log-map CLT.}
In the main text we set $\Omega=\hat \Omega_m(D_{\mathrm{pre}}^{(m)})\in\mathcal M$.
The conditional risk is
\[
R(D_{\mathrm{pre}}^{(m)},X_{1:n})
=
\E\!\left[\bigl(Y_{\mathrm{new}}-\hat f_{\hat \Omega_m,n}(X_{\mathrm{new}})\bigr)^2 \,\middle|\, D_{\mathrm{pre}}^{(m)},X_{1:n}\right],
\]
which is obtained by substituting $\Omega=\hat \Omega_m$ into \eqref{eq:exact-risk-general-main}.
The Riemannian structure enters when analyzing the fluctuations of $\hat \Omega_m$ around $\Omega_\star$ through the log map:
under the assumptions stated in the main text,
\[
\sqrt m\,\log_{\Omega_\star}(\hat \Omega_m)\distconv Z,
\qquad
Z\sim\mathcal N(0,V)\ \text{in }T_{\Omega_\star}\mathcal M.
\]
Since the decomposition \eqref{eq:exact-risk-general-main} holds conditionally on $(D_{\mathrm{pre}}^{(m)},X_{1:n})$ for each $m,n$,
one can combine this log-map CLT with a delta-method argument for the map $\Omega\mapsto R(D_{\mathrm{pre}}^{(m)},X_{1:n})$
in a normal neighborhood of $\Omega_\star$, without taking expectations over $D_{\mathrm{pre}}^{(m)}$.

\section{Proof of Theorem~\ref{thm:master-compatible}}
\label{app:proof-main-thm}
This appendix proves a more general version of~\Cref{thm:master-compatible} by analyzing the three terms in the exact conditional risk decomposition in~\Cref{prop:exact-decomp}. The proof separates the contribution of the pre-training randomness from the two downstream estimation effects. The representation term captures how the random pre-trained feature parameter $\hat\Omega_m(D_\mathrm{pre}^{(m)})$ changes the population approximation error. The variance and leakage terms capture, respectively, the usual downstream noise-fitting effect and the additional finite-sample interaction between the downstream empirical projection and the pre-trained residual.

The statement in the main text corresponds to the regular case \(\mathcal B_0=\{0\}\). In the more general stable-null-span regime treated below, directions that are null at \(\Omega_\star\) may have a limiting activated
span \(\mathcal B_0\), and the downstream OLS degrees-of-freedom term is then
\[
\sigma^2 d_{\mathrm{act}},
\qquad
d_{\mathrm{act}}
\coloneqq
\dim(\mathcal H_{\Omega_\star}\oplus\mathcal B_0).
\]
Specializing to \(\mathcal B_0=\{0\}\) gives \(d_{\mathrm{act}}=d_{\mathrm{eff}}(\Omega_\star)\), recovering \Cref{thm:master-compatible}. We analyze all terms along coupled sequences
\(m,n\to\infty\) with \(m/n\to\alpha\).
Throughout, we work on a coupled probability space supporting three mutually independent objects:
\begin{enumerate}[label=(\roman*)]
    \item  An i.i.d.\ pre-training sequence $(Z^{\mathrm{pre}}_j)_{j\ge 1}$,
    \item An i.i.d.\ downstream sequence $(X_i,\varepsilon_i)_{i\ge 1}$ with $X_i\sim\mu_{\mathrm{down}}$ and $\E[\varepsilon_i\mid X_i]=0$,
    \item  An independent test covariate $X_{\mathrm{new}}\sim\mu_{\mathrm{down}}$ (and an independent noise $\varepsilon_{\mathrm{new}}$)
\end{enumerate}
For each $m$, define the pre-training dataset $D_{\mathrm{pre}}^{(m)}\coloneqq (Z^{\mathrm{pre}}_1,\dots,Z^{\mathrm{pre}}_m)$ and the feature parameter $\Omega_m\coloneqq \hat \Omega_m\!\big(D_{\mathrm{pre}}^{(m)}\big)$.

\paragraph{Conditioning convention.}
All conditional expectations in this appendix are taken \emph{given} $D_{\mathrm{pre}}^{(m)}$
(and, when appropriate, given $X_{1:n}$ as well).
Once we condition on $D_{\mathrm{pre}}^{(m)}$, $\Omega_{m}$ is fixed, and the downstream sample
$(X_i,\varepsilon_i)_{i=1}^{n}$ remains i.i.d.\ and independent of $D_{\mathrm{pre}}^{(m)}$.

We first state the regularity assumptions in~\Cref{app:downstream-terms:setup}. We then analyze the representation term in~\Cref{app:rep}, followed by the downstream variance and leakage terms in~\Cref{app:downstream-terms}.
\subsection{Setup and standing regularity }
\label{app:downstream-terms:setup}

For each $n$, define the population and empirical covariances
\begin{align*}
\Sigma_{m}&\coloneqq \Sigma(\Omega_{m})\coloneqq \E\!\left[\phi(X,\Omega_{m})\phi(X,\Omega_{m})^\top\,\middle|\,D_{\mathrm{pre}}^{(m)}\right],\\
\qquad
\Sigma_{m,n}&\coloneqq \Sigma_{n}(\Omega_{m})\coloneqq \frac1{n}\sum_{i=1}^{n}\phi(X_i,\Omega_{m})\phi(X_i,\Omega_{m})^\top,   
\end{align*}
where the expectation in $\Sigma_{m}$ is over $X\sim\mu_{\mathrm{down}}$.

\begin{assmp}[Well-posedness for identifying the leakage term]
\label{ass:well-posedness}
With probability tending to $1$ as $m\to\infty$ (under the joint law of $(D_{\mathrm{pre}}^{(m)},X_{1:n})$),
the empirical inner product $\langle\cdot,\cdot\rangle_{n}$ is non-degenerate on $\mathcal H_{\Omega_m}$
(equivalently, the evaluation map on $\{X_i\}_{i=1}^{n}$ is injective on $\mathcal H_{\Omega_m}$).
On this event, $\Pi_{\Omega_m,n}h=h$ for all $h\in\mathcal H_{\Omega_m}$, and hence \eqref{eq:app:wellposed-bridge} holds (Appendix~\ref{app:decomp-risk-wp}). 
\end{assmp}

\begin{assmp}[Local-uniform moment and leverage bounds near \(\Omega_\star\)]
\label{ass:moments}
There exist \(\delta>0\), \(\eta>0\), a neighborhood \(\mathcal U\) of
\(\Omega_\star\) in \(\mathcal M\), and constants
\(C_\phi,C_q,C_e<\infty\) such that
\[
\sup_{\Omega\in\mathcal U}
\E\!\left[
\|\phi(X,\Omega)\|^{4+\delta}
\right]\le C_\phi .
\]
For each \(\Omega\in\mathcal U\), define the population leverage score
\begin{align}\label{eq:leverage-score}
q_\Omega(X)
\coloneqq
\phi(X,\Omega)^\top
\Sigma(\Omega)^+
\phi(X,\Omega).    
\end{align}
Assume the local leverage moment bound
\[
\sup_{\Omega\in\mathcal U}
\E\!\left[
q_\Omega(X)^{2+\eta}
\right]\le C_q .
\]
Moreover, for the signal term, assume the leverage-weighted signal bound
\[
\sup_{\Omega\in\mathcal U}
\E\!\left[
e_\Omega(X)^2 q_\Omega(X)^{1+\eta}
\right]\le C_e .
\]
\end{assmp}

\begin{assmp}[Local $C^1$ regularity of the feature map in $\Omega$ with moment control]
\label{ass:phi-C1-M}
Work in normal coordinates on a neighborhood $\mathcal U$ of $\Omega_\star$.
Assume that for $\mu_{\mathrm{down}}$-a.e.\ $x$, the map $\Omega\mapsto \phi(x,\Omega)\in\R^p$ is continuously differentiable on $\mathcal U$.
Moreover, there exist $\delta>0$ and a constant $C_{\partial\phi}<\infty$ such that
\begin{equation}
\label{eq:ass:moment-deriv}
\sup_{\Omega\in\mathcal U}\ \E\!\Big[\|D_\Omega\phi(X,\Omega)\|_{\mathrm{op}}^{4+\delta}\Big]\le C_{\partial\phi},
\end{equation}
where $D_\Omega\phi(x,\Omega):T_\Omega\mathcal M\to\R^p$ denotes the derivative in normal coordinates.
\end{assmp}

The preceding assumptions control the empirical downstream quantities and the local
smoothness of the feature map. We now introduce the geometric notation needed to
analyze the representation term. Since \(f_\star\in\mathcal H_{\Omega_\star}\),
choose \(\theta_\star\in\R^p\) such that
\(f_\star=T_{\Omega_\star}\theta_\star\). The map \(T_{\Omega_\star}\) need not be
injective. Let
\[
N_\star\coloneqq \ker(T_{\Omega_\star}),
\qquad
E_\star\coloneqq N_\star^\perp .
\]
Thus \(\R^p=N_\star\oplus E_\star\): directions in \(E_\star\) are identifiable at
\(\Omega_\star\), whereas directions in \(N_\star\) are invisible at
\(\Omega_\star\).

For a nearby descriptor \(\Omega\), one could similarly define
\(N_\Omega=\ker(T_\Omega)\) and \(E_\Omega=N_\Omega^\perp\). If \(T_\Omega\) has
locally constant rank, these subspaces vary smoothly with
\(\Omega\), and the image spaces
\(\mathcal H_\Omega=\operatorname{Im}(T_\Omega)\) vary smoothly as well. This is the
standard regime in which the full projector map \(\Omega\mapsto\Pi_\Omega\) is
Fréchet differentiable. In this appendix, however, we do not impose local constant
rank. Instead of tracking the moving decomposition \(N_\Omega\oplus E_\Omega\), we fix the splitting at $\Omega_\star$ and use it to describe how the
perturbed image space \(\mathcal H_\Omega\) is generated.

For \(v\in T_{\Omega_\star}\mathcal M\) near \(0\), set
\[
\Omega(v)\coloneqq\exp_{\Omega_\star}(v),
\qquad
T_v\coloneqq T_{\Omega(v)},
\qquad
\mathcal H_v\coloneqq \mathcal H_{\Omega(v)},
\qquad
\Pi_v\coloneqq \Pi_{\Omega(v)}.
\]
The identifiable directions \(E_\star\) generate the subspace $\mathcal A_v
\coloneqq
\operatorname{Im}\left(T_v\big|_{E_\star}\right)$. Since \(T_{\Omega_\star}\big|_{E_\star}\) is injective, this is the stable part of
the perturbed image: under the local smoothness assumptions,
\(\mathcal A_v\) converges to \(\mathcal H_{\Omega_\star}\) as \(v\to0\).

The remaining issue is the contribution of the directions that were null at
\(\Omega_\star\). Because \(\R^p=E_\star\oplus N_\star\), the full perturbed image
\(\mathcal H_v\) is generated by \(T_vE_\star\) together with \(T_vN_\star\). The
subspace generated by \(T_vE_\star\) is \(\mathcal A_v\). We
therefore isolate the orthogonal residual contribution of the perturbed null
directions by defining
\[
\mathcal B_v
\coloneqq
\operatorname{Im}\left(
(I-\Pi_{\mathcal A_v})T_v\big|_{N_\star}
\right),
\qquad v\neq0.
\]
By construction, \(\mathcal B_v\subseteq\mathcal A_v^\perp\), and the image space
decomposes as $\mathcal H_v=\mathcal A_v\oplus\mathcal B_v$. Thus \(\mathcal A_v\) records the part of \(\mathcal H_v\) generated by directions
already identifiable at \(\Omega_\star\), while \(\mathcal B_v\) records the
additional directions produced by perturbing directions that were null
at \(\Omega_\star\). The stable null-span assumption below requires only this
second component to have a limiting span. It is therefore weaker than requiring
smooth variation of the full image space, or equivalently differentiability of the
full projector map \(\Omega\mapsto\Pi_\Omega\).

\begin{assmp}[Stable limiting span of null directions]
\label{ass:stable-null-span}
There exists a finite-dimensional subspace $\mathcal B_0\subseteq \mathcal H_{\Omega_\star}^\perp$ such that
\[
\|\Pi_{\mathcal B_v}-\Pi_{\mathcal B_0}\|_{\mathrm{op}}
\to 0
\qquad\text{as }v\to0,
\]
where $\Pi_{\mathcal B_v}$ and $\Pi_{\mathcal B_0}$ denote the
$L^2(\mu_{\mathrm{down}})$-orthogonal projectors onto
$\mathcal B_v$ and $\mathcal B_0$, respectively.
\end{assmp}

\begin{lemma}[Differentiability of the feature operator]
\label{lem:T-diff}
Under~\Cref{ass:phi-C1-M}, the map $\Omega\mapsto T_\Omega$ is Fr\'echet differentiable at \(\Omega_\star\) as a map from \(\mathcal M\) to
\(\mathcal B(\R^p,L^2)\), the space of bounded linear operators form $\R^p$ to $L^2$. In normal coordinates around
\(\Omega_\star\), its derivative is given by
\[
(DT_{\Omega_\star}[v]\theta)(x)
=
\theta^\top D_\Omega\phi(x,\Omega_\star)[v], \qquad v\in T_{\Omega_\star}\mathcal M.
\]
\end{lemma}

\begin{proof}
Fix \(v\in T_{\Omega_\star}\mathcal M\) sufficiently small and define
\[
\rho_v(x)
\coloneqq
\frac{
\phi(x,\Omega(v))-\phi(x,\Omega_\star)-D_\Omega\phi(x,\Omega_\star)[v]
}{\|v\|}.
\]
By \Cref{ass:phi-C1-M}, for \(\mu_{\mathrm{down}}\)-a.e.\ \(x\),
\[
\rho_v(x)\to0
\qquad
\text{as } v\to0.
\]

We first show that \(\rho_v\to0\) in
\(L^2(\mu_{\mathrm{down}};\R^p)\). By the fundamental theorem of calculus in normal
coordinates,
\[
\rho_v(x)
=
\int_0^1
\left(
D_\Omega\phi(x,\Omega(tv))-D_\Omega\phi(x,\Omega_\star)
\right)
\left[\frac{v}{\|v\|}\right]
\,dt .
\]
Therefore,
\begin{align*}
\|\rho_v(x)\|&=\Bigg\| \int_0^1
\left(
D_\Omega\phi(x,\Omega(tv))-D_\Omega\phi(x,\Omega_\star)
\right)
\left[\frac{v}{\|v\|}\right]
\,dt\Bigg\|\\&\leq \int_0^1
\Bigg\|\left(
D_\Omega\phi(x,\Omega(tv))-D_\Omega\phi(x,\Omega_\star)
\right)
\left[\frac{v}{\|v\|}\right]\Bigg\|
\,dt \\&\leq \int_0^1
\left\|D_\Omega\phi(x,\Omega(tv))-D_\Omega\phi(x,\Omega_\star)\right\|_{\mathrm{op}}
\,dt .
\end{align*}
Let \(q=4+\delta\), where \(\delta>0\) is from \Cref{ass:phi-C1-M}. Jensen's
inequality and the elementary bound
\(\|A-B\|^q\le 2^{q-1}(\|A\|^q+\|B\|^q)\) give
\[
\E\|\rho_v(X)\|^q
\le
2^{q-1}
\int_0^1
\E\|D_\Omega\phi(X,\Omega(tv))\|_{\mathrm{op}}^q\,dt
+
2^{q-1}
\E\|D_\Omega\phi(X,\Omega_\star)\|_{\mathrm{op}}^q .
\]
The right-hand side is uniformly bounded for \(v\) sufficiently small by
\Cref{ass:phi-C1-M}. Hence the family \(\{\|\rho_v(X)\|^2\}\) is uniformly
integrable. Since \(\rho_v(X)\to0\) almost surely, Vitali's theorem yields
\[
\E\|\rho_v(X)\|^2\to0.
\]
Thus,
\[
\left\|
\phi(\cdot,\Omega(v))-\phi(\cdot,\Omega_\star)
-D_\Omega\phi(\cdot,\Omega_\star)[v]
\right\|_{L^2(\mu_{\mathrm{down}};\R^p)}
=
o(\|v\|).
\]

Now let \(\theta\in\R^p\) with \(\|\theta\|_2=1\). Then
\begin{align*}
&\left\|
\left(T_{\Omega(v)}-T_{\Omega_\star}-DT_{\Omega_\star}[v]\right)\theta
\right\|_{L^2}
\\
&\le
\left\|
\phi(\cdot,\Omega(v))-\phi(\cdot,\Omega_\star)
-D_\Omega\phi(\cdot,\Omega_\star)[v]
\right\|_{L^2(\mu_{\mathrm{down}};\R^p)}
=
o(\|v\|).
\end{align*}
Taking the supremum over \(\|\theta\|_2=1\) gives
\[
\|T_{\Omega(v)}-T_{\Omega_\star}-DT_{\Omega_\star}[v]\|_{\mathrm{op}}
=
o(\|v\|).
\]
Thus \(\Omega\mapsto T_\Omega\) is Fr\'echet differentiable at \(\Omega_\star\), with
derivative
\[
(DT_{\Omega_\star}[v]\theta)(x)
=
\theta^\top D_\Omega\phi(x,\Omega_\star)[v].
\]
\end{proof}

\subsection{Pretraining fluctuations}
\label{app:rep}

This subsection analyzes the representation term~\eqref{eq:representation error} in the compatible regime, where
$\mathrm{Rep}(\Omega_\star)=0$. The key step is to linearize the residual
$e_\Omega=(I-\Pi_\Omega)f_\star$ around $\Omega_\star$. Under the stable-null-span
condition (\Cref{ass:stable-null-span}), this residual admits a first-order expansion along perturbations of the
pre-trained descriptor, and the $1/m$-scale limit of $\mathrm{Rep}(\Omega_m)$ follows by applying the pre-training CLT.

This approach is weaker than requiring Fréchet differentiability of the full
projector map $\Omega\mapsto\Pi_\Omega$ as an operator on
$L^2(\mu_{\mathrm{down}})$. We only need first-order control of the single target
$f_\star$, or equivalently of the residual $e_\Omega$. After proving the representation limit under the stable null-span condition (\Cref{ass:stable-null-span}), we record constant-rank
differentiability of $\Pi_\Omega$ as a simpler sufficient condition used in examples.

\subsubsection{Main statements}
\label{app:pretrain-fluct:main}
\begin{proposition}[Representation term: distributional limit in the compatible case]
\label{prop:rep-term-dist}
Assume the pretraining distribution and loss function satisfy~\Cref{ass:riem_mestimation}
\[
Z_m\coloneqq \sqrt m\,\log_{\Omega_\star}(\Omega_m)\distconv Z,
\qquad
Z\sim\mathcal N(0,V),
\]
in $T_{\Omega_\star}\mathcal M$. Moreover, assume Assumptions~\ref{ass:moments}--\ref{ass:stable-null-span}. Let \(\mathcal L\) be the first-order residual map
defined as
\begin{align}\label{eq:linear-L-def}
\mathcal L(v)
\coloneqq
-
(I-\Pi_{\Omega_\star}-\Pi_{\mathcal B_0})DT_{\Omega_\star}[v]\theta_\star.    
\end{align}
Then
\[
m\,\mathrm{Rep}(\Omega_m)\distconv \|\mathcal L(Z)\|_{L^2(\mu_{\mathrm{down}})}^2.
\]
\end{proposition}

\subsubsection{Linearization in normal coordinates}
\label{app:pretrain-fluct:linearization}

\begin{lemma}[Residual expansion under stable limiting span of null directions]
\label{lem:stable-null-span-residual-expansion}
Assume Assumptions \ref{ass:phi-C1-M} and~\ref{ass:stable-null-span}, and recall  the linear map $\mathcal{L}:T_{\Omega_\star}\mathcal M\to L^2$~\eqref{eq:linear-L-def}. Then, in normal coordinates around
$\Omega_\star$,
\[
e_{\Omega(v)}
=
\mathcal L(v)+o(\|v\|)
\]
in $L^2(\mu_{\mathrm{down}})$.

\end{lemma}
\begin{proof}
Since $f_\star = T_{\Omega_\star}\theta_\star$, we have
\[
f_\star = T_v \theta_\star + (T_{\Omega_\star}-T_v)\theta_\star.
\]
We apply $\Pi_v$ to both sides and use $T_v \theta_\star\in \mathcal H_{\Omega(v)}$ to get
\[
\Pi_v f_\star = \Pi_v T_v \theta_\star + \Pi_v (T_{\Omega_\star}-T_v)\theta_\star = T_v \theta_\star + \Pi_v (T_{\Omega_\star}-T_v)\theta_\star.
\]
Therefore
\[
e_{\Omega(v)}=f_\star - \Pi_v f_\star = -(I - \Pi_v)(T_v - T_{\Omega_\star})\theta_\star.
\]
By~\Cref{lem:T-diff}, we have
\[
T_v = T_{\Omega_\star} + DT_{\Omega_\star}[v]+r_T(v)
\]
where $\norm{r_T (v)}_\mathrm{op}=o(\norm{v})$. Since $I-\Pi_{\Omega(v)}$ is an orthogonal projector, and hence it is non-expansive, we have
\[
e_{\Omega(v)} = -(I - \Pi_v)(DT_{\Omega_\star}[v] + r_T(v))\theta_\star = -(I - \Pi_v)(DT_{\Omega_\star}[v])\theta_\star+o(\norm{v}).
\]

We now identify the limit of $\Pi_{\Omega(v)}$. Since $\R^p = E_\star\oplus N_\star$, every element of $\mathcal{H}_{\Omega(v)}=\operatorname{Im}(T_v)$ can be written as 
\[
T_v\theta_E+T_v\theta_N
\]
where $\theta_E\in E_\star$ and $\theta_N\in N_\star$. The first term belongs to $\mathcal A_v = \operatorname{Im}\left(T_v\big|_{E_\star}\right)$. For the second term,
\[
T_v\theta_N = \Pi_{\mathcal{A}_v}T_v \theta_N + (I-\Pi_{\mathcal A_v})T_v\theta_N
\]
where the first summand lies in $\mathcal{A}_v$ and the second lies in $\mathcal B_v=
\operatorname{Im}((I-\Pi_{\mathcal A_v})T_v\big|_{N_\star})$. Hence, $\mathcal{H}_v = \mathcal{A}_v + \mathcal{B}_v$. Moreover, by construction, $\mathcal{B}_v\subseteq\mathcal{A}_v^\perp$, so we have
\[
\mathcal{H}_v = \mathcal{A}_v \oplus \mathcal{B}_v.
\]
Consequently, we have
\[
\Pi_v = \Pi_{\mathcal{A}_v} + \Pi_{\mathcal{B}_v}.
\]

It remains to identify the limiting behavior of the two projections. We first
consider \(\mathcal A_v\). The restriction
\(T_{\Omega_\star}\big|_{E_\star}\) is injective, because
\(E_\star=N_\star^\perp\) and \(N_\star=\ker(T_{\Omega_\star})\). Since \(E_\star\)
is finite-dimensional, this injectivity is stable under small operator-norm
perturbations. In particular, by \Cref{lem:T-diff},
\[
T_v\big|_{E_\star}\to T_{\Omega_\star}\big|_{E_\star}
\]
in operator norm, and hence \(T_v\big|_{E_\star}\) remains injective for all sufficiently small \(v\). Thus \(\mathcal A_v\) has the same dimension as
\(\mathcal H_\star\) for small \(v\). Consequently, \(\mathcal A_v\) converge to \(\mathcal H_\star\) in the Grassmannian topology, or
equivalently,
\[
\|\Pi_{\mathcal A_v}-\Pi_{\Omega_\star}\|_{\mathrm{op}}\to0.
\]
The second projection is controlled directly by \Cref{ass:stable-null-span}, which
gives
\[
\|\Pi_{\mathcal B_v}-\Pi_{\mathcal B_0}\|_{\mathrm{op}}\to0.
\]
Since \(\mathcal H_v=\mathcal A_v\oplus\mathcal B_v\), by orthogonality, we have
\[
\Pi_{\Omega(v)}=\Pi_{\mathcal A_v}+\Pi_{\mathcal B_v}.
\]
Combining the two projection limits yields
\[
\Pi_{\Omega(v)}
\to
\Pi_{\Omega_\star}+\Pi_{\mathcal B_0}
\]
in operator norm.

For small enough $v$,~\Cref{ass:phi-C1-M} yields 
\begin{align*}
    \norm{DT_{\Omega_\star}[v]\theta_\star}_{L^2}^2 = \E_{X}\big(\theta_\star^\top D_\Omega\phi(X,\Omega_\star)[v] \big)^2 \leq \norm{D_{\Omega_\star}\phi}_\mathrm{op}^2\, \norm{\theta_\star}_2^2\,\norm{v}_2^2 = O(\norm{v}^2_2).
\end{align*}
Therefore,
\begin{align*}
(I-\Pi_{\Omega(v)})DT_{\Omega_\star}[v]\theta_\star &= (I-\Pi_{\Omega_\star}-\Pi_{\mathcal{B}_0})DT_{\Omega_\star}[v]\theta_\star + o(\norm{v}).    
\end{align*}
Combining this with the earlier expansion for $e_{\Omega(v)}$ yields
\begin{align*}
e_{\Omega(v)}
=
-
(I-\Pi_{\Omega_\star}-\Pi_{\mathcal B_0})
DT_{\Omega_\star}[v]\theta_\star
+
o(\|v\|)
=
\mathcal L(v)+o(\|v\|).    
\end{align*}

\end{proof}

\begin{lemma}[Well-definedness of \(\mathcal L\)]
\label{lem:L-well-defined}
Assume Assumptions \ref{ass:moments}, \ref{ass:phi-C1-M}, and \ref{ass:stable-null-span}. The map
\[
\mathcal L(v)
\coloneqq
-
(I-\Pi_{\Omega_\star}-\Pi_{\mathcal B_0})DT_{\Omega_\star}[v]\theta_\star
\]
does not depend on the choice of \(\theta_\star\in\R^p\) satisfying
\(f_\star=T_{\Omega_\star}\theta_\star\).
\end{lemma}

\begin{proof}
Let \(\theta_\star'\in\R^p\) be another coefficient vector such that
\(f_\star=T_{\Omega_\star}\theta_\star'\). Then
\[
\eta\coloneqq \theta_\star'-\theta_\star\in \ker(T_{\Omega_\star})=N_\star.
\]
It suffices to show that, for every \(\eta\in N_\star\),
\[
(I-\Pi_{\Omega_\star}-\Pi_{\mathcal B_0})DT_{\Omega_\star}[v]\eta=0 .
\]

Fix \(\eta\in N_\star\). Since \(T_{\Omega_\star}\eta=0\), \Cref{lem:T-diff} gives
\[
T_v \eta
=T_{\Omega_\star}\eta+DT_{\Omega_\star}[v]\eta+r_T(v)\eta
=DT_{\Omega_\star}[v]\eta+r_T(v)\eta,
\qquad
\|r_T(v)\eta\|_{L^2}=o(\|v\|_2).
\]
On the other hand, by the orthogonal decomposition
\(\mathcal H_v=\mathcal A_v\oplus\mathcal B_v\), the projector onto
\(\mathcal H_v\) is
\[
\Pi_v=\Pi_{\mathcal A_v}+\Pi_{\mathcal B_v}.
\]
Since \(T_v \eta\in\mathcal H_v\), we have
\[
(I-\Pi_{\mathcal A_v}-\Pi_{\mathcal B_v})T_v \eta=0.
\]
Substituting the expansion of \(T_v\eta\) yields
\[
(I-\Pi_{\mathcal A_v}-\Pi_{\mathcal B_v})DT_{\Omega_\star}[v]\eta
=
-(I-\Pi_{\mathcal A_v}-\Pi_{\mathcal B_v})r_T(v)\eta
=
o(\|v\|_2),
\]
because \(I-\Pi_{\mathcal A_v}-\Pi_{\mathcal B_v}\) is an orthogonal projector and
therefore has operator norm at most one.

By the convergence
\[
\Pi_{\mathcal A_v}\to\Pi_{\Omega_\star},
\qquad
\Pi_{\mathcal B_v}\to\Pi_{\mathcal B_0}
\]
in operator norm, and since
\[
\|DT_{\Omega_\star}[v]\eta\|_{L^2}=O(\|v\|_2),
\]
we also have
\[
\begin{aligned}
&(I-\Pi_{\Omega_\star}-\Pi_{\mathcal B_0})DT_{\Omega_\star}[v]\eta \\
&\quad =
(I-\Pi_{\mathcal A_v}-\Pi_{\mathcal B_v})DT_{\Omega_\star}[v]\eta
+
(\Pi_{\mathcal A_v}-\Pi_{\Omega_\star}
+\Pi_{\mathcal B_v}-\Pi_{\mathcal B_0})
DT_{\Omega_\star}[v]\eta
=
o(\|v\|_2).
\end{aligned}
\]
The left-hand side is linear in \(v\). A linear map that is \(o(\|v\|_2)\) at the
origin must be identically zero. Therefore
\[
(I-\Pi_{\Omega_\star}-\Pi_{\mathcal B_0})DT_{\Omega_\star}[v]\eta=0
\]
for every \(v\in T_{\Omega_\star}\mathcal M\) and every \(\eta\in N_\star\). Thus replacing \(\theta_\star\) by \(\theta_\star+\eta\) does not change
\(\mathcal L(v)\), and \(\mathcal L\) is well-defined.
\end{proof}

\begin{corollary}[Deterministic first-order remainder bound]  \label{lem:rep-linearization}
Assume Assumptions~\ref{ass:phi-C1-M} and~\ref{ass:stable-null-span}. Then there exists a function
$\omega:[0,\infty)\to[0,\infty)$ with $\omega(t)\to 0$ as $t\to 0$ such that for all
$v\in T_{\Omega_\star}\mathcal M$ in a neighborhood of $0$,
\[
\bigl\|e_{\exp_{\Omega_\star}(v)}-\mathcal L(v)\bigr\|_{L^2(\mu_{\mathrm{down}})}
\le
\omega(\|v\|)\,\|v\|.
\]  
\end{corollary}

\begin{proof}
By~\Cref{lem:stable-null-span-residual-expansion}, $e_{\exp_{\Omega_\star}(v)}=\mathcal{L}(v) + r(v)$ where $r(v)=o(\norm{v})$. Define
\[
\omega(t)\coloneqq
\sup_{0<\|v\|\le t}
\frac{\|r(v)\|_{L^2(\mu_{\mathrm{down}})}}{\|v\|}.
\]
Then $\omega(t)\to0$ as $t\to0$, and the desired bound follows.
\end{proof}

\subsubsection{Distributional limit of the representation term}
\label{app:pretrain-fluct:dist}

By~\Cref{lem:rep-linearization},
\[
e_{\Omega_m}
=
e_{\exp_{\Omega_\star}(v_m)}
=
\mathcal L(v_m)+r(v_m),
\qquad
\|r(v_m)\|_{L^2(\mu_{\mathrm{down}})}\le \omega(\|v_m\|)\,\|v_m\|.
\]
Multiplying by $\sqrt m$ yields
\[
\sqrt m\,e_{\Omega_m}
=
\mathcal L(Z_m)+\sqrt m\,r(v_m),
\qquad
\|\sqrt m\,r(v_m)\|_{L^2(\mu_{\mathrm{down}})}\le \omega(\|v_m\|)\,\|Z_m\|.
\]

\begin{lemma}[The linearization remainder is negligible in probability]
\label{lem:rep-rem-op}
Assume Assumptions~\ref{ass:riem_mestimation},~\ref{ass:phi-C1-M}, and~\ref{ass:stable-null-span}. Then
\[
\|\sqrt m\,r(v_m)\|_{L^2(\mu_{\mathrm{down}})}\to 0
\qquad\text{in probability.}
\]
\end{lemma}

\begin{proof}
Since $\Omega_m\to \Omega_\star$ in probability and $v_m=\log_{\Omega_\star}(\Omega_m)$ on a normal neighborhood, we have $\|v_m\|\to 0$
in probability. Moreover, $\|Z_m\|$ is tight by the CLT. Fix $\varepsilon>0$ and $\delta>0$. By tightness of $\|Z_m\|$, choose
$0<M<\infty$ such that
\[
\limsup_{m\to\infty}\Pr(\|Z_m\|>M)\le \delta .
\]
Since $\omega(t)\to0$, choose $t_0>0$ such that
\[
\omega(t)\le \frac{\varepsilon}{M}
\qquad\text{for all }0\le t\le t_0 .
\]
On the event
\[
\{\|v_m\|\le t_0\}\cap\{\|Z_m\|\le M\},
\]
we have
\[
\|\sqrt m\,r(v_m)\|_{L^2(\mu_{\mathrm{down}})}
\le
\omega(\|v_m\|)\|Z_m\|
\le
\varepsilon .
\]
Therefore
\[
\Pr\!\left(
\|\sqrt m\,r(v_m)\|_{L^2(\mu_{\mathrm{down}})}>\varepsilon
\right)
\le
\Pr(\|v_m\|>t_0)+\Pr(\|Z_m\|>M).
\]
Taking $\limsup_{m\to\infty}$ gives
\[
\limsup_{m\to\infty}
\Pr\!\left(
\|\sqrt m\,r(v_m)\|_{L^2(\mu_{\mathrm{down}})}>\varepsilon
\right)
\le
\delta .
\]
Since $\delta>0$ is arbitrary, the claim follows.
\end{proof}

\begin{lemma}[Quadratic approximation]
\label{lem:rep-quad-approx}
Assume Assumptions~\ref{ass:riem_mestimation},~\ref{ass:phi-C1-M}, and~\ref{ass:stable-null-span}. Then
\[
m\,\mathrm{Rep}(\Omega_m)
-
\|\mathcal L(Z_m)\|_{L^2(\mu_{\mathrm{down}})}^2
\to 0
\qquad\text{in probability.}
\]
\end{lemma}

\begin{proof}
Expanding the square,
\[
m\,\mathrm{Rep}(\Omega_m)
=
\|\sqrt m\,e_{\Omega_m}\|_{L^2}^2
=
\|\mathcal L(Z_m)\|_{L^2}^2
+
2\langle \mathcal L(Z_m),\sqrt m\,r(v_m)\rangle_{L^2}
+
\|\sqrt m\,r(v_m)\|_{L^2}^2.
\]
Since $\mathcal L$ is bounded,
\[
\|\mathcal L(Z_m)\|_{L^2}
\le
\|\mathcal L\|_{\mathrm{op}}\|Z_m\|.
\]
Since $\|Z_m\|$ is tight, $\|\mathcal L(Z_m)\|_{L^2}$ is tight.
By Lemma~\ref{lem:rep-rem-op}, $\|\sqrt m\,r(v_m)\|_{L^2}\to 0$ in probability, which implies both the cross term and
the squared term vanish in probability.
\end{proof}

\begin{proof}[Proof of Proposition~\ref{prop:rep-term-dist}]
By Lemma~\ref{lem:rep-quad-approx}, it suffices to identify the limit law of $\|\mathcal L(Z_m)\|_{L^2}^2$.
Since $Z_m\distconv Z$ in $T_{\Omega_\star}\mathcal M$ and $z\mapsto \|\mathcal L(z)\|_{L^2}^2$ is continuous, the continuous mapping theorem yields
\[
\|\mathcal L(Z_m)\|_{L^2}^2\distconv \|\mathcal L(Z)\|_{L^2}^2.
\]
Combining with Lemma~\ref{lem:rep-quad-approx} gives the claim.
\end{proof}

\subsubsection{A constant-rank sufficient condition}\label{sec:diff-proj-ass} 
We next connect the residual expansion above to the more familiar regular
constant-rank setting. Under the stable-null-span condition, the special case
\(\mathcal B_0=\{0\}\) rules out the activation of directions that are null at
\(\Omega_\star\). In this case, the downstream feature covariance has locally
constant rank and a positive eigengap; consequently, the Moore--Penrose inverse
is smooth on the corresponding rank stratum, and the full population projector
map \(\Omega\mapsto\Pi_\Omega\) is Fr\'echet differentiable in operator norm.
Thus, the abstract residual map in the expansion above agrees with the usual
projector derivative:
\[
\mathcal L(v)=-D\Pi_{\Omega_\star}[v]f_\star = -\big(I-\Pi_{\Omega_\star}\big)DT_{\Omega_\star}[v]\theta_\star.
\]

\begin{lemma}\label{lem:B0-zero-implies-rank-stability}
Assume Assumptions~\ref{ass:moments},~\ref{ass:phi-C1-M}, and~\ref{ass:stable-null-span}. Suppose that $\mathcal B_0 = \{0\}$. Then, there exists a neighborhood $\mathcal U_0\subset\mathcal U$ of $\Omega_\star$ and a constant $\kappa>0$ such that, for all $\Omega \in \mathcal U_0$,
\[
\mathrm{rank}\big(\Sigma(\Omega)\big) = \mathrm{rank}\big(\Sigma(\Omega_\star)\big) =: r, \qquad \lambda_r(\Sigma(\Omega))\ge \kappa .
\] 
\end{lemma}

\begin{proof}
By \Cref{lem:T-diff},
\[
\|T_v-T_\star\|_{\mathrm{op}}\to0
\qquad
\text{as }v\to0.
\]
First, we show the local constancy of the rank. Since $\mathcal B_0=\{0\}$, we have $\Pi_{\mathcal B_0}=0$. By~\Cref{ass:stable-null-span}, we have
\[
\|\Pi_{\mathcal B_v}\|_{\mathrm{op}}
=
\|\Pi_{\mathcal B_v}-\Pi_{\mathcal B_0}\|_{\mathrm{op}}
\to 0.
\]
But, an orthogonal projector has operator norm either $0$ or $1$. Hence, for all sufficiently small non-zero $v$, we have $\Pi_{\mathcal{B}_v}=0$, and therefore, we have $\mathcal{B}_v =0$. For such $v$,
\[
\mathcal H_v=\mathcal A_v\oplus\mathcal B_v=\mathcal A_v.
\]
Moreover, \(T_\star|_{E_\star}\) is injective, and injectivity is stable under small operator-norm perturbations on a finite-dimensional domain. Hence, for all sufficiently small
\(v\),
\[
T_v|_{E_\star}
\]
remains injective. Therefore
\[
\rank(T_v)=\dim\mathcal A_v=\dim E_\star=\rank(T_{\Omega_\star})
\]
for all sufficiently small \(v\). Since
\[
\rank(\Sigma(\Omega(v)))=\rank(T_v^\ast T_v)=\rank(T_v),
\]
we obtain
\[
\rank(\Sigma(\Omega(v)))=r
\]
for all sufficiently small \(v\).

It remains to prove a uniform positive lower bound on the \(r\)-th positive eigenvalue. Let
\(s_r(T_\star)>0\) denote the smallest nonzero singular value of \(T_\star\). Since
\[
\|T_v-T_\star\|_{\mathrm{op}}\to0,
\]
the singular-value perturbation bound gives, for sufficiently small \(v\), we have $s_r(T_v)\ge \frac12 s_r(T_\star)$. Because
\[
\lambda_r(\Sigma(\Omega(v)))=s_r(T_v)^2,
\]
we get
\[
\lambda_r(\Sigma(\Omega(v)))
\ge
\frac14 s_r(T_\star)^2
\]
for all sufficiently small \(v\). Taking
\[
\kappa\coloneqq \frac14 s_r(T_\star)^2
\]
and shrinking the normal neighborhood if necessary proves the claim.
\end{proof}

\begin{lemma}[Differentiability of the pseudoinverse on the stable-rank region]
\label{lem:pinv-diff}
Let $A\succeq 0$ be symmetric with $\rank(A)=r$ and $\lambda_r(A)\ge \kappa>0$.
Then the restriction of the Moore--Penrose map to the stable-rank region
\[
\mathcal R_{r,\kappa}\coloneqq \{B\succeq 0:\ B=B^\top,\ \rank(B)=r,\ \lambda_r(B)\ge \kappa\}
\]
is Fr\'echet differentiable at $A$ (in operator norm). Its derivative in direction $H$ is
\begin{equation}
\label{eq:Dpinv}
D(A^+)[H]
=
-\,A^+ H A^+
+
A^{+2}H(I-AA^+)
+
(I-A^+A)H A^{+2}.
\end{equation}
In particular, there exists a constant $C_\kappa<\infty$ depending only on $\kappa$ such that
\[
\|D(A^+)[H]\|_{\mathrm{op}}\le C_\kappa\,\|H\|_{\mathrm{op}}.
\]
\end{lemma}

\begin{proof}
This is a standard result in perturbation theory. For the proof see~\cite{golub1973differentiation}.
\end{proof}

\begin{proposition}[Differentiability of \(\Omega\mapsto \Pi_\Omega\) on a constant-rank region]
\label{prop:proj-diff}
Fix \(\Omega_\star\in\mathcal M\). Assume Assumptions~\ref{ass:moments}--\ref{ass:stable-null-span}. Assume moreover that $\mathcal B_0 =0$. Then, the map \(\Omega\mapsto\Pi_\Omega\) is Fr\'echet differentiable at
\(\Omega_\star\) as a map into \(\mathcal B(L^2(\mu_{\mathrm{down}}))\). In
particular, for \(v\in T_{\Omega_\star}\mathcal M\),
\[
\bigl\|\Pi_{\exp_{\Omega_\star}(v)}
-\Pi_{\Omega_\star}
-D\Pi_{\Omega_\star}[v]\bigr\|_{\mathrm{op}}
=
o(\|v\|)
\qquad
\text{as } v\to0 .
\]
\end{proposition}

\begin{proof}
Write \(\Omega(v)\coloneqq\exp_{\Omega_\star}(v)\). By \Cref{lem:T-diff}, the map $\Omega\mapsto T_\Omega$ is Fr\'echet differentiable at \(\Omega_\star\) as a map into
\(\mathcal B(\R^p,L^2(\mu_{\mathrm{down}}))\). Since taking adjoints is a bounded
linear map between operator spaces, it follows that $\Omega\mapsto T_\Omega^\mathrm{adj}$ is Fr\'echet differentiable at \(\Omega_\star\) as a map into
\(\mathcal B(L^2(\mu_{\mathrm{down}}),\R^p)\). Next, using the identity $\Sigma(\Omega)=T_\Omega^\mathrm{adj}T_\Omega$ and the product rule for bounded operators, the map $\Omega\mapsto \Sigma(\Omega)$ is Fr\'echet differentiable at \(\Omega_\star\) as a map into \(\mathbb R^{p\times p}\).

By~\Cref{lem:B0-zero-implies-rank-stability}, the nonzero spectrum of
\(\Sigma(\Omega)\) is bounded away from zero on \(\mathcal U\). Therefore, by
\Cref{lem:pinv-diff}, the map $\Omega\mapsto \Sigma(\Omega)^+$ is Fr\'echet differentiable at \(\Omega_\star\). Finally, for every \(\Omega\in\mathcal U\), the population projection onto
\(\mathcal H_\Omega=\operatorname{Im}(T_\Omega)\) is $\Pi_\Omega
=
T_\Omega\Sigma(\Omega)^+T_\Omega^\mathrm{adj}$. The right-hand side is a composition and product of Fr\'echet differentiable maps between Banach spaces. Hence
\[
\Omega\mapsto
T_\Omega\Sigma(\Omega)^+T_\Omega^\mathrm{adj}
\]
is Fr\'echet differentiable at \(\Omega_\star\) as a map into
\(\mathcal B(L^2(\mu_{\mathrm{down}}))\). Therefore \(\Omega\mapsto\Pi_\Omega\) is
Fr\'echet differentiable at \(\Omega_\star\), which gives
\[
\bigl\|\Pi_{\Omega(v)}
-\Pi_{\Omega_\star}
-D\Pi_{\Omega_\star}[v]\bigr\|_{\mathrm{op}}
=
o(\|v\|).
\]
\end{proof}

\subsection{Downstream estimation terms}
\label{app:downstream-terms}

This subsection studies the two downstream \emph{estimation} terms that appear in the exact conditional risk decomposition
(Proposition~\ref{prop:exact-decomp}) where \emph{both} the pre-training and downstream sample sizes diverge with an asymptotic constant rate $\frac{m}{n}\to\alpha\in(0,\infty)$.

\subsubsection{Estimation terms and the compatible regime}
\label{app:downstream-terms:main}

The conditional risk is
\[
R(D_{\mathrm{pre}}^{(m)},X_{1:n})
\coloneqq 
\E\!\left[
\bigl(Y_{\mathrm{new}}-\hat f_{\Omega_{m},n}(X_{\mathrm{new}})\bigr)^2
\,\middle|\,
D_{\mathrm{pre}}^{(m)},X_{1:n}
\right].
\]
Recall from Proposition~\ref{prop:exact-decomp} that the two downstream estimation terms are
\begin{align}
\label{eq:app:var-term-def}
\mathrm{Var}_{n}
&\coloneqq 
\E\!\left[
\bigl(\Pi_{\Omega_{m},n}\varepsilon\bigr)(X_{\mathrm{new}})^2
\,\middle|\,
D_{\mathrm{pre}}^{(m)},X_{1:n}
\right],\\
\label{eq:app:leak-term-def}
\mathrm{Leakage}_{n}
&\coloneqq 
\E\!\left[
\bigl(\Pi_{\Omega_{m},n}f_\star-\Pi_{\Omega_{m}} f_\star\bigr)(X_{\mathrm{new}})^2
\,\middle|\,
D_{\mathrm{pre}}^{(m)},X_{1:n}
\right].
\end{align}
Here $\Pi_{\Omega_{m}}$ is the population $L^2(\mu_{\mathrm{down}})$ projector onto $\mathcal H_{\Omega_{m}}$ and
$\Pi_{\Omega_{m},n}$ is the canonical empirical projector from Appendix~\ref{app:prelim}
(Definition~\ref{def:canonical-emp-proj}).

Define the population residual $e_{\Omega}\coloneqq (I-\Pi_{\Omega})f_\star$, so that $e_{\Omega}\perp \mathcal H_{\Omega}$ in $L^2(\mu_{\mathrm{down}})$. Under the well-posedness condition in \Cref{ass:well-posedness}, one has
\begin{equation}
\label{eq:app:wellposed-bridge}
\Pi_{\Omega_{m},n}f_\star-\Pi_{\Omega_{m}} f_\star=\Pi_{\Omega_{m},n}e_{\Omega_{m}},
\end{equation}
and hence $\mathrm{Leakage}_n$ reduces to a residual-leakage term.

\paragraph{Compatible limit.}
In the main theorem we work in the \emph{compatible} regime where $f_\star\in\mathcal H_{\Omega_\star}$ for the limit feature parameter
$\Omega_\star$, equivalently
\begin{equation}
\label{eq:app:compatibility}
e_{\Omega_\star}=(I-\Pi_{\Omega_\star})f_\star=0.
\end{equation}
As a consequence of the residual expansion proved in \Cref{app:rep}, we have
\[
\|e_{\Omega_m}\|_{L^2(\mu_{\mathrm{down}})}\xrightarrow[]{\mathbb P}\,0.
\]
This residual vanishing is the ingredient that makes the signal estimation
term \(\mathrm{Leakage}_n\) vanish at the scale \(n\).

\subsubsection{Perturbation lemmas for pseudoinverses}
\label{app:downstream-terms:pinv-perturb}

\begin{lemma}[Empirical span is contained in the population span]
\label{lem:emp-span-contained}
Fix $m$ and condition on $D_{\mathrm{pre}}^{(m)}$.
Assume $\E[\|\phi(X,\Omega_m)\|^2 \mid D_{\mathrm{pre}}^{(m)}] < \infty$, so that $\Sigma_m$ is well-defined.
Let $S\coloneqq \mathrm{Im}(\Sigma_m)$. Then $\phi(X,\Omega_m)\in S$ almost surely.
In particular, $\mathrm{Im}(\Sigma_{m,n})\subseteq S$ and $\rank(\Sigma_{m,n})\le\rank(\Sigma_m)$ almost surely.
\end{lemma}

\begin{proof}
Let $K\coloneqq \ker(\Sigma_m)$ and let $v$ be any unit vector in $K$. Then
\[
0=v^\top \Sigma_m v=\E\!\left[(v^\top \phi(X,\Omega_m))^2\right],
\]
so $v^\top \phi(X,\Omega_m)=0$ almost surely. Hence $\phi(X,\Omega_m)\in K^\perp=\mathrm{Im}(\Sigma_m)$ almost surely.
Applying the same argument to each $X_i$ yields the claims for $\Sigma_{m,n}$.
\end{proof}

\begin{lemma}[Relative empirical covariance from leverage moments]
\label{lem:relative-covariance-lln}
Assume \Cref{ass:moments}. Let $S_m\coloneqq \operatorname{Im}(\Sigma_m)$ and $P_m\coloneqq \Pi_{S_m}$ where \(P_m\) is the Euclidean orthogonal projector onto \(S_m\). Define
\[
C_{m,n}
\coloneqq
\Sigma_m^{+/2}\Sigma_{m,n}\Sigma_m^{+/2}.
\]
If \(\Omega_m\to\Omega_\star\) in probability, then
\[
\bigl\|C_{m,n}-P_m\bigr\|_{\mathrm{op}}
\xrightarrow{\mathbb P}0.
\]
\end{lemma}

\begin{proof}
Fix \(m,n\) and condition on \(D_{\mathrm{pre}}^{(m)}\). Write
\[
\phi_m(X)\coloneqq \phi(X,\Omega_m),
\qquad
\widetilde\phi_m(X)
\coloneqq
\Sigma_m^{+/2}\phi_m(X).
\]
Then
\[
C_{m,n}
=
\frac1n\sum_{i=1}^n
\widetilde\phi_m(X_i)\widetilde\phi_m(X_i)^\top .
\]
Moreover,
\[
\E\!\left[
\widetilde\phi_m(X)\widetilde\phi_m(X)^\top
\,\middle|\,
D_{\mathrm{pre}}^{(m)}
\right]
=
\Sigma_m^{+/2}\Sigma_m\Sigma_m^{+/2}
=
P_m.
\]
Therefore
\[
C_{m,n}-P_m
=
\frac1n\sum_{i=1}^n U_{m,i},
\]
where
\[
U_{m,i}
\coloneqq
\widetilde\phi_m(X_i)\widetilde\phi_m(X_i)^\top-P_m
\]
are conditionally i.i.d. mean-zero matrices.

For each entry \((j,k)\), write \(U_{m,i}^{(jk)}\) for the \((j,k)\)-entry.
On the event \(\{\Omega_m\in\mathcal U\}\), Assumption~\ref{ass:moments} gives
\[
\E\!\left[
\|\widetilde\phi_m(X)\|^{4}
\,\middle|\,
D_{\mathrm{pre}}^{(m)}
\right]
=
\E\!\left[
q_{\Omega_m}(X)^2
\,\middle|\,
D_{\mathrm{pre}}^{(m)}
\right]
\le
C_q^{2/(2+\eta)} .
\]
Hence there exists a finite constant \(C\), depending only on \(C_q\), such that
on \(\{\Omega_m\in\mathcal U\}\),
\[
\E\!\left[
\bigl(U_{m,1}^{(jk)}\bigr)^2
\,\middle|\,
D_{\mathrm{pre}}^{(m)}
\right]
\le C
\]
for all \(j,k\). Therefore, by Chebyshev's inequality,
\[
\Pr\!\left(
\left|
\frac1n\sum_{i=1}^n U_{m,i}^{(jk)}
\right|
>
\frac{\varepsilon}{p}
\,\middle|\,
D_{\mathrm{pre}}^{(m)}
\right)
\le
\frac{C p^2}{n\varepsilon^2}
\]
on \(\{\Omega_m\in\mathcal U\}\). Taking a union bound over all \(p^2\) entries gives
\[
\Pr\!\left(
\|C_{m,n}-P_m\|_{\mathrm F}>\varepsilon
\,\middle|\,
D_{\mathrm{pre}}^{(m)}
\right)
\le
\frac{C p^4}{n\varepsilon^2}
\]
on \(\{\Omega_m\in\mathcal U\}\). Since
\[
\|C_{m,n}-P_m\|_{\mathrm{op}}
\le
\|C_{m,n}-P_m\|_{\mathrm F},
\]
we obtain
\[
\Pr\!\left(
\|C_{m,n}-P_m\|_{\mathrm{op}}>\varepsilon
\right)
\le
\Pr(\Omega_m\notin\mathcal U)
+
\frac{C p^4}{n\varepsilon^2}.
\]
Because \(\Omega_m\to\Omega_\star\) in probability and \(\mathcal U\) is a neighborhood of
\(\Omega_\star\), the first term tends to zero. The second term also tends to zero as
\(n\to\infty\). Hence
\[
\|C_{m,n}-P_m\|_{\mathrm{op}}
\xrightarrow{\mathbb P}0.
\]
\end{proof}
\begin{lemma}[Relative trace approximation without an eigengap]
\label{lem:relative-trace-approx}
Assume \Cref{ass:moments}. If \(\Omega_m\to\Omega_\star\) in probability, then
\[
\tr(\Sigma_m\Sigma_{m,n}^+)-\rank(\Sigma_m)
\xrightarrow{\mathbb P}0.
\]
\end{lemma}

\begin{proof}
Let
\[
S_m\coloneqq \operatorname{Im}(\Sigma_m),
\qquad
P_m\coloneqq \Pi_{S_m},
\qquad
r_m\coloneqq \rank(\Sigma_m),
\]
and write
\[
C_{m,n}
=
\Sigma_m^{+/2}\Sigma_{m,n}\Sigma_m^{+/2}.
\]
By \Cref{lem:emp-span-contained}, $\operatorname{Im}(\Sigma_{m,n})\subseteq S_m$ almost surely.

Fix \(0<\rho<1\), and work on the event
\[
\mathcal E_{m,n}^{\mathrm{rel}}(\rho)
\coloneqq
\left\{
\|C_{m,n}-P_m\|_{\mathrm{op}}\le \rho
\right\}.
\]
On \(S_m\), the projector \(P_m\) is the identity. Hence every eigenvalue of
\(C_{m,n}|_{S_m}\) lies in
\[
[1-\rho,1+\rho].
\]
In particular, \(C_{m,n}|_{S_m}\) is invertible.

Since \(\operatorname{Im}(\Sigma_{m,n})\subseteq S_m\), we have the factorization
\[
\Sigma_{m,n}
=
\Sigma_m^{1/2}
C_{m,n}
\Sigma_m^{1/2}
\]
on \(S_m\), and both sides vanish on \(S_m^\perp\). Therefore, on \(S_m\),
\[
\Sigma_{m,n}^+
=
\Sigma_m^{+/2}
\bigl(C_{m,n}|_{S_m}\bigr)^{-1}
\Sigma_m^{+/2},
\]
and both sides vanish on \(S_m^\perp\). Consequently,
\begin{align*}
\tr(\Sigma_m\Sigma_{m,n}^+)
&=
\tr\!\left(
\Sigma_m
\Sigma_m^{+/2}
\bigl(C_{m,n}|_{S_m}\bigr)^{-1}
\Sigma_m^{+/2}
\right)=
\tr\!\left(
\Sigma_m^{+/2}\Sigma_m\Sigma_m^{+/2}
\bigl(C_{m,n}|_{S_m}\bigr)^{-1}
\right)\\
&=
\tr\!\left(
P_m
\bigl(C_{m,n}|_{S_m}\bigr)^{-1}
\right)=
\tr\!\left(
\bigl(C_{m,n}|_{S_m}\bigr)^{-1}
\right).
\end{align*}
Let \(\lambda_1,\dots,\lambda_{r_m}\) be the eigenvalues of
\(C_{m,n}|_{S_m}\). Then
\[
\tr(\Sigma_m\Sigma_{m,n}^+)
=
\sum_{j=1}^{r_m}\frac1{\lambda_j}.
\]
Hence, on \(\mathcal E_{m,n}^{\mathrm{rel}}(\rho)\),
\begin{align*}
\left|
\tr(\Sigma_m\Sigma_{m,n}^+)-r_m
\right|
&=
\left|
\sum_{j=1}^{r_m}
\left(\frac1{\lambda_j}-1\right)
\right|
\le
\sum_{j=1}^{r_m}
\frac{|\lambda_j-1|}{\lambda_j}
\le
r_m\frac{\rho}{1-\rho}
\le
p\frac{\rho}{1-\rho}.
\end{align*}
By \Cref{lem:relative-covariance-lln},
\[
\Pr\!\left(
\mathcal E_{m,n}^{\mathrm{rel}}(\rho)
\right)\to1
\]
for every fixed \(0<\rho<1\). Since \(\rho>0\) can be taken arbitrarily small,
\[
\tr(\Sigma_m\Sigma_{m,n}^+)-r_m
\xrightarrow{\mathbb P}0.
\]
This proves the claim.
\end{proof}

\begin{lemma}[Active effective dimension under stable null span]
\label{lem:deff-null-stable}
Assume Assumptions~\ref{ass:moments}--\ref{ass:stable-null-span}. Let
\[
d_\star
\coloneqq
\dim\mathcal H_{\Omega_\star}
=
d_{\mathrm{eff}}(\Omega_\star),
\qquad
b_0\coloneqq \dim\mathcal B_0,
\]
and define
\[
d_{\mathrm{act}}
\coloneqq
\dim(\mathcal H_{\Omega_\star}\oplus\mathcal B_0)
=
d_\star+b_0.
\]Then
\[
\rank(\Sigma_m)
\xrightarrow{\mathbb P}
d_{\mathrm{act}}.
\]
\end{lemma}

\begin{proof}
Since \(T_{\Omega_\star}|_{E_\star}\) is injective and \(E_\star\) is finite-dimensional,
injectivity is stable under sufficiently small operator-norm perturbations. By
\Cref{lem:T-diff},
\[
T_v|_{E_\star}\to T_{\Omega_\star}|_{E_\star}
\]
in operator norm. Therefore, for all sufficiently small \(v\),
\[
\dim\mathcal A_v
=
\dim\mathcal H_{\Omega_\star}
=
d_\star.
\]

By \Cref{ass:stable-null-span},
\[
\|\Pi_{\mathcal B_v}-\Pi_{\mathcal B_0}\|_{\mathrm{op}}\to0.
\]
For orthogonal projectors, if \(\|P-Q\|_{\mathrm{op}}<1\), then \(P\) and \(Q\)
have the same rank. Hence, for all sufficiently small nonzero \(v\),
\[
\dim\mathcal B_v
=
\dim\mathcal B_0
=
b_0.
\]
Therefore, for all sufficiently small nonzero \(v\),
\[
\dim\mathcal H_v
=
\dim\mathcal A_v+\dim\mathcal B_v
=
d_\star+b_0
=
d_{\mathrm{act}}.
\]

Finally,
\[
\rank(\Sigma(\Omega(v)))
=
\rank(T_v^{\mathrm{adj}}T_v)
=
\rank(T_v)
=
\dim\mathcal H_v.
\]
Thus, for all sufficiently small nonzero \(v\),
\[
\rank(\Sigma(\Omega(v)))=d_{\mathrm{act}}.
\]
Applying this with \(v=\log_{\Omega_\star}(\Omega_m)\), and using \(\log_{\Omega_\star}(\Omega_m)\to0\) in probability gives
\[
\rank(\Sigma_m)
\xrightarrow{\mathbb P}
d_{\mathrm{act}}.
\]
\end{proof}

\subsubsection{Noise estimation term}
\label{app:downstream-terms:noise}

\begin{lemma}[Closed form for the noise term]
\label{lem:noise-closed}
For each $(n,m)$,
\[
\mathrm{Var}_{n,m}
=
\frac{\sigma^2}{n}\,\tr\!\bigl(\Sigma_{m}\,\Sigma_{m,n}^+\bigr).
\]
\end{lemma}

\begin{proof}
Fix $n$ and condition on $(D_{\mathrm{pre}}^{(m)},X_{1:n},X_{\mathrm{new}})$.
By Definition~\ref{def:canonical-emp-proj} (Appendix~\ref{app:prelim}),
\[
\Pi_{\Omega_{m},n}\varepsilon=T_{\Omega_{m}}\hat\theta_{\varepsilon},
\qquad
\hat\theta_{\varepsilon}
=
\Sigma_{m,n}^+\,
\frac1{n}\sum_{i=1}^{n} \varepsilon_i\,\phi(X_i,\Omega_{m}).
\]
Let $g_{m}\coloneqq \frac{1}{n}\sum_{i=1}^{n}\varepsilon_i\phi(X_i,\Omega_{m})$, so $\hat\theta_{\varepsilon}=\Sigma_{m,n}^+ g_{m}$ and
\[
(\Pi_{\Omega_{m},n}\varepsilon)(X_{\mathrm{new}})
=
\phi(X_{\mathrm{new}},\Omega_{m})^\top \Sigma_{m,n}^+ g_{m}.
\]
Since $\E[\varepsilon_i\mid X_i]=0$ and $\E[\varepsilon_i^2\mid X_i]=\sigma^2$,
\[
\E[g_{m} g_{m}^\top \mid D_{\mathrm{pre}}^{(m)},X_{1:n}]
=
\frac{1}{n^2}\sum_{i=1}^{n}\E[\varepsilon_i^2\mid X_i]\phi(X_i,\Omega_{m})\phi(X_i,\Omega_{m})^\top
=
\frac{\sigma^2}{n}\,\Sigma_{m,n}.
\]
Therefore,
\begin{align*}
\mathrm{Var}_{n,m}
&= \E\left[\phi(X_{\mathrm{new}},\Omega_{m})^\top \Sigma_{m,n}^+ g_{m} g_{m}^\top \Sigma_{m,n}^+\phi(X_{\mathrm{new}},\Omega_{m})\,\middle|\,
D_{\mathrm{pre}}^{(m)},X_{1:n},X_{\mathrm{new}}\right]\\
&=\frac{\sigma^2}{n}\,
\phi(X_{\mathrm{new}},\Omega_{m})^\top \Sigma_{{m},n}^+ \phi(X_{\mathrm{new}},\Omega_m),
\end{align*}

using $\Sigma_{m,n}^+\Sigma_{m,n}\Sigma_{m,n}^+=\Sigma_{m,n}^+$.
Taking conditional expectation over $X_{\mathrm{new}}$ yields
\[
\mathrm{Var}_{n,m}
=
\frac{\sigma^2}{n}\,
\tr\!\bigl(\Sigma_m\,\Sigma_{m,n}^+\bigr).
\]
\end{proof}

\begin{corollary}[Noise estimation term under stable null span]
\label{cor:noise-term}
Assume Assumptions~\ref{ass:moments}, \ref{ass:phi-C1-M}, and
\ref{ass:stable-null-span}. Let
\[
d_{\mathrm{act}}
\coloneqq
\dim(\mathcal H_{\Omega_\star}\oplus\mathcal B_0)
=
d_{\mathrm{eff}}(\Omega_\star)+\dim\mathcal B_0.
\]Then, along \(n,m\to\infty\) with \(m/n\to\alpha\),
\[
n\,\mathrm{Var}_{n,m}
\xrightarrow{\mathbb P}
\sigma^2 d_{\mathrm{act}}.
\]
\end{corollary}

\begin{proof}
By \Cref{lem:noise-closed},
\[
n\,\mathrm{Var}_{n,m}
=
\sigma^2\,\tr(\Sigma_m\Sigma_{m,n}^+).
\]
By \Cref{lem:relative-trace-approx},
\[
\tr(\Sigma_m\Sigma_{m,n}^+)-\rank(\Sigma_m)
\xrightarrow{\mathbb P}0.
\]
By \Cref{lem:deff-null-stable},
\[
\rank(\Sigma_m)
\xrightarrow{\mathbb P}
d_{\mathrm{act}}.
\]
Therefore,
\[
\tr(\Sigma_m\Sigma_{m,n}^+)
\xrightarrow{\mathbb P}
d_{\mathrm{act}}.
\]
Multiplying by \(\sigma^2\) yields
\[
n\,\mathrm{Var}_{n,m}
\xrightarrow{\mathbb P}
\sigma^2d_{\mathrm{act}}.
\]
\end{proof}

\subsubsection{Leakage Estimation Term}
\label{app:downstream-terms:signal}

Under \Cref{ass:well-posedness}, \eqref{eq:app:wellposed-bridge} yields
\[
\mathrm{Leakage}_n
=
\E\!\left[
\bigl(\Pi_{\Omega_m,n}e_{\Omega_m}\bigr)(X_{\mathrm{new}})^2
\,\middle|\,
D_{\mathrm{pre}}^{(m)},X_{1:n}
\right].
\]
Define the (random) population matrix
\[
\Sigma_{e,m}\coloneqq \Sigma_e(\Omega_m)\coloneqq \E\!\left[e_{\Omega_m}(X)^2\,\phi(X,\Omega_m)\phi(X,\Omega_m)^\top\,\middle|\,D_{\mathrm{pre}}^{(m)}\right],
\qquad X\sim\mu_{\mathrm{down}}.
\]

\begin{lemma}[Vanishing whitened signal covariance in the compatible limit]
\label{lem:whitened-sigmae-vanish}
Assume Assumptions~\ref{ass:moments}, \ref{ass:phi-C1-M}, and
\ref{ass:stable-null-span}. Then
\[
\tr(\Sigma_m^+\Sigma_{e,m})
\xrightarrow{\mathbb P}0.
\]
\end{lemma}

\begin{proof}
Write
\[
e_m(X)\coloneqq e_{\Omega_m}(X),
\qquad
\phi_m(X)\coloneqq \phi(X,\Omega_m),
\qquad
q_m(X)\coloneqq q_{\Omega_m}(X).
\]
By definition,
\[
\tr(\Sigma_m^+\Sigma_{e,m})
=
\E\!\left[
e_m(X)^2
\phi_m(X)^\top\Sigma_m^+\phi_m(X)
\,\middle|\,
D_{\mathrm{pre}}^{(m)}
\right]
=
\E\!\left[
e_m(X)^2q_m(X)
\,\middle|\,
D_{\mathrm{pre}}^{(m)}
\right].
\]
Apply H\"older's inequality with conjugate exponents $\frac{1+\eta}{\eta}$ and $1+\eta$. Then
\begin{align*}
\E\!\left[
e_m(X)^2q_m(X)
\,\middle|\,
D_{\mathrm{pre}}^{(m)}
\right]
&=
\E\!\left[
\bigl(e_m(X)^2\bigr)^{\frac{\eta}{1+\eta}}
\bigl(e_m(X)^2q_m(X)^{1+\eta}\bigr)^{\frac1{1+\eta}}
\,\middle|\,
D_{\mathrm{pre}}^{(m)}
\right] \\
&\le
\left(
\E\!\left[
e_m(X)^2
\,\middle|\,
D_{\mathrm{pre}}^{(m)}
\right]
\right)^{\frac{\eta}{1+\eta}}
\left(
\E\!\left[
e_m(X)^2q_m(X)^{1+\eta}
\,\middle|\,
D_{\mathrm{pre}}^{(m)}
\right]
\right)^{\frac1{1+\eta}} .
\end{align*}

Let $v_m = \log_{\Omega_\star}(\Omega_m)$. By the residual expansion in \Cref{lem:rep-linearization},
\[
e_{\Omega_m}
=
\mathcal L(v_m)+r(v_m),
\qquad
\|r(v_m)\|_{L^2(\mu_{\mathrm{down}})}
\le
\omega(\|v_m\|)\|v_m\|.
\]
Since \(v_m\to0\) in probability and \(\mathcal L\) is bounded, this implies
\[
\E\!\left[
e_m(X)^2
\,\middle|\,
D_{\mathrm{pre}}^{(m)}
\right]
=
\|e_{\Omega_m}\|_{L^2(\mu_{\mathrm{down}})}^2
\xrightarrow{\mathbb P}0.
\]

On the event \(\{\Omega_m\in\mathcal U\}\), Assumption~\ref{ass:moments} gives
\[
\E\!\left[
e_m(X)^2q_m(X)^{1+\eta}
\,\middle|\,
D_{\mathrm{pre}}^{(m)}
\right]
\le
C_{\mathrm{lev}}.
\]
Therefore, on \(\{\Omega_m\in\mathcal U\}\),
\[
\tr(\Sigma_m^+\Sigma_{e,m})
\le
C_{\mathrm{lev}}^{\frac1{1+\eta}}
\left(
\E\!\left[
e_m(X)^2
\,\middle|\,
D_{\mathrm{pre}}^{(m)}
\right]
\right)^{\frac{\eta}{1+\eta}}.
\]
Because \(\Omega_m\to\Omega_\star\) in probability, we have $\Pr(\Omega_m\notin\mathcal U)\to0$. Combining the preceding display with
\[
\E\!\left[
e_m(X)^2
\,\middle|\,
D_{\mathrm{pre}}^{(m)}
\right]
\xrightarrow{\mathbb P}0
\]
yields
\[
\tr(\Sigma_m^+\Sigma_{e,m})
\xrightarrow{\mathbb P}0.
\]
\end{proof}

\begin{proposition}[Leakage estimation term under stable null span]
\label{prop:signal-term}
Assume Assumptions~\ref{ass:moments}--\ref{ass:stable-null-span}. Then, along \(m,n\to\infty\) with
\(m/n\to\alpha\),
\[
n\,\mathrm{Leakage}_n
\xrightarrow{\mathbb P}0.
\]
\end{proposition}

\begin{proof}
Fix \(n\) and condition on \((D_{\mathrm{pre}}^{(m)},X_{1:n})\). Under
\Cref{ass:well-posedness}, \eqref{eq:app:wellposed-bridge} gives
\[
\mathrm{Leakage}_n
=
\E\!\left[
\bigl(\Pi_{\Omega_m,n}e_{\Omega_m}\bigr)(X_{\mathrm{new}})^2
\,\middle|\,
D_{\mathrm{pre}}^{(m)},X_{1:n}
\right].
\]
Write
\[
e_m(X)\coloneqq e_{\Omega_m}(X),
\qquad
\phi_m(X)\coloneqq \phi(X,\Omega_m),
\]
and define
\[
g_m
\coloneqq
\frac1n\sum_{i=1}^n e_m(X_i)\phi_m(X_i).
\]
By the definition of the canonical empirical projector,
\[
\bigl(\Pi_{\Omega_m,n}e_m\bigr)(X_{\mathrm{new}})
=
\phi_m(X_{\mathrm{new}})^\top\Sigma_{m,n}^+g_m.
\]
Therefore,
\begin{align*}
\mathrm{Leakage}_n
&=
\E\!\left[
g_m^\top
\Sigma_{m,n}^+
\phi_m(X_{\mathrm{new}})\phi_m(X_{\mathrm{new}})^\top
\Sigma_{m,n}^+
g_m
\,\middle|\,
D_{\mathrm{pre}}^{(m)},X_{1:n}
\right]\\
&=
g_m^\top
\Sigma_{m,n}^+\Sigma_m\Sigma_{m,n}^+
g_m.
\end{align*}
Thus
\[
n\,\mathrm{Leakage}_n
=
(\sqrt n\,g_m)^\top
\Sigma_{m,n}^+\Sigma_m\Sigma_{m,n}^+
(\sqrt n\,g_m).
\]

Define the whitened residual-correlation vector
\[
\widetilde g_m
\coloneqq
\Sigma_m^{+/2}g_m.
\]
We first show that
\[
\sqrt n\,\widetilde g_m\xrightarrow{\mathbb P}0.
\]
Condition only on \(D_{\mathrm{pre}}^{(m)}\). The random vectors
\[
\widetilde W_{m,i}
\coloneqq
e_m(X_i)\Sigma_m^{+/2}\phi_m(X_i),
\qquad
i=1,\dots,n,
\]
are conditionally i.i.d. Since \(e_m\perp\mathcal H_{\Omega_m}\) in
\(L^2(\mu_{\mathrm{down}})\), and each coordinate of \(\phi_m\) belongs to
\(\mathcal H_{\Omega_m}\), we have
\[
\E\!\left[
e_m(X_i)\phi_m(X_i)
\,\middle|\,
D_{\mathrm{pre}}^{(m)}
\right]
=0.
\]
Hence
\[
\E\!\left[
\widetilde W_{m,i}
\,\middle|\,
D_{\mathrm{pre}}^{(m)}
\right]
=0.
\]
Moreover,
\[
\sqrt n\,\widetilde g_m
=
\frac1{\sqrt n}\sum_{i=1}^n \widetilde W_{m,i}.
\]
Therefore,
\begin{align*}
\E\!\left[
\|\sqrt n\,\widetilde g_m\|^2
\,\middle|\,
D_{\mathrm{pre}}^{(m)}
\right]
&=
\E\!\left[
\|\widetilde W_{m,1}\|^2
\,\middle|\,
D_{\mathrm{pre}}^{(m)}
\right]\\
&=
\E\!\left[
e_m(X)^2
\phi_m(X)^\top\Sigma_m^+\phi_m(X)
\,\middle|\,
D_{\mathrm{pre}}^{(m)}
\right]\\
&=
\tr(\Sigma_m^+\Sigma_{e,m}).
\end{align*}
By \Cref{lem:whitened-sigmae-vanish},
\[
\tr(\Sigma_m^+\Sigma_{e,m})
\xrightarrow{\mathbb P}0.
\]
For every \(\varepsilon>0\),
\[
\Pr(\|\sqrt n\,\widetilde g_m\|>\varepsilon)
\le
\E\!\left[
\min\left\{
\frac{
\E[\|\sqrt n\,\widetilde g_m\|^2\mid D_{\mathrm{pre}}^{(m)}]
}{\varepsilon^2},
1
\right\}
\right].
\]
The conditional second moment is
\(\tr(\Sigma_m^+\Sigma_{e,m})\to_{\mathbb P}0\), and the expression inside
the expectation is bounded by \(1\). Hence dominated convergence along
convergence in probability gives
\[
\sqrt n\,\widetilde g_m\to_{\mathbb P}0.
\]

It remains to control the quadratic form. Let
\[
S_m\coloneqq \operatorname{Im}(\Sigma_m),
\qquad
P_m\coloneqq \Pi_{S_m},
\qquad
C_{m,n}
\coloneqq
\Sigma_m^{+/2}\Sigma_{m,n}\Sigma_m^{+/2}.
\]
By \Cref{lem:emp-span-contained},
\[
\operatorname{Im}(\Sigma_{m,n})\subseteq S_m
\]
almost surely. Work on the event
\[
\mathcal E_{m,n}^{\mathrm{rel}}
\coloneqq
\left\{
\left\|
C_{m,n}-P_m
\right\|_{\mathrm{op}}
\le \frac12
\right\}.
\]
By \Cref{lem:relative-covariance-lln},
\[
\Pr(\mathcal E_{m,n}^{\mathrm{rel}})\to1.
\]
On this event, \(C_{m,n}|_{S_m}\) is invertible and
\[
\left\|
\bigl(C_{m,n}|_{S_m}\bigr)^{-1}
\right\|_{\mathrm{op}}
\le 2.
\]
As in \Cref{lem:relative-trace-approx},
\[
\Sigma_{m,n}^+
=
\Sigma_m^{+/2}
\bigl(C_{m,n}|_{S_m}\bigr)^{-1}
\Sigma_m^{+/2}
\]
on \(S_m\), and both sides vanish on \(S_m^\perp\). Therefore,
\[
\Sigma_{m,n}^+\Sigma_m\Sigma_{m,n}^+
=
\Sigma_m^{+/2}
\bigl(C_{m,n}|_{S_m}\bigr)^{-2}
\Sigma_m^{+/2}.
\]
Substituting into the leakage quadratic form yields
\[
n\,\mathrm{Leakage}_n
=
(\sqrt n\,\widetilde g_m)^\top
\bigl(C_{m,n}|_{S_m}\bigr)^{-2}
(\sqrt n\,\widetilde g_m).
\]
Hence, on \(\mathcal E_{m,n}^{\mathrm{rel}}\),
\[
0
\le
n\,\mathrm{Leakage}_n
\le
4\|\sqrt n\,\widetilde g_m\|^2.
\]
Since
\[
\sqrt n\,\widetilde g_m\xrightarrow{\mathbb P}0
\]
and \(\Pr(\mathcal E_{m,n}^{\mathrm{rel}})\to1\), we conclude that
\[
n\,\mathrm{Leakage}_n\xrightarrow{\mathbb P}0.
\]
\end{proof}

\subsection{Dimension gap}
\label{app:dimension-gap}

We record a simple observation explaining the role of the dimension gap in the
downstream-only comparison. Work locally around a regular realization
$(\theta_\star,\Omega_\star)$ of $f_\star$, so that
\[
f_{\theta,\Omega}(x)=\ip{\theta}{\phi(x,\Omega)},
\qquad
f_\star=f_{\theta_\star,\Omega_\star}.
\]
Let
\[
\mathcal H_{\Omega_\star}
\coloneqq
\left\{
\ip{\theta}{\phi(\cdot,\Omega_\star)}:\theta\in\R^p
\right\}
\subset L^2(\mu_{\mathrm{down}})
\]
be the fixed-representation downstream class of the optimal representation $\Omega_\star$. We have define
\[
d_{\mathrm{eff}}(\Omega_\star)
\coloneqq
\dim \mathcal H_{\Omega_\star}.
\]
The downstream-only tangent space at $(\theta_\star,\Omega_\star)$ is
\[
\mathcal T_{\mathrm{base}}
\coloneqq
\Big\{
\ip{\delta\theta}{\phi(\cdot,\Omega_\star)}
+
\langle\theta_\star,D\phi(\cdot,\Omega_\star)[v]\rangle\Big|
\delta\theta\in\R^p,\;
v\in T_{\Omega_\star}\mathcal M
\Big\}
\subset L^2(\mu_{\mathrm{down}}),
\]
and we define
\[
d_{\mathrm{base}}
\coloneqq
\dim \mathcal T_{\mathrm{base}}.
\]
Finally, recall the linear map for the case \(\mathcal B_0=\{0\}\)
\[
\mathcal L(v)
\coloneqq
-D\Pi_{\Omega_\star}[v]f_\star,
\qquad
v\in T_{\Omega_\star}\mathcal M.
\]

\begin{proposition}
\label{prop:dimension-gap-L}
We have $d_{\mathrm{base}}\ge d_{\mathrm{eff}}(\Omega_\star)$. Moreover,
\[
d_{\mathrm{base}}=d_{\mathrm{eff}}(\Omega_\star)
\qquad\Longleftrightarrow\qquad
\mathcal L\equiv 0.
\]
\end{proposition}

\begin{proof}
The directions obtained by varying only the readout parameter are
\[
\left\{
\ip{\delta\theta}{\phi(\cdot,\Omega_\star)}
:
\delta\theta\in\R^p
\right\}
=
\mathcal H_{\Omega_\star}.
\]
These directions are contained in $\mathcal T_{\mathrm{base}}$, hence $\mathcal H_{\Omega_\star}\subseteq \mathcal T_{\mathrm{base}}$, which implies
\[
d_{\mathrm{base}}\ge d_{\mathrm{eff}}(\Omega_\star).
\]

For the equivalence, fix $v\in T_{\Omega_\star}\mathcal M$ and define
\[
h_v
\coloneqq
\langle\theta_\star,D\phi(\cdot,\Omega_\star)[v]\rangle.
\]
By definition,
\[
\mathcal T_{\mathrm{base}}
=
\mathcal H_{\Omega_\star}
+
\{h_v:v\in T_{\Omega_\star}\mathcal M\}.
\]
Since $\mathcal H_{\Omega_\star}\subseteq\mathcal T_{\mathrm{base}}$, we have
$d_{\mathrm{base}}=d_{\mathrm{eff}}(\Omega_\star)$ if and only if
\[
h_v\in\mathcal H_{\Omega_\star}
\qquad
\text{for every }v\in T_{\Omega_\star}\mathcal M.
\]

Now let $\Omega_t=\exp_{\Omega_\star}(tv)$. Since
$f_{\theta_\star,\Omega_t}\in\mathcal H_{\Omega_t}$, we have
\[
\Pi_{\Omega_t}f_{\theta_\star,\Omega_t}
=
f_{\theta_\star,\Omega_t}.
\]
Differentiating at $t=0$ gives
\[
D\Pi_{\Omega_\star}[v]f_\star
+
\Pi_{\Omega_\star}h_v
=
h_v.
\]
Therefore,
\[
D\Pi_{\Omega_\star}[v]f_\star
=
(I-\Pi_{\Omega_\star})h_v,
\]
and hence
\[
\mathcal L(v)
=
-D\Pi_{\Omega_\star}[v]f_\star
=
-(I-\Pi_{\Omega_\star})h_v.
\]
Thus $\mathcal L(v)=0$ if and only if $h_v\in\mathcal H_{\Omega_\star}$.
Since this holds for every $v\in T_{\Omega_\star}\mathcal M$, we conclude that
\[
\mathcal L\equiv 0
\qquad\Longleftrightarrow\qquad
d_{\mathrm{base}}=d_{\mathrm{eff}}(\Omega_\star).
\]
The final equivalence follows by taking the contrapositive together with
$d_{\mathrm{base}}\ge d_{\mathrm{eff}}(\Omega_\star)$.
\end{proof}

\section{Proofs of Section~\ref{sec:example_linear_spectral}}
\label{app:proof-cor-linear}

This appendix verifies that the linear spectral contrastive model of
\Cref{sec:example_linear_spectral} satisfies the standing assumptions of Theorem~\ref{thm:master-compatible}. We separate the verification into two parts. First, we verify the quotient geometry and downstream assumptions (Assumptions~\ref{ass:well-posedness}--\ref{ass:stable-null-span}). Second, we verify the pre-training CLT assumptions, namely~\Cref{ass:riem_mestimation}. Once these checks are in place,~\Cref{cor:linear-spectral-model} follows by a direct invocation of~\Cref{thm:master-compatible}.

\subsection{Geometry, descriptor, and quotient feature map}\label{sec:quot-feat-linear}
\paragraph{Quotient feature map $\phi(x,M)$ via a local section.}
To express downstream prediction purely in quotient coordinates, we define a feature map that depends on the descriptor $M\in\mathcal{M}_{d,k}$. Fix a regular point $M_\star\in\mathcal{M}_{d,k}$ and a neighborhood $\mathcal W\subset\mathcal{M}_{d,k}$ of $M_\star$ on which there exists a $C^2$ \emph{local section}
\[
s:\mathcal W\to\R^{k\times d},
\qquad s(M)^\T s(M)=M\ \ \text{for all }M\in\mathcal W,
\]
as constructed in Appendix~\ref{app:riem_bundle}.
We then define the \emph{quotient feature map}
\begin{align}
\label{eq:linear_spec_phi_def}
\phi(x,M)\coloneqq s(M)\,x\in\R^k,
\qquad (x,M)\in\R^d\times\mathcal W.
\end{align}
Any two choices of local section differ by a left-multiplication $s'(M)=Q(M)s(M)$ with $Q(M)\in O(k)$,
hence they generate features related by an orthogonal transform.
\paragraph{Concrete local choice of the section $s(M)$.}
Fix a reference point $M_\star\in\mathcal{M}_{d,k}$ and choose an orthonormal basis
$U_\star\in\R^{d\times k}$ of $\mathrm{range}(M_\star)$ (so $U_\star^\top U_\star=I_k$).
For $M$ in a sufficiently small neighborhood $\mathcal W$ of $M_\star$, let $P(M)$ denote the orthogonal projector
onto $\mathrm{range}(M)$. Define
\begin{align*}
B(M)\coloneqq U_\star^\top P(M)U_\star\in\R^{k\times k},\;
U(M)\coloneqq P(M)U_\star\,B(M)^{-1/2}\in\R^{d\times k},  
\end{align*}
so that $U(M)^\top U(M)=I_k$ and $\mathrm{range}(U(M))=\mathrm{range}(M)$.
Next set
\[
\Lambda(M)\coloneqq U(M)^\top M U(M)\in\R^{k\times k},
\]
which is positive definite on $\mathcal W$ and satisfies $M=U(M)\Lambda(M)U(M)^\top$.
A concrete local section is then
\[
s(M)\coloneqq \Lambda(M)^{1/2}U(M)^\top\in\R^{k\times d},
\]
so that $s(M)^\top s(M)=M$. Consequently, the quotient-level feature map can be written explicitly as
\[
\phi(x,M)=s(M)x=\Lambda(M)^{1/2}U(M)^\top x\in\R^k.
\]
All smoothness claims (existence of $\mathcal W$ and differentiability of $M\mapsto s(M)$) are proved in
Appendix~\ref{app:riem_bundle}.

\subsection{Population descriptor problem and regularity of $M_\star$}
\label{app:proof-cor-linear:target}

\paragraph{Population loss in descriptor space.} 
The following lemma expresses the population objective in terms of $M=A^\top A$.

\begin{lemma}
\label{lem:spec-loss-in-M}
Assume $\E\|x\|^4<\infty$.
Then
\begin{align}
\label{eq:Lspec-M-form}
L_{\mathrm{spec}}(M)
&=
-2\,\Tr\!\big(M \Sigma_{\mathrm{pre}}^{+}\big)
+\Tr\!\big((M \Sigma_{\mathrm{pre}})^2\big).
\end{align}
If furthermore $\Sigma_{\mathrm{pre}}$ is full-rank, then
\begin{align}
\label{eq:Lspec-square}
L_{\mathrm{spec}}(M)
&=
\big\| \Sigma_{\mathrm{pre}}^{1/2} M \Sigma_{\mathrm{pre}}^{1/2} - C \big\|_F^2
- \| C \|_F^2,
\end{align}
where $C \coloneqq  \Sigma_{\mathrm{pre}}^{-1/2}\, \Sigma_{\mathrm{pre}}^{+}\, \Sigma_{\mathrm{pre}}^{-1/2}$.
\end{lemma}

\begin{proof}
Let $\Sigma\coloneqq\Sigma_{\mathrm{pre}}=\E[xx^\top]$. Since $x^-$ is an independent copy of $x$, we have
\begin{align*}
    L_{\mathrm{spec}}(M) &= -2\E[x^\top Mx^+] + \E[(x^\top Mx^-)^2]=-2\tr(\E[Mx^+x^\top) +\tr(\E[Mx(x^-)^\top M xx^\top])\\
    &=-2\tr(M\Sigma_{\mathrm{pre}}^+) + \tr(M\Sigma M \Sigma)
\end{align*}
which proves \eqref{eq:Lspec-M-form}.
If $\Sigma$ is full rank, write
\[
\Tr\!\big((M\Sigma)^2\big)=\|\Sigma^{1/2}M\Sigma^{1/2}\|_F^2,
\qquad
\Tr(M\Sigma^+)=\big\langle \Sigma^{1/2}M\Sigma^{1/2},\,C\big\rangle_F,
\]
with $C=\Sigma^{-1/2}\Sigma^+\Sigma^{-1/2}$.
Completing the square gives \eqref{eq:Lspec-square}.
\end{proof}

\begin{remark}
Even though $\Sigma_{\mathrm{pre}}$ and $M$ are PSD, the matrix $C$ can be indefinite:
augmentations may induce negative correlations along some directions.
\end{remark}

\paragraph{Uniqueness of the descriptor minimizer.}
By \Cref{lem:spec-loss-in-M}, the population minimization in descriptor space reduces to
\begin{align}
M_\star
\in \argmin_{M \in \mathcal{M}_{d,k}}
\ \big\| \Sigma_{\mathrm{pre}}^{1/2} M \Sigma_{\mathrm{pre}}^{1/2} - C \big\|_F^2.
\label{eq:spec_loss_manifold}
\end{align}
We now state a simple eigengap condition that guarantees uniqueness.

Assume throughout $\Sigma_{\mathrm{pre}}\succ 0$ and \Cref{ass:C-positive-k}. Recall the whitened matrix
\[
C\coloneqq \Sigma_{\mathrm{pre}}^{-1/2}\Sigma_{\mathrm{pre}}^+\Sigma_{\mathrm{pre}}^{-1/2},
\]
and the descriptor manifold $\mathcal M_{d,k}\coloneqq \{M\succeq 0:\rank(M)=k\}$.
Thus, the population minimizer $M_\star$ exists, is unique,
and satisfies
\[
\Sigma_{\mathrm{pre}}^{1/2}M_\star\Sigma_{\mathrm{pre}}^{1/2}=U_k\Lambda_kU_k^\top,
\]
where $U_k\Lambda_kU_k^\top$ is the rank-$k$ truncation onto the top-$k$ \emph{positive} eigenvalues of $C$.
In particular, $M_\star\in\mathcal M_{d,k}$ is a regular point of the fixed-rank PSD manifold.

\subsection{Verification of Assumptions~\ref{ass:well-posedness}--\ref{ass:stable-null-span}}
This subsection verifies all assumptions of Theorem~\ref{thm:master-compatible} except the
pre-training CLT assumption~\ref{ass:riem_mestimation}.
\subsubsection{Local feature regularity and moments}
\label{app:proof-cor-linear:features}

Fix the regular point $M_\star\in\mathcal M_{d,k}$ and let $s(\cdot)$ be the $C^2$ local section from
Appendix~\ref{app:riem_bundle} (equivalently, the concrete construction in \Cref{sec:example_linear_spectral}),
defined on a neighborhood $\mathcal W\subset\mathcal M_{d,k}$ of $M_\star$ and satisfying $s(M)^\top s(M)=M$.
Define the quotient feature map
\[
\phi(x,M)\coloneqq s(M)x\in\R^k,
\qquad
(x,M)\in\R^d\times\mathcal W.
\]
By Appendix~\ref{app:riem_bundle}, the map $M\mapsto s(M)$ is $C^2$ on $\mathcal W$, hence
$M\mapsto \phi(x,M)$ is $C^1$ for each fixed $x$.

\begin{lemma}[Local-uniform feature and derivative moments]
\label{lem:linear:moments}
Assume $\E\|X\|^{4+\delta}<\infty$ for some $\delta>0$, where $X\sim\mu_{\mathrm{down}}$.
Then there exists a neighborhood $\mathcal U\subseteq\mathcal W$ of $M_\star$ and a constant $C<\infty$ such that
\[
\sup_{M\in\mathcal U}\E\|\phi(X,M)\|^{4+\delta}\le C,
\qquad
\sup_{M\in\mathcal U}\E\|D_M\phi(X,M)\|_{\mathrm{op}}^{4+\delta}\le C.
\]
In particular, the local $C^1$ feature regularity condition in~\Cref{ass:phi-C1-M} is satisfied.
\end{lemma}

\begin{proof}
Since $s(\cdot)$ is $C^2$ on $\mathcal W$, its operator norm and the operator norm of its derivative are locally bounded.
Choose a relatively compact neighborhood $\mathcal U\Subset\mathcal W$ of $M_\star$ so that
\[
\sup_{M\in\mathcal U}\|s(M)\|_{\mathrm{op}}<\infty,
\qquad
\sup_{M\in\mathcal U}\|D_M s(M)\|_{\mathrm{op}}<\infty.
\]
Then $\|\phi(X,M)\|_2\le \|s(M)\|_{\mathrm{op}}\|X\|_2$ and
$\|D_M\phi(X,M)\|_{\mathrm{op}}\le \|D_M s(M)\|_{\mathrm{op}}\|X\|_2$.
Raising to the power $4+\delta$ and taking expectations yields the stated bounds.
\end{proof}

\subsubsection{Downstream covariance, stable rank, and effective dimension}
\label{app:proof-cor-linear:downstream-rank}

Write $\Sigma_{\mathrm{down}}\coloneqq \E[XX^\top]$ and assume $\Sigma_{\mathrm{down}}\succ0$.
For $M\in\mathcal U$, define the downstream feature covariance
\[
\Sigma(M)\coloneqq \E[\phi(X,M)\phi(X,M)^\top]=\E[s(M)XX^\top s(M)^\top]=s(M)\Sigma_{\mathrm{down}}s(M)^\top\in\R^{k\times k}.
\]

\begin{lemma}[Stable rank/eigengap for $\Sigma(M)$ near $M_\star$]
\label{lem:linear:Sigma-stable}
Assume $\Sigma_{\mathrm{down}}\succ0$ and let $\mathcal U$ be as in Lemma~\ref{lem:linear:moments}.
Then there exist constants $\kappa_\Sigma,K_\Sigma\in(0,\infty)$ such that, for all $M\in\mathcal U$,
\[
\rank(\Sigma(M))=k,
\qquad
\lambda_k(\Sigma(M))\ge \kappa_\Sigma,
\qquad
\|\Sigma(M)\|_{\mathrm{op}}\le K_\Sigma.
\]
In particular, $d_{\mathrm{eff}}(M)=\tr(\Sigma(M)\Sigma(M)^+)=k$ for all $M\in\mathcal U$.
\end{lemma}

\begin{proof}
For each \(M\in\mathcal U\), the matrix \(s(M)\in\R^{k\times d}\) has rank \(k\)
because \(s(M)^\top s(M)=M\) and \(\rank(M)=k\). Since
\(\Sigma_{\mathrm{down}}\succ0\), for every nonzero \(u\in\R^k\),
\[
u^\top \Sigma(M)u
=
u^\top s(M)\Sigma_{\mathrm{down}}s(M)^\top u
=
\|\Sigma_{\mathrm{down}}^{1/2}s(M)^\top u\|_2^2
>0.
\]
Hence \(\Sigma(M)\succ0\) and \(\rank(\Sigma(M))=k\).

Shrink \(\mathcal U\), if necessary, so that its closure
\(\overline{\mathcal U}\) is compact. Since
\(M\mapsto s(M)\) is continuous, the maps
\[
M\mapsto \lambda_{\min}(\Sigma(M)),
\qquad
M\mapsto \|\Sigma(M)\|_{\mathrm{op}}
\]
are continuous on \(\overline{\mathcal U}\). The first is strictly positive on
\(\overline{\mathcal U}\), and the second is finite there. Therefore there exist
\(\kappa_\Sigma,K_\Sigma\in(0,\infty)\) such that, for all \(M\in\mathcal U\),
\[
\lambda_k(\Sigma(M))\ge \kappa_\Sigma,
\qquad
\|\Sigma(M)\|_{\mathrm{op}}\le K_\Sigma.
\]
Finally, since \(\Sigma(M)\) is invertible,
\[
d_{\mathrm{eff}}(M)
=
\Tr(\Sigma(M)\Sigma(M)^+)
=
\Tr(I_k)
=
k.
\]
\end{proof}
\subsubsection{Local leverage and leverage-weighted signal moments}
Recall that for $M\in \mathcal U$, we defined the population leverage score as
\[
q_M(X)\coloneqq \phi(X,M)^\top\Sigma(M)^+\phi(X,M).
\]

\begin{lemma}[Local-uniform leverage and signal moments]
\label{lem:linear:leverage-moments}
Assume $\E\|X\|^{4+\delta}<\infty$ for some $\delta>0$. Then there exist $\eta>0$ and constants $C_q,C_e<\infty$ such that, after possibly shrinking
$\mathcal U$,
\[
\sup_{M\in\mathcal U}\E\!\left[q_M(X)^{2+\eta}\right]\le C_q,
\qquad
\sup_{M\in\mathcal U}\E\!\left[e_M(X)^2q_M(X)^{1+\eta}\right]\le C_e.
\]
Thus, the local-uniform moment and leverage conditions in~\Cref{ass:moments} are satisfied.
\end{lemma}

\begin{proof}
By \Cref{lem:linear:Sigma-stable}, $\|\Sigma(M)^+\|_{\mathrm{op}}\le \kappa_\Sigma^{-1}$ uniformly over
$M\in\mathcal U$. Hence
\[
q_M(X)\le \kappa_\Sigma^{-1}\|\phi(X,M)\|^2
\le C\|X\|^2
\]
uniformly over $M\in\mathcal U$.
Choose any $\eta\in(0,\delta/2]$. Then
\[
\sup_{M\in\mathcal U}\E q_M(X)^{2+\eta}
\le C\,\E\|X\|^{4+2\eta}
\le C\,\E(1+\|X\|^{4+\delta})<\infty.
\]

It remains to control the signal term. Since $f_\star(x)=\theta_\star^\top\phi(x,M_\star)$ and
$\phi(x,M)=s(M)x$ with $s(M)$ uniformly bounded on $\mathcal U$, we have
$|f_\star(X)|\le C\|X\|$.
Moreover,
\[
\Pi_M f_\star(x)
=
\phi(x,M)^\top\Sigma(M)^+\,\E[\phi(X,M)f_\star(X)].
\]
The vector $\Sigma(M)^+\E[\phi(X,M)f_\star(X)]$ is uniformly bounded over $M\in\mathcal U$ by
\Cref{lem:linear:Sigma-stable}, \Cref{lem:linear:moments}, and Cauchy--Schwarz. Therefore
$|\Pi_M f_\star(X)|\le C\|X\|$ uniformly over $M\in\mathcal U$, and hence
$|e_M(X)|\le C\|X\|$ uniformly over $M\in\mathcal U$. Combining this bound with
$q_M(X)\le C\|X\|^2$ gives
\[
e_M(X)^2q_M(X)^{1+\eta}
\le C\|X\|^{4+2\eta}.
\]
Taking expectations and using $2\eta\le\delta$ proves the claim.
\end{proof}

\subsubsection{Well-posedness of the empirical projector}
\label{app:proof-cor-linear:wellposed}

Recall that $\mathcal H_M=\{x\mapsto\theta^\top\phi(x,M):\theta\in\R^k\}$ is a $k$-dimensional linear class.
The well-posedness assumption of Appendix~\ref{app:downstream-terms} is equivalent to requiring that the
empirical inner product is non-degenerate on $\mathcal H_M$, or equivalently that the $k\times k$ empirical
covariance $\Sigma_n(M)$ has full rank.

\begin{lemma}[Well-posedness holds almost surely for nondegenerate designs (\Cref{ass:well-posedness})]
\label{lem:linear:wellposed}
Assume $\mu_{\mathrm{down}}$ is nondegenerate in the sense that $X$ has a density on $\R^d$ and
$\Sigma_{\mathrm{down}}\succ0$. Fix any $M\in\mathcal U$ and any $n\ge k$.
Then, with probability one over $X_{1:n}$,
\[
\rank\!\Big(\Sigma_n(M)\Big)=k,
\qquad
\Sigma_n(M)\coloneqq \frac1n\sum_{i=1}^n \phi(X_i,M)\phi(X_i,M)^\top.
\]
Consequently, the empirical projector $\Pi_{M,n}$ acts as the identity on $\mathcal H_M$.
\end{lemma}

\begin{proof}
Write $\Phi\in\R^{n\times k}$ for the design matrix with rows $\phi(X_i,M)^\top$.
Then $\Sigma_n(M)=\frac1n\Phi^\top\Phi$ has rank $k$ if and only if $\Phi$ has rank $k$.
Since $\phi(X_i,M)=s(M)X_i$ and $s(M)$ has rank $k$, the random vector $\phi(X_i,M)$ has a density on $\R^k$
(because $X_i$ has a density on $\R^d$ and $s(M)$ is a surjective linear map $\R^d\to\R^k$).
For i.i.d.\ vectors in $\R^k$ with a density, the event that $k$ of them fall into a common proper hyperplane has probability $0$,
so $\Phi$ has rank $k$ almost surely when $n\ge k$.
The final claim is exactly the well-posedness implication used in Appendix~\ref{app:downstream-terms}.
\end{proof}

Since $M_m$ is independent of the downstream sample, Lemma~\ref{lem:linear:wellposed} applies conditionally on $M_m$
and yields the well-posedness requirement along the triangular array $(M_m,n)$ used in the master theorem.

Next, we show that the moment condition $\E[\|X\|^4] < \infty$ is in fact sufficient for~\Cref{ass:well-posedness}. 

\begin{proposition}[{\citet{oliveira2016lower}}]
Fix a $\delta \in (0, 1)$ and suppose $\norm{x}_{L^4} < \infty$.
Define the constants:
\begin{align}
    C_X \coloneq \sup_{v \in \mathbb{S}^{d-1}} \sqrt{\E[\ip{(\Sigma^+)^{1/2} x}{v}^4]}, \quad k \coloneq \rank(\Sigma).
\end{align}
Suppose that $n \geq c_0 C_X^2( k + \log(1/\delta) )$ for a universal $c_0$. Then with probability at least $1-\delta$:
\begin{align*}
    \Sigma_n \succcurlyeq \frac{1}{4} \Sigma.
\end{align*}
On this event, we also have that $\mathrm{Col}(\Sigma_n) = \mathrm{Col}(\Sigma)$.
\end{proposition}
\begin{proof}
The first part of the claim, that $\Sigma_n \succcurlyeq \frac{1}{4} \Sigma$, is immediate from \citet[Theorem 3.1]{oliveira2016lower}.
To finish, suppose that $\Sigma_n \succcurlyeq \frac{1}{4} \Sigma$ holds.
Now, let $q \in \mathrm{Kern}(\Sigma_n)$. By the above, this implies that
\begin{align*}
    0 = q^\T \Sigma_n q \geq \frac{1}{4} q^\T \Sigma q \geq 0,
\end{align*}
and hence $q^\T \Sigma q = 0$, which implies $q \in \mathrm{Kern}(\Sigma^{1/2}) = \mathrm{Kern}(\Sigma)$.
Therefore, $\mathrm{Kern}(\Sigma_n) \subseteq \mathrm{Kern}(\Sigma)$, which by \Cref{lem:emp-span-contained} implies
$\mathrm{Col}(\Sigma_n) = \mathrm{Col}(\Sigma)$.
\end{proof}

\subsubsection{Stable limiting span of null directions}

\begin{lemma}[Stable limiting span of null directions]
\label{lem:linear:stable-null-span}
Assume $\Sigma_{\mathrm{down}}\succ0$. Then the stable limiting span condition in Assumption~\ref{ass:stable-null-span} holds with limiting null-residual span equal to $\{0\}$.
\end{lemma}

\begin{proof}
For every $M\in\mathcal U$, \Cref{lem:linear:Sigma-stable} gives $\Sigma(M)\succ0$. Recall that $(T_M \theta)(x) =\theta^\top \phi(x,M)$. Thus, for any $\theta\neq 0$, we have
\[
\norm{T_M \theta}_{L^2}^2 = \E\big[(\theta^\top \phi(X,M))^2 \big]= \theta^\top \Sigma(M) \theta >0.
\]
Equivalently, the map $T_M:\R^k\to L^2(\mu_{\mathrm{down}})$ is injective.
Thus the population null space $N_M=\ker(T_M)$ is $\{0\}$ for every $M\in\mathcal U$, and in particular
$N_\star=\{0\}$.
The residual subspace generated by perturbing null directions is therefore identically zero:
\[
\mathcal B_v
=
\operatorname{Im}\!\left((I-\Pi_{\mathcal A_v})T_v\big|_{N_\star}\right)
=
\{0\}.
\]
Hence the required limiting span exists and equals $\{0\}$.
\end{proof}

\subsection{Pre-training consistency and manifold CLT for the linear spectral loss}
\label{app:proof-cor-linear:pretrain-clt}

This subsection verifies~\Cref{ass:riem_mestimation} for the linear spectral descriptor estimator $\hat M_m$.
\paragraph{Model and loss.}
A pre-training observation is \(z=(x,x^+,x^-)\in(\R^d)^3\), where \((x,x^+)\) is a positive pair and \(x^-\) is an
independent negative: \(x^-\) is an independent copy of \(x\), independent of \((x,x^+)\).
Assume \(\E[x]=0\) and define
\[
\Sigma_{\mathrm{pre}}\coloneqq \E[xx^\top],
\qquad
\Sigma_{\mathrm{pre}}^+\coloneqq \E[x^+(x)^\top].
\]
For \(M\in\mathcal M_{d,k}\), recall the per-sample loss (well-defined for all symmetric \(M\))
\[
\ell_{\mathrm{spec}}(M;z)= -2\,x^\top M x^+ + (x^\top M x^-)^2.
\]
We minimize the empirical loss over the rank-\(k\) PSD manifold \(\mathcal M_{d,k}=\{M\succeq 0:\rank(M)=k\}\):
\[
\hat L_m(M)\coloneqq \frac1m\sum_{j=1}^m \ell_{\mathrm{spec}}(M;z_j),
\qquad
\hat M_m\in\argmin_{M\in\mathcal M_{d,k}}\hat L_m(M).
\]
Let \(L(M)\coloneqq \E[\ell_{\mathrm{spec}}(M;z)]\) be the population loss.

\paragraph{Moment assumption (for LLN and CLT of derivatives).}
Assume there exists \(\delta>0\) such that
\begin{equation}
\label{eq:ass:spec-moments}
\E\|x\|^{8+\delta}<\infty,
\qquad
\E\big[\|x\|^{4+\delta}\|x^+\|^{4+\delta}\big]<\infty.
\end{equation}
This ensures integrability of the score and Hessian random fields used below.

Let \(\operatorname{sym}(A)\coloneqq (A+A^\top)/2\). Viewing \(\ell_{\mathrm{spec}}(\cdot;z)\) as a function on \(\mathcal M_{d,k}\)
with the Frobenius inner product, its Euclidean gradient is
\begin{equation}
\label{eq:spec-eucl-grad}
\nabla_M \ell_{\mathrm{spec}}(M;z)
=
-2\,\operatorname{sym}(x(x^+)^\top)
+
2\,(x^\top M x^-)\,\operatorname{sym}(x(x^-)^\top).
\end{equation}
Its Euclidean Hessian is the linear map \(H\mapsto D(\nabla_M \ell_{\mathrm{spec}})(M;z)[H]\) given by
\begin{equation}
\label{eq:spec-eucl-hess}
D(\nabla_M \ell_{\mathrm{spec}})(M;z)[H]
=
2\,(x^\top H x^-)\,\operatorname{sym}(x(x^-)^\top).
\end{equation}
Taking expectations and using \(x^-\) independent of \((x,x^+)\) with \(\E[x^- (x^-)^\top]=\Sigma_{\mathrm{pre}}\), we obtain
\begin{align}
\label{eq:spec-pop-grad-hess}
\nabla L(M)
&=
-2\,\Sigma_{\mathrm{pre}}^+
+
2\,\Sigma_{\mathrm{pre}}\,M\,\Sigma_{\mathrm{pre}},
\\
\label{eq:spec-pop-hess}
D(\nabla L)(M)[H]
&=
2\,\Sigma_{\mathrm{pre}}\,H\,\Sigma_{\mathrm{pre}}.
\end{align}

Let
\begin{equation}
\label{eq:spec-tangent-score}
\varphi(z)
\coloneqq
\operatorname{Proj}_{T_{M_\star}\mathcal M_{d,k}}
\big(\nabla_M \ell_{\mathrm{spec}}(M_\star;z)\big)
\in T_{M_\star}\mathcal M_{d,k},
\qquad
\Sigma_\star\coloneqq \operatorname{Cov}(\varphi(z)).
\end{equation}
Also define
\begin{equation}
\label{eq:spec-tangent-hess}
H_\star
\coloneqq
\operatorname{Hess} L(M_\star):T_{M_\star}\mathcal M_{d,k}\to T_{M_\star}\mathcal M_{d,k}.
\end{equation}

\begin{lemma}[Verification of Assumption~\ref{ass:riem_mestimation} for the linear spectral loss]
\label{lem:linear:D5}
Assume \(\Sigma_{\mathrm{pre}}\succ0\), \Cref{ass:C-positive-k}, and
\eqref{eq:ass:spec-moments}.  Then the empirical objective
\[
\hat L_m(M)
=
\frac1m\sum_{j=1}^m \ell_{\mathrm{spec}}(M;z_j)
\]
satisfies~\Cref{ass:riem_mestimation} on \(\mathcal M_{d,k}\) at \(M_\star\).
\end{lemma}

\begin{proof}
We verify the five parts of~\Cref{ass:riem_mestimation}.

\paragraph{(i) Identification and separation.}
By \Cref{lem:spec-loss-in-M}, the population objective satisfies
\[
L(M)
=
\left\|
\Sigma_{\mathrm{pre}}^{1/2}M\Sigma_{\mathrm{pre}}^{1/2}
-
C
\right\|_F^2
-
\|C\|_F^2.
\]
Under \Cref{ass:C-positive-k}, the rank-\(k\) positive truncation of \(C\) is
unique.  Hence the population minimizer \(M_\star\in\mathcal M_{d,k}\) is unique. Moreover, since \(\Sigma_{\mathrm{pre}}\succ0\), the objective is coercive on
\(\mathcal M_{d,k}\), that is, \(L(M)\to\infty\) whenever \(\|M\|_F\to\infty\) within
\(\mathcal M_{d,k}\).  Therefore, all sufficiently low sublevel sets are compact
after closure.  If the separation condition failed for some \(\epsilon>0\), then
there would exist \(M_j\in\mathcal M_{d,k}\) such that
\[
d_{\mathcal M}(M_j,M_\star)\ge\epsilon,
\qquad
L(M_j)\downarrow L(M_\star).
\]
By coercivity, a subsequence is bounded and has a limit point in the closure of
the rank-\(k\) PSD stratum.  The continuity of the squared-distance objective
would make this limit point another minimizer, contradicting the uniqueness of
the rank-\(k\) positive truncation under \Cref{ass:C-positive-k}.  Hence, for
every \(\epsilon>0\),
\[
\inf_{M\in\mathcal M_{d,k}:d_{\mathcal M}(M,M_\star)\ge\epsilon}
\big(L(M)-L(M_\star)\big)>0.
\]

\paragraph{(ii) Uniform LLN on a compact set and localization.}
Let \(U=\exp_{M_\star}(B(M_\star,\epsilon_0))\) be a normal neighborhood of
\(M_\star\), and fix \(\epsilon'\in(0,\epsilon_0)\).  Set
\[
K_{\epsilon'}
\coloneqq
\exp_{M_\star}\big(\overline B(M_\star,\epsilon')\big).
\]
By \citet[Lemma~2.4]{newey1994largesample}, it is enough to verify that the class
\[
\mathcal F
\coloneqq
\{\ell_{\mathrm{spec}}(M;\cdot):M\in K_{\epsilon'}\}
\]
is pointwise continuous in \(M\) and dominated by an integrable envelope.  The
continuity is immediate because \(M\mapsto \ell_{\mathrm{spec}}(M;z)\) is a
polynomial for every fixed \(z\).  Since \(K_{\epsilon'}\) is compact, there is
\(R_K<\infty\) such that \(\sup_{M\in K_{\epsilon'}}\|M\|_{\mathrm{op}}\le R_K\).
Therefore, for all \(M\in K_{\epsilon'}\),
\[
|\ell_{\mathrm{spec}}(M;z)|
\le
2R_K\|x\|\,\|x^+\|
+
R_K^2\|x\|^2\|x^-\|^2.
\]
The right-hand side is integrable under \eqref{eq:ass:spec-moments}, using that
\(x^-\) is an independent copy of \(x\).  Hence
\begin{align}\label{eq:linear-LLN-loss}
 \sup_{M\in K_{\epsilon'}}
|\hat L_m(M)-L(M)|
\xrightarrow{\mathbb P}0.   
\end{align}

Equation~\ref{eq:linear-LLN-loss} with the separation in
part (i), gives the argmin consistency:
\[
\hat M_m\xrightarrow{\mathbb P} M_\star.
\]
Consequently, after fixing \(\epsilon'\in(0,\epsilon_0)\),
\[
\Pr\!\left(
\hat M_m\in \exp_{M_\star}\big(\overline B(M_\star,\epsilon')\big)
\right)\to 1.
\]

\paragraph{(iii) Local \(C^2\) smoothness and score moments.}
For every \(z=(x,x^+,x^-)\), the map
\(M\mapsto \ell_{\mathrm{spec}}(M;z)\) is a polynomial in \(M\).  Hence, in any
normal coordinate chart around \(M_\star\), it is \(C^2\).  The first two ambient
derivatives are given by \eqref{eq:spec-eucl-grad} and \eqref{eq:spec-eucl-hess}.

Since the Riemannian gradient is the tangent projection of the Euclidean
gradient,
\[
\|\operatorname{grad}\ell_{\mathrm{spec}}(M_\star;z)\|_{M_\star}
\le
\|\nabla_M\ell_{\mathrm{spec}}(M_\star;z)\|_F.
\]
Using \eqref{eq:spec-eucl-grad},
\[
\|\nabla_M\ell_{\mathrm{spec}}(M_\star;z)\|_F
\le
C\Big(
\|x\|\|x^+\|
+
\|x\|^2\|x^-\|^2
\Big),
\]
where \(C\) depends only on \(M_\star\).  The moment condition
\eqref{eq:ass:spec-moments}, together with independence of \(x^-\), implies
\[
\E\big[
\|\operatorname{grad}\ell_{\mathrm{spec}}(M_\star;Z)\|_{M_\star}^2
\big]
<\infty.
\]
Thus, the score moment condition in Assumption~\ref{ass:riem_mestimation}(iii) holds.

\paragraph{(iv) Nondegenerate minimizer.}
By \eqref{eq:Lspec-square}, the population objective is 
\[
L_{\mathrm{spec}}(w)=\big\| \Sigma_{\mathrm{pre}}^{1/2} M \Sigma_{\mathrm{pre}}^{1/2} - C \big\|_F^2
- \| C \|_F^2,
\]
restricted to $\mathcal{M}_{d,k}$.  The map $M\mapsto \Sigma_{\mathrm{pre}}^{1/2}M\Sigma_{\mathrm{pre}}^{1/2}$ is a local diffeomorphism on \(\mathcal M_{d,k}\).  Under
\Cref{ass:C-positive-k}, the rank-\(k\) positive truncation is separated by an
eigengap.  Therefore the restricted squared-distance objective has a
nondegenerate strict local minimum at \(M_\star\).  Equivalently, the Riemannian
Hessian
\[
H_\star
=
\operatorname{Hess}L(M_\star):
T_{M_\star}\mathcal M_{d,k}\to T_{M_\star}\mathcal M_{d,k}
\]
is positive definite, and in particular invertible.

\paragraph{(v) Uniform transported Hessian convergence.}
Let
\[
K_{\epsilon'}
=
\exp_{M_\star}\big(\overline B(M_\star,\epsilon')\big)
\]
as above.  In normal coordinates, the pulled-back loss $v\mapsto \ell_{\mathrm{spec}}(\exp_{M_\star}(v);z)$ has second derivatives that are continuous in \(v\).  Since \(K_{\epsilon'}\) is
compact and the exponential map and its derivatives are bounded on
\(\overline B(M_\star,\epsilon')\), these second derivatives admit an integrable
envelope of the form
\[
C_K\Big(
\|x\|\|x^+\|
+
\|x\|^2\|x^-\|^2
\Big),
\]
again by \eqref{eq:spec-eucl-grad}--\eqref{eq:spec-eucl-hess}.
Therefore, by \citet[Lemma~2.4]{newey1994largesample},
\[
\sup_{M\in K_{\epsilon'}}
\big\|
\widetilde H_m(M)-\widetilde H(M)
\big\|
\xrightarrow{\ \mathbb P\ } 0,
\]
where
\[
\widetilde H_m(M)
=
\mathcal P_{M\to M_\star}\circ
\operatorname{Hess}\hat L_m(M)\circ
\mathcal P_{M_\star\to M},
\qquad
\widetilde H(M)
=
\mathcal P_{M\to M_\star}\circ
\operatorname{Hess}L(M)\circ
\mathcal P_{M_\star\to M}.
\]
This is precisely Assumption~\ref{ass:riem_mestimation}(v).

Combining parts (i)--(v), Assumption~\ref{ass:riem_mestimation} holds.
\end{proof}

\begin{proposition}[Descriptor CLT for the linear spectral estimator]
\label{prop:spec-manifold-clt}
Assume \(\Sigma_{\mathrm{pre}}\succ0\), \Cref{ass:C-positive-k}, and
\eqref{eq:ass:spec-moments}.  Let \(\hat M_m\) be a measurable empirical
minimizer of \(\hat L_m\) over \(\mathcal M_{d,k}\).  Then
\[
\sqrt m\,\log_{M_\star}(\hat M_m)
\distconv
Z,
\qquad
Z\sim\mathcal N(0,V),
\quad
V\coloneqq H_\star^{-1}\Sigma_\star H_\star^{-1},
\]
as a random element in \(T_{M_\star}\mathcal M_{d,k}\).
\end{proposition}

\begin{proof}
By \Cref{lem:linear:D5}, Assumption~\ref{ass:riem_mestimation} holds for the linear spectral empirical
objective on \(\mathcal M_{d,k}\).  The abstract Riemannian \(M\)-estimation CLT
therefore gives
\[
\sqrt m\,\log_{M_\star}(\hat M_m)
\distconv
\mathcal N(0,H_\star^{-1}\Sigma_\star H_\star^{-1}),
\]
where \(\Sigma_\star=\operatorname{Cov}(\varphi(Z))\) and \(H_\star\) is the
Riemannian Hessian of \(L\) at \(M_\star\).
\end{proof}

\subsection{Proof of \Cref{cor:linear-spectral-model}}
\label{app:proof-cor-linear:conclusion}

\begin{proof}[Proof of~\Cref{cor:linear-spectral-model}]
In Lemmas~\ref{lem:linear:moments}--\ref{lem:linear:stable-null-span}, we proved that all non-CLT assumptions of Theorem~\ref{thm:master-compatible} hold for the quotient feature map
\(\phi(x,M)=s(M)x\). By \Cref{prop:spec-manifold-clt}, the pre-training CLT assumption~\ref{ass:riem_mestimation} holds
for the descriptor estimator \(\hat M_m\). Therefore, all assumptions of Theorem~\ref{thm:master-compatible} are satisfied.
\end{proof}

\subsection{Explicit calculations for a concrete example}
\label{app:toy-scaling}

In this section, we first derive an asymptotic distribution for the spectral pre-training estimator in a general Gaussian model by computing the closed-form characterization of the limiting covariance operator $V_\star=H_\star^{-1}\Sigma_\star H_\star^{-1}$ in Proposition~\ref{prop:spec-manifold-clt}. We then specialize the result to the concrete diagonal example
presented in~\Cref{sec:example_linear_spectral}.
\paragraph{Fully Gaussian assumption.}
Assume $(X,X^+)$ is jointly Gaussian, and $X^-$ is an independent copy of $X$,
independent of $(X,X^+)$. In particular,
\[
\E[XX^\top]=\Sigma_{\mathrm{pre}},
\qquad
\E[X(X^+)^\top]=\Sigma_{\mathrm{pre}}^+.
\]
\paragraph{Bilinear form for the score covariance.}
Let $\mathcal{S}^k$ be the space of $d$ by $d$ symmetric matrices. For a symmetric direction $H\in\mathcal S^d$, define the scalar score functional at $M_\star$ by
\[
S_H(Z)
\coloneqq
\langle \nabla_M \ell_{\mathrm{spec}}(M_\star;Z), H\rangle
=
-2\,X^\top H X^+
+
2\,(X^\top M_\star X^-)\,(X^\top H X^-).
\]
where $Z=(X,X^+,X^-)$. Recall that $\Sigma_\star=\operatorname{Cov}(\varphi(Z))$ with $\varphi(Z)=\operatorname{Proj}_T(\nabla_M\ell_{\mathrm{spec}}(M_\star;Z))$ where $\operatorname{Proj}_T$ is the orthogonal projection onto $T_{M_\star}\mathcal M_{d,k}$. We know that $\operatorname{Proj}_T$ is self-adjoint and $\E[\varphi(Z)]=0$. Thus, for all $v_1,v_2\in T_{M_\star}\mathcal M_{d,k}$ we have
\begin{equation}
\label{eq:Sigma-star-bilin}
\langle v_1,\Sigma_\star v_2\rangle
=
\operatorname{Cov}\!\big(S_{v_1}(Z),S_{v_2}(Z)\big).
\end{equation}
The following proposition gives $\operatorname{Cov}(S_{H_1},S_{H_2})$ in closed form under joint Gaussianity.

\begin{proposition}[Exact score covariance under joint Gaussianity]
\label{prop:spec-gauss-general-score-cov}
Under the fully Gaussian assumption, for any $H_1,H_2\in\mathcal S^d$,
\begin{align}\label{eq:score-cov-gaussian}
\operatorname{Cov}\!\big(S_{H_1}(Z),S_{H_2}(Z)\big)
=
4\Big(
C_{aa}(H_1,H_2)
-
C_{ab}(H_1,H_2)
-
C_{ab}(H_2,H_1)
+
C_{bb}(H_1,H_2)
\Big),   
\end{align}

where the terms are given as follows:
\begin{enumerate}[label=(\roman*)]
    \item Positive-pair term: 
\begin{align}
\label{eq:Caa}
C_{aa}(H_1,H_2)
&\coloneqq
\tr (H_1\Sigma_{\mathrm{pre}}H_2\Sigma_{\mathrm{pre}})
+
\tr (H_1\Sigma_{\mathrm{pre}}^+H_2\Sigma_{\mathrm{pre}}^+).
\end{align}

\item Cross term: Let $P_i\coloneqq M_\star \Sigma_{\mathrm{pre}} H_i$. Then
\begin{align}
\label{eq:Cab}
C_{ab}(H_1,H_2)
&\coloneqq
\Tr(P_2\Sigma_{\mathrm{pre}}H_1\Sigma_{\mathrm{pre}}^+)
+
\Tr(P_2\Sigma_{\mathrm{pre}}^+H_1\Sigma_{\mathrm{pre}}).
\end{align}
\item Negative-sample term:
Let $Q\coloneqq M_\star \Sigma_{\mathrm{pre}} M_\star$.
Then
\begin{align}
\label{eq:Cbb}
C_{bb}(H_1,H_2) &= \tr(P_1 \Sigma_\mathrm{pre})\tr(P_2 \Sigma_\mathrm{pre})+2\tr(P_1 \Sigma_\mathrm{pre} P_2 \Sigma_\mathrm{pre})+\tr(H_1\Sigma_\mathrm{pre}H_2\Sigma_\mathrm{pre})\tr(Q\Sigma_{\mathrm{pre}})\notag\\&\quad+4\Tr(H_1\Sigma_\mathrm{pre}H_2\Sigma_\mathrm{pre}Q\Sigma_\mathrm{pre}).
\end{align}
\end{enumerate}
\end{proposition}

We will repeatedly use the following standard lemmas in the proof.

\begin{lemma}[Bilinear--bilinear moment]
\label{lem:gauss-bilin-bilin}
Let $(U,V)$ be jointly Gaussian with $\E[U]=\E[V]=0$,
$\E[UU^\top]=\Sigma_U$, $\E[VV^\top]=\Sigma_V$, and $\E[UV^\top]=\Sigma_{UV}$.
Then for any $A,B\in\R^{d\times d}$,
\begin{align}
\label{eq:bilin-bilin}
\E\big[(U^\top A V)(U^\top B V)\big]
&=
\Tr(A\Sigma_{UV})\Tr(B\Sigma_{UV})
+
\Tr(A\Sigma_U B\Sigma_V)
+
\Tr(A\Sigma_{UV}B^\top \Sigma_{UV}^\top).
\end{align}
Consequently,
\begin{align}
\label{eq:bilin-bilin-cov}
\Cov(U^\top A V,U^\top B V)
&=
\Tr(A\Sigma_U B\Sigma_V)
+
\Tr(A\Sigma_{UV}B^\top \Sigma_{UV}^\top).
\end{align}
\end{lemma}

\begin{lemma}[Quadratic--quadratic moment]
\label{lem:gauss-quad-quad}
Let $U\sim\mathcal N(0,\Sigma)$ and let $A,B\in\R^{d\times d}$.
Then
\begin{align}
\label{eq:quad-quad}
\E\big[(U^\top A U)(U^\top B U)\big]
&=
\Tr(A\Sigma)\Tr(B\Sigma)
+
\Tr(A\Sigma B\Sigma)
+
\Tr(A\Sigma B^\top\Sigma).
\end{align}
Consequently,
\begin{align}
\label{eq:quad-quad-cov}
\Cov(U^\top A U,U^\top B U)
&=
\Tr(A\Sigma B\Sigma)
+
\Tr(A\Sigma B^\top\Sigma).
\end{align}
\end{lemma}

\begin{proof}[Proof of Lemmas~\ref{lem:gauss-bilin-bilin}--\ref{lem:gauss-quad-quad}]
Both identities follow by expanding the products componentwise and applying Isserlis' formula to fourth moments.
For instance, for Lemma~\ref{lem:gauss-bilin-bilin},
\[
\E[(U^\top A V)(U^\top B V)]
=
\sum_{i,j,p,q} A_{ij}B_{pq}\,\E[U_i V_j U_p V_q],
\]
and Isserlis gives
$\E[U_i V_j U_p V_q]=\E[U_i V_j]\E[U_p V_q]+\E[U_i U_p]\E[V_j V_q]+\E[U_i V_q]\E[V_j U_p]$,
which sums to \eqref{eq:bilin-bilin}. Lemma~\ref{lem:gauss-quad-quad} is analogous.
\end{proof}

\begin{proof}[Proof of~\Cref{prop:spec-gauss-general-score-cov}]
Write $S_H(Z)=-2a_H(Z)+2b_H(Z)$ with
\[
a_H(Z)\coloneqq X^\top H X^+,
\qquad
b_H(Z)\coloneqq (X^\top M_\star X^-)(X^\top H X^-).
\]
Then, for any $H_1,H_2\in\mathcal S^d$,
\begin{align}
\label{eq:cov-expand}
\Cov(S_{H_1}(Z),S_{H_2}(Z))
&=
4\Big(
\Cov(a_{H_1},a_{H_2})
-\Cov(a_{H_1},b_{H_2})
-\Cov(a_{H_2},b_{H_1})
+\Cov(b_{H_1},b_{H_2})
\Big).
\end{align}
We compute the four covariance terms under joint Gaussianity using Isserlis' formula.

First, we will compute $\Cov(a_{H_1},a_{H_2})$. Apply Lemma~\ref{lem:gauss-bilin-bilin} with $(U,V)=(X,X^+)$,
$\Sigma_U=\Sigma_V=\Sigma_{\mathrm{pre}}$ and $\Sigma_{UV}=\Sigma_{\mathrm{pre}}^+$.
Since $H_1,H_2$ are symmetric and $\Sigma_{\mathrm{pre}}^+$ is symmetric,
\eqref{eq:bilin-bilin-cov} yields
\[
C_{aa}(H_1,H_2)=\Cov(a_{H_1},a_{H_2})
=
\Tr(H_1\Sigma_{\mathrm{pre}}H_2\Sigma_{\mathrm{pre}})
+
\Tr(H_1\Sigma_{\mathrm{pre}}^+H_2\Sigma_{\mathrm{pre}}^+).
\]

Next, we will compute the cross term $\Cov(a_{H_1},b_{H_2})$. Fix $H_1,H_2$ and write
\[
b_{H_2}
=
(X^\top M_\star X^-)(X^\top H_2 X^-)
=
(X^-)^\top(M_\star X X^\top H_2)X^-.
\]
Condition on $(X,X^+)$ and use that $X^-$ is independent of $(X,X^+)$ with covariance $\Sigma_{\mathrm{pre}}$:
\begin{align}
\label{eq:cond-b}
\E[b_{H_2}\mid X,X^+]
&=
\E\big[(X^\top M_\star X^-)(X^\top H_2 X^-)\mid X,X^+\big]
=
X^\top M_\star \Sigma_{\mathrm{pre}} H_2 X.
\end{align}
Hence,
\begin{align}
\label{eq:Eab}
\E[a_{H_1}b_{H_2}]
&=
\E\Big[(X^\top H_1 X^+)\, \E[b_{H_2}\mid X,X^+]\Big]
=
\E\Big[(X^\top H_1 X^+)\,(X^\top P_2 X)\Big],
\end{align}
where $B_2\coloneqq M_\star\Sigma_{\mathrm{pre}}H_2$.

We next compute the mixed moment in \eqref{eq:Eab}. Expand
\[
(X^\top H_1 X^+)(X^\top P_2 X)
=
\sum_{i,j,p,q} (H_1)_{ij}(P_2)_{pq}\,X_i X^+_j X_p X_q.
\]
Apply Isserlis to $\E[X_i X^+_j X_p X_q]$ for the jointly Gaussian pair $(X,X^+)$:
\[
\E[X_i X^+_j X_p X_q]
=
\E[X_i X^+_j]\E[X_p X_q]
+
\E[X_i X_p]\E[X^+_j X_q]
+
\E[X_i X_q]\E[X^+_j X_p].
\]
Using $\E[X_i X^+_j]=(\Sigma_{\mathrm{pre}}^+)_{ij}$ and $\E[X_p X_q]=(\Sigma_{\mathrm{pre}})_{pq}$,
this yields
\begin{align}
\label{eq:bilin-quad}
\E\big[(X^\top H_1 X^+)(X^\top P_2 X)\big]
&=
\Tr(H_1\Sigma_{\mathrm{pre}}^+)\Tr(P_2\Sigma_{\mathrm{pre}})
+
\Tr(P_2\Sigma_{\mathrm{pre}}H_1\Sigma_{\mathrm{pre}}^+)
+
\Tr(P_2\Sigma_{\mathrm{pre}}^+H_1\Sigma_{\mathrm{pre}}).
\end{align}
Moreover,
\[
\E[a_{H_1}]=\Tr(H_1\Sigma_{\mathrm{pre}}^+),
\qquad
\E[b_{H_2}]
=
\E[X^\top P_2 X]
=
\Tr(P_2\Sigma_{\mathrm{pre}}).
\]
Therefore, subtracting $\E[a_{H_1}]\E[b_{H_2}]$ from \eqref{eq:bilin-quad} gives
\[
C_{ab}(H_1,H_2) = \Cov(a_{H_1},b_{H_2})
=
\Tr(P_2\Sigma_{\mathrm{pre}}H_1\Sigma_{\mathrm{pre}}^+)
+
\Tr(P_2\Sigma_{\mathrm{pre}}^+H_1\Sigma_{\mathrm{pre}}).
\]

Finally, we will compute the negative-sample term $\Cov(b_{H_1},b_{H_2})$. Write
\[
b_{H_i}=(X^\top M_\star X^-)(X^\top H_i X^-)
=
(X^-)^\top A_i(X) X^-,
\qquad
A_i(X)\coloneqq M_\star X X^\top H_i.
\]
From the total law of covariance, we have
\[
C_{bb}(H_1,H_2)=\Cov(b_{H_1},b_{H_2})=\E[\Cov(b_{H_1},b_{H_2}|X)]+\Cov(\E[b_{H_1}|X],\E[b_{H_2}|X]).
\]
We apply Lemma~\ref{lem:gauss-quad-quad} to $X^-\sim\mathcal N(0,\Sigma_{\mathrm{pre}})$:
\begin{align*}
\Cov(b_{H_1},b_{H_2}\mid X)
&=
\Tr(A_1(X)\Sigma_\mathrm{pre}A_2(X)\Sigma_\mathrm{pre})+\Tr(A_1(X)\Sigma_\mathrm{pre}A_2(X)^T\Sigma_\mathrm{pre})\notag\\
&=(X^\T M_\star \Sigma_\mathrm{pre}H_1 X X^\T M_\star \Sigma_\mathrm{pre}H_2 X) + (X^\T H_1 \Sigma_\mathrm{pre}H_2 X X^\T M_\star \Sigma_\mathrm{pre} M_\star X)\\
&=(X^\T P_1 X X^\T P_2 X) + (X^\T H_1 \Sigma_\mathrm{pre}H_2 X X^\T Q X).
\end{align*}
Therefore, we have
\begin{align*}
    \E[\Cov(b_{H_1},b_{H_2}|X)] &= \tr(P_1 \Sigma_\mathrm{pre})\tr(P_2 \Sigma_\mathrm{pre})+\tr(P_1 \Sigma_\mathrm{pre} P_2 \Sigma_\mathrm{pre})+\tr(H_1\Sigma_\mathrm{pre}H_2\Sigma_\mathrm{pre})\tr(Q\Sigma_{\mathrm{pre}})\\&\quad+3\Tr(H_1\Sigma_\mathrm{pre}H_2\Sigma_\mathrm{pre}Q\Sigma_\mathrm{pre}).
\end{align*}
Similarly,
\begin{align*}
    \Cov(\E[b_{H_1}|X],\E[b_{H_2}|X]) = \Cov(X^\T P_1 X,X^\T P_2 X) = \tr(P_1\Sigma_\mathrm{pre}P_2\Sigma_\mathrm{pre})+\tr(H_1\Sigma_\mathrm{pre}H_2\Sigma_\mathrm{pre}Q\Sigma_\mathrm{pre}).
\end{align*}
Therefore, we get
\begin{align*}
    C_{bb}(H_1,H_2) &= \tr(P_1 \Sigma_\mathrm{pre})\tr(P_2 \Sigma_\mathrm{pre})+2\tr(P_1 \Sigma_\mathrm{pre} P_2 \Sigma_\mathrm{pre})+\tr(H_1\Sigma_\mathrm{pre}H_2\Sigma_\mathrm{pre})\tr(Q\Sigma_{\mathrm{pre}})\\&\quad+4\Tr(H_1\Sigma_\mathrm{pre}H_2\Sigma_\mathrm{pre}Q\Sigma_\mathrm{pre}).
\end{align*}
\end{proof}

\begin{corollary}
    Under the fully Gaussian assumption, for any $H_1,H_2\in \mathcal S^d$, we have
    \begin{align}\label{eq:full-gaussian-cov}
        \langle H_1,V_\star H_2\rangle = C_{aa}(H'_1,H'_2)-C_{ab}(H'_1,H'_2)-C_{ab}(H'_2,H'_1)+C_{bb}(H'_1,H'_2)
    \end{align}
    where $H_i'=\Sigma_\mathrm{pre}^{-1}H_i \Sigma_\mathrm{pre}^{-1}$ for $i=\{1,2\}$.
\end{corollary}
\begin{proof}
    For any $v\in\mathcal{S}^d$, we have
    \[
    H_\star v= 2\Sigma_\mathrm{pre} v \Sigma_\mathrm{pre}, \qquad H_\star^{-1} v= \frac12\Sigma_\mathrm{pre}^{-1} v \Sigma_\mathrm{pre}^{-1}.
    \]
    Therefore, we have
    \begin{align*}
        \langle H_1,V_\star H_2\rangle &=  \langle H_1,H_\star^{-1}\Sigma_\star H_\star^{-1} H_2\rangle = \langle H_\star^{-1}  H_1,\Sigma_\star H_\star^{-1} H_2\rangle\\ &= \frac{1}{4} \langle \Sigma_\mathrm{pre}^{-1}H_1 \Sigma_\mathrm{pre}^{-1} ,\Sigma_\star H_\star^{-1} \Sigma_\mathrm{pre}^{-1}H_2 \Sigma_\mathrm{pre}^{-1}\rangle = \frac{1}{4} \langle H_1',\Sigma_\star H_2'\rangle  = \frac{1}{4}\operatorname{Cov}\!\big(S_{H'_1}(Z),S_{H'_2}(Z)\big).
    \end{align*}
    \Cref{eq:full-gaussian-cov} follows after plugging this equation into~\Cref{eq:score-cov-gaussian}.
\end{proof}

\paragraph{Linearization of the representation error.}

Fix $M$ in a neighborhood of $M_\star$ and let $A:=s(M)\in\mathbb R^{k\times d}$ be the local
section so that $A^\top A = M$. Consider the operator
$T_M:\mathbb R^k \to L^2(\mu_{\mathrm{down}})$ defined by
\[
(T_M b)(x) := \langle b, Ax\rangle = b^\top A x ,
\qquad b\in\mathbb R^k .
\]
Its adjoint $T_M^\mathrm{adj} : L^2(\mu_{\mathrm{down}})\to \mathbb R^k$ satisfies, for any
$g\in L^2(\mu_{\mathrm{down}})$,
\[
T_M^\mathrm{adj} g
=
\mathbb E\big[g(X)\,AX\big]
=
A\,\mathbb E\big[g(X)\,X\big],
\]
where $X\sim \mu_{\mathrm{down}}$ and we used linearity of $A$.

Let
\[
\Sigma_{\mathrm{down}}:=\mathbb E[XX^\top],
\qquad
\Sigma(M):=T_M^\mathrm{adj} T_M = \mathbb E[(AX)(AX)^\top] = A\Sigma_{\mathrm{down}}A^\top .
\]
The orthogonal projector onto $\mathrm{Range}(T_M)$ is given by
\[
\Pi_M = T_M\,\Sigma(M)^+\,T_M^\mathrm{adj}.
\]
Therefore, for any $g\in L^2(\mu_{\mathrm{down}})$,
\begin{align}
(\Pi_M g)(x)
&=
\big(T_M \Sigma(M)^+ T_M^\mathrm{adj} g\big)(x)=\left\langle \Sigma(M)^+\,T_M^\mathrm{adj} g,\ Ax \right\rangle \notag\\
&=
\left\langle \Sigma(M)^+\,A\,\mathbb E[g(X)X],\ Ax \right\rangle =
x^\top A^\top \Sigma(M)^+ A\,\mathbb E[g(X)X].
\label{eq:PiM-g}
\end{align}
Define the induced matrix
\[
P(M):=A^\top \Sigma(M)^+ A \in \mathbb R^{d\times d}.
\]
Then \eqref{eq:PiM-g} becomes the compact form
\[
(\Pi_M g)(x) = x^\top P(M)\,\mathbb E[g(X)X].
\]

We know that the target function is $f_\star(x)=\beta_\star^\top A_\star x$ for some $\beta_\star\in\mathbb R^k$
and $A_\star:=s(M_\star)$. Then
\[
\mathbb E[f_\star(X)X]
=
\mathbb E[XX^\top]A_\star^\top \beta_\star
=
\Sigma_{\mathrm{down}}A_\star^\top \beta_\star,
\]
and therefore
\begin{equation}
(\Pi_M f_\star)(x)
=
x^\top P(M)\,\Sigma_{\mathrm{down}}\,A_\star^\top \beta_\star.
\label{eq:PiM-fstar}
\end{equation}

We linearize \eqref{eq:PiM-fstar} around $M_\star$ using
\Cref{prop:proj-diff}.  In the constant-rank region, the map
$\Omega\mapsto \Pi_\Omega$ is Fr\'echet differentiable at $\Omega_\star$,
and for $v\in T_{\Omega_\star}\mathcal M$,
\[
\Pi_{\exp_{\Omega_\star}(v)}
=
\Pi_{\Omega_\star}
+
D\Pi_{\Omega_\star}[v]
+
o(\|v\|).
\]
Define, for any
$M\in\mathcal S_+^d$ with $\rank(M)=k$, the \emph{whitened descriptor}
\[
B(M):=\Sigma_{\mathrm{down}}^{1/2}\,M\,\Sigma_{\mathrm{down}}^{1/2}\in\mathcal S_+^d .
\]
Let $\Pi(B)$ denote the Euclidean orthogonal projector onto $\mathrm{range}(B)$. Then the induced matrix in \eqref{eq:PiM-g} admits the
$M$--only representation
\begin{equation}
P(M)
=
\Sigma_{\mathrm{down}}^{-1/2}\,\Pi\!\big(B(M)\big)\,\Sigma_{\mathrm{down}}^{-1/2}
=
\Sigma_{\mathrm{down}}^{-1/2}\,B(M)\,B(M)^+\,\Sigma_{\mathrm{down}}^{-1/2}.
\label{eq:P-M-only}
\end{equation}

For the fixed target $f_\star$, define $m_\star:=\E[f_\star(X)X]\in\R^d$.  Since $m_\star$ is
independent of $M$, we have for any tangent direction $v\in T_{M_\star}\mathcal M_{d,k}$,
\begin{align}    
D\Pi_{M_\star}[v]\,f_\star(x)
&=
x^\top DP(M_\star)[v]\;m_\star,
\notag\\
(\mathcal L(v)f_\star)(x)&:=-D\Pi_{M_\star}[v]\,f_\star(x)=
-\,x^\top DP(M_\star)[v]\;m_\star .
\label{eq:L-via-DP-Monly}
\end{align}

\paragraph{Derivative of $P(M)$.}
Let $B_\star:=B(M_\star)=\Sigma_{\mathrm{down}}^{1/2}M_\star\Sigma_{\mathrm{down}}^{1/2}$ and
$\dot B:=DB(M_\star)[v]=\Sigma_{\mathrm{down}}^{1/2}v\Sigma_{\mathrm{down}}^{1/2}$.
Differentiating \eqref{eq:P-M-only} yields
\begin{equation}
DP(M_\star)[v]
=
\Sigma_{\mathrm{down}}^{-1/2}\,D\Pi(B_\star)[\dot B]\;\Sigma_{\mathrm{down}}^{-1/2}.
\label{eq:DP-Monly}
\end{equation}

Assume an eigengap at $k$ for $B_\star$, so that $\Pi(\cdot)$ is Fr\'echet differentiable at
$B_\star$.  Let $B_\star=U\mathrm{diag}(\Lambda_1,\Lambda_2)U^\top$ with
$U=[U_1,U_2]$, $\Lambda_1=\mathrm{diag}(\lambda_1,\ldots,\lambda_k)$,
$\Lambda_2=\mathrm{diag}(\lambda_{k+1},\ldots,\lambda_d)$, and
$\min_{i\le k<j}|\lambda_i-\lambda_j|>0$. 

\begin{lemma}
\label{lem:spectral-projector-diff}
For any symmetric direction
$\dot B\in\mathcal S^d$,
\begin{equation}
D\Pi(B_\star)[\dot B]
=
U_2\,(G\oslash \Delta)\,U_1^\top
+
U_1\,\big(G^\top\oslash \Delta^\top\big)\,U_2^\top,
\label{eq:DProj-gap-proof}
\end{equation}
where
\[
G:=U_2^\top \dot B\,U_1\in\R^{(d-k)\times k},
\qquad
\Delta_{ai}:=\lambda_i-\lambda_{k+a},
\]
and $\oslash$ is the elementwise division.
\end{lemma}

\begin{proof}
We work in the eigenbasis of $B_\star$, i.e.\ replace $\dot B$ by
$\widetilde{\dot B}:=U^\top \dot B\,U$ and write the projector in this basis.
Since $\Pi(B)$ is orthogonally equivariant, it suffices to prove the claim for
$B_\star=\mathrm{diag}(\Lambda_1,\Lambda_2)$, in which case $\Pi(B_\star)=P_\star:=\mathrm{diag}(I_k,0)$.

Let $\Gamma$ be a positively oriented contour in the complex plane enclosing
$\{\lambda_1,\ldots,\lambda_k\}$ and excluding $\{\lambda_{k+1},\ldots,\lambda_d\}$.
By the Riesz projector formula,
\begin{equation}
\Pi(B)=\frac{1}{2\pi i}\oint_\Gamma (zI-B)^{-1}\,dz
\label{eq:Riesz}
\end{equation}
for all $B$ in a neighborhood of $B_\star$. Differentiating \eqref{eq:Riesz} at $B_\star$
in direction $\dot B$ and using the standard resolvent identity
\[
D\big((zI-B)^{-1}\big)\Big|_{B=B_\star}[\dot B]
=
(zI-B_\star)^{-1}\,\dot B\,(zI-B_\star)^{-1},
\]
we obtain
\begin{equation}
D\Pi(B_\star)[\dot B]
=
\frac{1}{2\pi i}\oint_\Gamma
(zI-B_\star)^{-1}\,\dot B\,(zI-B_\star)^{-1}\,dz.
\label{eq:DProj-Riesz}
\end{equation}

Since $B_\star=\mathrm{diag}(\lambda_1,\ldots,\lambda_d)$ is diagonal in this basis,
$(zI-B_\star)^{-1}$ is also diagonal with entries $(z-\lambda_j)^{-1}$. Therefore,
the $(p,q)$ entry of the integrand in \eqref{eq:DProj-Riesz} equals
\[
\big[(zI-B_\star)^{-1}\,\dot B\,(zI-B_\star)^{-1}\big]_{pq}
=
\frac{\dot B_{pq}}{(z-\lambda_p)(z-\lambda_q)}.
\]
Hence
\begin{equation}
\big[D\Pi(B_\star)[\dot B]\big]_{pq}
=
\dot B_{pq}\cdot \frac{1}{2\pi i}\oint_\Gamma \frac{dz}{(z-\lambda_p)(z-\lambda_q)}.
\label{eq:entrywise}
\end{equation}

Now consider cases.
\begin{enumerate}[label=(\roman*)]
    \item $p,q\le k$: If both $\lambda_p$ and $\lambda_q$ lie inside of the contour $\Gamma$, then we have
    \begin{align*}
        I_{pq}&\coloneqq \frac{1}{2\pi i}\oint_\Gamma \frac{dz}{(z-\lambda_p)(z-\lambda_q)} = \frac{1}{2\pi i}\oint_\Gamma \frac{1}{\lambda_p - \lambda_q}(\frac{1}{z-\lambda_p}-\frac{1}{z-\lambda_q})dz\\&= \frac{1}{2\pi i(\lambda_p - \lambda_q)}(1-1)=0.
    \end{align*}
    \item $p,q>k$: If both $\lambda_p$ and $\lambda_q$ lie outside of the contour $\Gamma$, then the integrand in \eqref{eq:entrywise} is holomorphic inside $\Gamma$ and the contour integral vanishes.
    \item $p\le k<q$: Then $\lambda_p$ is inside $\Gamma$ and $\lambda_q$ is outside, so the integrand in \eqref{eq:entrywise} has a single pole inside $\Gamma$ at $z=\lambda_p$ with residue $1/(\lambda_p-\lambda_q)$. Therefore
\[
\frac{1}{2\pi i}\oint_\Gamma \frac{dz}{(z-\lambda_p)(z-\lambda_q)}
=
\frac{1}{\lambda_p-\lambda_q},
\]
and thus
\[
\big[D\Pi(B_\star)[\dot B]\big]_{pq}
=
\frac{\dot B_{pq}}{\lambda_p-\lambda_q}.
\]
By symmetry, the case $q\le k<p$ gives the transpose.
\end{enumerate}

Writing this in block form with respect to the split $\R^d=\R^k\oplus\R^{d-k}$ yields
\[
D\Pi(B_\star)[\dot B]
=
\begin{pmatrix}
0 & H^\top\\
H & 0
\end{pmatrix},
\qquad
H_{ai}=\frac{\dot B_{k+a,i}}{\lambda_i-\lambda_{k+a}}.
\]
Returning to the original basis (undoing the $U$-conjugation) gives \eqref{eq:DProj-gap-proof}
with $G:=U_2^\top \dot B\,U_1$ and $\Delta_{ai}:=\lambda_i-\lambda_{k+a}$.
\end{proof}

\paragraph{pre-training fluctuations.}
We will derive the asymptotic distribution of $m\norm{\mathcal{L}(\mathrm{log}_{M_\star}(\hat M_m)}$.
\begin{align}\label{eq:linear-pre-fluctuation-exact}
    m\norm{\mathcal{L}(\mathrm{log}_{M_\star}(\hat M_m))}_{L^2(\mu_\mathrm{down})}^2 &= \norm{x^\top DP(M_\star)[\sqrt{m}\mathrm{log}_{M_\star}(\hat M_m)]\;m_\star}_{L^2(\mu_\mathrm{down})}^2\notag\\&= \beta_\star^\top A_\star \Sigma_{\mathrm{down}}DP(M_\star)[\sqrt{m}\mathrm{log}_{M_\star}(\hat M_m)]^TDP(M_\star)[\sqrt{m}\mathrm{log}_{M_\star}(\hat M_m)] \Sigma_{\mathrm{down}} A_\star^\top \beta_\star\notag\\
    &\distconv  \beta_\star^\top A_\star \Sigma_{\mathrm{down}}DP(M_\star)[Z]^TDP(M_\star)[Z] \Sigma_{\mathrm{down}} A_\star^\top \beta_\star
\end{align}
where the exact form of $DP(M_\star)$ is calculated in~\Cref{eq:DP-Monly} and~\eqref{eq:DProj-gap-proof}, and $Z$ is a mean-zero Gaussian with the covariance $V_\star$ (\Cref{eq:full-gaussian-cov}).

\subsection{Comparison to \cite{cabannes23SSLinterplay}}\label{sec:comparison-ssl}
We now study the following concrete example to compare our asymptotic result with the general upper bound of \citet[Thm.~4]{cabannes23SSLinterplay}. Assume that $\Sigma_{\mathrm{down}}=I_d$ and
\[
\Sigma_{\mathrm{pre}}
=
\begin{pmatrix}
\Sigma_{\mathrm{pre},1} & 0\\
0 & \Sigma_{\mathrm{pre},2}
\end{pmatrix},
\qquad
\Sigma_{\mathrm{pre}}^+
=
\begin{pmatrix}
\Sigma_{\mathrm{pre},1}^+ & 0\\
0 & \Sigma_{\mathrm{pre},2}^+
\end{pmatrix}
\]
with
\[
\Sigma_{\mathrm{pre},1}=\operatorname{diag}(a_1,\dots,a_k),
\qquad
\Sigma_{\mathrm{pre},2}=\operatorname{diag}(b_1,\dots,b_{d-k}),
\]
and
\[
\Sigma_{\mathrm{pre},1}^+=\operatorname{diag}(c_1,\dots,c_k),
\qquad
\Sigma_{\mathrm{pre},2}^+=\operatorname{diag}(e_1,\dots,e_{d-k}).
\]
Furthermore, we assume that all diagonal entries are strictly positive and 
\[
\frac{c_i}{a_i}>\frac{e_j}{b_j},
\qquad
\forall\, i\in\{1,\dots,k\},\ \forall\, j\in\{1,\dots,d-k\}.
\]

In this case, the whitened cross-covariance matrix $C
=
\Sigma_{\mathrm{pre}}^{-1/2}\Sigma_{\mathrm{pre}}^{+}\Sigma_{\mathrm{pre}}^{-1/2}$
is diagonal and given by
\[
C
=
\operatorname{diag}\!\left(
\frac{c_1}{a_1},\dots,\frac{c_k}{a_k},
\frac{e_1}{b_1},\dots,\frac{e_{d-k}}{b_{d-k}}
\right).
\]
By the ordering assumption, the top-$k$ eigenvalues of $C$ are $\frac{c_1}{a_1},\dots,\frac{c_k}{a_k}$, so the constrained population minimizer over $\mathcal M_{d,k}$ is
\[
M_\star
=
\begin{pmatrix}
R & 0\\
0 & 0
\end{pmatrix}
\]
where $R:=\operatorname{diag}(r_1,\dots,r_k)$ and $r_i:=\frac{c_i}{a_i^2}$. Any tangent matrix $H\in T_{M_\star}\mathcal M_{d,k}$ has the block form
\[
H=
\begin{pmatrix}
A & B\\
B^\top & 0
\end{pmatrix},
\qquad
A\in\mathbb S^k,\quad B\in\mathbb R^{k\times(d-k)}.
\]
Let $\mu_i\coloneqq\frac{e_i}{b_i}$, $\lambda_i\coloneqq\frac{c_i}{a_i}$, and $\tau\coloneqq\sum_{i=1}^k \lambda_i^2$.
By Corollary~\eqref{eq:full-gaussian-cov}, the covariance operator $V_\star$ is given by
\[
\langle H_1,V_\star H_2\rangle
=
C_{aa}(H_1',H_2')
-
C_{ab}(H_1',H_2')
-
C_{ab}(H_2',H_1')
+
C_{bb}(H_1',H_2'),
\]
where
\[
H_\ell'=\Sigma_{\mathrm{pre}}^{-1}H_\ell\Sigma_{\mathrm{pre}}^{-1}
=
\begin{pmatrix}
\Sigma_{\mathrm{pre},1}^{-1}A_\ell\Sigma_{\mathrm{pre},1}^{-1}
&
\Sigma_{\mathrm{pre},1}^{-1}B_\ell\Sigma_{\mathrm{pre},2}^{-1}\\
\Sigma_{\mathrm{pre},2}^{-1}B_\ell^\top\Sigma_{\mathrm{pre},1}^{-1}
&
0
\end{pmatrix}.
\]

We now compute the three terms in \eqref{eq:Caa}--\eqref{eq:Cbb}. Using the block-diagonal form of $\Sigma_{\mathrm{pre}}$ and $\Sigma_{\mathrm{pre}}^+$, a direct multiplication gives
\[
H_1'\Sigma_{\mathrm{pre}}
=
\begin{pmatrix}
\Sigma_{\mathrm{pre},1}^{-1}A_1 &
\Sigma_{\mathrm{pre},1}^{-1}B_1\\
\Sigma_{\mathrm{pre},2}^{-1}B_1^\top & 0
\end{pmatrix},
\qquad
H_2'\Sigma_{\mathrm{pre}}
=
\begin{pmatrix}
\Sigma_{\mathrm{pre},1}^{-1}A_2 &
\Sigma_{\mathrm{pre},1}^{-1}B_2\\
\Sigma_{\mathrm{pre},2}^{-1}B_2^\top & 0
\end{pmatrix}.
\]
Hence
\[
\tr(H_1'\Sigma_{\mathrm{pre}}H_2'\Sigma_{\mathrm{pre}})
=
\tr(\Sigma_{\mathrm{pre},1}^{-1}A_1\Sigma_{\mathrm{pre},1}^{-1}A_2)
+
2\tr(\Sigma_{\mathrm{pre},1}^{-1}B_1\Sigma_{\mathrm{pre},2}^{-1}B_2^\top).
\]
Similarly,
\[
\tr(H_1'\Sigma_{\mathrm{pre}}^+H_2'\Sigma_{\mathrm{pre}}^+)
=
\tr(\Sigma_{\mathrm{pre},1}^{-1}A_1\Sigma_{\mathrm{pre},1}^{+}\Sigma_{\mathrm{pre},1}^{-1}A_2\Sigma_{\mathrm{pre},1}^{+})
+
2\tr(\Sigma_{\mathrm{pre},1}^{-1}B_1\Sigma_{\mathrm{pre},2}^{+}\Sigma_{\mathrm{pre},2}^{-1}
B_2^\top
\Sigma_{\mathrm{pre},1}^{+}\Sigma_{\mathrm{pre},1}^{-1}).
\]
Since all these matrices are diagonal, this becomes entrywise
\[
\begin{aligned}
C_{aa}(H_1',H_2')
&=
\sum_{i,j=1}^k
\frac{1+\lambda_i\lambda_j}{a_i a_j}\,
(A_1)_{ij}(A_2)_{ij} +
2\sum_{i=1}^k\sum_{j=1}^{d-k}
\frac{1+\lambda_i \mu_j}{a_i b_j}\,
(B_1)_{ij}(B_2)_{ij}.
\end{aligned}
\]
Since $M_\star = \begin{pmatrix}
R\Sigma_{\mathrm{pre}}&0\\0&0    
\end{pmatrix}$, we obtain
\[
P_\ell=M_\star\Sigma_{\mathrm{pre}}H_\ell'
=
\begin{pmatrix}
RA_\ell\Sigma_{\mathrm{pre},1}^{-1}
&
RB_\ell\Sigma_{\mathrm{pre},2}^{-1}\\
0 & 0
\end{pmatrix}.
\]
Then
\[
\begin{aligned}
C_{ab}(H_1',H_2')
&=
\tr(P_2\Sigma_{\mathrm{pre}}H_1'\Sigma_{\mathrm{pre}}^+)
+
\tr(P_2\Sigma_{\mathrm{pre}}^+H_1'\Sigma_{\mathrm{pre}})\\
&=
\sum_{i,j=1}^k
\frac{r_i(\lambda_i+\lambda_j)}{a_j}\,
(A_1)_{ij}(A_2)_{ij} +
2\sum_{i=1}^k\sum_{j=1}^{d-k}
\frac{\lambda_i^2}{a_i b_j}\,
(B_1)_{ij}(B_2)_{ij}.
\end{aligned}
\]
By symmetry,
\[
\begin{aligned}
C_{ab}(H_2',H_1')
&=
\sum_{i,j=1}^k
\frac{r_j(\lambda_i+\lambda_j)}{a_i}\,
(A_1)_{ij}(A_2)_{ij}+
2\sum_{i=1}^k\sum_{j=1}^{d-k}
\frac{\lambda_i^2}{a_i b_j}\,
(B_1)_{ij}(B_2)_{ij}.
\end{aligned}
\]
We have
\[
Q:=M_\star\Sigma_{\mathrm{pre}}M_\star
=
\begin{pmatrix}
R\Sigma_{\mathrm{pre},1}R & 0\\
0 & 0
\end{pmatrix},
\]
so that $\tr(Q\Sigma_{\mathrm{pre}})=\sum_{i=1}^k \lambda_i^2=\tau$. Also, $\tr(P_\ell\Sigma_{\mathrm{pre}})
=
\sum_{i=1}^k r_i (A_\ell)_{ii}$. Further,
\[
\tr(P_1\Sigma_{\mathrm{pre}}P_2\Sigma_{\mathrm{pre}})
=
\sum_{i,j=1}^k r_i r_j\,(A_1)_{ij}(A_2)_{ij}.
\]
Finally, using again that all matrices are diagonal in block form,
\[
\tr(H_1'\Sigma_{\mathrm{pre}}H_2'\Sigma_{\mathrm{pre}}Q\Sigma_{\mathrm{pre}})
=
\sum_{i,j=1}^k
\frac{\lambda_i^2}{a_i a_j}\,(A_1)_{ij}(A_2)_{ij}
+
\sum_{i=1}^k\sum_{j=1}^{d-k}
\frac{\lambda_i^2}{a_i b_j}\,(B_1)_{ij}(B_2)_{ij}.
\]
Substituting into \eqref{eq:Cbb}, we obtain
\[
\begin{aligned}
C_{bb}(H_1',H_2')
&=
\Big(\sum_{i=1}^k r_i (A_1)_{ii}\Big)
\Big(\sum_{j=1}^k r_j (A_2)_{jj}\Big)\\
&\quad
+
\tau\sum_{i,j=1}^k \frac{1}{a_i a_j}(A_1)_{ij}(A_2)_{ij}
+
2\tau\sum_{i=1}^k\sum_{j=1}^{d-k}\frac{1}{a_i b_j}(B_1)_{ij}(B_2)_{ij}\\
&\quad
+
4\sum_{i,j=1}^k
\frac{\lambda_i^2}{a_i a_j}(A_1)_{ij}(A_2)_{ij}
+
4\sum_{i=1}^k\sum_{j=1}^{d-k}
\frac{\lambda_i^2}{a_i b_j}(B_1)_{ij}(B_2)_{ij}.
\end{aligned}
\]

Substituting the previous expressions into Corollary~\eqref{eq:full-gaussian-cov} yields

\begin{align*}
\langle H_1,V_\star H_2\rangle
&=
\Big(\sum_{i=1}^k r_i (A_1)_{ii}\Big)
\Big(\sum_{j=1}^k r_j (A_2)_{jj}\Big)+\sum_{i=1}^k
\frac{1+\tau+3\lambda_i^2}{a_i^2}\,
(A_1)_{ii}(A_2)_{ii}\\
&\quad
+2\sum_{1\le i<j\le k}
\frac{1+\tau+\lambda_i^2+\lambda_i\lambda_j+\lambda_j^2}{a_i a_j}\,
(A_1)_{ij}(A_2)_{ij}+2\sum_{i=1}^k\sum_{j=1}^{d-k}
\frac{1+\tau+\lambda_i^2}{a_i b_j}\,
(B_1)_{ij}(B_2)_{ij}.
\end{align*}

Therefore, if
\[
Z=
\begin{pmatrix}
Z_{11} & Z_{12}\\
Z_{12}^\top & 0
\end{pmatrix}
\sim \mathcal N(0,V_\star),
\]
then the covariance structure is as follows:
\begin{enumerate}[label=(\roman*)]
\item  The diagonal entries of $Z_{11}$ are jointly Gaussian with covariance
\[
\operatorname{Diag}\!\left(\frac{1+\tau+3\lambda_i^2}{a_i^2}\right)_{i=1}^k
+
rr^\top,
\qquad
r=(r_1,\dots,r_k)^\top.
\]
\item The off-diagonal entries $(Z_{11})_{ij}$, $1\le i<j\le k$, are independent with variances
\[
\frac{2(1+\tau+\lambda_i^2+\lambda_i\lambda_j+\lambda_j^2)}{a_i a_j}.
\]
\item The entries of $Z_{12}$ are independent with variances
\[
\frac{2(1+\tau+\lambda_i^2)}{a_i b_j}.
\]
\end{enumerate}

We now specialize \eqref{eq:linear-pre-fluctuation-exact} to this setting. In this case, $B(M)=\Sigma_{\mathrm{down}}^{1/2}M\Sigma_{\mathrm{down}}^{1/2}=M$ so that the derivative of the spectral projector simplifies to
\[
DP(M_\star)[H]=D\Pi(M_\star)[H].
\]

By \Cref{eq:DProj-gap-proof}, the Fréchet derivative of the projector at \(M_\star\) in direction \(Z\) is
\[
DP(M_\star)[Z]
=
\begin{pmatrix}
0 & R^{-1}Z_{12}\\
Z_{12}^\top R^{-1} & 0
\end{pmatrix}.
\]
Consequently,
\[
DP(M_\star)[Z]^\top DP(M_\star)[Z]
=
\begin{pmatrix}
R^{-1}Z_{12}Z_{12}^\top R^{-1} & 0\\
0 & Z_{12}^\top R^{-2}Z_{12}
\end{pmatrix}.
\]

Substituting this into \eqref{eq:linear-pre-fluctuation-exact} yields
\[
L:=
\beta_\star^\top A_\star \Sigma_{\mathrm{down}}DP(M_\star)[Z]^\top DP(M_\star)[Z]\Sigma_{\mathrm{down}}A_\star^\top\beta_\star
=
\beta_\star^\top A_\star DP(M_\star)[Z]^\top DP(M_\star)[Z]A_\star^\top\beta_\star.
\]
Writing
\[
m_\star:=A_\star^\top\beta_\star
=
\begin{pmatrix}
m_1\\
m_2
\end{pmatrix},
\qquad
m_1\in\mathbb R^k,\quad m_2\in\mathbb R^{d-k},
\]
we obtain
\[
L
=
m_1^\top R^{-1}Z_{12}Z_{12}^\top R^{-1}m_1
+
m_2^\top Z_{12}^\top R^{-2}Z_{12}m_2.
\]
In particular, in the aligned basis where \(m_\star\) lies entirely in the leading \(k\)-dimensional eigenspace, $m_\star=
\begin{pmatrix}
R^{1/2}\beta_\star\\
0
\end{pmatrix}$, the second term vanishes and therefore
\begin{equation}\label{eq:L-isotropic-explicit}
L
=
\beta_\star^\top R^{-1/2}Z_{12}Z_{12}^\top R^{-1/2}\beta_\star
=
\|Z_{12}^\top R^{-1/2}\beta_\star\|_2^2=
\sum_{j=1}^{d-k}
\left(
\sum_{i=1}^k \frac{\beta_{\star,i}}{\sqrt{r_i}}(Z_{12})_{ij}
\right)^2.
\end{equation}

By the explicit covariance structure of \(V_\star\), the entries of \(Z_{12}\) are independent centered Gaussians with variances
\[
\operatorname{Var}\bigl((Z_{12})_{ij}\bigr)
=
\frac{2(1+\tau+\lambda_i^2)}{a_i b_j},
\qquad
1\le i\le k,\quad 1\le j\le d-k.
\]
Hence, for each \(j\),
\[
Y_j:=\sum_{i=1}^k \frac{\beta_{\star,i}}{\sqrt{r_i}}(Z_{12})_{ij}
\]
is centered Gaussian with variance
\[
\operatorname{Var}(Y_j)
=
\sum_{i=1}^k \frac{\beta_{\star,i}^2}{r_i}\,
\frac{2(1+\tau+\lambda_i^2)}{a_i b_j}.
\]
Using \(r_i=\lambda_i/a_i\), this simplifies to
\[
\operatorname{Var}(Y_j)
=
\frac{2}{b_j}
\sum_{i=1}^k
\beta_{\star,i}^2
\frac{(1+\tau+\lambda_i^2)}{\lambda_i}.
\]
Therefore, $L=\sum_{j=1}^{d-k} Y_j^2$ is a weighted chi-square random variable, namely
\begin{equation}\label{eq:L-isotropic-chi-square}
L
\overset{d}{=}
\sum_{j=1}^{d-k} \sigma_j^2 \chi_j^2(1),
\qquad
\sigma_j^2
=
\frac{2}{b_j}
\sum_{i=1}^k
\beta_{\star,i}^2
\frac{(1+\tau+\lambda_i^2)}{\lambda_i},
\end{equation}
where the \(\chi_j^2(1)\) are independent chi-square random variables with one degree of freedom.
\begin{proof}[Proof of~\Cref{cor:linear-concrete}]
The result follows by a direct substitution into the general expression for the pre-training interaction term given in~\Cref{eq:L-isotropic-chi-square}.

Under the diagonal model specified in~\Cref{sec:example_linear_spectral}, we have
\[
b_j = 1, \qquad \beta_{\star,i} = 1, \qquad \lambda_i = \frac{1}{i},
\qquad \text{and} \qquad \tau = \sum_{j=1}^k j^{-2}.
\]
Substituting these quantities into~\Cref{eq:L-isotropic-chi-square}, we obtain
\[
\sigma_j^2
=
2 \sum_{i=1}^k i \bigl(1 + i^{-2} + \tau\bigr),
\]
which does not depend on \(j\). Therefore, the weighted chi-square random variable reduces to
\[
L
=
\left(
2 \sum_{i=1}^k i \bigl(1 + i^{-2} + \tau\bigr)
\right)
\sum_{j=1}^{d-k} \chi_j^2(1)=\left(
2 \sum_{i=1}^k i \bigl(1 + i^{-2} + \tau\bigr)
\right)
\chi_{(d-k)}^2,
\]
establishing the stated convergence in distribution.

The scaling of the expectation follows immediately from the above expression, yielding the claimed order.

\end{proof}

\paragraph{Pre-training asymptotic sub-optimality.}Let
\[
\mathcal Q
:=
\frac12\bigl\langle Z,\,2\Sigma_{\mathrm{pre}}Z\Sigma_{\mathrm{pre}}\bigr\rangle_F
=
\bigl\langle Z,\Sigma_{\mathrm{pre}}Z\Sigma_{\mathrm{pre}}\bigr\rangle_F.
\]
We note that the population Hessian operator at $M_\star$ is given by the map $H\mapsto\langle H,2\Sigma_\mathrm{pre}H\Sigma_\mathrm{pre}\rangle$. A second-order Taylor expansion of $L_\mathrm{spec}$ around $M_\star$, together with the first-order optimality condition, yields

\[
m\bigl(L_\mathrm{spec}(\hat M_m)-L_\mathrm{spec}(M_\star)\bigr)\;\distconv\;\mathcal Q.
\]
We have
\begin{align*}
\mathcal Q
&=
\sum_{i,j=1}^k a_i a_j (Z_{11})_{ij}^2
+
2\sum_{i=1}^k\sum_{u=1}^{d-k} a_i b_u (Z_{12})_{iu}^2\\&=\sum_{i=1}^k a_i^2 (Z_{11})_{ii}^2
+
2\sum_{1\le i<j\le k} a_i a_j (Z_{11})_{ij}^2
+
2\sum_{i=1}^k\sum_{u=1}^{d-k} a_i b_u (Z_{12})_{iu}^2.
\end{align*}

We now compute the law of each term.

Let
\[
z_{\mathrm{diag}}
:=
\bigl((Z_{11})_{11},\dots,(Z_{11})_{kk}\bigr)^\top .
\]
By assumption,
\[
z_{\mathrm{diag}}\sim\mathcal N(0,K),
\qquad
K
=
\operatorname{Diag}\!\left(\frac{1+\tau+3\lambda_i^2}{a_i^2}\right)_{i=1}^k
+
rr^\top,
\qquad
r_i=\frac{\lambda_i}{a_i}.
\]
Hence
\[
\sum_{i=1}^k a_i^2 (Z_{11})_{ii}^2
=
z_{\mathrm{diag}}^\top \operatorname{Diag}(a_1^2,\dots,a_k^2)\,z_{\mathrm{diag}}.
\]
If \(\xi\sim\mathcal N(0,I_k)\), this quadratic form may be written as
\[
z_{\mathrm{diag}}^\top \operatorname{Diag}(a_1^2,\dots,a_k^2)\,z_{\mathrm{diag}}
\;\overset{d}{=}\;
\xi^\top M \xi,
\]
where
\[
M
=
\operatorname{Diag}(a_1,\dots,a_k)\,K\,\operatorname{Diag}(a_1,\dots,a_k)
=
\operatorname{Diag}(1+\tau+3\lambda_i^2)_{i=1}^k
+
\lambda\lambda^\top,
\]
with \(\lambda=(\lambda_1,\dots,\lambda_k)^\top\). Therefore, if
\(\rho_1,\dots,\rho_k\) denote the eigenvalues of \(M\), then
\[
\sum_{i=1}^k a_i^2 (Z_{11})_{ii}^2
\;\overset{d}{=}\;
\sum_{\ell=1}^k \rho_\ell \chi_\ell^2(1).
\]

For \(1\le i<j\le k\), the variables \((Z_{11})_{ij}\) are independent centered
Gaussians with variance
\[
\operatorname{Var}\bigl((Z_{11})_{ij}\bigr)
=
\frac{2(1+\tau+\lambda_i^2+\lambda_i\lambda_j+\lambda_j^2)}{a_i a_j}.
\]
Hence
\[
2a_i a_j (Z_{11})_{ij}^2
\;\overset{d}{=}\;
4\bigl(1+\tau+\lambda_i^2+\lambda_i\lambda_j+\lambda_j^2\bigr)\chi_{ij}^2(1),
\]
where the \(\chi_{ij}^2(1)\) are independent.

For \(1\le i\le k\) and \(1\le j\le d-k\), the variables \((Z_{12})_{ij}\) are
independent centered Gaussians with variance
\[
\operatorname{Var}\bigl((Z_{12})_{ij}\bigr)
=
\frac{2(1+\tau+\lambda_i^2)}{a_i b_j}.
\]
Therefore
\[
2a_i b_j (Z_{12})_{ij}^2
\;\overset{d}{=}\;
4(1+\tau+\lambda_i^2)\widetilde\chi_{ij}^2(1),
\]
where the \(\widetilde\chi_{ij}^2(1)\) are independent. Using Gaussian independence of the orthogonal coordinates, we obtain
\begin{equation}\label{eq:Q-exact-law}
\mathcal Q
\;\overset{d}{=}\;
\sum_{\ell=1}^k \rho_\ell \chi_\ell^2(1)
+
\sum_{1\le i<j\le k}
4\bigl(1+\tau+\lambda_i^2+\lambda_i\lambda_j+\lambda_j^2\bigr)\chi_{ij}^2(1)
+
\sum_{i=1}^k\sum_{j=1}^{d-k}
4(1+\tau+\lambda_i^2)\widetilde\chi_{ij}^2(1),
\end{equation}
where \(\rho_1,\dots,\rho_k\) are the eigenvalues of
\[
M
=
\operatorname{Diag}(1+\tau+3\lambda_i^2)_{i=1}^k+\lambda\lambda^\top.
\]

Taking expectations yields
\begin{align}
\mathbb E[\mathcal Q]
&=
\operatorname{tr}(M)
+
4\sum_{1\le i<j\le k}
\bigl(1+\tau+\lambda_i^2+\lambda_i\lambda_j+\lambda_j^2\bigr)
+
4\sum_{i=1}^k\sum_{u=1}^{d-k}(1+\tau+\lambda_i^2)\notag\\
&=
k(4d-2k-1)(1+\tau)+2(2d-1)\tau+2(\sum_{i=1}^k\lambda_i)^2.
\label{eq:Q-expectation}
\end{align}

\paragraph{Connection to~\cite{cabannes23SSLinterplay}.} We compare our result to the generalization bound of~\citet[Thm. 4]{cabannes23SSLinterplay}. In their result, the contribution of pre-training appears through a product of a conditioning factor and a
sub-optimality term, of the form
\[
\|T_\lambda^{-1}\Pi_{\mathcal F_\lambda} f_\star\|^2
\bigl(L_{\mathrm{spec}}(\hat M)-L_{\mathrm{spec}}(M_\star)\bigr).
\]
We now show how the conditioning factor $\|T_\lambda^{-1}\Pi_{\mathcal F_\lambda} f_\star\|^2$ reduces in our setting.

In the present model with $\lambda=0$, the operator $T_\lambda$
coincides with the covariance matrix $C= \operatorname{diag}(R,0)$. Moreover, the projection onto the feature space is given by
\[
\Pi_{\mathcal F_\lambda} = \operatorname{diag}(I_k,0).
\]
Since the target function satisfies $f_\star(x)=\beta_\star^\top A_\star x$,
its projection onto the feature space depends only on the first $k$
coordinates. Therefore, the action of $T_\lambda^{-1}\Pi_{\mathcal F_\lambda}$
amounts to inverting $R$ on the top-$k$ subspace, yielding
\[
\|T_\lambda^{-1}\Pi_{\mathcal F_\lambda} f_\star\|^2
=
\|C^{-1}A_\star^\top\beta_\star\|_2^2=\|R^{-\frac12}\beta_\star\|_2^2.
\]
Consequently,
\[
m\|T_\lambda^{-1}\Pi_{\mathcal F_\lambda} f_\star\|^2
\bigl(L_{\mathrm{spec}}(\hat M)-L_{\mathrm{spec}}(M_\star)\bigr)
\;\distconv\;
\|R^{-\frac12}\beta_\star\|_2^2\,\mathcal Q.
\]
In the setting where $\Sigma_{\mathrm{pre}} = I_d$, $\Sigma_{\mathrm{pre}}^{+} = \diag(1,1/2,\ldots,1/d)$, and $\beta_\star=(1,\cdots,1)^\top$, we have 
\[
\|R^{-\frac12}\beta_\star\|_2^2 = \sum_{i=1}^k i = \Theta(k^2),
\]
while
\[\E[\mathcal{Q}]=k(4d-2k-1)(1+\tau)+2(2d-1)\tau+2(\sum_{i=1}^k\lambda_i)^2=\Theta\big(k(d-k)\big).\]
Consequently, the expected pre-training contribution satisfies
\[
\mathbb E\!\left[\|R^{-\frac12}\beta_\star\|_2^2\,\mathcal Q\right]
=
\|R^{-\frac12}\beta_\star\|_2^2\,\mathbb E[\mathcal Q] = O\big(k^3(d-k)\big).
\]
Thus, the pre-training contribution scales as $O(k^3(d-k))$ in this model.

\section{Proofs of Section~\ref{sec:ex:factor}}
\label{app:proof-factor}

This appendix verifies that the factor-model example in~\Cref{sec:ex:factor}
satisfies the assumptions of~\Cref{thm:master-compatible}.  The downstream verification is reduced to the linear-feature argument used in
Appendix~\ref{app:proof-cor-linear}: after passing to the quotient descriptor \(M=AA^\top\), the feature map can be written as
\[
\phi(x,M)=s(M)^\top(I_d+M)^{-1}x=B_M^\top x,
\qquad
B_M\coloneqq (I_d+M)^{-1}s(M),
\]
where \(M\mapsto B_M\) is smooth and \(B_M\) has rank \(k\) locally around
\(M_\star\).  Consequently, the quotient geometry and
Assumptions~\ref{ass:well-posedness}--\ref{ass:stable-null-span} follow by the same linear-feature verification, which we summarize below.  The only
model-specific pretraining step is to verify the Riemannian \(M\)-estimation
condition, Assumption~\ref{ass:riem_mestimation}, for the quotient
estimator \(\hat M_m\).  After these checks, the corollary follows by directly invoking~\Cref{thm:master-compatible}.

Throughout this appendix,
\[
M_\star \coloneqq A_\star A_\star^\top,
\qquad
\Sigma_x\coloneqq I_d+M_\star,
\]
and $A_\star\in\R^{d\times k}$ has full column rank.  Thus
$M_\star\in\mathcal M_{d,k}\coloneqq\{M\succeq0:\rank(M)=k\}$ and
$\Sigma_x\succ0$.

We now derive the exact maximum likelihood estimator of the
quotient-level parameter $M = AA^\top$. Recall the unlabeled pretraining model
\begin{equation}
\label{eq:factor-pretrain-model}
X = A_\star Z + \mu,
\qquad
Z \sim \mathcal N(0,I_k),
\quad
\mu \sim \mathcal N(0,I_d),
\end{equation}
with $Z$ and $\mu$ independent and $k \ll d$. By marginalizing over the latent factor $Z$, the pretraining covariate $X$
is Gaussian with covariance
\begin{equation}
\label{eq:factor-marginal}
X \sim \mathcal N\bigl(0,\; \Sigma_\star\bigr),
\qquad
\Sigma_\star = A_\star A_\star^\top + I_d = M_\star + I_d,
\end{equation}
where $M_\star = A_\star A_\star^\top$ is a rank-$k$ positive semidefinite matrix.

The negative log-likelihood of the pretraining sample
$\{X_i\}_{i=1}^m$ as a function of $M \succeq 0$ is
\begin{equation}
\label{eq:factor-loglik}
\hat L_m(M)
=
\frac{1}{2}\Big(
\log\det(I_d + M)
+
\tr\!\big(S_m (I_d + M)^{-1}\big)
\Big),
\end{equation}
up to an additive constant independent of $M$. Here, we defined $S_m\coloneqq \frac{1}{m}\sum_{i=1}^{m}X_iX_i^\top$. This maximum likelihood problem is exactly \emph{probabilistic principal component analysis} (PPCA); see \citet{tipping1999probabilistic}.

Let $S_m = U\Lambda U^\top $ be the eigenvalue decomposition of the sample covariance, and $\Lambda=\mathrm{diag}(\lambda_1,\cdots,\lambda_d)$ where $\lambda_1\ge\cdots\geq\lambda_d$. Let $U_k\in\R^{d\times k}$ be the $k$-dimensional leading eigenspace.
\begin{lemma}[PPCA maximum likelihood estimator for $M$]\label{lem:ppca-mle}
For the model $X\sim\mathcal N(0,I_d+M_\star)$ with $\mathrm{rank}(M_\star) = k$, a maximum likelihood estimator is
\begin{align}\label{eq:mle-factor-pre}
\hat M_m \;=\; U_k\bigl(\mathrm{diag}((\lambda_1 - 1)_+ \cdots,(\lambda_k - 1)_+)\bigr)U_k^\top.
\end{align}
\end{lemma}
\begin{proof}
See \citet{tipping1999probabilistic}.
\end{proof}

\begin{lemma}[Downstream reduction for the factor model]
\label{lem:factor:downstream-reduction}
The factor-model quotient feature map
\[
\phi(x,M)=s(M)^\top(I_d+M)^{-1}x
\]
satisfies Assumptions~\ref{ass:well-posedness}--\ref{ass:stable-null-span} by
the same argument as the linear spectral feature map, with $s(M)$ replaced by
$B_M^\top$, where
\[
B_M\coloneqq (I_d+M)^{-1}s(M).
\]
\end{lemma}

\begin{proof}
The map $M\mapsto s(M)$ is $C^2$ by the local-section construction, and
$M\mapsto(I_d+M)^{-1}$ is smooth because $I_d+M\succ0$ for all $M\succeq0$.
Hence
\[
M\mapsto B_M=(I_d+M)^{-1}s(M)
\]
is $C^2$ locally.  Since $s(M)$ has rank $k$ and $I_d+M$ is invertible, $B_M$
has rank $k$ locally.  The quotient feature map is therefore a rank-$k$ linear
feature map in the covariate:
\[
\phi(x,M)=B_M^\top x.
\]

All downstream verifications are then identical to the linear spectral case.
Indeed, local boundedness of $B_M$ and $D_MB_M$ gives the local-uniform feature and derivative moment bounds from Gaussian moments of $X$ (\Cref{lem:linear:moments}). The downstream
covariance is
\[
\Sigma(M)
=
\E[\phi(X,M)\phi(X,M)^\top]
=
B_M^\top\Sigma_xB_M,
\]
which is positive definite because $B_M$ has rank $k$ and $\Sigma_x\succ0$. Thus, the stable rank/eigengap condition holds locally and $d_{\mathrm{eff}}(M)=k$ (\Cref{lem:linear:Sigma-stable}). The leverage and leverage-weighted signal moment bounds follow as in the linear case from
\[
q_M(X)\le C\|\phi(X,M)\|^2\le C\|X\|^2
\]
and from the fact that both $f_\star$ and $\Pi_Mf_\star$ are linear functions of $X$ with locally bounded coefficients (\Cref{lem:linear:leverage-moments}). Empirical well-posedness follows because
$B_M^\top X$ has a density on $\R^k$, so the feature design has rank $k$ almost
surely for $n\ge k$. Finally, since $\Sigma(M)\succ0$, the
coefficient null space $\ker(T_M)$ is $\{0\}$, so the stable null-span condition holds trivially (\Cref{lem:linear:stable-null-span}).
\end{proof}

\paragraph{Verification of~\Cref{ass:riem_mestimation}.} Next, we will prove the maximum likelihood estimator in Equation~\ref{eq:mle-factor-pre} satisfies~\Cref{ass:riem_mestimation}.
\begin{lemma}\label{lem:factor:D5}
Assume \(A_\star\) has full column rank, so that
\(M_\star=A_\star A_\star^\top\in\mathcal M_{d,k}\) and
\(\lambda_k(M_\star)>0\).  Let \(L(M)=\E[\hat L_m(M)]\).  Then Assumption~\ref{ass:riem_mestimation} holds at \(M_\star\).    
\end{lemma}
\begin{proof}
    We verify the five parts of~\Cref{ass:riem_mestimation}.

    \paragraph{(i) Identification and separation.}
The population objective is
\[
L(M)
=
\frac12\left[
\log\det(I_d+M)
+
\tr\!\big(\Sigma_x(I_d+M)^{-1}\big)
\right],
\qquad
\Sigma_x=I_d+M_\star .
\]
This is the Gaussian negative log-likelihood for covariance \(I_d+M\), up to an
additive constant.  Equivalently,
\[
L(M)-L(M_\star)
=
\frac12\left[
\tr\!\big(\Sigma_x(I_d+M)^{-1}\big)
-
\log\det\!\big(\Sigma_x(I_d+M)^{-1}\big)
-d
\right].
\]
The bracketed term is nonnegative and vanishes if and only if
\(I_d+M=\Sigma_x\), i.e. \(M=M_\star\).  Thus \(M_\star\) is the unique minimizer
of \(L\) on \(\mathcal M_{d,k}\).

The separation condition follows from uniqueness and coercivity.  Indeed,
\(L(M)\to\infty\) whenever \(\|M\|_F\to\infty\) within \(\mathcal M_{d,k}\),
because \(\log\det(I_d+M)\to\infty\).  If separation failed for some
\(\epsilon>0\), there would exist \(M_j\in\mathcal M_{d,k}\) such that
\[
d_{\mathcal M}(M_j,M_\star)\ge \epsilon,
\qquad
L(M_j)\downarrow L(M_\star).
\]
By coercivity, a subsequence is bounded.  Since the PSD cone is closed, a
further subsequence converges to some \(M_\infty\succeq0\).  By continuity of
the Gaussian likelihood on the PSD cone,
\(L(M_\infty)=L(M_\star)\), so the uniqueness statement above gives
\(M_\infty=M_\star\), contradicting
\(d_{\mathcal M}(M_j,M_\star)\ge\epsilon\).  Hence, for every \(\epsilon>0\),
\[
\inf_{M\in\mathcal M_{d,k}:d_{\mathcal M}(M,M_\star)\ge \epsilon}
\bigl(L(M)-L(M_\star)\bigr)>0.
\]
\paragraph{(ii) Uniform LLN on a compact set.}
Let $K_{\epsilon'}
\coloneqq
\exp_{M_\star}\!\big(\overline B(M_\star,\epsilon')\big)$ be the compact normal-coordinate set appearing in Assumption~\ref{ass:riem_mestimation}.
For each fixed \(x\), the map
\[
M\mapsto
\ell_{\mathrm{pre}}(M;x)
=
\frac12\left[
\log\det(I_d+M)+x^\top(I_d+M)^{-1}x
\right]
\]
is continuous on \(K_{\epsilon'}\).  Moreover, since \(K_{\epsilon'}\) is compact,
\(\log\det(I_d+M)\) is uniformly bounded on \(K_{\epsilon'}\), and
\(\|(I_d+M)^{-1}\|_{\mathrm{op}}\le1\) for all \(M\succeq0\).  Therefore
\[
\sup_{M\in K_{\epsilon'}}
|\ell_{\mathrm{pre}}(M;x)|
\le
C_{K}(1+\|x\|^2).
\]
The right-hand side is integrable because \(X\sim\mathcal N(0,\Sigma_x)\).
Thus, by \citet[Lemma~2.4]{newey1994largesample},
\[
\sup_{M\in K_{\epsilon'}}
|\hat L_m(M)-L(M)|
\xrightarrow{\mathbb P}0.
\]

It remains to verify localization of \(\hat M_m\).  We prove consistency of the
PPCA estimator.  By the strong law of large numbers,
\(S_m\to\Sigma_x\) almost surely in operator norm.  Fix an outcome in this
almost-sure event and write $\Delta_m\coloneqq S_m-\Sigma_x$. Thus, $\|\Delta_m\|_{\mathrm{op}}\to0$. Let \(P_\star\) denote the orthogonal projector onto the top-\(k\) eigenspace of
\(\Sigma_x\), and let $\hat P_m\coloneqq \sum_{j=1}^k \hat u_j\hat u_j^\top$ be the projector onto the top-\(k\) eigenspace of \(S_m\).  By the
Davis--Kahan \(\sin\Theta\) theorem~\citep{davis1970rotation}, for all \(m\) large enough,
\begin{equation}
\label{eq:dk}
\|\hat P_m-P_\star\|_{\mathrm{op}}
\le
\frac{2\|\Delta_m\|_{\mathrm{op}}}{\lambda_k(\Sigma_x)-\lambda_{k+1}(\Sigma_x)}
=
\frac{2\|\Delta_m\|_{\mathrm{op}}}{\lambda_k(\Sigma_x)-1}
\longrightarrow 0.
\end{equation}
Thus \(\hat P_m\to P_\star\) in operator norm, hence also in Frobenius norm.

Let $\gamma\coloneqq \lambda_k(\Sigma_x)-1>0$. By Weyl's inequality, $|\hat\lambda_i-\lambda_i(\Sigma_x)|
\le
\|\Delta_m\|_{\mathrm{op}}$ for all $i\in[d]$. Hence, for all \(m\) large enough so that
\(\|\Delta_m\|_{\mathrm{op}}\le \gamma/2\),
\[
\hat\lambda_k
\ge
\lambda_k(\Sigma_x)-\|\Delta_m\|_{\mathrm{op}}
\ge
1+\gamma/2.
\]
Therefore \(\hat\lambda_i>1\) for all \(i\le k\), and the \((\cdot)_+\)
truncation in the PPCA estimator is inactive.

Define the rank-\(k\) truncation map
\[
\Pi_k(B)
\coloneqq
\sum_{j=1}^k \lambda_j(B)u_j(B)u_j(B)^\top .
\]
Then, $\hat M_m=\Pi_k(S_m)-\hat P_m$. Similarly, since
\(\lambda_{k+1}(\Sigma_x)=\cdots=\lambda_d(\Sigma_x)=1\),
\[
\Pi_k(\Sigma_x)
=
M_\star+P_\star,
\qquad
M_\star=\Pi_k(\Sigma_x)-P_\star.
\]
Consequently,
\begin{equation}
\label{eq:factor-consistency-decomp}
\hat M_m-M_\star
=
\bigl(\Pi_k(S_m)-\Pi_k(\Sigma_x)\bigr)
-
\bigl(\hat P_m-P_\star\bigr).
\end{equation}
The projector term converges to zero by \eqref{eq:dk}.  For the truncation term,
the top-\(k\) spectral truncation is continuous at \(\Sigma_x\) because of the
eigengap at \(k\).  Hence
\[
\|\Pi_k(S_m)-\Pi_k(\Sigma_x)\|_F\to0.
\]
Combining this with \eqref{eq:factor-consistency-decomp} gives $\hat M_m\xrightarrow{\mathbb P} M_\star$. Therefore, for the compact normal-coordinate set \(K_{\epsilon'}\),
\[
\Pr(\hat M_m\in K_{\epsilon'})\to1
\]
after fixing \(\epsilon'>0\) sufficiently small.

\paragraph{(iii) Local \(C^2\) smoothness and score moments.}
For every \(x\), the map
\[
M\mapsto \ell_{\mathrm{pre}}(M;x)
=
\frac12\left[
\log\det(I_d+M)+x^\top(I_d+M)^{-1}x
\right]
\]
is \(C^2\) on a neighborhood of \(M_\star\).  Its Euclidean gradient is
\[
\nabla_M\ell_{\mathrm{pre}}(M;x)
=
\frac12\left[
(I_d+M)^{-1}
-
(I_d+M)^{-1}xx^\top(I_d+M)^{-1}
\right].
\]
Thus, at \(M_\star\),
\[
\|\operatorname{grad}\ell_{\mathrm{pre}}(M_\star;X)\|_{M_\star}
\le
C(1+\|X\|^2).
\]
Since \(X\) is Gaussian,
\[
\E\Big[
\|\operatorname{grad}\ell_{\mathrm{pre}}(M_\star;X)\|_{M_\star}^2
\Big]
<\infty.
\]

\paragraph{(iv) Nondegenerate minimizer.}
The population gradient is
\[
\nabla L(M)
=
\frac12\left[
(I_d+M)^{-1}
-
(I_d+M)^{-1}\Sigma_x(I_d+M)^{-1}
\right],
\]
which vanishes at \(M=M_\star\).  The population Hessian at \(M_\star\) has
quadratic form
\[
D^2L(M_\star)[H,H]
=
\frac12\Tr\!\left(\Sigma_x^{-1}H\Sigma_x^{-1}H\right),
\qquad
H\in T_{M_\star}\mathcal M_{d,k}.
\]
Because \(\Sigma_x\succ0\), this quadratic form is strictly positive for every
nonzero \(H\in T_{M_\star}\mathcal M_{d,k}\).  Hence the Riemannian Hessian
\[
H_\star
=
\operatorname{Hess}L(M_\star):
T_{M_\star}\mathcal M_{d,k}\to T_{M_\star}\mathcal M_{d,k}
\]
is positive definite, and in particular invertible.

\paragraph{(v) Uniform transported Hessian convergence.}
On the compact normal-coordinate set \(K_{\epsilon'}\), the transported Hessian
entries are continuous functions of \(M\).  Moreover, the derivatives of the
exponential map, parallel transport, and the local coordinate maps are bounded
on \(K_{\epsilon'}\), while the first two derivatives of
\(\ell_{\mathrm{pre}}(M;X)\) are dominated by an integrable envelope of the form $C_K(1+\|X\|^2)$. Applying \citet[Lemma~2.4]{newey1994largesample} entrywise to the finite-dimensional
matrix representation of the transported Hessian gives
\[
\sup_{M\in K_{\epsilon'}}
\bigl\|
\widetilde H_m(M)-\widetilde H(M)
\bigr\|
\xrightarrow{\mathbb P}0,
\]
where
\[
\widetilde H_m(M)
=
\mathcal P_{M\to M_\star}
\circ
\operatorname{Hess}\hat L_m(M)
\circ
\mathcal P_{M_\star\to M},
\]
and
\[
\widetilde H(M)
=
\mathcal P_{M\to M_\star}
\circ
\operatorname{Hess}L(M)
\circ
\mathcal P_{M_\star\to M}.
\]
This verifies Assumption~\ref{ass:riem_mestimation}(v).
\end{proof}

\paragraph{Explicit form of the Gaussian tangent limit.} \Cref{lem:factor:D5} gives the abstract descriptor CLT.  We now record the explicit form of the limiting Gaussian tangent vector \(Z\), which is used in the evaluation of the pretraining interaction term. Let $X\sim\mathcal N(0,\Sigma_x)$ with
\[
\Sigma_x = I_d + M_\star,\qquad \operatorname{rank}(M_\star)=k.
\]
First, we write the eigendecomposition
\[
\Sigma_x = U_\star
\begin{pmatrix}
I_k + D & 0\\
0 & I_{d-k}
\end{pmatrix}
U_\star^\top,
\qquad
D=\operatorname{diag}(\sigma_1,\dots,\sigma_k),\ \ \sigma_k>0,
\]
so $M_\star = U_\star \begin{psmallmatrix} D & 0\\ 0 & 0\end{psmallmatrix}U_\star^\top$.
Let the sample covariance be $S_m=\frac1m\sum_{i=1}^m X_iX_i^\top$.
Recall the (fixed-$k$) PPCA MLE can be written as
\[
\hat M_m = \hat U_k\,\operatorname{diag}(\hat\lambda_1-1,\dots,\hat\lambda_k-1)\,\hat U_k^\top,
\]
where $\hat\lambda_1\ge\cdots\ge\hat\lambda_d$ and $\hat U_k$ are the top-$k$ eigenpairs of $S_m$.
Since $\lambda_k(\Sigma_x)=1+\sigma_k>1$, we have $\hat\lambda_i>1$ for $i\le k$ with probability $\to 1$,
so the ``$(\cdot)_+$'' truncation is asymptotically inactive and the map is smooth at $\Sigma_x$.

The Gaussian fourth-moment identity gives
\[
\sqrt m\,\mathrm{vec}(S_m-\Sigma_x)\ \distconv\ \mathcal N\!\big(0,\ (I+K)(\Sigma_x\otimes\Sigma_x)\big),
\]
where $K$ is the commutation matrix.

Define the smooth map (near $\Sigma_x$)
\[
g(\Sigma)\coloneqq \text{``best rank-$k$ PSD approximation of }\Sigma-I_d\text{''},
\qquad
\hat M_m=g(S_m),\quad M_\star=g(\Sigma_x).
\]
Then
\[
Z_m \coloneqq \sqrt m\,(\hat M_m - M_\star)\ \distconv\ Z,
\qquad
Z = Dg_{\Sigma_x}[G],
\]
where $G$ is the Gaussian matrix limit of $\sqrt m(S_m-\Sigma_x)$.

Let $U_\star=[U_1\ U_2]$ with $U_1\in\R^{d\times k}$ spanning the signal subspace and
$U_2\in\R^{d\times(d-k)}$ spanning its orthogonal complement, and rotate
\[
\tilde G \coloneqq U_\star^\top G U_\star
=
\begin{pmatrix}
\tilde G_{11} & \tilde G_{12}\\
\tilde G_{21} & \tilde G_{22}
\end{pmatrix}.
\]
A first-order perturbation calculation (using the eigengap $\sigma_k>0$) yields the simple derivative
\[
Dg_{\Sigma_x}[G]
=
U_\star
\begin{pmatrix}
\tilde G_{11} & \tilde G_{12}\\
\tilde G_{21} & 0
\end{pmatrix}
U_\star^\top.
\]
Therefore the asymptotic fluctuation has the explicit form
\[
Z
=
U_\star
\begin{pmatrix}
\tilde G_{11} & \tilde G_{12}\\
\tilde G_{21} & 0
\end{pmatrix}
U_\star^\top,
\qquad
\tilde G = U_\star^\top G U_\star,
\]
and in particular $Z$ is mean-zero Gaussian (as a vector in $\R^{d(d+1)/2}$).

If we define the linear operator $\mathcal P:\R^{d\times d}\to\R^{d\times d}$ by
\[
\mathcal P(Y)\coloneqq U_\star
\begin{pmatrix}
Y_{11} & Y_{12}\\
Y_{21} & 0
\end{pmatrix}
U_\star^\top
\quad\text{(blocks taken in the $U_\star$-basis),}
\]
then $\mathrm{vec}(Z)= (\mathcal P\otimes I)\mathrm{vec}(G)$ and hence
\[
\mathrm{vec}(Z)\sim \mathcal N\!\big(0,\ (\mathcal P\otimes I)\,(I+K)(\Sigma_x\otimes\Sigma_x)\,(\mathcal P\otimes I)^\top\big).
\]
Equivalently, Let $P_\star$ be the orthogonal projector onto $\operatorname{range}(M_\star)$. Then, we have 

\begin{align}\label{eq:asymp-dist-factor-Z}
Z \;=\; P_\star G + G P_\star - P_\star G P_\star.
\end{align}

Let $W\in\R^{d\times d}$ have i.i.d.\ $\mathcal N(0,1)$ entries. Thus, 
\begin{align}\label{eq:asymp-dist-factor-G}
G \;\coloneqq\; \frac{1}{\sqrt{2}}\,
\Sigma_x^{1/2}(W+W^\top)\Sigma_x^{1/2}.
\end{align}

Substituting the expression for $G$ gives the explicit representation of $Z$
as a linear transform of a standard Gaussian matrix:
\begin{align}\label{eq:asymp-dist-factor}
Z
=
\frac{1}{\sqrt{2}}\Big[
P_\star \Sigma_x^{1/2}(W+W^\top)\Sigma_x^{1/2}
+
\Sigma_x^{1/2}(W+W^\top)\Sigma_x^{1/2}P_\star
-
P_\star \Sigma_x^{1/2}(W+W^\top)\Sigma_x^{1/2}P_\star
\Big].   
\end{align}

\paragraph{Pretraining fluctuations.}Fix $\Sigma_x \coloneqq I_d+M_\star$ and write $f_\star(x)=w_\star^\top x$ with
\[
w_\star=\Sigma_x^{-1}A_\star\beta_\star.
\]
At $M_\star$, choose a local section with $s(M_\star)=A_\star$ and define
\[
U_\star \coloneqq (I_d+M_\star)^{-1}A_\star=\Sigma_x^{-1}A_\star.
\]
Let $P_\star$ be the $\Sigma_x$-orthogonal projector onto
$\operatorname{range}(U_\star)$:
\[
P_{\star}
=
U_\star\big(U_\star^\top \Sigma_x U_\star\big)^{-1}U_\star^\top \Sigma_x.
\]
Then $(\Pi_{M_\star}f_\star)(x)=(P_\star^{(\Sigma_x)}w_\star)^\top x$ and
$P_\star^{(\Sigma_x)}w_\star=w_\star$.

Recall $\mathcal L(v)\coloneqq -D\Pi_{M_\star}[v]\,f_\star$. Since
$D\Pi_{M_\star}[v]\,f_\star(x)=(DP_\star^{(\Sigma_x)}[v]\,w_\star)^\top x$, we have
\[
\mathcal L(v)(x)=-(DP_\star^{(\Sigma_x)}[v]\,w_\star)^\top x.
\]
Define for each $M$
\[
U_M \coloneqq (I_d+M)^{-1}s(M)\in\R^{d\times k},
\qquad
\mathcal S_M \coloneqq \operatorname{range}(U_M),
\]
and let $P_M^{(\Sigma_x)}$ denote the $\Sigma_x$-orthogonal projector onto $\mathcal S_M$.
By definition of an orthogonal projector onto $\mathcal S_M$, we have the matrix identity
\begin{equation}
\label{eq:projector-fixes-range}
P_M^{(\Sigma_x)} U_M = U_M .
\end{equation}
Fix $v\in T_{M_\star}\mathcal M_{d,k}$ and consider a smooth curve $M(t)$ in $\mathcal M_{d,k}$
with $M(0)=M_\star$ and $\dot M(0)=v$. Define $U(t)\coloneqq U_{M(t)}$ and
$P(t)\coloneqq P_{M(t)}^{(\Sigma_x)}$. Then \eqref{eq:projector-fixes-range} becomes
\[
P(t)\,U(t)=U(t)\qquad\text{for all $t$ near $0$.}
\]
Differentiating at $t=0$ and applying the product rule gives
\[
\dot P(0)\,U_\star + P_\star\,\dot U(0)=\dot U(0),
\]
where $U_\star\coloneqq U_{M_\star}$ and $P_\star\coloneqq P_{M_\star}^{(\Sigma_x)}$.
Rearranging yields the key identity
\begin{equation}
\label{eq:DP-times-U}
\dot P(0)\,U_\star = (I_d-P_\star)\,\dot U(0).
\end{equation}
Interpreting $\dot P(0)=DP_\star^{(\Sigma_x)}[v]$ and $\dot U(0)=DU_\star[v]$, we can rewrite
\eqref{eq:DP-times-U} as
\begin{equation}
\label{eq:DPU-general}
DP_\star^{(\Sigma_x)}[v]\;U_\star = (I_d-P_\star)\;DU_\star[v].
\end{equation}
Finally, since in our model $w_\star\in\mathcal S_{M_\star}$ we can write
\[
w_\star = U_\star\beta_\star
\]
for the same $\beta_\star$ appearing in $f_\star(x)=\beta_\star^\top U_\star^\top x$.
Multiplying \eqref{eq:DPU-general} on the right by $\beta_\star$ gives
\begin{equation}
\label{eq:DPwstar}
DP_\star^{(\Sigma_x)}[v]\;w_\star
=
DP_\star^{(\Sigma_x)}[v]\;U_\star\beta_\star
=
(I_d-P_\star)\;DU_\star[v]\;\beta_\star.
\end{equation}
Moreover $U_M=(I_d+M)^{-1}s(M)$ implies
\[
DU_\star[v]\,\beta_\star=-\Sigma_x^{-1}v\,\Sigma_x^{-1}A_\star\beta_\star + \Sigma_x^{-1}Ds(M_\star)[v]\,\beta_\star.
\]
The section-dependent term does not change $\operatorname{range}(U_M)$ and is killed by
$(I_d-P_\star^{(\Sigma_x)})$.
Consequently,
\[
DP_\star^{(\Sigma_x)}[v]\,\Sigma_x^{-1}A_\star\beta_\star
=
-(I_d-P_\star^{(\Sigma_x)})\Sigma_x^{-1}v\,\Sigma_x^{-1}A_\star\beta_\star,
\]
and therefore, for $v=Z$,
\[
\mathcal L(Z)(x)
=
x^\top (I_d-P_\star^{(\Sigma_x)})\Sigma_x^{-1}Z\,\Sigma_x^{-1}A_\star\beta_\star.
\]

Define the coefficient vector
\[
g(Z)\;\coloneqq\;(I_d-P_\star^{(\Sigma_x)})\Sigma_x^{-1}Z\,\Sigma_x^{-1}A_\star\beta_\star\in\R^d,
\]
so that $\mathcal L(Z)(x)=x^\top g(Z)$. Then
\[
\|\mathcal L(Z)\|_{L^2(\mu_{\mathrm{down}})}^2
=
\E\big[(x^\top g(Z))^2\big]
=
g(Z)^\top \Sigma_x\, g(Z).
\]

Equivalently, expanding $g(Z)$,
\begin{align*}
\|\mathcal L(Z)\|_{L^2(\mu_{\mathrm{down}})}^2
&=
\Big(\Sigma_x^{-1}A_\star\beta_\star\Big)^\top
Z^\top \Sigma_x^{-1}
\,(I_d-P_\star^{(\Sigma_x)})^\top \Sigma_x (I_d-P_\star^{(\Sigma_x)})
\,\Sigma_x^{-1} Z
\Big(\Sigma_x^{-1}A_\star\beta_\star\Big).
\end{align*}

Since $P_\star^{(\Sigma_x)}$ is $\Sigma_x$-self-adjoint and idempotent,
\[
(P_\star^{(\Sigma_x)})^\top \Sigma_x=\Sigma_x P_\star^{(\Sigma_x)},
\qquad
(P_\star^{(\Sigma_x)})^2=P_\star^{(\Sigma_x)},
\]
we have the simplification
\[
(I_d-P_\star^{(\Sigma_x)})^\top \Sigma_x (I_d-P_\star^{(\Sigma_x)})
=
\Sigma_x (I_d-P_\star^{(\Sigma_x)}).
\]
Moreover, at $M_\star$, we have $P_\star^{(\Sigma_x)}=P_\star$ is the orthogonal projection onto the span of $\mathrm{image}(M_\star)$. Therefore,
\begin{align}\label{eq:norm-L-factor}
\|\mathcal L(Z)\|_{L^2(\mu_{\mathrm{down}})}^2
&=
\Big(\Sigma_x^{-1}A_\star\beta_\star\Big)^\top
Z^\top \Sigma_x^{-1}\,(I_d-P_\star)\,\Sigma_x^{-1} Z
\Big(\Sigma_x^{-1}A_\star\beta_\star\Big)\notag\\
&=\norm{(I_d-P_\star)\,\Sigma_x^{-1} Z
\Big(\Sigma_x^{-1}A_\star\beta_\star\Big)}^2 =\norm{(I_d-P_\star)\, Z
\Big(\Sigma_x^{-1}A_\star\beta_\star\Big)}^2.
\end{align}
Since $\Sigma_\star^{-1}$ acts as identity on $I-P_\star$. Now, we plug~\cref{eq:asymp-dist-factor-Z} into~\Cref{eq:norm-L-factor} to get
\begin{align}\label{eq:norm-L-factor-simple}
\|\mathcal L(Z)\|_{L^2(\mu_{\mathrm{down}})}^2 &= \norm{(I_d-P_\star)\, (P_\star G + G P_\star - P_\star G P_\star)
\Big(\Sigma_x^{-1}A_\star\beta_\star\Big)}^2\notag\\
&=\norm{(I_d-P_\star)\, G P_\star
\Sigma_x^{-1}A_\star\beta_\star}^2.
\end{align}

Finally, we plug~\Cref{eq:asymp-dist-factor-G} into~\Cref{eq:norm-L-factor-simple}
\begin{align}\label{eq:norm-L-factor-final}
\|\mathcal L(Z)\|_{L^2(\mu_{\mathrm{down}})}^2 &= \norm{(I_d-P_\star)\, (\frac{1}{\sqrt{2}}\,
\Sigma_x^{1/2}(W+W^\top)\Sigma_x^{1/2}) P_\star
\Sigma_x^{-1}A_\star\beta_\star}^2\notag\\&=\frac12\norm{(I_d-P_\star)\, 
(W+W^\top)P_\star
\Sigma_x^{-1/2}A_\star\beta_\star}^2.
\end{align}

Let $M_\star = U_\star
\begin{pmatrix}
D & 0\\
0 & 0
\end{pmatrix}
U_\star^\top$ with $D=\operatorname{diag}(\sigma_1,\dots,\sigma_k)$ and
$U_\star=[U_1\ U_2]$, where $U_1\in\R^{d\times k}$ spans $\operatorname{range}(M_\star)$.
Then
\[
P_\star = U_1U_1^\top
=U_\star
\begin{pmatrix}
I_k & 0\\
0 & 0
\end{pmatrix}
U_\star^\top,
\qquad
\Sigma_x=I_d+M_\star
=
U_\star
\begin{pmatrix}
I_k + D & 0\\
0 & I_{d-k}
\end{pmatrix}
U_\star^\top,
\]
so
\[
\Sigma_x^{-1/2}
=
U_\star
\begin{pmatrix}
(I_k + D)^{-1/2} & 0\\
0 & I_{d-k}
\end{pmatrix}
U_\star^\top.
\]
Rotate the standard Gaussian matrix by $\tilde W\coloneqq U_\star^\top W U_\star$ (still i.i.d.\ $\mathcal N(0,1)$).
Then
\[
(I_d-P_\star)(W+W^\top)P_\star\Sigma_x^{-1/2}A_\star\beta_\star
=
U_\star
\begin{pmatrix}
0 & 0\\
(\tilde W_{21}+\tilde W_{12}^\top) & 0
\end{pmatrix}
U_\star^\top
\,
U_\star
\begin{pmatrix}
(I_k+D)^{-1/2} & 0\\
0 & 0
\end{pmatrix}
U_\star^\top
A_\star\beta_\star.
\]
Since $A_\star\beta_\star\in\operatorname{range}(U_1)$, letting $a\coloneqq U_1^\top A_\star\beta_\star\in\R^k$ gives
\[
(I_d-P_\star)(W+W^\top)P_\star\Sigma_x^{-1/2}A_\star\beta_\star
=
U_2\,(\tilde W_{21}+\tilde W_{12}^\top)\,(I_k+D)^{-1/2}\,a.
\]
Because $U_2$ is orthonormal, $\|U_2 v\|_2=\|v\|_2$, hence the original quantity simplifies to
\[
\frac12\big\|(I_d-P_\star)(W+W^\top)P_\star\Sigma_x^{-1/2}A_\star\beta_\star\big\|_2^2
=
\frac12\big\|(\tilde W_{21}+\tilde W_{12}^\top)\,(I_k+D)^{-1/2}\,a\big\|_2^2,
\qquad
a=U_1^\top A_\star\beta_\star.
\]
In particular, since $\tilde W_{21}$ has i.i.d.\ $\mathcal N(0,1)$ entries and
$\tilde W_{12}$ is independent with the same law, the matrix
$\frac{1}{\sqrt2}(\tilde W_{21}+\tilde W_{12}^\top)$ has i.i.d.\ $\mathcal N(0,1)$ entries.
Therefore
\[
\frac12\big\|(\tilde W_{21}+\tilde W_{12}^\top)\,(I_k+D)^{-1/2}\,a\big\|_2^2
\ \stackrel{d}{=}\
\big\|\Xi\,(I_k+D)^{-1/2}\,U_1^\top A_\star\,\beta_\star\big\|_2^2,
\qquad
\Xi_{ij}\stackrel{\mathrm{iid}}{\sim}\mathcal N(0,1).
\]
Equivalently, we have
\begin{align}\label{eq:factor-fluct}
    \|\mathcal L(Z)\|_{L^2(\mu_{\mathrm{down}})}^2 \stackrel{d}{=}\ \|(I_k+D)^{-1/2}\,U_1^\top A_\star\,\beta_\star\|_2^2\,\chi^2_{d-k}.
\end{align}

\section{Proofs of~\Cref{sec:ex-mog}}
\label{app:proof-cor-mog}

This appendix verifies that the Gaussian-mixture example from \Cref{sec:ex-mog} satisfies the hypotheses of~\Cref{thm:master-compatible}, and hence yields Corollary~\ref{cor:GoM-model}.
The verification is modular:
(i) local-uniform moment bounds for the feature map,
(ii) rank stability/eigengap for the population feature second-moment $\Sigma(M)$,
(iii) a manifold CLT for the (pretraining) descriptor estimator $\underline{U}_m=D(\hat U_m)$ via an $M$-estimation CLT and a delta method,
and (iv) verification of the hypotheses of the general projector differentiability result proved in \Cref{app:rep}.

Throughout, we work on the regular set from \Cref{ass:ex-mog-regular} and write $\underline{U}_\star\coloneqq D(U^\star)$.

\subsection{Model, features, and the quotient-level descriptor}
\label{app:proof-cor-mog:model}

\paragraph{Unlabeled distribution.}
Let $K\ge 2$ and $d\ge 1$. The unlabeled distribution is the spherical Gaussian mixture
\[
X\mid(Z=i)\sim \mathcal N(u_i^\star,I_d),
\qquad
Z\sim \mathrm{Unif}([K]),
\]
with unknown centers $U^\star=(u_1^\star,\dots,u_K^\star)\in(\R^d)^K$.

\paragraph{Centered-mean subspace.}
Define the empirical mean and centered second-moment matrix
\[
\bar u(U):=\frac1K\sum_{i=1}^K u_i,
\qquad
S(U):=\sum_{i=1}^K \big(u_i-\bar u(U)\big)\big(u_i-\bar u(U)\big)^\top.
\]
Let $r_\star:=r(U^\star)$ and define $P_U$ to be the orthogonal projector onto the leading $r_\star$-dimensional eigenspace of $S(U)$
(on the regular neighborhood where this eigenspace is well-defined).
Write $P_\star:=P_{U^\star}$ and $V_\star:=\mathrm{Im}(P_\star)$.

\paragraph{Subspace-aware responsibilities.}
For $i\in[K]$, define
\[
\pi_i(x;U)
=
\frac{\exp\!\big(\langle P_Uu_i,P_Ux\rangle-\tfrac12\|P_Uu_i\|_2^2\big)}
{\sum_{j=1}^K \exp\!\big(\langle P_Uu_j,P_Ux\rangle-\tfrac12\|P_Uu_j\|_2^2\big)}.
\]
These satisfy $0\le \pi_i\le 1$ and $\sum_{i=1}^K \pi_i=1$.

\paragraph{Feature map and hypothesis class.}
The induced feature map has dimension $p(U)=K(r(U)+1)$:
\[
\psi_U(x)
=
\Big(
\pi_1(x;U)\big(P_Ux-P_Uu_1\big),\,\pi_1(x;U),\dots,
\pi_K(x;U)\big(P_Ux-P_Uu_K\big),\,\pi_K(x;U)
\Big).
\]
Let $\mathcal H_U=\{\langle\theta,\psi_U(\cdot)\rangle:\theta\in\R^{p(U)}\}$ and $\Pi_U:L^2(\mu_{\mathrm{down}})\to L^2(\mu_{\mathrm{down}})$ be the orthogonal projector onto $\mathcal{H}_U$.

\paragraph{Group action and quotient parameter.} Fix $U=(u_1,\dots,u_K)\in(\R^d)^K$.  
The unlabeled mixture likelihood is invariant under relabeling of mixture components.
Accordingly, we consider the action of the permutation group $S_K$ on $(\R^d)^K$ defined by
\[
(g\cdot U)_i := u_{g^{-1}(i)}, \qquad g\in S_K .
\]
We write
\[
\underline U := [U] \;\in\; (\R^d)^K / S_K
\]
for the equivalence class (orbit) of $U$ under this action, and refer to $\underline U$ as the
\emph{quotient-level parameter}. 

\paragraph{Regular regime and local lifts.} Throughout this appendix we restrict attention to the regular regime in which the centers are
distinct:
\[
u_i \neq u_j \qquad \text{for all } i\neq j .
\]
In this regime the action of $S_K$ is free, and the quotient $(\R^d)^K/S_K$ is a smooth manifold of
dimension $Kd$. In particular, for any fixed $\underline U^\star=[U^\star]$ there exists a
neighborhood $\mathcal U$ and a smooth local section
\[
s:\mathcal U\to(\R^d)^K,
\qquad s(\underline U)\in\underline U ,
\]
which amounts to choosing a deterministic ordering of the centers in a neighborhood of $U^\star$.
All constructions below are independent of the specific choice of section.

\paragraph{Permutation-invariant geometric quantities.} Both $\bar u(U)$ and $P_U$ are invariant under the action of $S_K$, and therefore depend only on the
quotient parameter $\underline U$. We may thus unambiguously write $\bar u(\underline U)$ and
$P_{\underline U}$.

\paragraph{Responsibilities and equivariant feature map.} For a representative $U\in\underline U$, define the projected responsibilities
\[
\pi_i(x;U)
:=
\frac{\exp\!\big(\langle P_U u_i,\;P_U x\rangle-\tfrac12\|P_Uu_i\|_2^2\big)}
{\sum_{j=1}^K \exp\!\big(\langle P_U u_j,\;P_U x\rangle-\tfrac12\|P_Uu_j\|_2^2\big)},
\qquad i\in[K].
\]
Let $\psi_U:\R^d\to\R^{K(r_\star+1)}$ denote the block-structured feature map
\[
\psi_U(x)
=
\Big(
\pi_1(x;U)(P_Ux-P_Uu_1),\, \pi_1(x;U);\;
\dots;\;
\pi_K(x;U)(P_Ux-P_Uu_K),\, \pi_K(x;U)
\Big).
\]

For any $g\in S_K$, let $\rho(g)$ denote the block-permutation matrix acting on
$\R^{K(r_\star+1)}$. A direct computation shows that the feature map is $S_K$-equivariant:
\[
\psi_{g\cdot U}(x) = \rho(g)\,\psi_U(x).
\]
To be more precise, $\rho(G)$ consists of those permutations in $S_{2K}$ that relabel $\pi_i(x,U)$ and $\pi_i(x;U)(P_Ux-P_Uu_i)$ two the same label. Therefore, we can define a feature map $\phi(\cdot,\underline U)$ on the quotient manifold that satisfies all the assumptions of~\Cref{app:riem_bundle}.

\subsection{Local-uniform moment bounds for $\psi(\cdot,\underline U)$}
\label{app:proof-cor-mog:moments}

We verify the local-uniform moment bounds from \Cref{ass:moments} for this model (with
$X\sim\mu_{\mathrm{down}}$ equal to the Gaussian mixture described in the example).

\begin{lemma}[Local-uniform polynomial moments of the Gaussian-mixture features(\Cref{ass:moments}]
\label{lem:mog-moments}
Fix a compact neighborhood $\mathcal U$ of $U^\star$ in $(\R^d)^K$. Then for every $q\ge 1$,
\[
\sup_{U\in\mathcal U}\ \E\big[\|\psi_U(X)\|^q\big]<\infty.
\]
Consequently, for every neighborhood $\underline{\mathcal U}$ of $\underline U_\star=[U^\star]$ contained in the quotient image
$[\mathcal U]\subset (\R^d)^K/S_K$ and every $q\geq1$
\[
\sup_{\underline U\in \underline{\mathcal U}}\ \E\big[\|\psi_{s(\underline U)}(X)\|^{q}\big]<\infty,
\]
where $s:\underline{\mathcal U}\to (\R^d)^K$ is any local section (lift) with $s(\underline U)\in \underline U$.
\end{lemma}

\begin{proof}
Fix $U\in\mathcal U$. Since $0\le \pi_i(x;U)\le 1$ and $\|P_Ux\|_2\le \|x\|_2$, we have for each $i$
\[
\big\|\pi_i(X;U)\big(P_UX-P_Uu_i\big)\big\|_2
\le
\|X\|_2+\|u_i\|_2,
\qquad
|\pi_i(X;U)|\le 1.
\]
Thus there exists a constant $C<\infty$ depending only on $K$ such that
\[
\|\psi_U(X)\|
\le
C\Big(1+\|X\|_2+\max_{i\in[K]}\|u_i\|_2\Big).
\]
Taking $q$-th moments and using $\sup_{U\in\mathcal U}\max_i\|u_i\|_2<\infty$, it remains to note that all polynomial moments of $X$
are finite under a spherical Gaussian mixture, giving the stated uniform bound.

For the quotient-level statement, let $\underline{\mathcal U}\subset [\mathcal U]$ and let $s$ be a local section on
$\underline{\mathcal U}$. Then $s(\underline U)\in \mathcal U$ for all $\underline U\in \underline{\mathcal U}$, so the same bound
applies uniformly to $\psi_{s(\underline U)}(X)$.
\end{proof}

\subsection{Rank stability and eigengap for the downstream second moment}
\label{app:proof-cor-mog:rank}

Define the population second moment (feature covariance)
\[
\Sigma(\underline U)
\;\coloneqq\;
\E\big[\phi_{\underline U}(X)\phi_{\underline U}(X)^\top\big]
\in \R^{p\times p},
\]
where $p$ is the feature dimension. To invoke~\Cref{thm:master-compatible}, we need rank stability and an eigengap at $0$ for $\Sigma(\underline U)$ on a neighborhood of
$\underline U_\star$. We reduce this to (a) continuity of $\underline U\mapsto\Sigma(\underline U)$ and (b) an eigengap at $\underline U_\star$.

\begin{lemma}[Continuity of $\underline U\mapsto \Sigma(\underline U)$ in operator norm]
\label{lem:mog-sigma-cont}
Let $\underline{\mathcal U}$ be a comapct neighborhood of $\underline U_\star$ such that the conclusion of \Cref{lem:mog-moments} holds on
$s(\underline{\mathcal U})$ for some exponent $4+\delta$. Then $\underline U\mapsto \Sigma(\underline U)$ is continuous at $\underline U_\star$ in operator norm.
\end{lemma}

\begin{proof}
Fix $\underline U\in\underline{\mathcal U}$ and write $Y_{\underline U}\coloneqq \phi_{\underline U}(X)\phi_{\underline U}(X)^\top$.
By \Cref{lem:mog-moments} (with exponent $1+\delta/4$) and Cauchy--Schwarz,
\[
\sup_{\underline U\in\underline{\mathcal U}} \E\big[\|Y_{\underline U}\|_{\mathrm{op}}^{1+\delta/4}\big]
\;\le\;
\sup_{\underline U\in\underline{\mathcal U}} \E\big[\|\phi_{\underline U}(X)\|^{2+\delta/2}\big]
\;<\;\infty,
\]
so $\{\|Y_{\underline U}\|_{\mathrm{op}}:\underline U\in\underline{\mathcal U}\}$ is uniformly integrable.
Since $\phi_{\underline U}$ is a continuous function of $\underline U$ in $\underline{\mathcal{U}}$, $Y_{\underline U}\to Y_{\underline U_\star}$ almost surely as $\underline U\to \underline U_\star$.
Uniform integrability then yields
\[
\big\|\Sigma(\underline U)-\Sigma(\underline U_\star)\big\|_{\mathrm{op}}
=
\Big\|\E\big[Y_{\underline U}-Y_{\underline U_\star}\big]\Big\|_{\mathrm{op}}
\;\le\;
\E\big[\|Y_{\underline U}-Y_{\underline U_\star}\|_{\mathrm{op}}\big]
\;\to\;0,
\]
proving operator-norm continuity at $\underline U_\star$.
\end{proof}

\begin{lemma}[Local uniform invertibility from full rank at $\underline U_\star$]
\label{lem:mog-gap-stable}
Assume $\Sigma(\underline U_\star)\succ 0$.
If $\underline U\mapsto \Sigma(\underline U)$ is continuous at $\underline U_\star$ in operator norm
(e.g.\ by \Cref{lem:mog-sigma-cont}), then there exist a neighborhood $\underline{\mathcal U}$ of $\underline U_\star$
and a constant $\kappa>0$ such that for all $\underline U\in\underline{\mathcal U}$,
\[
\lambda_{\min}\!\big(\Sigma(\underline U)\big)\ge \kappa .
\]
In particular, $\Sigma(\underline U)$ is invertible on $\underline{\mathcal U}$ and $\Sigma(\underline U)^+=\Sigma(\underline U)^{-1}$.
\end{lemma}

\begin{proof}
Since $\Sigma(\underline U_\star)\succ 0$, we have $\lambda_{\min}(\Sigma(\underline U_\star))>0$.
Let $\kappa := \tfrac12\,\lambda_{\min}(\Sigma(\underline U_\star))$.
By continuity in operator norm, for $\underline U$ close enough to $\underline U_\star$ we have
$\|\Sigma(\underline U)-\Sigma(\underline U_\star)\|_{\mathrm{op}}\le \kappa$.
Weyl's inequality then yields
\[
\lambda_{\min}\!\big(\Sigma(\underline U)\big)
\ge
\lambda_{\min}\!\big(\Sigma(\underline U_\star)\big)-\|\Sigma(\underline U)-\Sigma(\underline U_\star)\|_{\mathrm{op}}
\ge 2\kappa-\kappa=\kappa,
\]
proving the claim.
\end{proof}

\subsection{Verification of Assumption~\ref{ass:riem_mestimation}}
\label{app:proof-cor-mog:clt}

We now verify the Riemannian \(M\)-estimation condition, Assumption~\ref{ass:riem_mestimation}, for the quotient estimator
\(\hat{\underline U}_m=[\hat U_m]\).  We first record the score/Hessian formulas and the local consistency input, and then verify the five parts of Assumption~\ref{ass:riem_mestimation} one by one.

\paragraph{Pretraining estimator.}
For concreteness, let $\hat U_m$ be a (measurable) local minimizer of the empirical negative log-likelihood
for the mixture with equal weights and known covariance:
\[
\ell(U;x)=-\log\Big(\frac1K\sum_{i=1}^K \exp\big(-\tfrac12\|x-u_i\|_2^2\big)\Big),
\qquad
\hat U_m\in\argmin_U \frac1m\sum_{j=1}^m \ell(U;X_j).
\]
Define the quotient estimator $\hat{\underline U}_m \coloneqq [\hat U_m]\in(\R^d)^K/S_K$.

\begin{lemma}[Score and Hessian for the spherical equal-weight MoG likelihood]
\label{lem:mog-score-hess}
Let
\[
p_U(x)\coloneqq \frac1K\sum_{i=1}^K \varphi_d(x-u_i),
\qquad
\varphi_d(t)\coloneqq (2\pi)^{-d/2}\exp\!\Big(-\tfrac12\|t\|_2^2\Big),
\qquad
\ell(U;x)\coloneqq -\log p_U(x).
\]
Define the responsibilities
\[
\pi_i(x;U)\coloneqq \frac{\varphi_d(x-u_i)}{\sum_{j=1}^K \varphi_d(x-u_j)}
=
\frac{\exp(-\tfrac12\|x-u_i\|_2^2)}{\sum_{j=1}^K \exp(-\tfrac12\|x-u_j\|_2^2)},
\qquad i\in[K].
\]
Then $\ell(U;x)$ is $C^\infty$ in $U$, and the gradient blocks satisfy
\[
\nabla_{u_i}\ell(U;x)=-\pi_i(x;U)\,(x-u_i),
\qquad i\in[K].
\]
Moreover, the block Hessian $\nabla^2_{u_i u_j}\ell(U;x)\in\R^{d\times d}$ is given by
\[
\nabla^2_{u_i u_j}\ell(U;x)
=
\begin{cases}
\pi_i(x;U)\,I_d-\pi_i(x;U)\bigl(1-\pi_i(x;U)\bigr)\,(x-u_i)(x-u_i)^\top,
& i=j,\\[0.35em]
\pi_i(x;U)\pi_j(x;U)\,(x-u_i)(x-u_j)^\top,
& i\neq j.
\end{cases}
\]
In particular, for any bounded set $\mathcal U\subset(\R^d)^K$ there exists $C<\infty$ (depending on $\mathcal U$ and $K$) such that
for all $U\in\mathcal U$ and all $x\in\R^d$,
\[
\|\nabla \ell(U;x)\|_2 \le C\bigl(1+\|x\|_2\bigr),
\qquad
\|\nabla^2 \ell(U;x)\|_{\mathrm{op}} \le C\bigl(1+\|x\|_2^2\bigr).
\]
\end{lemma}

\begin{proof}
The smoothness of $U\mapsto \ell(U;x)$ follows from $p_U(x)>0$ and smoothness of $u\mapsto \varphi_d(x-u)$.
Write $Z(U;x)\coloneqq \sum_{j=1}^K \varphi_d(x-u_j)$ so that $\ell(U;x)=-\log Z(U;x)+\log K$.
Since $\nabla_{u_i}\varphi_d(x-u_i)=\varphi_d(x-u_i)(x-u_i)$, we have
\[
\nabla_{u_i}\ell(U;x)
=-\frac{\nabla_{u_i} Z(U;x)}{Z(U;x)}
=-\frac{\varphi_d(x-u_i)}{\sum_{j=1}^K \varphi_d(x-u_j)}(x-u_i)
=-\pi_i(x;U)(x-u_i).
\]
For the Hessian, first note that
\[
\nabla_{u_j}\pi_i(x;U)
=
\pi_i(x;U)\bigl(\delta_{ij}-\pi_j(x;U)\bigr)(x-u_j),
\]
which follows by differentiating the softmax form of $\pi_i$.
Using $\nabla_{u_j}(x-u_i)=-\delta_{ij}I_d$, we obtain
\[
\nabla^2_{u_i u_j}\ell(U;x)
=-\nabla_{u_j}\!\bigl(\pi_i(x;U)(x-u_i)\bigr)
=
-\bigl(\nabla_{u_j}\pi_i(x;U)\bigr)(x-u_i)^\top+\pi_i(x;U)\delta_{ij}I_d.
\]
Substituting the expression for $\nabla_{u_j}\pi_i$ yields the displayed block formulas for $i=j$ and $i\neq j$.

Finally, the bounds use $0\le \pi_i\le 1$ and, on a bounded set $\mathcal U$, the uniform bound $\max_i\|u_i\|_2\le C_\mathcal U$:
\[
\|\nabla_{u_i}\ell(U;x)\|_2 \le \|x-u_i\|_2 \le \|x\|_2 + C_\mathcal U,
\]
and each Hessian block is a sum of terms bounded by $1$ and $\|x-u_i\|_2\|x-u_j\|_2$, hence by
$C(1+\|x\|_2^2)$ uniformly over $U\in\mathcal U$. Summing over the $K\times K$ blocks gives the stated operator-norm bound.
\end{proof}

\begin{lemma}[Local uniqueness of the MLE after fixing the permutation symmetry]
\label{lem:mog-mle-unique-local}
Let $L(U)=\E[\ell(U;X)]$ and $\hat L_m(U)=\frac1m\sum_{j=1}^m \ell(U;X_j)$, where $X\sim p_{U^\star}$.
Fix a gauge-fixing set $\Theta\subset(\R^d)^K$ that removes the permutation symmetry locally (e.g.\ a local ordering rule),
so that $U^\star\in\Theta$ and no nontrivial permutation of $U^\star$ belongs to $\Theta$.
Assume:
\begin{enumerate}[label=(\roman*)]
\item if $p_U=p_{U^\star}$ then $U$ equals a permutation of $U^\star$ (identifiability up to permutation);
\item $L$ is twice differentiable on $\Theta$ and there exist $\mu>0$ and $\rho>0$ such that
$\nabla^2 L(U)\succeq \mu I$ for all $U\in\Theta\cap B(U^\star,\rho)$;
\item $\sup_{U\in\Theta\cap B(U^\star,\rho)}\|\nabla^2\hat L_m(U)-\nabla^2 L(U)\|_{\mathrm{op}}\to 0$ in probability.
\end{enumerate}
Then, with probability tending to one, $\hat L_m$ is $\mu/2$-strongly convex on $\Theta\cap B(U^\star,\rho)$
and has a unique minimizer there, denoted $\hat U_m$, and $\hat U_m\to U^\star$ in probability.
\end{lemma}

\begin{proof}
This is identical to the standard strong-convexity plus uniform-Hessian-consistency argument; we include it here for completeness.
By (iii), with probability tending to one,
\[
\sup_{U\in \Theta\cap B(U^\star,\rho)}
\big\|\nabla^2\hat L_m(U)-\nabla^2 L(U)\big\|_{\mathrm{op}}
\le \mu/2.
\]
On this event, (ii) implies $\nabla^2\hat L_m(U)\succeq (\mu/2)I$ for all $U\in\Theta\cap B(U^\star,\rho)$, hence $\hat L_m$
is $\mu/2$-strongly convex there and has a unique minimizer. Consistency follows from (i) together with uniform convergence of $\hat L_m$ to $L$
on $\Theta\cap\overline{B}(U^\star,\rho)$, which is implied by the same regularity underlying (iii) and the polynomial tail control of $X$ under the mixture.
\end{proof}

\begin{lemma}
\label{lem:mog:D5}
Assume \Cref{ass:ex-mog-regular} and the conditions of
\Cref{lem:mog-mle-unique-local} for a gauge-fixing set \(\Theta\).  Let
\[
F(\underline U)
\coloneqq
\E[\ell(s(\underline U);X)],
\qquad
\hat F_m(\underline U)
\coloneqq
\frac1m\sum_{j=1}^m \ell(s(\underline U);X_j),
\]
where \(s\) is any local section of the quotient map near
\(\underline U_\star=[U^\star]\).  Then Assumption~\ref{ass:riem_mestimation}
holds for the quotient objective \(F\) at \(\underline U_\star\).
\end{lemma}

\begin{proof}
We verify the five parts of Assumption~\ref{ass:riem_mestimation}.

\paragraph{(i) Identification and separation.}
The population objective is the cross-entropy
\[
F(\underline U)
=
\E_{X\sim p_{U^\star}}\big[-\log p_{s(\underline U)}(X)\big].
\]
Equivalently,
\[
F(\underline U)-F(\underline U_\star)
=
\mathrm{KL}\!\left(p_{U^\star}\,\|\,p_{s(\underline U)}\right)\ge 0.
\]
Equality holds if and only if \(p_{s(\underline U)}=p_{U^\star}\).  By the
identifiability condition in \Cref{lem:mog-mle-unique-local}, this happens if
and only if \(s(\underline U)\) equals a permutation of \(U^\star\), i.e.
\(\underline U=\underline U_\star\).  Thus \(\underline U_\star\) is the unique
minimizer of \(F\) on the quotient.

The separation condition follows from the same identifiability and the local
gauge fixing.  Indeed, if for some \(\epsilon>0\) the separation failed, there
would exist \(\underline U_j\) such that
\[
d(\underline U_j,\underline U_\star)\ge \epsilon,
\qquad
F(\underline U_j)\downarrow F(\underline U_\star).
\]
Choosing the local gauge representative whenever \(\underline U_j\) lies near
\(\underline U_\star\), the compactness of local sublevel sets and continuity of
the likelihood give a limit point \(\underline U_\infty\) satisfying
\(F(\underline U_\infty)=F(\underline U_\star)\).  By the uniqueness just proved,
\(\underline U_\infty=\underline U_\star\), contradicting the distance lower
bound.  Hence, for every \(\epsilon>0\),
\[
\inf_{\underline U:\,d(\underline U,\underline U_\star)\ge \epsilon}
\big(F(\underline U)-F(\underline U_\star)\big)>0.
\]

\paragraph{(ii) Uniform LLN on a compact set.}
Let $K_{\epsilon'}
\coloneqq
\exp_{\underline U_\star}
\big(\overline B(\underline U_\star,\epsilon')\big)$ be the compact normal-coordinate set appearing in
Assumption~\ref{ass:riem_mestimation}.  On this set, choose a smooth local
section \(s\).  For each fixed \(x\), the map $\underline U\mapsto \ell(s(\underline U);x)$ is continuous.  Moreover, the centers \(s(\underline U)\) remain in a bounded
subset of \((\R^d)^K\) as \(\underline U\) ranges over \(K_{\epsilon'}\).  Hence
there exists \(C_K<\infty\) such that
\[
\sup_{\underline U\in K_{\epsilon'}}
|\ell(s(\underline U);x)|
\le
C_K(1+\|x\|^2).
\]
The right-hand side is integrable under the spherical Gaussian mixture.  Thus,
by \citet[Lemma~2.4]{newey1994largesample},
\[
\sup_{\underline U\in K_{\epsilon'}}
|\hat F_m(\underline U)-F(\underline U)|
\xrightarrow{\mathbb P}0.
\]

It remains to check localization.  By \Cref{lem:mog-mle-unique-local}, after
fixing the local gauge, the empirical likelihood has a unique local minimizer
\(\hat U_m\) with probability tending to one and $\hat U_m\xrightarrow{\mathbb P}U^\star$. Therefore
\[
\hat{\underline U}_m=[\hat U_m]\xrightarrow{\mathbb P} [U^\star]=\underline U_\star.
\]
Consequently, $\Pr\!\left(
\hat{\underline U}_m\in K_{\epsilon'}
\right)\to 1$ for every sufficiently small fixed \(\epsilon'>0\).

\paragraph{(iii) Local \(C^2\) smoothness and score moments.}
By \Cref{lem:mog-score-hess}, \(U\mapsto \ell(U;x)\) is \(C^\infty\) for every
\(x\).  Since the local section \(s\) is smooth, the quotient-coordinate map $\underline U\mapsto \ell(s(\underline U);x)$ is \(C^2\) on a normal neighborhood of \(\underline U_\star\).

The Riemannian gradient in quotient coordinates is obtained from the Euclidean
score in the local gauge by a smooth change of coordinates.  On bounded gauge
neighborhoods, \Cref{lem:mog-score-hess} gives
\[
\|\nabla \ell(U;x)\|_2
\le
C(1+\|x\|).
\]
Therefore,
\[
\|\operatorname{grad}\ell(\underline U_\star;X)\|_{\underline U_\star}
\le
C(1+\|X\|).
\]
Since \(X\) has finite moments of all orders under the Gaussian mixture,
\[
\E\!\left[
\|\operatorname{grad}\ell(\underline U_\star;X)\|_{\underline U_\star}^2
\right]<\infty.
\]
Thus Assumption~\ref{ass:riem_mestimation}(iii) holds.

\paragraph{(iv) Nondegenerate minimizer.}
By \Cref{lem:mog-mle-unique-local}, in the gauge-fixed coordinates there exist
\(\mu>0\) and \(\rho>0\) such that
\[
\nabla^2 L(U)\succeq \mu I
\qquad
\text{for all }U\in\Theta\cap B(U^\star,\rho).
\]
In particular, the Hessian at \(U^\star\) is positive definite in the
gauge-fixed coordinates.  Since the gauge chart is a smooth local coordinate
chart for the quotient manifold, this is equivalent to positive definiteness of
the Riemannian Hessian
\[
H_\star
=
\operatorname{Hess}F(\underline U_\star):
T_{\underline U_\star}\big((\R^d)^K/S_K\big)
\to
T_{\underline U_\star}\big((\R^d)^K/S_K\big).
\]
Thus \(H_\star\) is invertible.

\paragraph{(v) Uniform transported Hessian convergence.}
On the compact normal-coordinate set \(K_{\epsilon'}\), the local section,
the exponential map, and parallel transport are smooth with uniformly bounded
derivatives.  Therefore the transported Hessian entries are finite-dimensional
linear combinations of the Euclidean Hessian entries in the gauge chart with
smooth bounded coefficients.

By \Cref{lem:mog-score-hess}, on bounded gauge neighborhoods,
\[
\|\nabla^2\ell(U;X)\|_{\mathrm{op}}
\le
C(1+\|X\|^2),
\]
and the right-hand side is integrable under the Gaussian mixture.  Hence, again
by \citet[Lemma~2.4]{newey1994largesample}, applied entrywise to the transported
Hessian matrix,
\[
\sup_{\underline U\in K_{\epsilon'}}
\big\|
\widetilde H_m(\underline U)-\widetilde H(\underline U)
\big\|
\xrightarrow{\mathbb P}0,
\]
where
\[
\widetilde H_m(\underline U)
=
\mathcal P_{\underline U\to\underline U_\star}
\circ
\operatorname{Hess}\hat F_m(\underline U)
\circ
\mathcal P_{\underline U_\star\to\underline U},
\]
and
\[
\widetilde H(\underline U)
=
\mathcal P_{\underline U\to\underline U_\star}
\circ
\operatorname{Hess}F(\underline U)
\circ
\mathcal P_{\underline U_\star\to\underline U}.
\]
This is exactly Assumption~\ref{ass:riem_mestimation}(v).
\end{proof}

\subsection{Fr\'echet differentiability of $\underline U\mapsto \Pi_{\underline U}$}
\label{app:proof-cor-mog:proj}

The master theorem requires Fr\'echet differentiability of $\underline U\mapsto\Pi_{\underline U}$ at $\underline U_\star$ in operator norm on $L^2(\mu_{\mathrm{down}})$.
We do not re-prove this here; instead we verify the hypotheses of the general differentiability result proved in \Cref{app:rep}.

\begin{lemma}[Verification of the projector differentiability hypotheses]
\label{lem:mog-proj-diff-hyp}
Let $\mathcal V$ be a neighborhood of $\underline U_\star$ on which \Cref{lem:mog-moments} holds.
Then:
\begin{enumerate}[label=(\roman*)]
\item In the local chart fixing the permutation ambiguity, $\underline U\mapsto \phi(\cdot,\underline U)$ is $C^1$ on $\mathcal V$.
\item There exist $\delta>0$ and $C_{\partial\phi}<\infty$ such that
\[
\sup_{\underline U\in\mathcal V}\ \E\big[\|D_{\underline U}\phi(X,\underline U)\|_{\mathrm{op}}^{4+\delta}\big]\le C_{\partial\phi}.
\]
\item If, additionally, $\Sigma(\underline U)$ has stable rank/eigengap on $\mathcal V$ (as in \Cref{lem:mog-gap-stable}),
then $\underline U\mapsto \Pi_{\underline U}$ is Fr\'echet differentiable at $\underline U_\star$ in operator norm.
\end{enumerate}
\end{lemma}

\begin{proof}
(i) On the regular set, $U\mapsto P_U$ and $U\mapsto (P_Uu_i)$ are smooth, and $(x,U)\mapsto \pi_i(x;U)$ is a softmax of smooth functions.
In a local chart around $U^\star$ fixing the permutation ambiguity, these objects become smooth functions of $\underline U$, hence
$\underline U\mapsto \phi(x,\underline U)$ is $C^1$ for each $x$.

(ii) Differentiating $\phi$ produces finite sums of terms involving $\pi_i$, $D\pi_i$, $P_U$, and $DP_U$, multiplied by $x$ and the bounded centers.
On any bounded neighborhood of $U^\star$, these derivatives are bounded by polynomials in $\|x\|_2$, while $0\le \pi_i\le 1$.
Since $X$ under the mixture has finite moments of all orders, we obtain the stated bound for some $\delta>0$.

(iii) This is exactly the implication of the general projector differentiability result in \Cref{app:rep} once the moment bounds and
the stable-rank/eigengap condition are in place.
\end{proof}

\subsection{Proof of Corollary~\ref{cor:GoM-model}}
\label{app:proof-cor-mog:conclude}

\begin{proof}[Proof of Corollary~\ref{cor:GoM-model}]
Under \Cref{ass:ex-mog-regular} and the model-specific verifications above:
\begin{itemize}
\item \Cref{lem:mog-moments} gives the local-uniform moment bounds in \Cref{ass:moments}.
\item \Cref{lem:mog-sigma-cont} and~\Cref{lem:mog-gap-stable} give rank stability/eigengap for $\Sigma(M)$ on a neighborhood of $M_\star$.
\item \Cref{lem:mog:D5} yields the manifold CLT
$Z_m=\sqrt m\,\log_{M_\star}(\hat M_m)\distconv Z\sim\mathcal N(0,V)$.
\item \Cref{lem:mog-proj-diff-hyp} verifies the hypotheses needed to invoke the general differentiability theorem for $M\mapsto\Pi_M$
from \Cref{app:rep}.
\end{itemize}
Assuming compatibility $\mathrm{Rep}(M_\star)=0$, the hypotheses of the master theorem in the compatible regime
(Theorem~\ref{thm:master-compatible} in the main text) hold for this model.
Applying that theorem along any joint limit $(m,n)\to(\infty,\infty)$ with $m/n\to\alpha\in(0,\infty)$ yields the claimed limit
in Corollary~\ref{cor:GoM-model}, with $\mathcal L(v)=-D\Pi_{M_\star}[v]f_\star$.
\end{proof}

\section{Experiment Details}
\label{app:experiment-details}

This section provides the full setup underlying
Figure~\ref{fig:gmm-rates-scan}. Our goal is to numerically evaluate, in the
synthetic Gaussian-mixture model of \Cref{sec:ex-mog}, the pretraining
interaction quantity
\[
\mathbb{E}\!\left[\|\mathcal{L}(Z)\|_{L^2(\mu_{\mathrm{down}})}^2\right],
\]
and to study how it depends on the number of mixture components \(K\), the
ambient dimension \(d\), and the separation parameter \(\beta\).

\paragraph{Gaussian-mixture model.}
We consider an equally weighted \(K\)-component Gaussian mixture in
\(\mathbb{R}^d\) with identity covariance and component means
\[
U^\star = \beta I_{K,d},
\]
where \(I_{K,d}\) denotes the \(K\times d\) rectangular identity matrix.
Thus the \(k\)th component is centered at \(\beta e_k\), for
\(k=1,\dots,K\), and in particular we require \(d \ge K\).

\paragraph{Downstream signal.}
The downstream target is taken to be linear in the projected feature map, with
a block structure aligned with the mixture components. For each
\(i \in [K]\), we associate a parameter block \((\theta_i^\star,b_i^\star)\),
with block magnitude proportional to \(1/i\). Equivalently, the concatenated
parameter vector
\[
(\theta_1^\star,b_1^\star,\ldots,\theta_K^\star,b_K^\star)
\]
is proportional to
\[
\left(\mathbf{1}_{d+1},\frac{1}{2}\mathbf{1}_{d+1},\ldots,
\frac{1}{K}\mathbf{1}_{d+1}\right),
\]
and is then normalized to have unit Euclidean norm.

\paragraph{Monte Carlo procedure.}
For each parameter triple \((K,d,\beta)\), we estimate
\(\mathbb{E}[\|\mathcal{L}(Z)\|_{L^2(\mu_{\mathrm{down}})}^2]\) using three Monte Carlo
stages:
\begin{enumerate}
    \item estimation of the Fisher information matrix;
    \item estimation of the projection coefficients defining the downstream predictor;
    \item estimation of the final quadratic quantity using fresh evaluation samples.
\end{enumerate}
In our main experiments, we use \(100{,}000\) samples for the
Fisher-information estimate, \(1{,}000{,}000\) samples for the projection step,
and \(1{,}000{,}000\) fresh samples for the evaluation step.

\paragraph{Descriptive fits and observed trends.}
To summarize the empirical scaling behavior, we overlay simple descriptive fits
in each panel. In panel (a), the dependence on \(K\) is monotone increasing and
concave over the plotted range, and is well captured by a quadratic fit
\(aK^2 + bK + c\) with \(R^2=0.99\). In panel (b), the dependence on \(d\)
shows slow growth and is reasonably described by a logarithmic fit
\(a + b \log d\) with \(R^2=0.89\). In panel (c), the dependence on
\(\beta\) decays rapidly and is well summarized by a power-law fit
\(C\beta^\alpha\) with \(R^2=0.90\). These fits should be interpreted as
compact summaries of the numerical trends rather than as formal asymptotic
claims.

\paragraph{Finite-sample distributional validation.}
We also simulate the full two-stage procedure in order to check the distributional prediction of~\Cref{thm:master-compatible}. We fix a representative GMM instance with
\[
K=4,\qquad d=20,\qquad \beta=2,
\qquad \alpha \coloneqq m/n = 2.
\]
For each downstream sample size $n$, we set $m=\lfloor \alpha n\rfloor$, draw $m$ unlabeled pretraining samples from the mixture, estimate the centers by EM, draw an independent downstream design $X_{1:n}$, and compute the conditional scaled excess risk
\[
n\bigl(R(D_{\mathrm{pre}}^{(m)},X_{1:n})-\sigma^2\bigr).
\]
The conditional expectation over downstream label noise is computed analytically using the exact conditional risk decomposition according to~\Cref{prop:exact-decomp}, rather than by repeatedly resampling labels. Thus, each Monte Carlo repetition produces the total scaled risk together with its finite-sample variance and pretraining components.

The asymptotic comparison is made against the limiting random variable
\[
E_\alpha
=
\sigma^2 d_{\mathrm{eff}}(\Omega_\star)
+
\alpha^{-1}\|\mathcal{L}(Z)\|^2_{L^2},
\qquad
Z\sim\mathcal N(0,H_\star^{-1}\Sigma_\star H_\star^{-1}).
\]
For the pretraining component alone, the corresponding limiting law is $\alpha^{-1}\|\mathcal{L}(Z)\|^2_{L^2}$, whereas the variance contribution converges in probability to the deterministic limit $\sigma^2 d_{\mathrm{eff}}(\Omega_\star)$. Figure~\ref{fig:gmm-cdfs} displays empirical CDFs for a representative subset of sample sizes from a logarithmic grid. The dashed black curves denote the asymptotic laws. For visual clarity, the CDF panels show only a subset of the simulated $n$-values, while the quantitative metrics below are computed on the full logarithmic grid. For the finite-sample validation, we use \(1000\) repetitions per \(n\), \(2\times 10^4\) Gaussian draws for the asymptotic CDFs, \(10^5\) samples for the Fisher estimate, and \(2\times 10^5\) samples for each projection, evaluation, and test-risk computation.

\begin{figure}[t]
    \centering
    \includegraphics[width=0.7\textwidth]{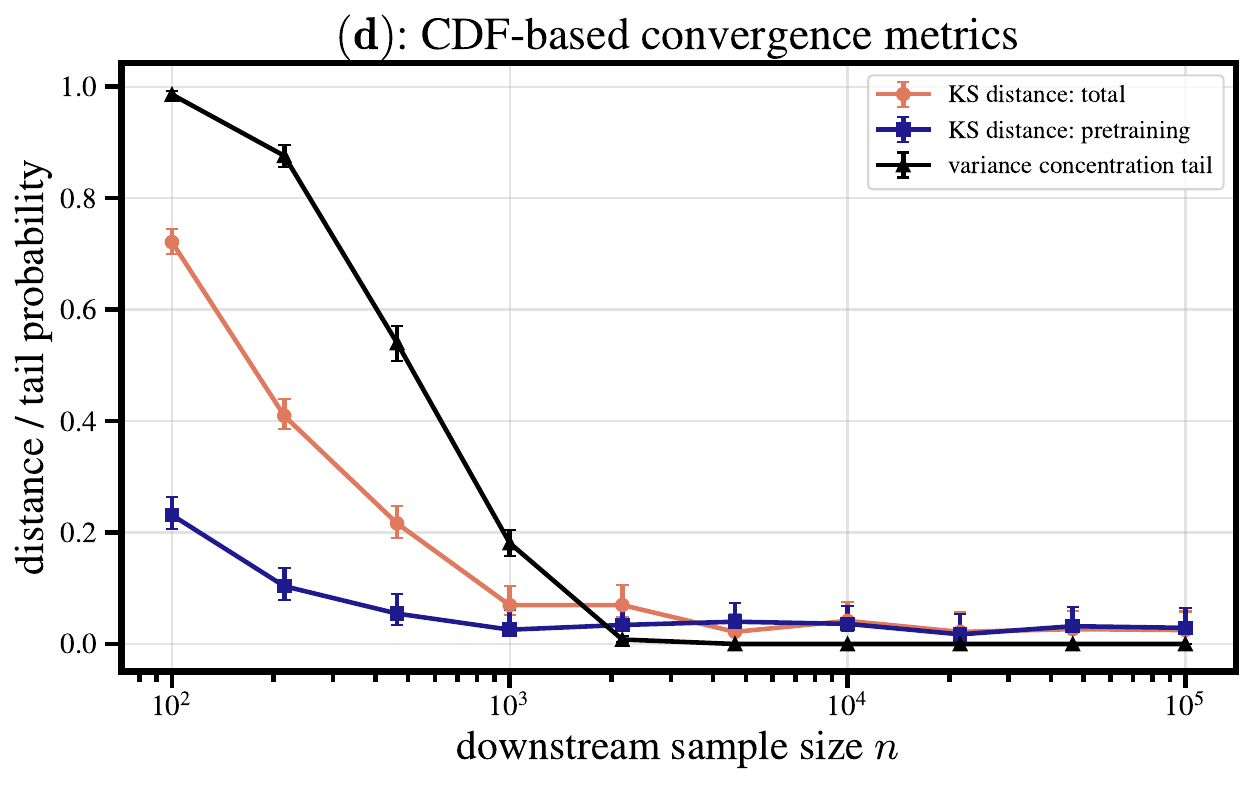}
    \caption{CDF-based convergence metrics for the GMM distributional validation. We plot the Kolmogorov distances \(D_{\mathrm{KS}}^{\mathrm{tot}}(n)\) and \(D_{\mathrm{KS}}^{\mathrm{pre}}(n)\) between the empirical CDFs and the corresponding asymptotic CDFs for the total scaled excess risk and pretraining contribution. Since the scaled variance contribution has the degenerate limit \(\sigma^2d_{\mathrm{eff}}(\Omega_\star)\), we instead plot the concentration probability \(p_{\mathrm{var}}(n;\varepsilon)\). Metrics are computed over the logarithmic grid $n \in \operatorname{round}\!\left(\operatorname{geomspace}(100,10^5,10)\right)$.}
    \label{fig:gmm-metric}
\end{figure}

\paragraph{Distributional metrics.}
To quantify the convergence observed in the CDF plots, we report Kolmogorov distances for the non-degenerate limiting distributions. For the total scaled excess risk, let \(\widehat F_{n,\mathrm{tot}}\) denote the empirical CDF of \(n\bigl(R(D_{\mathrm{pre}}^{(m)},X_{1:n})-\sigma^2\bigr)\) over repeated two-stage simulations, and let \(\widehat F_{\infty,\mathrm{tot}}\) denote the Monte Carlo CDF of the limiting law \(E_\alpha\). We define
\[
D_{\mathrm{KS}}^{\mathrm{tot}}(n)
\coloneqq
\sup_t
\left|
\widehat F_{n,\mathrm{tot}}(t)
-
\widehat F_{\infty,\mathrm{tot}}(t)
\right|.
\]
Similarly, for the pretraining component, we define
\[
D_{\mathrm{KS}}^{\mathrm{pre}}(n)
\coloneqq
\sup_t
\left|
\widehat F_{n,\mathrm{pre}}(t)
-
\widehat F_{\infty,\mathrm{pre}}(t)
\right|,
\]
where \(\widehat F_{n,\mathrm{pre}}\) denotes the empirical CDF of $n\mathrm{Rep}(\Omega_m)$ and \(\widehat F_{\infty,\mathrm{pre}}\) denotes the Monte Carlo CDF of
\(\alpha^{-1}\|\mathcal{L}(Z)\|^2_{L^2}\).

For the variance contribution, the limiting distribution is the point mass at
\(\sigma^2d_{\mathrm{eff}}(\Omega_\star)\). Since the limiting CDF is discontinuous, an ordinary Kolmogorov distance to this point mass is not a stable convergence metric. Instead, we report the concentration probability
\[
p_{\mathrm{var}}(n;\varepsilon)
\coloneqq
\widehat{\mathbb P}
\left(
\left|
V_{m,n}
-
\sigma^2d_{\mathrm{eff}}(\Omega_\star)
\right|
>
\varepsilon
\right),
\]
where \(V_{m,n}\) denotes the scaled variance contribution and we take
\[
\varepsilon
=
0.05\,\sigma^2d_{\mathrm{eff}}(\Omega_\star).
\]
The quantities \(D_{\mathrm{KS}}^{\mathrm{tot}}(n)\), \(D_{\mathrm{KS}}^{\mathrm{pre}}(n)\), and \(p_{\mathrm{var}}(n;\varepsilon)\) are reported in~\Cref{fig:gmm-metric}. Decreasing values across the logarithmic $n$-grid provide a quantitative summary of the distributional convergence shown in Figure~\ref{fig:gmm-cdfs}.

\section{AI Tool Usage}
\label{sec:appendix:AI_tool}

OpenAI's ChatGPT was used to 
assist in identifying relevant papers during the literature review,
brainstorm high-level proof strategies,
proofread mathematical arguments for clarity, 
edit portions of the manuscript,
and assist in writing code for experiments.
Anthropic's Claude was used to assist with generating schematic illustrations.

\end{document}